%% file: main.tex
\documentclass[11pt]{article}
\usepackage[margin =1in]{geometry}
\usepackage{amssymb,amsthm}
\usepackage{amsmath}
\input{liweimacros} 
\usepackage{pgfplots}
\usepackage{color}
\usepackage{algorithm}
\usepackage{enumitem}
\usepackage{algpseudocode}
\usepackage{caption}
\usepackage{subcaption}
\usepackage{tikz}
\usetikzlibrary{positioning,arrows.meta}

\usepackage{mathtools}
\usepackage{hyperref}
\usepackage{mathabx}
\usepackage[sort&compress]{natbib}
\pgfmathdeclarefunction{Phi}{1}{\pgfmathparse{0.5*(1 + erf(#1/sqrt(2)))}}
\pgfmathdeclarefunction{gelu}{1}{\pgfmathparse{#1 * (0.5*(1 + erf(#1/sqrt(2))))}}
\pgfmathsetmacro{\LeakP}{0.10} 
\pgfmathsetmacro{\LeakQ}{1.00} 

\usepackage{xcolor} 

\newcommand{\softmax}{\mathrm{softmax}}

\newcommand{\od}{{\rm offDiag}}

\newcommand{\Cov}{\mathrm{Cov}}

\newcommand{\had}{\circ}

\newcommand{\er}{\mathrm{er}}

\newcommand{\nr}{\mathrm{nr}}
\newcommand{\RMSNorm}{\mathrm{RMSNorm}}

\DeclareMathOperator{\op}{op}
\DeclareMathOperator{\tr}{tr}

\usepackage{etoolbox}    

\newcounter{steptag}

\begin{document}

\title{When do spectral gradient updates help in deep learning?}
\author{Damek Davis\thanks{Department of Statistics and Data Science, The Wharton School, University of Pennsylvania;
\texttt{www.damekdavis.com}. Research of Davis supported by NSF DMS award 2047637.}\and Dmitriy Drusvyatskiy\thanks{Halıcıoğlu Data Science Institute (HDSI),
University of California San Diego, La Jolla, CA, 
\texttt{sites.google.com/view/dmitriy-drusvyatskiy}.  Research of
Drusvyatskiy was supported by NSF DMS-2306322, NSF DMS-2023166, and AFOSR FA9550-24-1-0092 awards.}}	
\date{}

\maketitle

\begin{abstract}
Spectral gradient methods, such as the recently popularized Muon optimizer, are a promising alternative to standard Euclidean gradient descent for training deep neural networks and transformers, but it is still unclear in which regimes they are expected to perform better. We propose a simple layerwise condition that predicts when a spectral update yields a larger decrease in the loss than a Euclidean gradient step. This condition compares, for each parameter block, the squared nuclear-to-Frobenius ratio of the gradient to the stable rank of the incoming activations. To understand when this condition may be satisfied, we first prove that post-activation matrices have low stable rank at Gaussian initialization in random feature regression, feedforward networks, and transformer blocks. In spiked random feature models we then show that, after a short burn-in, the Euclidean gradient's nuclear-to-Frobenius ratio grows with the data dimension while the stable rank of the activations remains bounded, so the predicted advantage of spectral updates scales with dimension. We validate these predictions in synthetic regression experiments and in NanoGPT-scale language model training, where we find that intermediate activations have low-stable-rank throughout training and the corresponding gradients maintain large nuclear-to-Frobenius ratios. Together, these results identify conditions for spectral gradient methods, such as Muon, to be effective in training deep networks and transformers.
\end{abstract}

\newcommand\numeq[1]{\stackrel{\scriptscriptstyle(\mkern-1.5mu#1\mkern-1.5mu)}{\leq}}

\section{Introduction}
Modern deep neural networks are almost always trained with first-order methods, most prominently stochastic gradient descent (SGD) and its adaptive variants such as \texttt{Adam}~\cite{kingma2015adam}, \texttt{RMSprop}~\cite{tieleman2012rmsprop}, and \texttt{AdaGrad}~\cite{duchi2011adagrad}. These algorithms can be viewed as Euclidean gradient descent equipped with inexpensive preconditioners, typically diagonal rescalings or structured approximations such as Kronecker-factored schemes (\texttt{K-FAC}~\cite{martens2015kfac}, \texttt{Shampoo}~\cite{gupta2018shampoo}). Despite differences in their per-iteration cost, they share a common design principle: one follows the Euclidean gradient that is adjusted by a preconditioner derived from past gradients or curvature information. A complementary line of work takes a different viewpoint and modifies the \emph{geometry} of the update itself by operating on the spectrum of layerwise gradient matrices. The spectral gradient descent method (\texttt{SpecGD}~\cite{carlson2015preconditioned}) replaces the raw gradient with its polar factor, thus moving in a direction with the same singular vectors but unit singular values and step length given by the nuclear norm. The recently proposed optimizer \texttt{MUON}~\cite{jordan2024muon} implements a momentum-based variant of this spectral update, and has been observed to match or surpass variants of \texttt{Adam}~\cite{kingma2015adam} in large-language-model pretraining. Motivated by these empirical findings, we ask the following question.
\begin{quote}
\begin{center}
{In what regimes should one expect spectral updates to outperform standard Euclidean gradient methods for training deep neural networks and transformers?}
\end{center}
\end{quote}

The explanation we propose is based on a simple but prevalent structural property of the post-activation matrices. We show experimentally and theoretically that the \emph{post-activation matrices in intermediate layers of deep neural networks tend to have stable rank bounded by a small numerical constant that is independent of the ambient dimension.} This low stable rank reflects a strong degeneracy of the learned representations and, in particular, implies an ill-conditioned optimization landscape when viewed through the Euclidean lens. Spectral updates, such as \texttt{SpecGD}, which align with the singular vector directions and depend only on spectral and nuclear norms, are naturally adapted to this regime. We make this precise in a sequence of toy and semi-realistic settings, where the advantage of spectral over Euclidean updates can be quantified directly in terms of the stable rank of post-activation matrices. Our goal is not to analyze $\mathtt{MUON}$ in all of its practical details, but rather to understand why its underlying spectral update rule is well suited to the low-stable-rank structure that arises when training modern neural networks.

At a technical level, our results are organized around a one-step comparison of Euclidean and spectral updates acting on a single matrix block. Let $W$ be any matrix parameter that multiplies an incoming activation matrix $A$ (for example, a layer weight acting on the post-activations from the previous layer), and let $G = \nabla_W \mathcal{L}(W)$ denote the corresponding block of the gradient. 
 A simple argument (carried out in Section~\ref{sec:rand_regress} and extended in Section~\ref{sec:multiple}) shows that for a variety of models (random features, MLPs, and transformers) the guaranteed decrease in the loss after one Euclidean gradient step and one spectral step can be bounded as
\[
\Delta_{\mathrm{GD}} \;\asymp\; \frac{\|G\|_F^2}{\|A\|_{\op}^2}
\qquad\text{and}\qquad
\Delta_{\mathrm{Spec}} \;\asymp\; \frac{\|G\|_*^2}{\|A\|_{F}^2}.
\]
Therefore, spectral descent is favored whenever
\begin{equation}
\frac{\|G\|_*^2}{\|G\|_F^2} \;\ge\;  \frac{\|A\|_{F}^2}{\|A\|_{\op}^2}.
\label{eq:intro-key-condition}
\end{equation}
The right-hand side is precisely the stable rank
$\mathrm{st}(A):= \|A\|_{F}^2/\|A\|_{\op}^2$ of the incoming features, while the left-hand side
measures how spread out the singular values of the gradient are; we will refer to this ratio as
the \emph{nuclear rank} of $G$ and denote it by
\[
\nr(G) \;:=\; \frac{\|G\|_*^2}{\|G\|_F^2}.
\]
In terms of these quantities, the ratio of the one-step descent guarantees is
$
\frac{\Delta_{\mathrm{Spec}}}{\Delta_{\mathrm{GD}}}
\;\asymp\;
\frac{\nr(G)}{\mathrm{st}(A)}.
$
In particular, when $\mathrm{st}(A)$ is $O(1)$ and $\nr(G)$ grows with the
dimension—as we prove in random-feature models after a short-burn-in and observe empirically in
realistic networks—the predicted speedup of spectral over Euclidean updates is itself
dimension-dependent (and in our examples linear in $d$), matching the large early-time advantage
we see for spectral methods in practice.

In a multilayer network, the parameters naturally decompose into matrix blocks $W_\ell$ with
gradients $G_\ell$ that interact with post-activations $A_{\ell-1}$. The same one-step comparison
then yields a \emph{layerwise} condition
\begin{equation}\label{eqn:layer_cond}
\nr(G_\ell) \;\geq\; \mathrm{st}(A_{\ell-1}),
\end{equation}
which identifies when a blockwise spectral update on $W_\ell$ is favored over a
Euclidean one. For standard MLPs, $A_{\ell-1}$ is simply the usual post-activation of the previous
layer, so the inequality compares a gradient-based quantity $\nr(G_\ell)$ to the stable rank of the
propagated data seen by that layer. In Section~\ref{sec:multiple} we show that the same blockwise
condition continues to hold for a broad class of layered architectures, including transformers. In
that setting, the relevant $A_{\ell-1}$ is the normalized or MLP activation feeding each projection
(for example the RMS-normalized hidden states entering $W_Q,W_K,W_V$, or the MLP post-activations
entering $W_2$). The upshot is a simple rule of thumb: spectral updates are most advantageous on
those blocks whose incoming features have low-stable-rank and whose gradients have large nuclear
rank~$\nr(G_\ell)$.\footnote{Similar gradient nuclear ranks also appear implicitly or
explicitly in convergence analyses of spectral optimizers such as \texttt{MUON} and
Scion~\cite{shen2025convergence,pethick2025scion}; our focus here is to relate this ratio to the
\emph{stable rank of the propagated data} and to show that the resulting inequality is naturally
activated in realistic networks.}

Beyond this one-step picture, we also show that the nuclear-rank advantage of spectral updates is not a purely transient phenomenon. In spiked random-feature regression models, the gradient nuclear rank becomes large after a short burn-in period and then remains high along a macroscopic window of gradient-descent iterations. More precisely, in both realizable and teacher–student variants we prove that after a short $O(\log d)$ burn-in period there is a window of $\Theta(d)$ iterations on which $\nr(\nabla\mathcal{L}(W_t))=\Omega(d)$ and, for any fixed $\varepsilon>0$, a longer window of $\Theta(d\log d)$ iterations on which $\nr(\nabla\mathcal{L}(W_t))\ge d^{1-\varepsilon}$ (Theorem~\ref{thm:multimodel1} and Corollary~\ref{cor:main_multi_spike_hard}). In the same models, Euclidean gradient descent needs $\Theta(d\log(d/\delta))$ steps to reach relative error $\delta\in(0,1)$, so for any fixed target accuracy this $\Theta(d\log d)$ window represents a constant fraction  of the training time. Thus, if we were to restart a spectral method such as $\mathtt{SpecGD}$ or $\mathtt{MUON}$ from any point in this window, its one-step improvement over Euclidean gradient descent would itself be dimension-dependent.

The rest of this introductory section describes our results in more detail. The reader may find code to reproduce our experiments at \href{https://github.com/damek/specgd/}{https://github.com/damek/specgd/}.

\subsection{Post-activation matrices have low-stable rank.}\label{subsec:post_activ}

The layerwise criterion~\eqref{eqn:layer_cond} compares the gradient nuclear rank
$\nr(G_\ell)$ to the stable rank of the incoming activations.  To understand when this
criterion is activated in realistic architectures, we need structural control on the matrices
that propagate through the network.  Our main conclusion is that, at Gaussian initialization,
the matrices feeding the internal (two-dimensional) weight blocks of standard MLPs and
decoder-only transformers typically have stable rank bounded by a small numerical constant,
independent of width and (for transformers) sequence length.  The argument has the following two components:
\begin{enumerate}
\item[(i)] two base cases explaining how low-stable-rank representations arise in the first place, namely through (a) mean-induced spikes
and (b) token-indicator/embedding structure;  
\item[(ii)] a propagation calculus showing how the stable rank evolves as data propagates through standard network blocks.
\end{enumerate}

\paragraph{(1a) Mean-spike inducing activations.}
Call a pointwise activation $\sigma$ \emph{mean-spike} \emph{inducing} (MSI) if, for a  Gaussian input $\gamma\sim\cN(0,1)$, the mean is nonzero and the second moment is finite:
\[
m_1:=\EE[\sigma(\gamma)]\neq 0
\qquad\textrm{and}\qquad
m_2:=\EE[\sigma(\gamma)^2]<\infty.
\]
This property is generic for non-centered activations used in practice (e.g.\ ReLU).
Our first main theorem (Theorem~\ref{thm:gaussian_weights_low_sr}) shows that, for any fixed
input matrix $X$ and Gaussian weights $W$, the post-activation matrix
$A=\sigma(WX)$ has stable rank controlled by the one-dimensional moment ratio
$m_2/m_1^2$ up to a slack $(1\pm\varepsilon)$. For the concrete example of ReLU
this ratio is $\pi$.

Why does a nonzero mean force a spectral spike? To see this, let $a_1^\top,\dots,a_k^\top\in\R^n$ denote the rows of the post-activation matrix $A\in\R^{k\times n}$. It is straightforward to see that the row Gram matrix trivially satisfies
$$\frac{1}{d}A^\top A\succeq \bar a\,\bar a^\top,$$
where $\bar a:=\frac{1}{d}\sum_{i=1}^k a_i$ is the row mean.
Thus, we obtain the estimate
\[
\st(A)=\frac{\|A\|_F^2}{\|A\|_{\op}^2}
\le \frac{(1/d)\|A\|_F^2}{\|\bar a\|_2^2}.
\]
We later show that the numerator and denominator 
separately concentrate, so we may bound the ratio by the population counterpart. Writing $x_t$ for the $t$'th column of $X$ and $w\sim\cN(0,I)$, the
population counterparts are
\[
\EE\Big[\frac{1}{k}\|A\|_F^2\Big]=\sum_{t=1}^n \EE\big[\sigma(\langle w,x_t\rangle)^2\big]
\qquad\textrm{and}\qquad
\|\EE \bar a \|_2^2=\sum_{t=1}^n \EE\big[\sigma(\langle w,x_t\rangle)\big]^2.
\]
For $p$-homogeneous $\sigma$, we have $\EE[\sigma(\langle w,x_t\rangle)]=m_1\|x_t\|_2^p$ and
$\EE[\sigma(\langle w,x_t\rangle)^2]=m_2\|x_t\|_2^{2p}$, so the ratio above simplifies to
$m_2/m_1^2$. A slightly more involved arguments yield a similar conclusion for all the common non-homogeneous and non-centered functions used in practice (Theorem~\ref{thm:gaussian_weights_low_sr}).
Iterating this one-layer statement yields an $O(1)$ stable-rank bound for
\emph{every} hidden-layer post-activation matrix in an MLP with an MSI activation function at Gaussian initialization
(Corollary~\ref{cor:mlp-mean-spike}).

\paragraph{(1b) Token-indicator matrices as a pervasive low-stable-rank input.}
For transformers, low stable rank already appears at the input.  Let $H\in\R^{V\times n}$ be
the token-indicator matrix for a length-$n$ token sequence: the $t$'th column is the one-hot
vector for token $i_t$.  Writing $p_{\max}$ for the empirical frequency of the most common
token in the sequence, a direct calculation gives
\[
\st(H)=\frac{1}{p_{\max}},
\]
so the stable rank of the token-indicator matrix is the inverse frequency of the most common
token.  In natural language corpora, token frequencies are heavy-tailed (Zipf-type behavior),
so $p_{\max}$ is not extremely small and $\st(H)$ is modest.  Section~\ref{sec:embedding} shows
that a Gaussian embedding layer maps this indicator structure to an embedding matrix whose
stable rank is controlled (up to constants) by the same $p_{\max}$, and combining this base
case with the propagation rules in part~(2) below yields width- and sequence-length-independent
stable-rank bounds for the RMS-normalized hidden states and MLP activations throughout a randomly
initialized transformer (see Corollary~\ref{cor:deep-transformer-stablerank-linear}).  As an illustration, Figure~\ref{fig:embedding} tracks, in a NanoGPT training run, the stable rank of the token-indicator matrix and of the embedding activations throughout training.

\paragraph{(2) Stable-rank propagation through network blocks.}
The base cases (1a) and (1b) above explain why low-stable-rank structure can arise either from the activation
nonlinearity (mean spikes) or from the input representation (token indicators/embeddings).
To pass from these base cases to full architectures---especially transformers, where each block
contains normalization, attention mixing, residual updates, and (possibly gated) MLPs---we develop a
``calculus'' showing that stable rank remains small as the data is propagated through  standard
building blocks.  These results are proved in Section~\ref{sec:propagate}; Table~\ref{tab:sr-propagation}
informally summarizes the form of the bounds that we obtain.  In particular, at Gaussian initialization these rules show that the
stable rank of the representations entering each internal weight block is controlled by the stable
rank of the upstream representation, rather than inflating with width or sequence length.

\begin{table}[t]
\centering
\renewcommand{\arraystretch}{1.3}
\begin{tabular}{|c|c|c|}
\hline
{\bf Block} & {\bf Stable rank} & {\bf Assumptions} \\
\hline\hline
Linear map
& $\st(WX)\lesssim \st(X)$
& --- \\
Pointwise nonlinearity
& $\st(\sigma(WX))\lesssim \st(X)$
& $|\EE\sigma'(g)|>0$ \\
Residual connection
& $\st(X+WH)\lesssim \min\{\st(X),\st(H)\}$
& $\|X\|_F\asymp \|H\|_F$ \\
RMSNorm
& $\st\!\left(\begin{bmatrix} \tfrac{x_1}{\|x_1\|_2}& \cdots & \tfrac{x_n}{\|x_n\|_2}\end{bmatrix}\right)\lesssim \st(X)$
& nearly const.\ column norms \\
Gating
& $\st(\sigma(VZ)\odot WX)\lesssim \st(X)$
& nearly const.\ column norms,\ $|\EE\sigma(g)|>0$ \\
\hline
\end{tabular}
\caption{Stable-rank propagation rules for common network operations at Gaussian initialization; precise statements and constants appear in Section~\ref{sec:propagate}.
}
\label{tab:sr-propagation}
\end{table}

\paragraph{Empirical results.}
Although our sharpest guarantees apply at initialization, the same low-stable-rank structure
persists empirically well into training.  In a synthetic sparse-regression experiment (training a three layer neural network on the target $f(x)=x_1x_2x_3$),
Figure~\ref{fig:synthetic_stable_rank_sparse} shows that hidden-layer post-activations have
stable rank far below the ambient dimension throughout training.  In a NanoGPT-scale language
model run, Figure~\ref{fig:mlpnano} shows that MLP post-activations across all blocks maintain
low stable rank, and the embedding/token-indicator behavior is illustrated in
Figure~\ref{fig:embedding}.  RMS-normalized hidden states feeding attention and MLP projections
exhibit the same behavior; see Figure~\ref{fig:rmsactivations}.

Taken together, these results indicate that, already at initialization, the activation matrices
$A$ entering most internal weight blocks are typically low-stable-rank.  We now examine how this
degeneracy interacts with Euclidean and spectral updates in a simple, yet representative,
random-feature regression model.

\begin{figure}[h!]
\centering
\includegraphics[width=\textwidth]{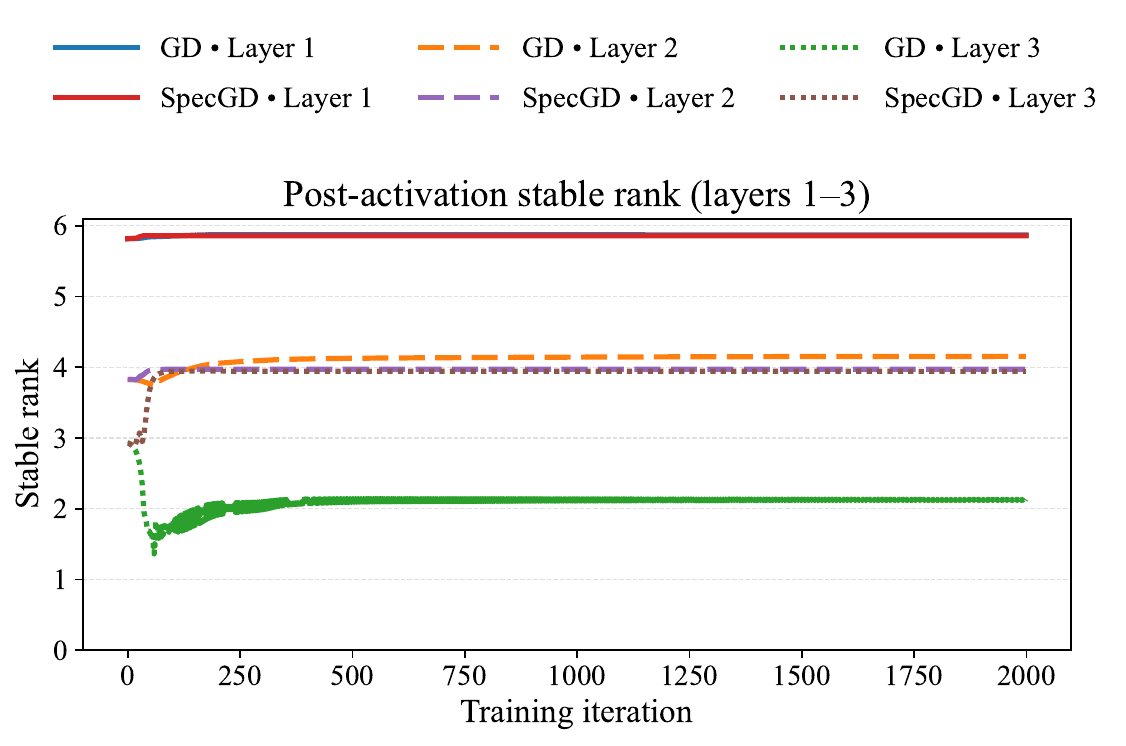}
 \caption{We generate $n= 512$ random gaussian vectors of dimension $d=128$ and use target $y$ which is the product of the first 3 coordinates. We then train a 4 layer feedforward neural network, with activation function $\sigma(t) = \max\{0, t\}^2$, mapping from input space $\RR^d$ to $\RR$ as follows: $\RR^d \rightarrow \RR^{4d}  \rightarrow \RR^{4d} \rightarrow \RR^{4d} \rightarrow \RR$. We start at the standard pytorch random initialization and run full batch algorithms: Gradient Descent and a Spectral descent method (only on layers 2 and 3, while layer 1 and 4 simply take Euclidean gradient steps; see Section~\ref{sec:multiple} for a justification of this choice) and observe the stable ranks of the activation matrices. We note that the maximum possible stable rank of such matrices is $256.$}
\label{fig:synthetic_stable_rank_sparse}
\end{figure}

\begin{figure}[h!]
\centering
\includegraphics[width=\textwidth]{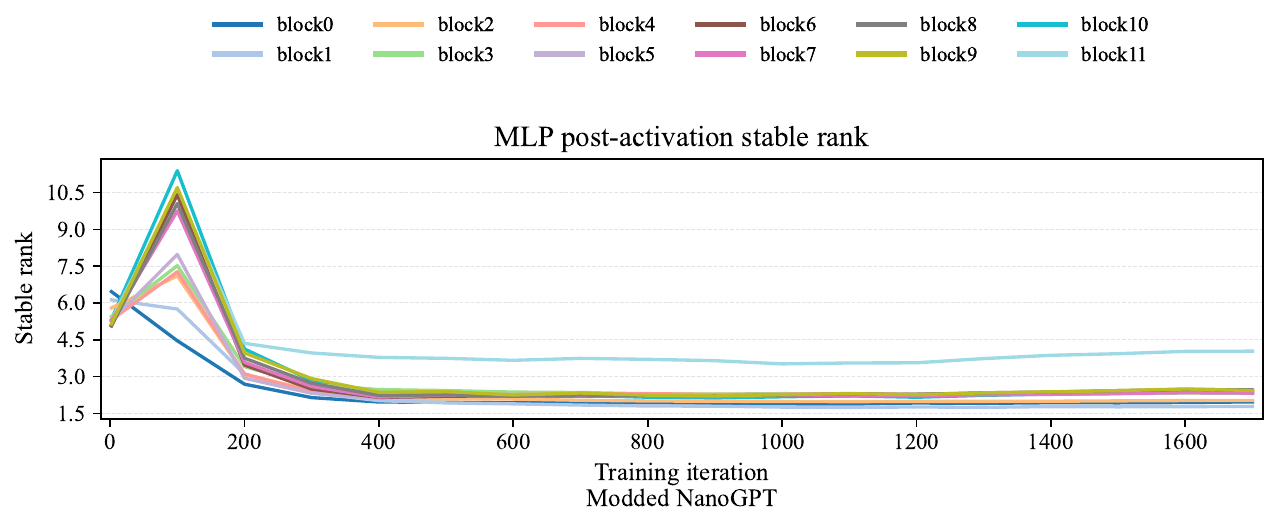}
\caption{Stable rank of MLP post activations while running a particular snapshot (specifically~\cite{modded_nanogpt_july18}) of the modded-NanoGPT repo~\cite{moddednanogpt2025}. We also include the July 18th snapshot of the architecture at the time of training in Section~\ref{figure:nanogpt}. We note that the maximal rank of any activation matrix in this plot is $3072$. Thus, the the stable rank of the post-activations are far below their maximal value.}
\label{fig:mlpnano}
\end{figure}

\subsection{Impact of post-activation degeneracy on training algorithms}\label{sec:rand_regress}
Low-stable-rank post-activations have a pronounced impact on the behavior of first-order training algorithms. We first make this precise in the random-feature regression model:
\begin{equation}\label{eqn:rfr_intro}
\min_{W\in \mathbb{R}^{m\times k}}~ \mathcal{L}(W):=\tfrac{1}{2n}\|WA-Y\|_F^2.
\end{equation}
We will think about $W$ as the $\ell$'th layer weight matrix to be trained and $A$ is the post-activation matrix of the propogated data from the previous layers. The random feature model is a standard and analytically tractable simplification of neural networks; see for example \cite{rahimi2007random,rahimi2008weighted,cho2009kernel,daniely2016toward,rudi2017generalization}.

Now being a quadratic, the function $\mathcal{L}$ may be rewritten as a Taylor-expansion:
\begin{equation}\label{eqn:expansion}
\mathcal{L}(W+U) = \mathcal{L}(W)+\langle \nabla \mathcal{L}(W),U\rangle+\tfrac{1}{2n}\|UA   \|_F^2.
\end{equation}
The usual starting point for first-order algorithms is based on the majorization principle: one upper-bounds the pure quadratic term $\|UA\|_F^2$ in \eqref{eqn:expansion} by a function that decouples $U$ from the data $A$. The update direction of the algorithm is then obtained by minimizing this simple upper model over $U$. Euclidean gradient descent ($\mathtt{GD}$) and spectral gradient descent ($\mathtt{SpecGD}$) proposed in \cite{carlson2015preconditioned} utilize the following two approximations, respectively:\footnote{Note that both of these inequalities are tight for some matrix $U$.}
\begin{equation}\label{eq:two-bounds}
\tfrac{1}{n}\|UA\|_F^2 \leq \underbrace{\tfrac{1}{n}\|A\|_{\rm op}^2\,}_{=:L_{F}}\cdot \|U\|_F^2
\qquad \textrm{and}\qquad
\tfrac{1}{n}\|UA\|_F^2 \leq \underbrace{\tfrac{1}{n}\|A\|_{F}^2}_{=:L_{\rm op}}\,\cdot \|U\|_{\rm op}^2.
\end{equation}
Thus $\mathtt{GD}$ and $\mathtt{SpecGD}$ impose the (Euclidean) Frobenius norm and the (non-Euclidean) operator norm on the domain, respectively.
Combining \eqref{eqn:expansion} with each of the inequalities \eqref{eq:two-bounds} and minimizing in $U$ yields the explicit updates for $\mathtt{GD}$ and $\mathtt{SpecGD}$, respectively:
\[
W_{\rm GD}:=W-\tfrac{1}{L_F}\nabla \mathcal{L}(W)\qquad\textrm{and}\qquad W_{\rm SD}=W-\tfrac{\|\nabla \mathcal{L}(W)\|_*}{L_{\rm op}}{\rm polar}(\nabla \mathcal{L}(W)).
\]
Here, the polar of the matrix is the product of the left and right orthogonal factors appearing in its singular value decomposition. $\mathtt{MUON}$ is a variant of $\mathtt{SpecGD}$, where the polar is approximated by the Newton-Schulz algorithm, and one further incorporates momentum.
Now the guaranteed function decrease achieved by $\mathtt{GD}$ and $\mathtt{SpecGD}$ is fully determined by the dual norm of the gradient:
\begin{align*} 
\mathcal{L}(W)-\mathcal{L}(W_{\rm GD})\geq \frac{1}{2L_F}\|\nabla\mathcal{L}(W)\|^2_F \qquad\text{and} \qquad \mathcal{L}(W)-\mathcal{L}(W_{\rm SD})\geq \frac{1}{2L_{\rm op}}\|\nabla\mathcal{L}(W)\|^2_{*}.
\end{align*}
Comparing the two bounds, we see that $\mathtt{SpecGD}$ promises more descent than $\mathtt{GD}$ whenever
\begin{equation}
\nr(\nabla\mathcal{L}(W)) \geq {\rm st}(A), 
\label{eq:sd-vs-gd-condition}
\end{equation}
When the stable rank of the post-activation ${\rm st}(A)$ is constant (independent of dimension), one expects that \eqref{eq:sd-vs-gd-condition} will hold, since the ratio on the left-side can scale with the dimension.

As a numerical illustration (Figure~\ref{fig:random_features}), we compare the performance of $\mathtt{GD}$ and $\mathtt{SpecGD}$ on random feature regression corresponding to training second layer weights, namely $A_{1}=\sigma(W_1 X)$, where $W_1\in \mathbb{R}^{100\times 50}$ and $X\in \mathbb{R}^{100\times 50}$ have iid standard Gaussian entries and $\sigma(t)=\max\{0,t\}$ is the ReLU activation function applied coordinatewise. It is clear from the figure that $(i)$ equation~\eqref{eq:sd-vs-gd-condition} holds along both $\mathtt{GD}$ and $\mathtt{SpecGD}$ trajectories and $(ii)$ $\mathtt{SpecGD}$ significantly outperforms $\mathtt{GD}$ as expected.

\begin{figure}[h]
\centering
\includegraphics[width=\textwidth]{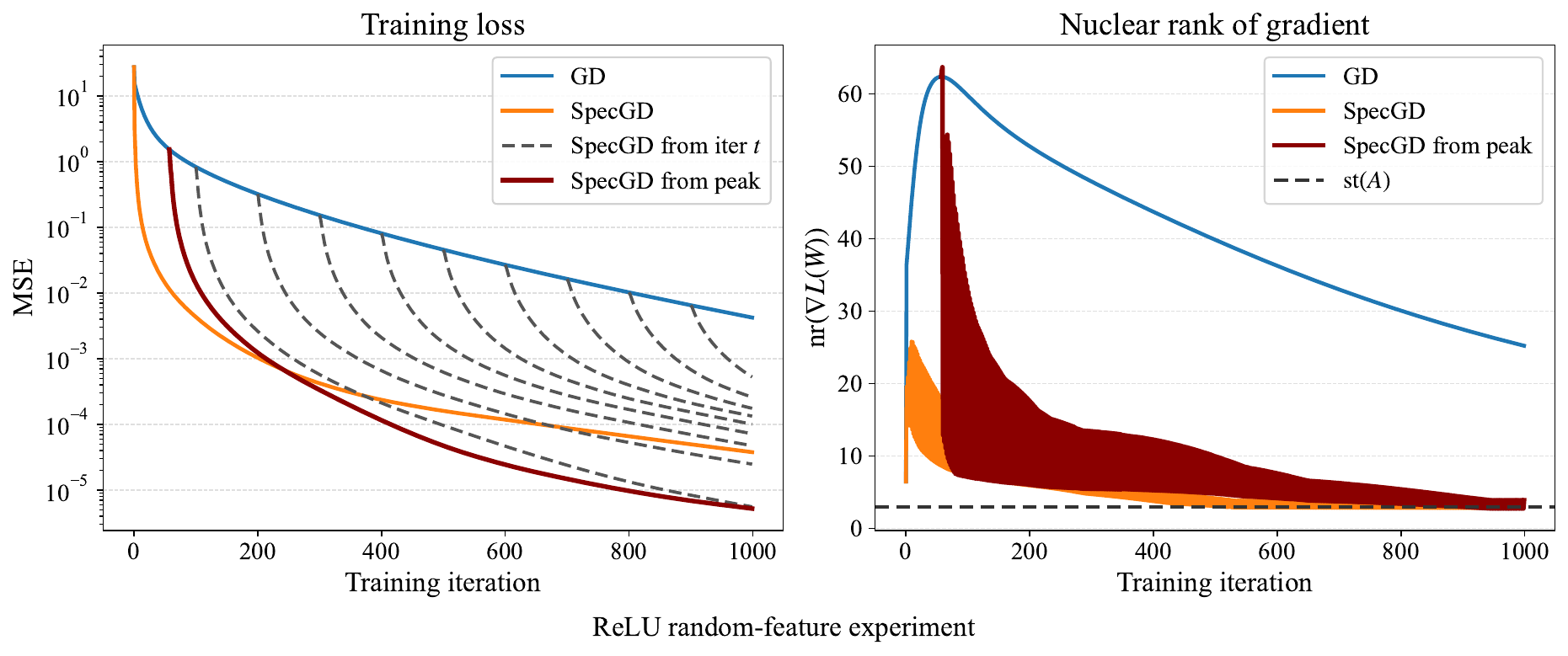}
\caption{Comparison of gradient descent ($\mathtt{GD}$) and spectral gradient descent $(\mathtt{SpecGD})$ on the random feature model $\min_{W}\mathcal{L}(W)=\tfrac{1}{2n}\|WA-Y\|_F^2$ with $W\in \mathbb{R}^{100\times 100}$. The ground truth matrix $W_{\sharp}\in \mathbb{R}^{100\times 100}$ is drawn with iid standard Gaussian entries. The data matrix is generated by $A=\sigma(W_1 X)$, where $W_1\in \mathbb{R}^{100\times 50}$ and $X\in \mathbb{R}^{100\times 400}$ have iid standard Gaussian entries and $\sigma(t)=\max\{0,t\}$ is the ReLU activation function applied coordinatewise. The target matrix $Y=W_{\sharp}A$ is generated from a ground truth matrix $W_{\sharp}\in \mathbb{R}^{100\times 100}$ that has iid standard Gaussian entries. Both methods are initialized at the all-zeros matrix.
Left: supoptimality gap in function value along the $\mathtt{GD}$ (solid blue) and $\mathtt{SpecGD}$ (solid gold) iterations. The dashed curves plot the suboptimality gap if we were to initialize $\mathtt{SpecGD}$ at the current $\mathtt{GD}$ step, plotted every $100$ iterations. The superior performance of $\mathtt{SpecGD}$ is clear from the figure. Right: nuclear rank $\nr(\nabla \mathcal{L}(W))$ of the gradient along the $\mathtt{GD}$ and  $\mathtt{SpecGD}$ iterations initialized at the all-zeros matrix; the black dashed line signifies the level ${\rm st}(A)$, above which $\mathtt{SpecGD}$ is superior to $\mathtt{GD}$. The nuclear rank $\nr(\nabla \mathcal{L}(W))$ can be large (with \eqref{eq:sd-vs-gd-condition} holding) along the trajectories.}
\label{fig:random_features}
\end{figure}

The argument we presented above is particularly clean for least squares, but it can be generalized much further. As a first example, we stay within the random feature model and consider a smooth loss $f$ on $WA$ for which the gradient $\nabla f$ is $L$-Lipschitz. Note that we can think of $W$ as  a fixed layer within a network, keeping the rest of the weights frozen, and $A$ would be the previous layers post-activations. Now, define $\cL(W) := f(WA)$ and note that
\begin{align}\label{eq:general_L}
\cL(W + U) \leq \cL(W) + \dotp{ \nabla L(W), U} + \frac{L}{2}\|UA\|_{F}^2.
\end{align}
Then, as before, we may bound $\|UA\|_{F}^2$ by $\|A\|_{\op}^2\|U\|_{F}^2$ or $\|A\|_F^2\|U\|_{\op}^2$, and minimize the two constructed bounds in the variable $U$. This gives us $\mathtt{GD}$ and $\mathtt{SpecGD}$, respectively:
\[
W_{\rm GD}:=W-\tfrac{1}{L_F\cdot L}\nabla \mathcal{L}(W)\qquad\textrm{and}\qquad W_{\rm SD}=W-\tfrac{\|\nabla \mathcal{L}(W)\|_*}{L_{\rm op}\cdot L}{\rm polar}(\nabla \mathcal{L}(W)),
\]
Plugging in the update in~\eqref{eq:general_L}, we find
\begin{align*} 
\mathcal{L}(W)-\mathcal{L}(W_{\rm GD})\geq \frac{1}{2L_F\cdot L}\|\nabla\mathcal{L}(W)\|^2_F \qquad\text{and} \qquad \mathcal{L}(W)-\mathcal{L}(W_{\rm SD})\geq \frac{1}{2L_{\rm op}\cdot L}\|\nabla\mathcal{L}(W)\|^2_{*}.
\end{align*}
Thus,  descent promised by $\mathtt{SpecGD}$ is again greater than that promised by $\mathtt{GD}$ whenever~\eqref{eq:sd-vs-gd-condition} holds.

Although the random feature model is idealized, the derived conclusions appear to hold more generally. For example, in Figure~\ref{fig:synthetic_run_grad_vs_training_loss_sparse} we build on the sparse regression setting of  Figure~\ref{fig:synthetic_stable_rank_sparse}, plotting the performance of $\mathtt{GD}$ and $\mathtt{SpecGD}$, along with the the nuclear rank $\nr(\nabla \mathcal{L}(W))$.  $\mathtt{SpecGD}$ exhibits significantly better performance than $\mathtt{GD}$, especially early on in training when the nuclear ranks of the gradients are large. Further, building on Figure~\ref{fig:mlpnano}, we plot in Figure~\ref{fig:mlpgrad}  the ratio $\nr(\nabla \mathcal{L}(W))$ for the MLP layers when training the NanoGPT model. We observe that this ratio can indeed be  large along the trajectory, thereby suggesting a significant advantage of $\mathtt{MUON}$/$\mathtt{SpecGD}$ over $\mathtt{GD}$.

Implementing spectral updates such as \texttt{SpecGD}/\texttt{Muon} in large distributed runs is nontrivial because the relevant gradient (or momentum) matrices are typically sharded across devices; in Appendix~\ref{app:distributed-specgd} we analyze a shardwise alternative that applies the polar-factor map independently to each local shard (incurring no incremental communication) and derive the corresponding analogue of the nuclear-rank versus stable-rank condition.

\begin{figure}[h]
\centering
\begin{subfigure}[t]{0.48\textwidth}
  \centering
\includegraphics[width=\linewidth]{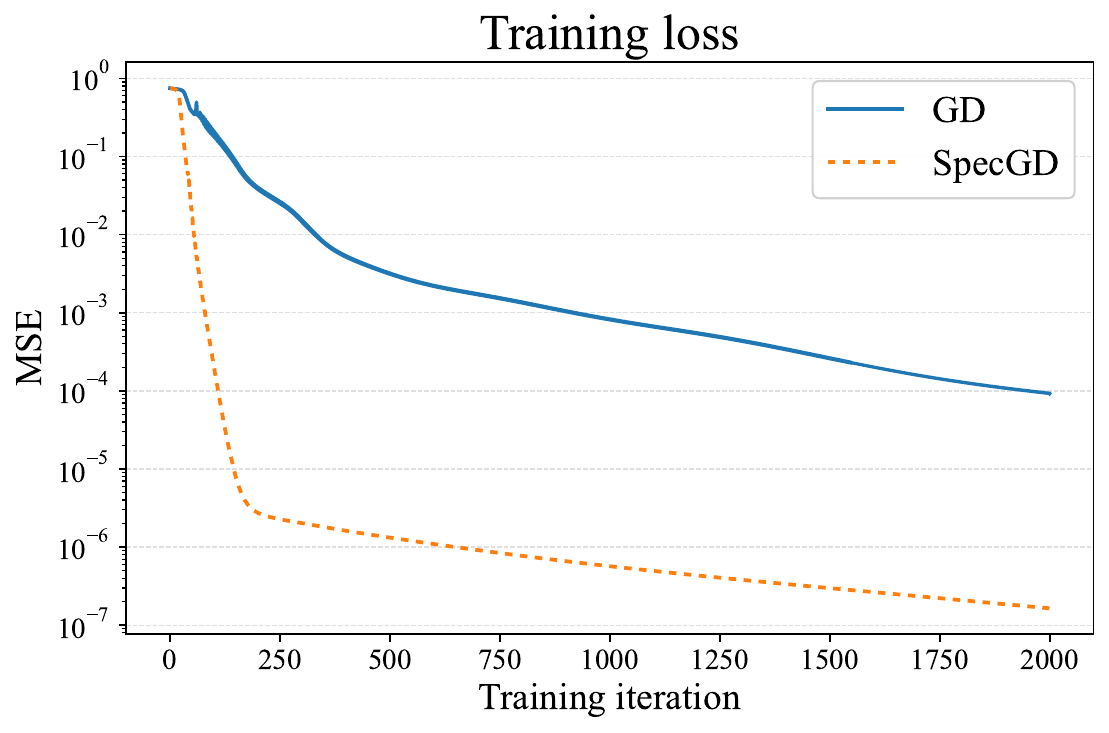}
  \caption{Layer 1}
\end{subfigure}\hfill
\begin{subfigure}[t]{0.48\textwidth}
  \centering
  \includegraphics[width=\linewidth]{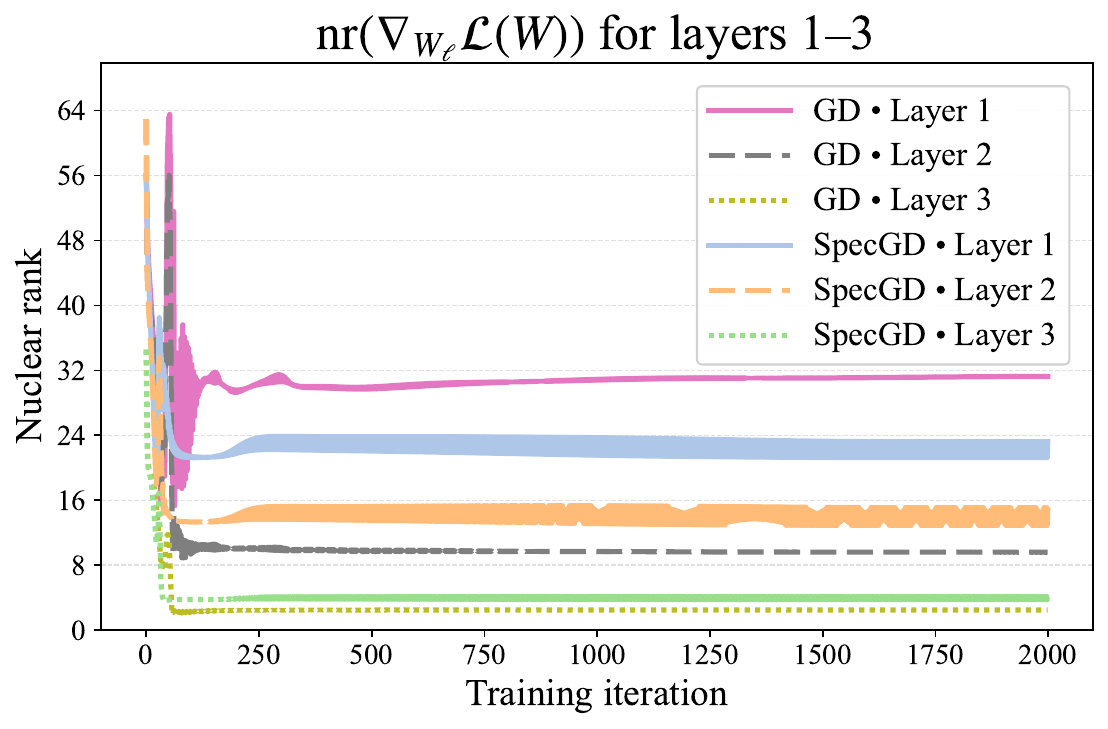}
  \caption{Layer 2}
\end{subfigure}
\caption{The training loss and gradient nuclear rank associated to the sparse regression problem in Figure~\ref{fig:synthetic_stable_rank_sparse}. We see that the training loss for $\mathtt{SpecGD}$ decreases significantly faster than for $\mathtt{GD}$ during the initial phase of training, when the nuclear rank of the gradient is large. The maximum possible nuclear rank at layer 1 is 128, while for layers 2 and 3, it is 256.}
\label{fig:synthetic_run_grad_vs_training_loss_sparse}
\end{figure}

\begin{figure}[h]
\centering
\includegraphics[width=\linewidth]{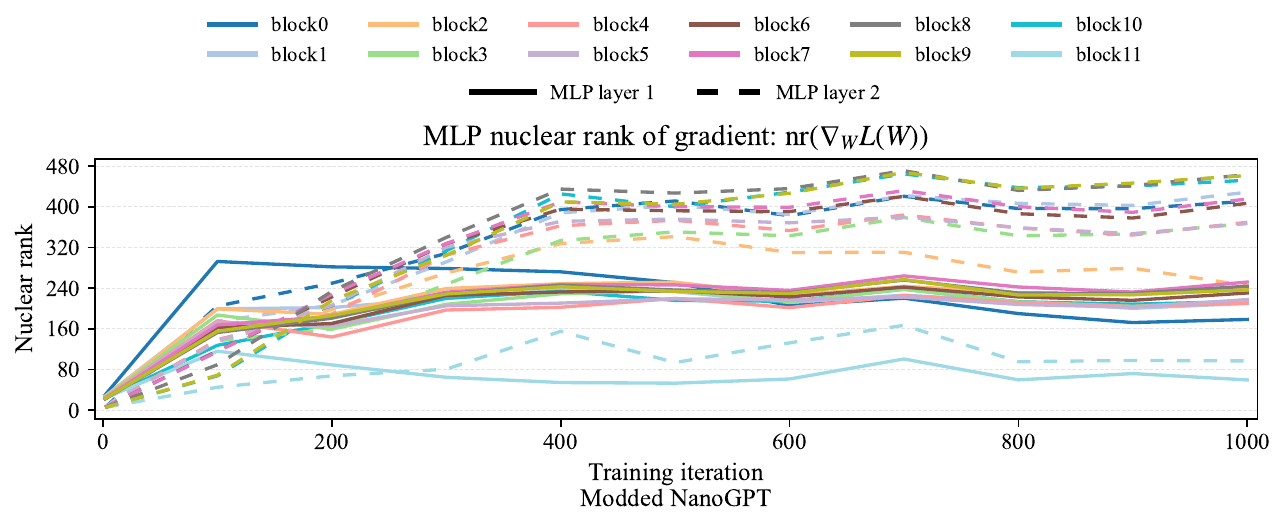}
%

\caption{Modded NanoGPT MLP gradient nuclear ranks.}
\label{fig:mlpgrad}
\end{figure}

\subsubsection{Multiple layers and transformer blocks}\label{sec:multiple}

The discussion so far has treated a single matrix block $W$ acting on a post-activation matrix $A$,
and has shown that whenever the incoming features have low stable rank and the gradient has large
nuclear rank, a spectral update on $W$ enjoys a strictly better one-step guarantee than
its Euclidean counterpart. In practice, however, we train \emph{all} layers of a deep architecture
simultaneously, and the parameters interact through a highly coupled nonlinearity. To understand
when spectral methods are still favored in this setting, we abstract the architecture into a
general layered model and study the full Hessian of the composite loss.

In Section~\ref{sec:layered-descent} we formalize this viewpoint. We fix preactivation matrices
$X_\ell(W)=W_\ell A_{\ell-1}(W)$ at each layer and let the activations $A_\ell(W)$ be obtained
from $(X_1(W),\dots,X_\ell(W))$ through a twice differentiable map $\Phi_\ell$. The outer loss
$f$ is defined on the collection of preactivations $(X_1,\dots,X_L)$, and the training objective is
$\mathcal L(W):=f(X_1(W),\dots,X_L(W))$. In this setting we show that the Hessian of $\mathcal L$
is controlled by a feature-based term plus a smaller parameter-dependent term:
\[
\big\langle \Delta W,\nabla^2\mathcal L(W)\,\Delta W\big\rangle
\;\le\; C_F(W)\sum_{\ell=1}^L \|\Delta W_\ell A_{\ell-1}(W)\|_F^2
\;+\; C_{\op}(W)\sum_{\ell=1}^L \|\Delta W_\ell\|_{\op}^2,
\]
for all perturbations $\Delta W=(\Delta W_1,\dots,\Delta W_L)$ and some finite constants
$C_F(W)$ and $C_{\op}(W)$ depending smoothly on the current iterate. For the standard $1/n$-averaged
losses we consider (squared loss, cross-entropy), $C_F(W)$ scales like $O(1/n)$, so the feature-based
curvature constants inherit the usual $1/n$ dependence on the batch size. Both $C_F(W)$ and $C_{\op}(W)$
are otherwise dimension-free in the sense that they depend only on the network architecture, depth, and current
weight operator norms, and not on width or sequence length.\footnote{For simplicity, we state our Hessian bounds with a single feature term
$\sum_{\ell} \|\Delta W_\ell A_{\ell-1}\|_F^2$ and a single parameter term
$\sum_{\ell} \|\Delta W_\ell\|_{\op}^2$. All of our arguments, most importantly the layerwise
criterion comparing $\nr(G_\ell)$ to $\st(A_{\ell-1})$, continue to hold if one
instead uses blockwise curvature weights $C_{\ell,F}(W)$ in front of
$\|\Delta W_\ell A_{\ell-1}\|_F^2$ and $C_{\ell,\op}(W)$ in front of $\|\Delta W_\ell\|_{\op}^2$.
This refinement only rescales the natural step sizes $L_\ell^{\mathrm F},L_\ell^{\mathrm op}$ for each
block and does not affect the nuclear-versus-stable-rank comparison.}

Applying Taylor's theorem with remainder, we may therefore write,
\begin{equation}\label{eq:intro-layered-taylor}
\mathcal L(W+\Delta W)
= \mathcal L(W)
+ \sum_{\ell=1}^L \big\langle G_\ell(W),\Delta W_\ell\big\rangle
+ \tfrac12\big\langle \Delta W,\nabla^2\mathcal L(W)\,\Delta W\big\rangle
+ R(\Delta W),
\end{equation}
where $G_\ell(W):=\nabla_{W_\ell}\mathcal L(W)$ and $R(\Delta W) = o(\|\Delta W\|_F)$ collects higher order
terms. Combining the Hessian bound with the basic norm inequalities
\[
\|\Delta W_\ell A_{\ell-1}\|_F
\le \|A_{\ell-1}\|_{\op}\,\|\Delta W_\ell\|_F,
\qquad
\|\Delta W_\ell A_{\ell-1}\|_F
\le \|A_{\ell-1}\|_F\,\|\Delta W_\ell\|_{\op},
\]
we obtain feature-based blockwise curvature coefficients
\[
L_\ell^{\mathrm F}(W) := C_F(W)\,\|A_{\ell-1}(W)\|_{\op}^2,
\qquad
L_\ell^{\mathrm op}(W) := C_F(W)\,\|A_{\ell-1}(W)\|_F^2.
\]
Their ratio is exactly the stable rank of the incoming features:
\[
\frac{L_\ell^{\mathrm op}(W)}{L_\ell^{\mathrm F}(W)}
= \frac{\|A_{\ell-1}(W)\|_F^2}{\|A_{\ell-1}(W)\|_{\op}^2}
= \mathrm{st}\!\big(A_{\ell-1}(W)\big).
\]

These feature-based constants suggest natural blockwise step sizes. In the
layered analysis of Section~\ref{sec:layered-descent} (see in particular
Section~\ref{subsec:descent-comparison}), we proceed as follows. For each block
$\ell$ we fix two descent directions:
\[
  \text{GD direction: } -G_\ell(W),
  \qquad
  \text{spectral direction: } -\mathrm{polar}\big(G_\ell(W)\big),
\]
and then choose scalar step sizes along these directions that minimize the
quadratic upper bound in the Taylor model~\eqref{eq:intro-layered-taylor}
under the mixed Hessian bound. This produces two full updates:
a Euclidean update $W^{\mathrm{GD}}$ and a mixed spectral update
$W^{\mathrm{Spec}}$ which applies spectral steps on blocks in
$\mathcal S \subseteq [L]$ and Euclidean steps on the remaining blocks.

Section~\ref{subsec:descent-comparison} shows that the spectral update has at
least as much predicted decrease as the Euclidean update yet again when
the nuclear-rank versus
stable-rank inequality holds:
\begin{equation}\label{eq:intro-key-condition}
  \nr(G_\ell(W))
  \;\ge\;
  \st\!\big(A_{\ell-1}(W)\big).
\end{equation}
Thus, already at the level of the layered quadratic model, we recover the same
criterion as in the random-feature setting: spectral updates are favored on
those blocks whose incoming activations $A_{\ell-1}(W)$ have low stable rank
while their gradients $G_\ell(W)$ have high nuclear rank.

For fully connected networks this picture is particularly transparent. The matrices
$A_{\ell-1}$ are precisely the layerwise post-activations introduced in
\eqref{eqn:post_activ}. Section~\ref{sec:spikes_and_embeddings} shows that with when the activation are mean spike indusing, the stable rank $\mathrm{st}(A_{\ell-1})$ is
bounded by a numerical constant for all internal layers at initialization, while the input $A_0=X$ need not have
any such structure. In this regime, the constants $L_\ell^{\mathrm op}$ and
$L_\ell^{\mathrm F}$ differ only by a fixed factor for $2\le \ell\le L-1$, but may be very
different at the input and output layers. Our implementation of $\mathtt{SpecGD}$  therefore uses the spectral step with step size
$\|G_\ell\|_*/L_\ell^{\mathrm op}$ on all internal weight matrices, and keeps the first and last
layers on standard Euclidean steps with step size $1/L_\ell^{\mathrm F}$. In other words, we
apply spectral updates exactly where the incoming post-activations are empirically low-stable-rank,
and retain Euclidean updates where the features (raw inputs or task-specific heads) are not
guaranteed to be degenerate.

Transformers require a slightly more detailed bookkeeping, because the matrices which appear
multiplying a given weight are composites of several operations: normalization, attention,
residual connections and MLPs. Nevertheless, the layered formalism and the seminorm
$\sum_\ell\|\Delta W_\ell A_{\ell-1}\|_F^2$ apply unchanged once we view each transformer block
as built from linear maps $W_\ell$ composed with smooth activation maps~$\Phi_\ell$. For concreteness, we write formulas for a single decoder block acting on a sequence representation
$X\in\RR^{d_{\mathrm{model}}\times T}$ (hidden dimension $d_{\mathrm{model}}$, sequence length
$T$), and we suppress masking, multiple-heads, and positional encodings (e.g.\ RoPE) from the notation, though we include them in Section~\ref{sec:transformers}.
The single-head attention sublayer can be written as
\begin{align*}
A^{\mathrm{rms}} &= \mathrm{RMSNorm}(X),\\
Q &= W_Q A^{\mathrm{rms}},\qquad
K = W_K A^{\mathrm{rms}},\qquad
V = W_V A^{\mathrm{rms}},\\
P &= \mathrm{softmax}\big(d_{\mathrm{model}}^{-1/2} K^\top Q\big),\\
H &= V P,\\
X^{\mathrm{att}} &= X + W_O H,
\end{align*}
where $W_Q,W_K,W_V,W_O$ are the attention weight matrices and $d_{\mathrm{embed}}$ is the embedding dimension. The
MLP sublayer then takes the form
\begin{align*}
A^{\mathrm{rms}}_{\mathrm{mlp}} &= \mathrm{RMSNorm}(X^{\mathrm{att}}),\\
B &= \sigma(W_1 A^{\mathrm{rms}}_{\mathrm{mlp}}),\\
X^+ &= X^{\mathrm{att}} + W_2 B,
\end{align*}
where $\sigma$ is a pointwise nonlinearity (e.g.\ GELU) and $W_1,W_2$ are the MLP weights. At the
 end of the network, a language-modeling head applies an RMS normalization followed by a
linear map
\[
A^{\mathrm{rms}}_{\mathrm{final}} = \mathrm{RMSNorm}(X^+),\qquad
Z = W_{\mathrm{lm}} A^{\mathrm{rms}}_{\mathrm{final}},
\]
producing logits $Z$ over the vocabulary.

In this block-level description, each trainable matrix appears exactly in the affine form
$X_\ell = W_\ell A_{\ell-1}$ required by our layered model. The corresponding incoming features
$A_{\ell-1}$ are:
\[
A_{\ell-1} \in
\bigl\{
A^{\mathrm{rms}},\,
A^{\mathrm{rms}}_{\mathrm{mlp}},\,
B,\,
H,\,
A^{\mathrm{rms}}_{\mathrm{final}}
\bigr\},
\]
depending on whether $W_\ell$ is one of $W_Q,W_K,W_V,W_O,W_1,W_2,W_{\mathrm{lm}}$. 
For simplicity we have written $\mathrm{RMSNorm}$ as parameter–free, but in practice each normalization can carry a learned weight vector $\gamma\in\R^{d_{\mathrm{model}}}$ acting elementwise as
\(
x \mapsto \gamma \had \frac{x}{\|x\|_{\mathrm{rms}}}.
\)
In our notation these weights can be treated as additional diagonal blocks
$W_{\mathrm{rms}}:=\Diag(\gamma)$ acting on the parameter–free RMS–normalized activations $\widetilde A^{\mathrm{rms}}$, so that $A^{\mathrm{rms}}=W_{\mathrm{rms}}\widetilde A^{\mathrm{rms}}$ and the affine form $X_\ell = W_\ell A_{\ell-1}$ continues to hold. The corresponding updates are slightly different. Indeed, on diagonal blocks spectral updates reduce to a sign–type update in the gain. We discuss this in more detail in Appendix~\ref{sec:rms-diagonal}. Note that the modded-NanoGPT repo~\cite{modded_nanogpt_july18} does not train the RMSNorm parameter, and instead uses fixed weight $\gamma = \mathbf{1}$.

Our
measurements in Figure~\ref{fig:mlpnano} (MLP post-activations $B$) and in Figure~\ref{fig:rmsactivations} (RMS activations $A^{\mathrm{rms}},\,
A^{\mathrm{rms}}_{\mathrm{mlp}},\,
A^{\mathrm{rms}}_{\mathrm{final}}$) show that the RMS-normalized activations
$A^{\mathrm{rms}},A^{\mathrm{rms}}_{\mathrm{mlp}},A^{\mathrm{rms}}_{\mathrm{final}}$ and the
MLP post-activations $B$ all have low stable rank throughout training. For the attention output
features $H$ we may further write
\[
H = V P = W_V A^{\mathrm{rms}} P,
\]
so that for perturbations of the output projection $W_O$ we have
\[
\Delta W_O H = \Delta W_O W_V A^{\mathrm{rms}} P.
\]
Here the RMS-normalized activations $A^{\mathrm{rms}}$ sit \emph{between} two linear operators,
but submultiplicativity still isolates their contribution:
\[
\|\Delta W_O H\|_F
= \|\Delta W_O W_V A^{\mathrm{rms}} P\|_F
\;\le\; \|\Delta W_O\|_{\op}\,\| W_V\|_{\op}\,\|A^{\mathrm{rms}}\|_F\,\|P\|_{\op}.
\]
More generally, each seminorm term $\|\Delta W_\ell A_{\ell-1}\|_F$ associated with a
transformer block can be written as a product of the form
\[
\Delta W_\ell A_{\ell-1}
= \Delta W_{\ell} M_{\ell,\mathrm{left}}\;\widetilde A_{\ell-1}\;M_{\ell,\mathrm{right}},
\]
where $\widetilde A_{\ell-1}$ is one of the low-stable-rank matrices
$A^{\mathrm{rms}},A^{\mathrm{rms}}_{\mathrm{mlp}},B,A^{\mathrm{rms}}_{\mathrm{final}}$, and
$M_{\ell,\mathrm{left}},M_{\ell,\mathrm{right}}$ are fixed linear maps (products of attention
projections, residual additions, or the attention kernel $P$). Using
\[
\|\Delta W_\ell A_{\ell-1}\|_F
\;=\; \|\Delta W_\ell\; M_{\ell,\mathrm{left}}\;\widetilde A_{\ell-1}\;M_{\ell,\mathrm{right}}\|_F
\;\le\; \|\Delta W_\ell\|_{\op}\|M_{\ell,\mathrm{left}}\|_{\op}\,\|M_{\ell,\mathrm{right}}\|_{\op}\,
\|\widetilde A_{\ell-1}\|_F,
\]
we see that the curvature constants relevant for block~$\ell$ scale with
$\|\widetilde A_{\ell-1}\|_F^2$, and the ratio $L_\ell^{\mathrm op}/L_\ell^{\mathrm F}$ is
again governed by the stable rank of these RMS and MLP activations.

To make the connection between gradient nuclear ranks and the upstream feature
matrices completely explicit, it is useful to summarize, for each parameter block, which
activation matrix enters the condition~\eqref{eq:intro-key-condition}. Table~\ref{tab:grad-activation}
records the correspondence we will use in our experiments.

\begin{table}[h]
\centering
\begin{tabular}{ll}
\hline
Parameters $W_\ell$ & Activation in \eqref{eq:intro-key-condition} (for $W_\ell$) \\
\hline
$W_Q$           & $A^{\mathrm{rms}}$ \\
$W_K$           & $A^{\mathrm{rms}}$ \\
$W_V$           & $A^{\mathrm{rms}}$ \\
$W_O$           & $A^{\mathrm{rms}}$ (via $H = W_V A^{\mathrm{rms}} P$) \\
$W_1$           & $A^{\mathrm{rms}}_{\mathrm{mlp}}$ \\
$W_2$           & $B$ \\
$W_{\mathrm{lm}}$ & $A^{\mathrm{rms}}_{\mathrm{final}}$ \\
Embedding       & $X$ (moderate stable-rank guarantee) \\
\hline
\end{tabular}
\caption{For each parameter block $W_\ell$, the activation matrix whose stable rank appears on the
right-hand side of~\eqref{eq:intro-key-condition}. The corresponding left-hand side is always the
gradient nuclear rank $\nr(G_\ell)$ for that same block.}
\label{tab:grad-activation}
\end{table}

Figures~\ref{fig:mlpgrad}, \ref{fig:rmsactivations}, \ref{fig:gradientembeddingslmhead}, 
\ref{fig:gradientattentionqk}, and 
\ref{fig:gradientattentionvo} illustrate how these quantities behave in a NanoGPT run. The plots
in Figure~\ref{fig:mlpgrad} show that the gradient nuclear ranks for the MLP weights remain
consistently large along training, while Figures~\ref{fig:mlpnano} and~\ref{fig:rmsactivations}
demonstrate that both the MLP post-activations $B$ and the RMS-normalized activations
$A^{\mathrm{rms}},A^{\mathrm{rms}}_{\mathrm{mlp}},A^{\mathrm{rms}}_{\mathrm{final}}$ have
stable rank bounded by a small numerical constant. Thus, for these blocks, the inequality
\eqref{eq:intro-key-condition} is strongly activated, and the simple GD-based model predicts a clear advantage
for spectral-style updates. This is consistent with the convergence bounds presented in~\cite{pethick2025scion}, where improvements over Euclidean methods are also driven by
large nuclear-to-Frobenius ratios. 

\begin{figure}[h]
\centering
\includegraphics[width=\linewidth]{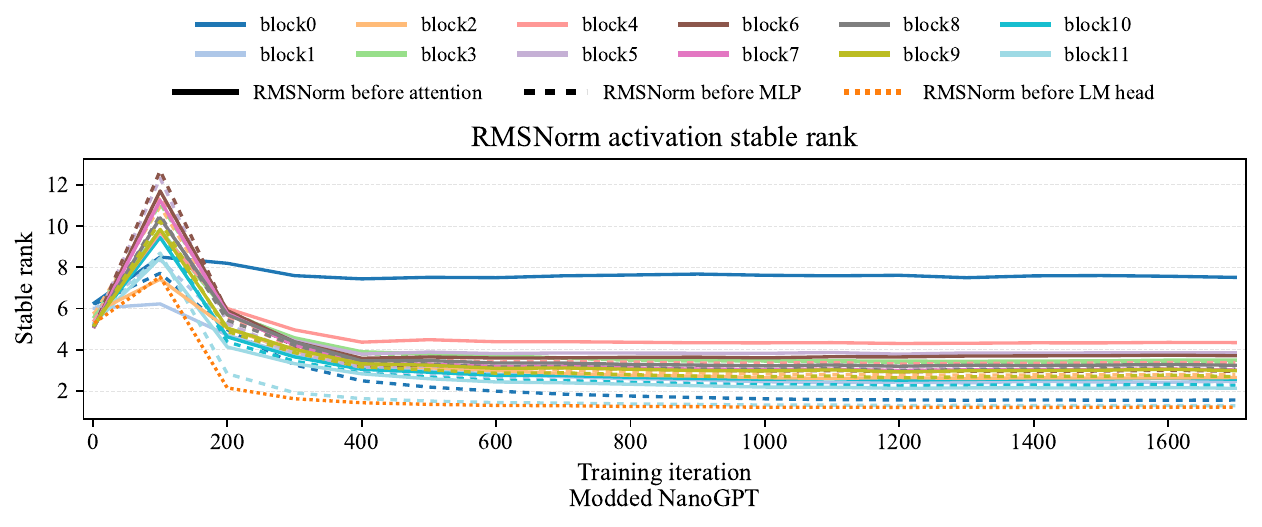}
\caption{Stable rank of RMS-normalized activations $A^{\mathrm{rms}},A^{\mathrm{rms}}_{\mathrm{mlp}},
A^{\mathrm{rms}}_{\mathrm{final}}$.}
\label{fig:rmsactivations}
\end{figure}

\begin{figure}[h]
\centering
\includegraphics[width=\linewidth]{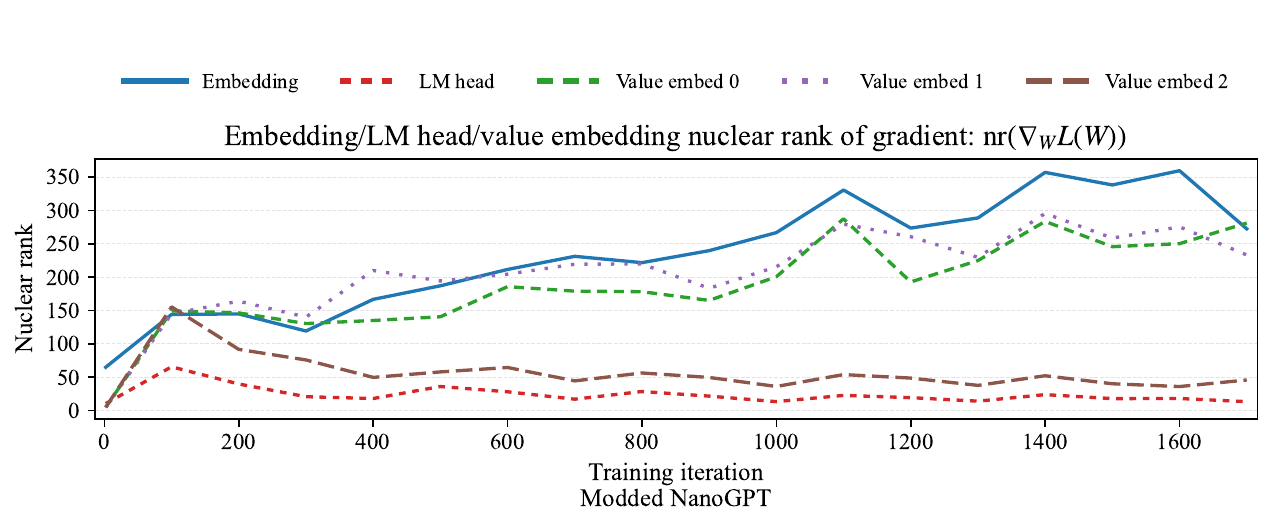}
\caption{Gradient nuclear ranks for the token embedding and language-model head
parameters in the same NanoGPT run. The mean ratio for the LM head over the last $1000$ steps
is approximately $42.68$.}
\label{fig:gradientembeddingslmhead}
\end{figure}

\begin{figure}[h]
\centering
\includegraphics[width=\linewidth]{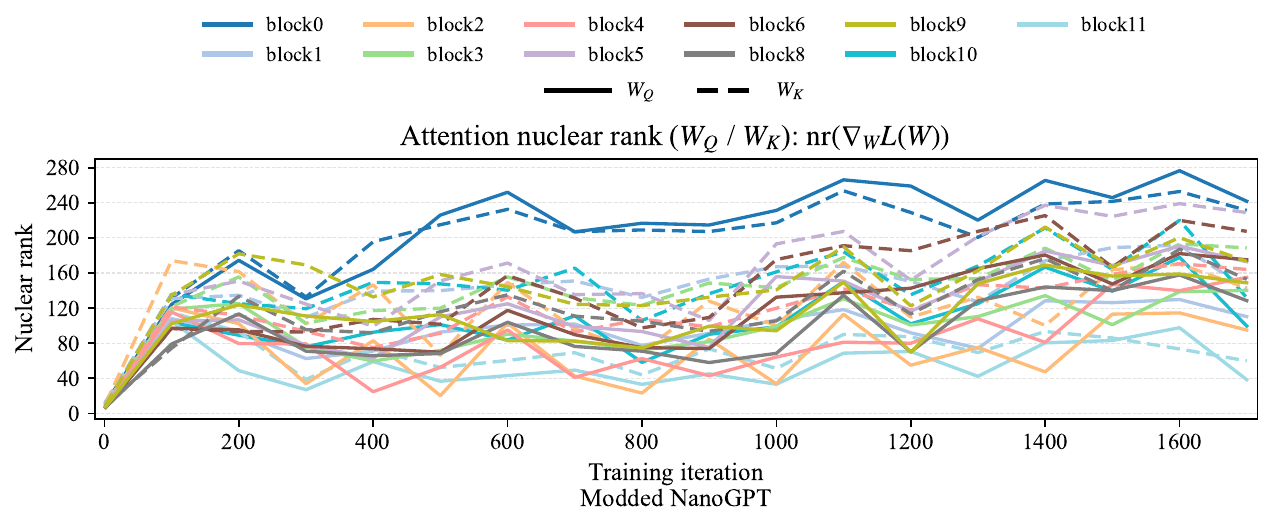}
\caption{Gradient nuclear ranks for the attention parameters
$W_Q,W_K$ in the same NanoGPT run.}
\label{fig:gradientattentionqk}
\end{figure}

\begin{figure}[h]
\centering
\includegraphics[width=\linewidth]{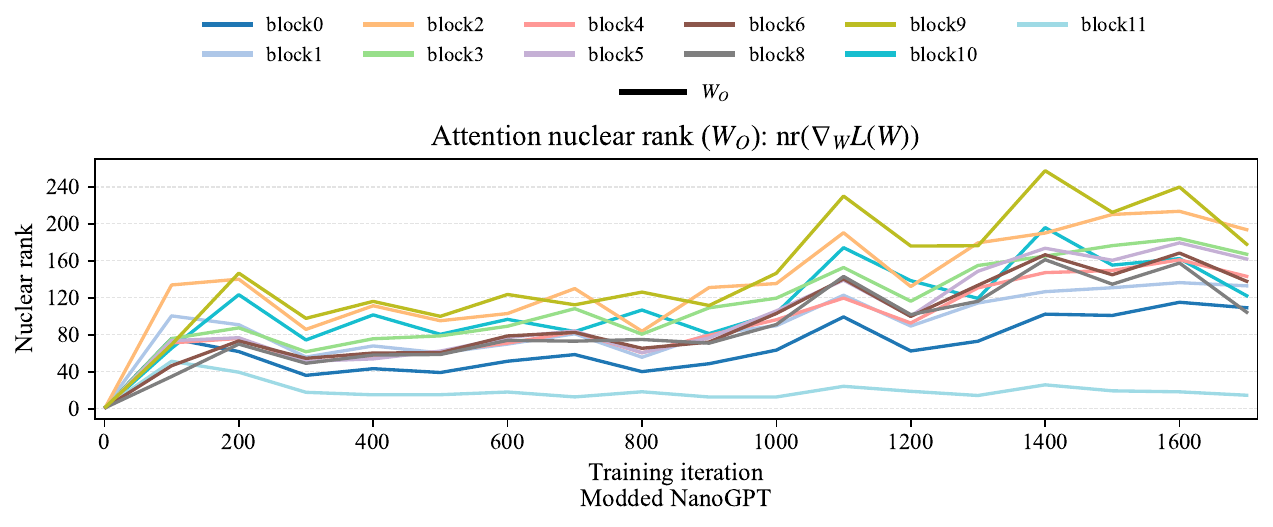}
\includegraphics[width=\linewidth]{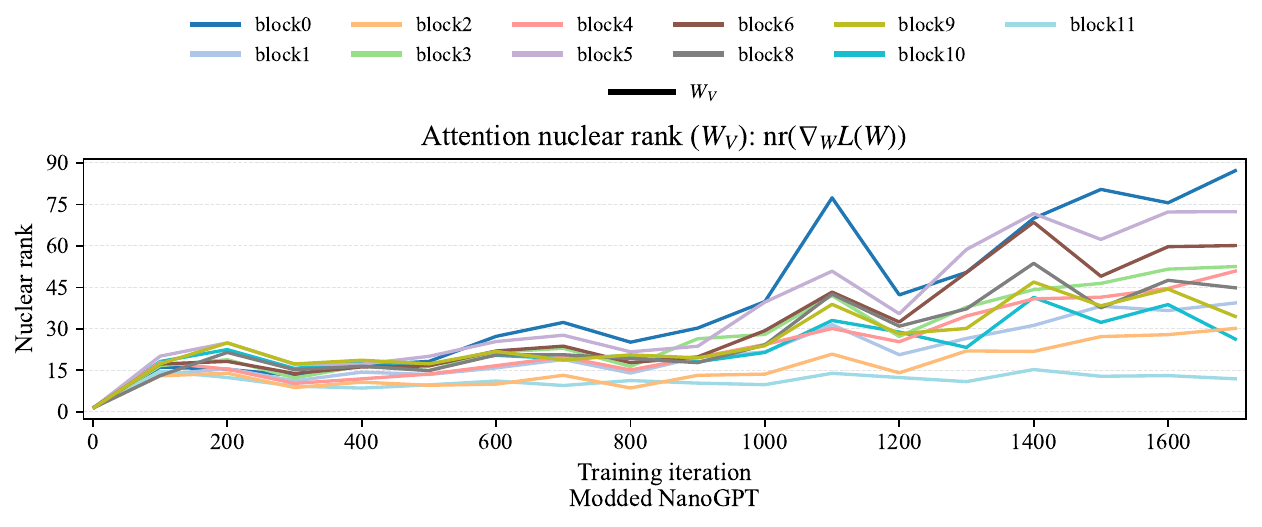}
\caption{Gradient nuclear ranks for the attention parameters
$W_O,W_V$ in the same NanoGPT run.}
\label{fig:gradientattentionvo}
\end{figure}

Figures~\ref{fig:gradientattentionqk} and \ref{fig:gradientattentionvo} shows a similar behavior for the attention weights:
the gradient nuclear ranks for $W_Q,W_K,W_O$ are typically large and $W_V$ is more modestly sized, while their
reference activations in Table~\ref{tab:grad-activation} are the low-stable-rank RMS features
of Figure~\ref{fig:rmsactivations}. In contrast, Figure~\ref{fig:gradientembeddingslmhead}
reveals a different picture for the input embedding weights and the language-model head. The embedding
matrix multiplies a raw token-indicator matrix whose stable rank hovers around $26$ (see section~\ref{sec:embedding} for more on this), which is modestly
low relative to the ambient dimension but noticeably larger than the stable ranks of the internal
post-activations; the nuclear rank of the embedding gradient grows from roughly $100$ to $400$
over the course of training (Figure~\ref{fig:embedding}). Thus the ratio
$\nr(G_{\mathrm{emb}})/\mathrm{st}(X)$ is only moderately larger than one and much smaller than
the corresponding ratios for the hidden MLP and attention blocks, so our condition
\eqref{eq:intro-key-condition} predicts a milder advantage for spectral updates on the
embedding matrix. This matches our empirical experience with the NanoGPT: applying $\mathtt{MUON}$-style spectral updates to
the input embedding in the Modded-NanoGPT run changes training and validation curves only slightly (not shown).
From the viewpoint of~\cite{pethick2025scion}, a more natural geometry for the
embedding block is the $\ell_1\!\to\!\ell_2$ operator norm, since the input tokens are one-hot
vectors of $\ell_1$-norm one; in that setting the relevant Lipschitz constants are controlled by
the maximum column norm of the embedding matrix rather than its spectral norm, and the
stable-rank heuristic is replaced by an $\ell_1$-based argument. We do not pursue this alternative
geometry here. 

\begin{figure}[h]
\centering
\includegraphics[width=\linewidth]{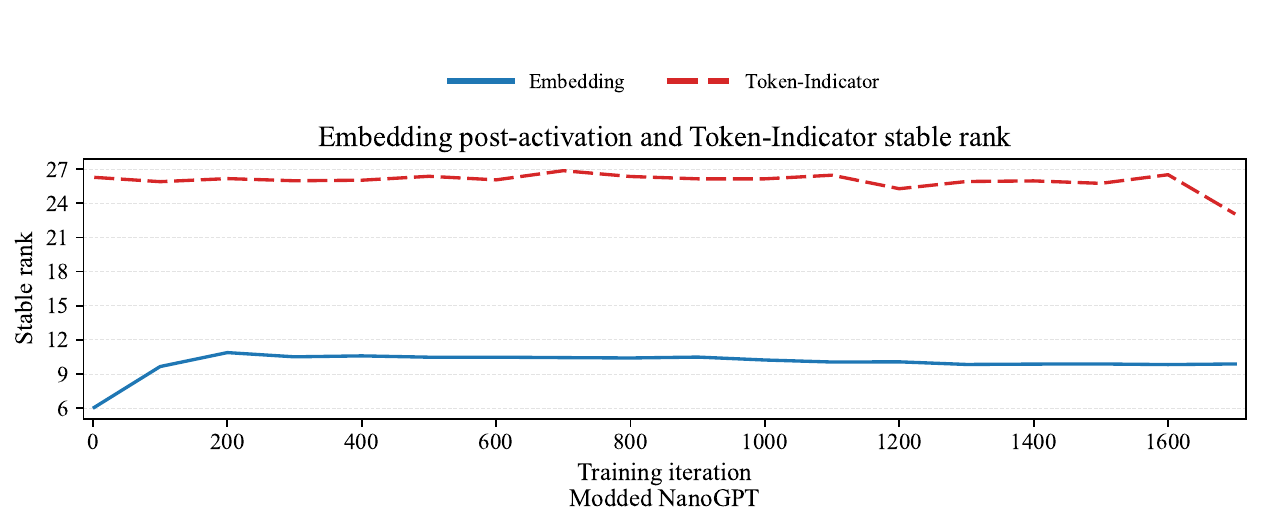}
\caption{Embeddings matrix post-activation and Token-Indicator stable rank throughout training. See Section~\ref{sec:embedding} for more on these matrices.}
\label{fig:embedding}
\end{figure}

To sum up, all matrix-valued weights
interior to the the transformer blocks---the attention projections and MLP matrices---see incoming features
that are effectively low-stable-rank in the sense relevant for the Hessian bound, and their
gradient nuclear ranks are typically large. For these blocks the inequality
\eqref{eq:intro-key-condition} is strongly activated, and spectral updates have a clear structural
advantage over Euclidean ones. The input embedding weight matrix, by contrast, multiplies raw
token indicators and lies only marginally in this regime, while the language-model head sits in an
intermediate regime where the same inequality holds but with a much smaller margin. The resulting
picture is consistent with current practice: the modded-NanoGPT repository applies spectral
updates (implemented via $\mathtt{MUON}$) to all $2$D weight blocks inside the transformer, while
keeping both the token embedding and the language-model head on \texttt{AdamW}-style updates. Moreover,
Section~\ref{sec:transformers} shows that,
at initialization, the transformer activations that feed these blocks have stable rank bounded by
numerical constants independent of width and sequence length, uniformly for the MLP post-activations,
and growing at most quadratically with the layer index for the remaining RMS-normalized and hidden
states. In view of the momentum-based implementation of $\mathtt{MUON}$ and the different
effective geometry of \texttt{AdamW}, our comparison should be read as a geometric heuristic
instead of a precise hyperparameter recommendation: the nuclear-rank versus stable-rank condition
explains why spectral geometry is well matched to the internal layers of MLPs
and transformers, even though our one-step comparison is carried out for Euclidean gradient
descent instead of \texttt{AdamW}.

\subsection{When do we expect the nuclear rank of the gradient to be large?}
So far, we have primarily focused on understanding when the stable rank of the post-activation matrices is small. On the other hand, in order for the spectral updates to be superior to Euclidean gradient updates according to the rule \eqref{eq:intro-key-condition}, the nuclear rank of the gradient should be large. Naively, one may think that the low stable rank of the post-activation matrix may inadvertently translate into a low nuclear rank of the gradient, thereby nullifying our thesis. On the contrary, in Section~\ref{sec:atinitialization} we show that in the representative model of random feature regression \eqref{eqn:rfr_intro} both in the realizable and student-teacher set-ups, the nuclear rank of the gradient is expected to be large through training. More precisely, consider running the Euclidean gradient method on the problem \eqref{eqn:rfr_intro} initialized with the all-zero matrix $W_0$ and with stepsize $\eta>0$. Let $W_t$ be the resulting weight matrices generated across the iterations. We show that if the activation function $\sigma$ is mean-spike inducing (or more generally if $AA^{\top}$ follows a spiked model), the following are true.

\begin{itemize}
    \item When using a nearly maximal stepsize of the form $\eta=\frac{1}{c+\|A\|_{\op}^2}$ for $c>0$, the weight matrix $W_1$ generated by a single step has nuclear rank that scales with the dimension $\nr(\nabla \mathcal{L}(W_1))\gtrsim d$.
    \item When using a more realistic stepsize $\eta=\frac{1}{C\|A\|_{\op}^2}$ for $C>1$, after a short $O(\log d)$ burn-in period there is a window of $\Theta(d)$ iterations on which $\nr(\nabla\mathcal{L}(W_t))=\Omega(d)$ and, for any fixed $\varepsilon>0$, a longer window of $\Theta(d\log d)$ iterations on which $\nr(\nabla\mathcal{L}(W_t))\ge d^{1-\varepsilon}$ (Theorem~\ref{thm:multimodel1} and Corollary~\ref{cor:main_multi_spike_hard}). 
\end{itemize}
 In the same models, Euclidean gradient descent needs $\Theta(d\log(d/\delta))$ steps to reach relative error $\delta\in(0,1)$, so for any target accuracy this $\Theta(d\log d)$ window represents a constant fraction  of the training time. Thus, if we were to restart  $\mathtt{SpecGD}$ from any point in this window, its one-step improvement over Euclidean gradient descent would be dimension-dependent.
Notice that these results are consistent with the nuclear ranks plotted in our experiments on random feature regression (Figures~\ref{fig:random_features}) and the MLP layer in modded NanoGPT (Figure~\ref{fig:mlpgrad}), wherein the gradient nuclear rank starts off low but quickly increases through training.

\subsection{When are spectral updates less helpful?}

Given the discussion thus far the reader might conclude that spectral updates are always preferred to Euclidean ones.  This is not the case.  The random-feature model with {\em centered activations}, already provides a simple counterexample in which the nuclear-rank condition~\eqref{eq:sd-vs-gd-condition} fails and spectral updates are \emph{slower} than Euclidean ones.

To demonstrate, we repeat the experiment from Figure~\ref{fig:random_features}, keeping the same dimensions, data distribution, and ground-truth matrix $W_\sharp$, but replacing ReLU with a SwiGLU activation; the result is summarized in Figure~\ref{fig:swiglu_random_features}.  Concretely, in this variant the random-feature matrix is
\[
A \;=\; \mathrm{SwiGLU}(W_1 X, W_2 X)
\;:=\; \mathrm{SiLU}(W_1 X)\had (W_2 X),
\]
where $W_1,W_2\in\R^{k\times d}$ are independent Gaussian weight matrices, the Hadamard product $\had$ is taken coordinatewise, and
\[
\mathrm{SiLU}(t) \;:=\; t\cdot (1+e^{-t})^{-1}
\]
is the usual Swish/SiLU nonlinearity applied elementwise.  Thus each feature coordinate is of the form
\[
a_{ij} \;=\; \mathrm{SiLU}((W_1 X)_{ij})\,(W_2 X)_{ij},
\]
in contrast to the ReLU experiment where $A=\sigma(W_1 X)$ depends on a single weight matrix.  We will refer to such activations of the form $\phi(W_1 X)\had (W_2 X)$ as \emph{gated activations}.  In particular, in the Gaussian setup above each feature row is (exactly) centered: $\EE[a]=0$.

In Figure~\ref{fig:random_features} (ReLU features), the post-activation matrices $A$ are empirically low--stable--rank.  As explained in Section~\ref{subsec:nsr_moment_ratio}, non-centered activations such as ReLU produce a large rank-one contribution to the second moment: equivalently, $(1/k)A^\top A$ develops a spike aligned with the mean row vector $\EE a$.  This spike dominates the spectrum and keeps $\st(A)$ bounded by a numerical constant even as the width grows.  Gated activations like SwiGLU behave differently.  Here $\EE a=0$, so the mean-induced spike in $A^\top A$ is absent.  The resulting Gram matrix has a more isotropic spectrum, and the stable rank $\st(A)$ now grows proportionally with the width instead of remaining $O(1)$, while the nuclear rank $\nr(\nabla\cL(W))$ of the gradient remains modest.  In the Gaussian random-feature experiment considered here this means that $A^\top A$ has essentially no spike at initialization: its spectrum is bulk-only.  For data that already carry a low-rank spike, such as the embedding/token-indicator input of a transformer at initialization (see Section~\ref{sec:embedding}), centered activations are known to propagate that \emph{informative} spike, as in the nonlinear spiked-covariance and conjugate-kernel analyses discussed in Section~\ref{sec:related-work}~\cite{benigni2022largest,wang2024nonlinear}; what they avoid is exactly the large, uninformative mean-induced spike that drives the low--stable--rank regime in our ReLU examples.

In this experiment the inequality
\[
\nr(\nabla\cL(W)) \;\geq\; \st(A)
\]
is violated: the right-hand panel of Figure~\ref{fig:swiglu_random_features} shows the nuclear-rank curves dipping below the dashed stable-rank level.  The one-step comparison in~\eqref{eq:sd-vs-gd-condition} therefore predicts that gradient descent should outperform spectral descent, and the left-hand panel confirms this prediction: \texttt{GD} reaches a much smaller objective value in far fewer iterations than \texttt{SpecGD}.

\begin{figure}[t]
\centering
\includegraphics[width=\textwidth]{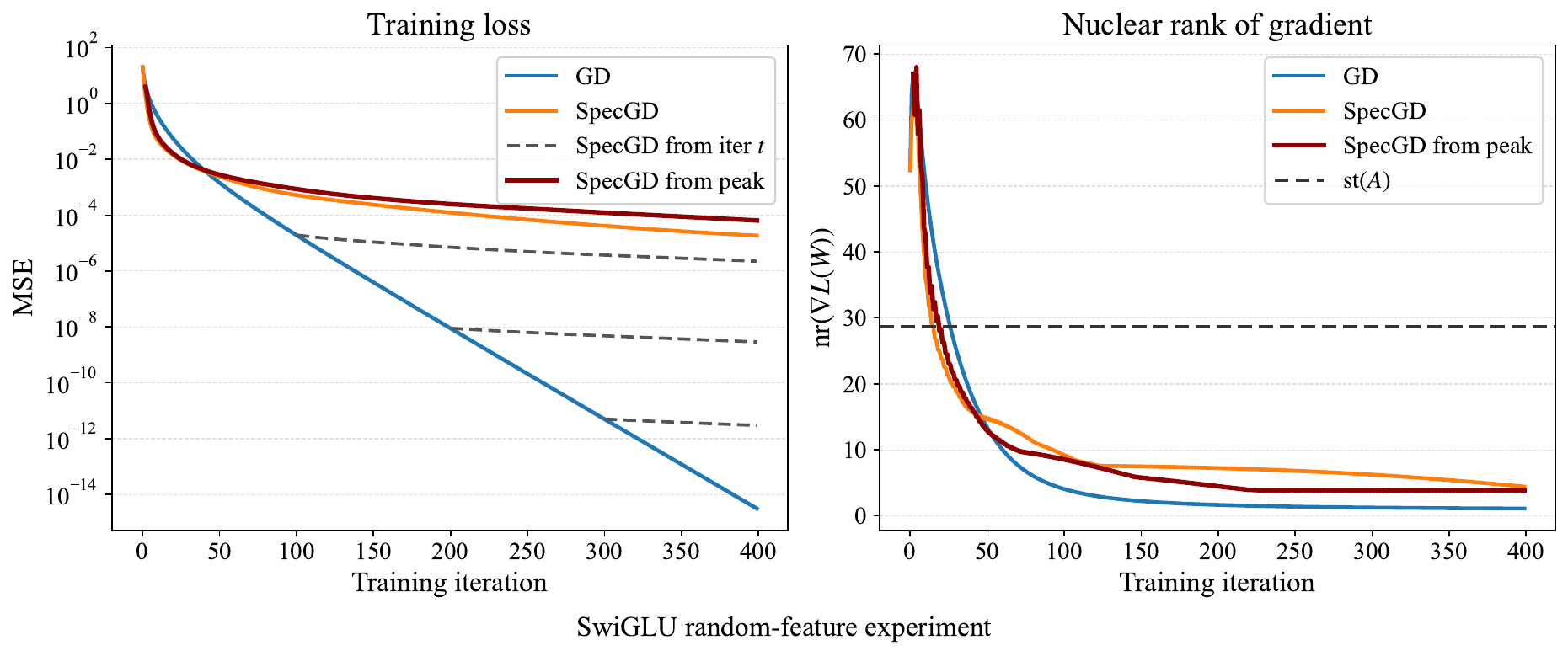}
\caption{\textbf{A cautionary tale.}  Random-feature regression with SwiGLU features
\mbox{$A=\mathrm{SwiGLU}(W_1 X,W_2 X)=\mathrm{SiLU}(W_1 X)\had (W_2 X)$},
in the same setup as Figure~\ref{fig:random_features} but with the ReLU nonlinearity replaced by SwiGLU.  \emph{Left:} objective gap $\cL(W)-\min\cL$ along gradient descent (\texttt{GD}, solid blue) and spectral gradient descent (\texttt{SpecGD}, solid orange).  Unlike the ReLU experiment in Figure~\ref{fig:random_features}, here \texttt{SpecGD} is \emph{slower} than \texttt{GD}.  \emph{Right:} nuclear rank $\nr(\nabla\cL(W))$ of the gradient (solid curves) together with the stable rank $\st(A)$ of the post-activations (black dashed line).  Because SwiGLU yields centered post-activations at Gaussian initialization, the mean-induced spike is absent and $\st(A)$ scales with width.  The inequality $\nr(\nabla\cL(W)) \gtrsim \st(A)$ is more easily violated during training, and the one-step guarantees from Section~\ref{sec:rand_regress} predict that Euclidean updates should outperform spectral ones.}
\label{fig:swiglu_random_features}
\end{figure}

In Appendix~\ref{sec:swiglu}, we also repeat the sparse regression experiment of Figure~\ref{fig:synthetic_stable_rank_sparse} with a SwiGLU activation.  There we see that \texttt{SpecGD} can still perform well at small batch sizes, where the post-activations remain moderately low--stable--rank, but that its advantage is quickly lost as the batch size (and hence the stable rank) increases.

For transformers, however, this toy counterexample does not preclude low--stable--rank internal features.  In Appendix~\ref{sec:swiglu-nanogpt}, we repeat the modded-NanoGPT diagnostics with a SwiGLU MLP, finding that the internal activations remain low--stable--rank (though typically larger than in the squared-ReLU run) and the corresponding gradient nuclear ranks remain large.  This is consistent with the results of Section~\ref{sec:transformers}: even in the nonMSI/gated regime, low stable rank is already present at the input through the token-indicator/embedding structure (Section~\ref{sec:embedding}) and is then preserved through attention and MLP blocks by the propagation bounds, yielding width- and sequence-length-independent stable-rank control at initialization.

\begin{table}[h]
\centering
\begin{tabular}{lcc}
\hline
Method & Non-gated MLP & Gated MLP \\
\hline
All Adam & 3.9242 & 3.8837 \\
Muon($W_{\mathrm{out}}$ only) & 3.7023 & 3.7843 \\
Muon(All Attn, MLP) & \textbf{3.5654} & \textbf{3.5125} \\
\hline
\end{tabular}
\caption{Validation loss at $10k$ training steps for the MLP ablations of Wang et al.~\cite{wang2025muon}.  The model is a $160$M-parameter NanoGPT-style decoder-only transformer trained on FineWeb, with columns corresponding to standard (non-gated) and gated MLP variants as in Eqns.~(3.2)–(3.3) of~\cite{wang2025muon}.  In the gated case, the MLP uses a gated activation of the form $\phi(W_{\mathrm{in}} h)\had (W_{\mathrm{gate}} h)$ feeding $W_{\mathrm{out}}$.}
\label{tab:muon_wout_gating}
\end{table}

As a point of comparison, Table~\ref{tab:muon_wout_gating} summarizes the MLP ablations reported by Wang et al.~\cite{wang2025muon}.  For the non-gated MLP, switching only $W_{\mathrm{out}}$ from Adam to Muon reduces the validation loss from $3.9242$ (All Adam) to $3.7023$, while full Muon on all attention and MLP weights reaches $3.5654$; thus about $62\%$ of the total improvement from Adam to full Muon is already captured by changing $W_{\mathrm{out}}$ alone.  In the gated MLP, the corresponding losses are $3.8837$ (All Adam), $3.7843$ (Muon on $W_{\mathrm{out}}$ only), and $3.5125$ (Muon on all attention and MLP weights), so the change on $W_{\mathrm{out}}$ accounts for only about $27\%$ of the full gain.  One plausible interpretation, in light of the SwiGLU random-feature example above, is that gated activations produce more nearly centered, higher--stable--rank post-activations feeding $W_{\mathrm{out}}$, which makes Euclidean updates on that block better conditioned and leaves less room for a spectral update to help.  A detailed verification of this mechanism in large transformers would require additional ablations and lies beyond our scope here.

\subsection{Related literature}
\label{sec:related-work}

We briefly discuss work on (i) spectral and norm-constrained optimizers such as
\texttt{SpecGD}, \texttt{Muon}, and \texttt{Scion}, and (ii) empirical and theoretical
studies of low-rank structure in gradients and activations. We focus only on results that are relevant to the structural phenomena studied in this paper.

\paragraph{Spectral and norm-constrained optimizers.}
Preconditioned Spectral Descent (PSD) and Stochastic Spectral Descent (SSD) were among the
first optimizers to use spectral geometry for training deep models. Carlson et al.\ formulate
non-Euclidean gradient descent in general Schatten norms and show that for objectives built
from log-sum-exp terms (including RBMs, feedforward networks, and CNNs), the loss admits
a tighter majorization bound in the Schatten-$\infty$ norm than in the Frobenius norm; this
motivates steepest descent in the spectral norm rather than the Euclidean norm
\cite{carlson2015psd,carlson2015ssd}. Their analysis is entirely in terms of norm-dependent
Lipschitz constants and matrix updates; they do not study nuclear ranks ratios of the
gradient, nor connect the choice of geometry to the stable rank of the propagated data.

\texttt{Muon}~\cite{jordan2024muon} applies a similar spectral update---implemented with a Newton--Schulz approximation to the polar factor---to the hidden layers of modern
architectures and has shown strong empirical performance for NanoGPT-scale models and other
large-scale training tasks. Related work by Bernstein and collaborators has advocated
norm-aware training via ``modular norms'' and ``modular duality,'' helping to popularize
spectral and operator-norm updates in large-scale settings and providing practical
Newton--Schulz-based orthogonalization routines used in \texttt{Muon} implementations
\cite{large2024modularnorm,bernstein2024oldoptimizer,bernstein2025modularduality}. A recent convergence analysis of \texttt{Muon} by
Shen et al.~\cite{shen2025convergence} is particularly close in spirit to our work. In a
general nonconvex setting, they derive iteration-complexity bounds for \texttt{Muon} and GD
under various smoothness assumptions and identify two key quantities:
a Hessian-based ratio $J_t/L_t$, and the gradient nuclear ranks
$\nr(\nabla f(W_t))$. Their central condition (Eq.~(4) in
\cite{shen2025convergence}) is
\begin{equation}
  \label{eq:shen-ratio}
  \frac{J_t}{L_t}
  \;\lesssim\;
  \frac{\|\nabla f(W_t)\|_*^2}{\|\nabla f(W_t)\|_F^2},
\end{equation}
and they show that when~\eqref{eq:shen-ratio} holds along the trajectory, \texttt{Muon} enjoys
better convergence guarantees than GD for reaching a nuclear-norm stationary point. They also
empirically track $J_t$, $L_t$, and the gradient nuclear ranks during GD and
\texttt{Muon} training of a small MLP, and observe that~\eqref{eq:shen-ratio} holds in the
early and middle phases of training, consistent with \texttt{Muon}'s speedup.

Our analysis is complementary. We derive a \emph{layerwise} inequality for \texttt{SpecGD}
versus GD that has the same gradient ratio on the right-hand side as
Shen et al.~\eqref{eq:shen-ratio}, but replaces the Hessian ratio $J_t/L_t$ by the \emph{stable
rank of the post-activation matrices}:
\[
  \mathrm{st}(A^{\ell-1})
  \;\lesssim\;
 \nr(\nabla_{W_\ell}\mathcal L(W))
\]
We do not claim that it is conceptually novel that the nuclear rank of the gradient
controls when spectral updates are better than Euclidean updates---this is already implicit in
PSD/SSD and explicit in the convergence bounds for \texttt{Muon} and related methods. Our
contribution is to show that, in the setting of random-feature regression, deep neural networks, and transformers, the \emph{left-hand side}
of this inequality is naturally small (post-activations have low stable rank), while the
\emph{right-hand side} can be large (gradient nuclear rank is dimension-dependent), both
theoretically (after a few GD steps in random-feature models) and empirically (in NanoGPT-scale
training). This provides a concrete, data-driven explanation for the regimes where spectral
updates outperform GD.

The \texttt{Scion} optimizer of Pethick et al.~\cite{pethick2025scion} builds a general
framework of stochastic conditional gradient (LMO-based) methods with norm constraints, in
which \texttt{Muon} is one special case and \texttt{Scion} corresponds to using operator norms
(e.g., spectral and $\ell_\infty$ norms) for different layers. Their convergence guarantees are
stated in terms of the dual norm of the gradient, $\|\nabla f(x)\|_*$, under $L$-smoothness in
the corresponding primal norm, and in the spectral case this dual norm is exactly the nuclear
norm. Thus their rates naturally control $\EE\|\nabla f(x)\|_*$, while classical GD rates
control $\EE\|\nabla f(x)\|_F$. To compare these guarantees, one must again relate nuclear and
Frobenius norms of the gradient, i.e., consider
$\|\nabla f(x)\|_* / \|\nabla f(x)\|_F$. Pethick et al.\ focus on the algorithmic framework,
norm choices, and hyperparameter transfer, and do not attempt to characterize when this ratio is
large in neural networks.

\paragraph{Relation to nonlinear spiked covariance and conjugate kernel work.}
Our perspective is complementary to recent analyses of nonlinear spiked covariance matrices and
conjugate kernels~\cite{benigni2022largest,wang2024nonlinear}. There, the goal
is to understand how a finite number of \emph{informative} spikes in the input data propagate
through random networks or how nonlinearity and non-Gaussianity affect the edge of the spectrum,
and for this reason the activations (and weights) are explicitly centered to suppress
large, non-informative outliers. 
Earlier work by Pennington and Worah~\cite{pennington2017nonlinear} analyzes the same
one-layer Gram matrix $M = \tfrac{1}{m}YY^\top$ with $Y=f(WX)$ in a proportional random
matrix limit and derives an explicit characterization of its limiting spectral density in
terms of two Gaussian functionals of the activation; like~\cite{benigni2022largest,wang2024nonlinear},
they work with centered nonlinearities and focus on the bulk spectrum rather than on stable
rank or mean-induced spikes.
By contrast, we work with the non-centered activations
used in practice (e.g.\ ReLU), for which even isotropic inputs generate mean-induced spikes in
the post-activation covariances at every hidden layer. These spikes need not carry task-relevant
signal, but they force the activation matrices to be intrinsically low–stable–rank, and it is
precisely this degenerate geometry that makes spectral updates more effective than Euclidean ones
in the regimes we study.

\paragraph{Low-rank structure in gradients and activations.}
There is a substantial literature documenting low-dimensional or low-rank structure in the
gradients and representations of deep networks. Gur-Ari et al.\cite{gurari2018tiny_subspace} showed that during training,
gradients quickly concentrate in a low-dimensional subspace spanned by the top Hessian
eigenvectors---a ``tiny subspace" whose dimension is often comparable to the number of
classes~\cite{gurari2018tiny_subspace}. This is a statement about \emph{gradient subspace}
dimension rather than nuclear rank, but it supports the general picture that most
of the optimization geometry is captured by a small number of directions. More recent work in
large-scale training leverages this low-rank structure explicitly: the GaLore algorithm \cite{zhao2024galore} performs gradient
low-rank projection to reduce optimizer memory while preserving pretraining performance, and
reports empirical evidence that gradients of large language models admit very low-rank
approximations; Sonthalia et al. \cite{sonthalia2025low_rank_gradients} analyze low-rank gradients in two-layer
networks in a spiked data model, proving that the gradient spectrum is sharply concentrated and
identifying parameter and data regimes where gradient rank is small.

Huh et al.~\cite{huh2021low_rank_simplicity} empirically observe that the Gram matrix of the final-layer embeddings of deep networks has low effective rank (Observation 3.1 and 3.2), both at initialization and after training, and conjecture that deeper networks are biased toward lower effective-rank embeddings. In our notation, their Gram matrices are constructed from the columns of the post-activation matrix $A_L$, and effective rank is a soft-rank functional of the same spectrum as the stable rank $\mathrm{st}(A_L)$. Our results are consistent with these observations and extend them in two directions: (i) we work with the stable rank $\mathrm{st}(A_\ell)$ of the layerwise post-activation matrices, which is a closely related notion of effective dimensionality of the embeddings, and show theoretically in several random-feature and quadratic-activation models that $\mathrm{st}(A_\ell)$ is bounded by a numerical constant independent of the layer width; and (ii) we verify empirically in realistic MLP and NanoGPT models that this low-stable-rank property holds across \emph{all} intermediate layers, not only for the final embedding, and that it persists along the optimization trajectory. While we are mainly interested here in the stable rank, our results immediately imply that the ``effective rank" in the sense of \cite{huh2021low_rank_simplicity} is of order $\log(d)$.

A complementary line of work seeks to \emph{remove} the kind of low-rank degeneracy we study by
introducing normalization layers such as batch normalization~\cite{ioffe2015batch}. For
example, Daneshmand et al.~\cite{daneshmand2020batch} show that in deep \emph{linear}
networks (i.e., without nonlinear activations), batch normalization provably prevents rank
collapse: the (soft) rank of the hidden activations remains of order $d$ at all depths, whereas
in the unnormalized case the representations become effectively rank one as depth grows. This
line of work is complementary to ours: rather than modifying the architecture to avoid
degeneracy, we take the activation matrices as they arise in standard, non-normalized networks
with non-centered nonlinearities, and analyze how their low-stable-rank, spiked covariance
structure influences optimization and favors spectral updates over Euclidean ones.

Finally, a parallel line of work examines the spectral structure of the Hessian itself and reports a
similar ``bulk + spikes" picture: Sagun et al.~\cite{sagun2016hessian} and Papyan~\cite{papyan2018fullspectrum}
find that deep networks have a large bulk of near-zero eigenvalues with a small number of large
outliers, while Ghorbani et al.~\cite{ghorbani2019hessian_spectrum} show that these outliers appear rapidly during training and that the gradient concentrates
in their eigenspaces. Our stable-rank perspective is consistent with this broader view of
degenerate curvature, but focuses on the propagated data matrices that directly enter the
layerwise Lipschitz constants relevant for spectral versus Euclidean updates.

\paragraph{Organization.}
Section~\ref{sec:spikes_and_embeddings} gives two base cases for low--stable--rank representations: mean-induced spikes for non-centered activations and low--stable--rank token indicators/embeddings in transformers.
Section~\ref{sec:propagate} develops propagation rules showing that stable rank (and the column-norm regularity needed for RMSNorm) is preserved, up to explicit factors, by the atomic operations appearing in transformer blocks, and packages these rules into attention- and MLP-sublayer bounds.
Sections~\ref{sec:deep-mlp} and~\ref{sec:transformers} iterate these one-layer inputs to control stable rank throughout deep MLPs and decoder-only transformers at Gaussian initialization, with a regime split between mean-spike-inducing and nonMSI/gated activations.
Turning to the optimization criterion, Section~\ref{sec:atinitialization} analyzes spiked random-feature regression models and shows that gradient descent quickly produces gradients with dimension-dependent nuclear rank, thereby activating the inequality~\eqref{eq:sd-vs-gd-condition} on a macroscopic window of iterates.
Finally, Section~\ref{sec:layered-descent} introduces a general layered model, proves a mixed Hessian bound, and derives a one-step descent comparison for mixed Euclidean/spectral updates that reduces blockwise gains to the same nuclear-rank versus stable-rank condition.
The appendices contain additional experiments (e.g., SwiGLU in Appendix~\ref{sec:swiglu}) and supporting concentration and linear-algebra lemmas.

\section{Low-stable-rank activations: mean spikes and token-embeddings}
\label{sec:spikes_and_embeddings}

This section identifies two sources of low-stable-rank matrices that arise naturally in the
settings of this paper, and that will serve as base cases for later arguments.
First, post-activation matrices for common non-centered activations, such as ReLU, have low-stable rank due to a large mean-induced spike. Second, for transformers, the input token-indicator matrix
and the initial token embeddings are low-stable-rank whenever a single token (e.g., ``the") makes up at least a constant fraction of the sequence.

These results complement those of the upcoming Section~\ref{sec:propagate}, which show that stable
rank propagates through common network operations with small inflation. Then in Sections~\ref{sec:deep-mlp} and~\ref{sec:transformers} we apply the results of this section and  Section~\ref{sec:propagate} to conclude that all relevant post-activation matrices
in MLPs and transformers have low stable rank.

\subsection{Concentration of the stable rank.}
A recurring object in this paper is an \emph{activation matrix with i.i.d.\ feature columns}:
\begin{equation}\label{eq:activation-matrix-iid-columns}
Z_n \;=\; [F(x_1),\dots,F(x_n)]\in\R^{d\times n},
\end{equation}
where $x_1,\dots,x_n$ are i.i.d.\ and $F:\R^p\to\R^d$ is a fixed measurable map with
$0<\EE\|F(x_1)\|_2^2<\infty$.
In the neural network setting, $F$ is the composition of linear maps and a pointwise activation, and
the canonical model is
\[
Z_n=\sigma(WX),
\qquad
X=[x_1,\dots,x_n],
\]
with $\sigma$ applied entrywise.
Depending on the model, randomness may enter through the data or through the weights:
if the columns of $X$ are random and $W$ is fixed, then $Z_n$ has i.i.d.\ columns;
if $W$ is random and $X$ is fixed, then $Z_n$ has i.i.d.\ rows.
Since $\st(Z)=\st(Z^\top)$, we may treat both situations using the same i.i.d.-column setup.

Henceforth, we write $Z_n=[z_1,\dots,z_n]\in\R^{d\times n}$ where $z_1,\dots,z_n$ are i.i.d.\ copies
of a random vector $z\in\R^d$ with $0<\EE\|z\|_2^2<\infty$.
Define the empirical and population second moment matrices, respectively:
\[
M_n \;:=\; \tfrac{1}{n}Z_nZ_n^\top
\;=\;\tfrac{1}{n}\sum_{i=1}^n z_i z_i^\top
\qquad\textrm{and}\qquad
M\;:=\;\EE[zz^\top].
\]
For a positive semidefinite matrix $B$, let ${\rm er}(B):=\tr(B)/\|B\|_{\op}$ denote its
\emph{effective rank}.
Since $\tr(M_n)=\frac{1}{n}\|Z_n\|_F^2$ and $\|M_n\|_{\op}=\frac{1}{n}\|Z_n\|_{\op}^2$, we have the identity
\begin{equation}\label{eq:sr-er-identity}
\st(Z_n)=\frac{\|Z_n\|_F^2}{\|Z_n\|_{\op}^2}=\frac{\tr(M_n)}{\|M_n\|_{\op}}={\rm er}(M_n).
\end{equation}
Thus, stable-rank control for $Z_n$ reduces to effective-rank control for $M_n$.
Since $M$ is deterministic and typically easier to analyze than $M_n$, it is natural to
first understand ${\rm er}(M)$ and then transfer to ${\rm er}(M_n)$ with reasonably high probability. Reassuringly, the next lemma records the
asymptotic consistency of the estimator ${\rm er}(M_n)$.

\begin{lem}[Consistency of the stable rank]\label{lem:MSLLN}
Let $Z_n=[z_1,\dots,z_n]$ have i.i.d.\ columns distributed as $z$ with
$0<\EE\|z\|_2^2<\infty$, and let $M=\EE[zz^\top]$. Then
\[
\st(Z_n)\xrightarrow{\ \mathrm{a.s.}\ }\ {\rm er}(M)
\qquad \textrm{as}~n\to \infty.
\]
\end{lem}
\begin{proof}
By the strong law of large numbers applied entrywise, $M_n\to M$ almost surely.
Since $\tr(M)=\EE\|z\|_2^2>0$, we have $\|M\|_{\op}>0$. On the same event,
$\|M_n\|_{\op}\to \|M\|_{\op}$ and $\tr(M_n)\to \tr(M)$, hence
$\tr(M_n)/\|M_n\|_{\op}\to \tr(M)/\|M\|_{\op}$. Using \eqref{eq:sr-er-identity} completes the proof.
\end{proof}

Next, we aim to obtain finite sample guarantees. Since we are primarily interested in upper bounding $\er(M_n)$, an upper-bound on the ratio $\er(M_n)/\er(M)$ suffices. Such a one-sided bound holds under very mild assumptions. In particular, a small-ball condition suffices to obtain a result of this form with constant probability; this is the content of the following theorem.

\begin{thm}[Constant probability from a small-ball condition]\label{thm:small_balls}
Suppose there exists a unit vector $u_{\star}\in\R^d$ and constants $p_0\in (0,1)$ and $c_0>0$
satisfying the small-ball condition
\begin{equation}\label{eqn:small_ball}
\PP\left(\langle z,u_{\star}\rangle^2\geq c_0 \|M\|_{\op}\right)\geq p_0.
\end{equation}
Then for any $c\in (0,\tfrac{1}{2})$, as long as $n\geq \frac{8}{p_0}\log\frac{1}{c}$, the upper bound
\begin{equation}\label{eqn:upper_smallballs}
\st(Z_n)={\rm er}(M_n)\leq \tfrac{2}{p_0c_0c}\cdot{\rm er}(M),
\end{equation}
holds with probability at least $1-2c$.
\end{thm}
\begin{proof}
Since $\EE\tr(M_n)=\tr(M)$, Markov's inequality implies
\[
\PP\!\left(\tr(M_n)>c^{-1}\tr(M)\right)\le c.
\]
Next, we estimate
\[
\|M_n\|_{\op}\ge u_\star^\top M_n u_\star = \frac{1}{n}\sum_{i=1}^n\langle z_i,u_{\star}\rangle^2.
\]
Let $I_i:=\mathbf 1_{\{\langle z_i,u_{\star}\rangle^2\ge c_0\|M\|_{\op}\}}$, so that $\EE I_i\ge p_0$ by
\eqref{eqn:small_ball}. Chernoff's inequality yields
\[
\PP\!\left(\sum_{i=1}^n I_i \le \frac{p_0 n}{2}\right)\le \exp\!\left(-\frac{p_0 n}{8}\right).
\]
On the event $\{\sum_{i=1}^n I_i \ge p_0 n/2\}$, we have
\[
\|M_n\|_{\op}
\ge \frac{1}{n}\sum_{i=1}^n\langle z_i,u_{\star}\rangle^2
\ge \frac{p_0 n}{2}\cdot \frac{c_0\|M\|_{\op}}{n}
=\frac{p_0c_0}{2}\|M\|_{\op}.
\]
Combining this lower bound on $\|M_n\|_{\op}$ with $\tr(M_n)\le c^{-1}\tr(M)$ yields
\eqref{eqn:upper_smallballs}.
Finally, if $n\ge \frac{8}{p_0}\log\frac{1}{c}$, then $\exp(-p_0 n/8)\le c$, so a union bound gives
probability at least $1-2c$.
\end{proof}

Although Theorem~\ref{thm:small_balls} holds under a very weak small-ball condition, it has a limitation. Namely, the probability of \eqref{eqn:upper_smallballs} not holding scales
the right side of \eqref{eqn:upper_smallballs} multiplicatively. The following theorem shows that this limitation can be removed under light tail assumptions.

\begin{thm}[High probability upper-bound on the stable rank]
Suppose that there for some values $\theta, \phi>0$ and $K,L>0$ the light-tail conditions hold:
$$\langle z-\EE z,u\rangle_{\psi_{\theta}}\leq K\|u\|_2\qquad\forall u\in\R^d\qquad \textrm{and}\qquad \|\|z\|_2^2-\EE\|z\|^2_2\|_{\psi_{\phi}}\leq L.$$
Then there exists a constant $c$, depending only on $\theta$ and $\phi$, such that for any $\varepsilon\in (0,1)$ the estimate
$$\er(M_n)\leq \frac{1+\varepsilon}{1-\varepsilon}\cdot \er(M),$$
holds with probability at least $$1-6\exp\left(-c\min\left\{\tfrac{n\varepsilon^2\|M\|^2_{\op}}{K^4}, \tfrac{n\varepsilon^2\|M\|_{\op}}{K^2}, \tfrac{n\varepsilon^2\tr(M)^2}{L^2}, n^{1\wedge \tfrac{\theta}{2}}(\tfrac{\varepsilon\|M\|_{\op}}{K^2})^{\theta/2}, n^{1\wedge \theta}(\tfrac{\varepsilon \sqrt{\|M\|_{\op}}}{K})^{\theta}, n^{1\wedge \phi} (\tfrac{\varepsilon\tr(M)}{L })^{\phi}\right\}\right).$$
\end{thm}
\begin{proof}
First, Theorem~\ref{lem:lower_op_norm} in the appendix directly implies $\|M_n\|_{\op}\geq (1-\varepsilon) \|M\|_{\op}$ with probability at least 
$$1-2\exp\left(-c\left(\tfrac{n\varepsilon^2\|M\|^2_{\op}}{K^4}\wedge n^{1\wedge \tfrac{\theta}{2}}\left(\tfrac{\varepsilon\|M\|_{\op}}{K^2}\right)^{\theta/2}\right)\right)-2\exp\left(- c\left(\tfrac{n\varepsilon^2\|M\|_{\op}}{K^2}\wedge n^{1\wedge \theta}\left(\tfrac{\varepsilon \sqrt{\|M\|_{\op}}}{K}\right)^{\theta}\right)\right).$$
for some numerical constant $c$ depending only on $\theta$.
Next, we upper bound the trace of $M_n$. To this end, we write 
$$\tr(M_n)-\tr(M)=\frac{1}{n}\sum_{i=1}^n (\|z_i\|_2^2-\EE \|z\|^2_2).$$
Bernstein's inequality \cite[Theorem 3.1]{kuchibhotla2022moving} then implies $\tr(M_n)\leq (1+\varepsilon)\tr(M)$ with probability 
$$1-2\exp\left(-c\left(\tfrac{n\varepsilon^2\tr(M)^2}{L^2}\wedge n^{1\wedge \phi} \left(\tfrac{\varepsilon\tr(M)}{L }\right)^{\phi}\right)\right),$$
which completes the proof.
\end{proof}

With the concentration results in mind, we can focus on estimating the effective rank of $M$ and then transfer back to $M_n$ using the results we have obtained. With this in mind, we now isolate two sources of low stable ranks in post-activation matrices: mean-induced spikes (Section~\ref{subsec:nsr_moment_ratio}) and token embedding in transformers (Section~\ref{sec:embedding}).

\subsection{Mean spikes}
\label{subsec:nsr_moment_ratio}

The quantity ${\rm er}(M)$ is small whenever $\tr(M)$ is moderate and $\|M\|_{\op}$ is large.
By linearity, we have
\begin{equation}\label{eq:trace-is-second-moment}
\tr(M)=\EE\tr(zz^\top)=\EE\|z\|_2^2.
\end{equation}
If the mean $\mu:=\EE z$ is nonzero, we can estimate
\[
M=\EE[zz^\top]=\mu\mu^\top+\Cov(z)\succeq \mu\mu^\top,
\]
and therefore $\|M\|_{\op}\ge \|\mu\|_2^2$. Consequently, we deduce
\begin{equation}\label{eq:er-leq-pop-nsr}
{\rm er}(M)
=\frac{\EE\|z\|_2^2}{\|M\|_{\op}}
\le
\frac{\EE\|z\|_2^2}{\|\EE z\|_2^2}.
\end{equation}
We call the ratio on the right the \emph{(population) noise-to-signal ratio} of $z$:
\[
\mathrm{NSR}(z):=\frac{\EE\|z\|_2^2}{\|\EE z\|_2^2}\qquad (\EE z\neq 0).
\]
Thus, whenever the squared norm of the mean is comparable to the second moment, the effective rank of $M$
is small. Now, the population NSR can be approximated by its empirical counterpart. Namely, define
\[
\hat\mu := \frac1n\sum_{i=1}^n z_i,
\qquad
\hat m_2 := \frac1n\sum_{i=1}^n \|z_i\|_2^2,
\qquad
\widehat{\mathrm{NSR}}(Z):=\frac{\hat m_2}{\|\hat\mu\|_2^2}.
\]
Then the same argument as leading to \eqref{eq:er-leq-pop-nsr} shows 
$$\st(Z_n)\leq \widehat{\mathrm{NSR}}(Z_n).$$

Reassuringly, under mild light tail assumptions, $\widehat{\mathrm{NSR}}(Z_n)$ and ${\mathrm{NSR}}(z)$ are indeed close, as the following lemma shows. We will use the usual Orlicz quasi-norm
$\|z\|_{\psi_{\theta}}$ of a random vector $z$ for $\theta>0$.
Namely, $\|z\|_{\psi_{\theta}}$ is the infimum over all values $\sigma\geq 0$ satisfying $\EE\exp(|z/\sigma|^{\theta})\leq 2$.

\begin{lem}[Consistency of the Empirical NSR]\label{lem:snr_stability}
Let $z\in\R^d$ be a random vector with mean $\mu:=\EE z\neq 0$ and finite second moment
$m_2:=\EE\|z\|_2^2$. Fix parameters $\theta,\phi>0$ and constants $K,L>0$ such that
\[
\|\langle z-\mu,u\rangle\|_{\psi_\theta}\leq K\|u\|_2
\qquad\text{for all }u\in\R^d,
\qquad\text{and}\qquad
\big\|\|z\|_2^2-\EE\|z\|_2^2\big\|_{\psi_\phi}\leq L.
\]
Let $z_1,\ldots,z_n$ be iid copies of $z$, and define the empirical mean and the second moment:
\[
\hat\mu=\frac{1}{n}\sum_{i=1}^n z_i,
\qquad`
\hat m_2=\frac{1}{n}\sum_{i=1}^n \|z_i\|_2^2.
\]
Then there exist constants $c_1=c_1(\theta)>0$ and $c_2=c_2(\phi)>0$ such that for any $0<\varepsilon_1\leq 1$ and $0<\varepsilon_2<1$, the estimate
\[
\frac{\hat m_2}{\|\hat\mu\|_2^2}
\leq
\frac{1+\varepsilon_1}{(1-\varepsilon_2)^2}\cdot
\frac{m_2}{\|\mu\|_2^2}
\]
holds with probability at least
\[
1
-
2\exp\!\left(
  -c_1\min\Big\{
     n\,\varepsilon_2^2\,\frac{\|\mu\|_2^2}{K^2},
     \ n^{1\wedge \theta}\varepsilon_2^{\theta}\,\frac{\|\mu\|_2^{\theta}}{K^{\theta}}
  \Big\}
\right)
-
2\exp\!\left(
  -c_2\min\Big\{
     n\,\varepsilon_1^2\,\frac{m_2^2}{L^2},
     \ n^{1\wedge \phi}\varepsilon_1^{\phi}\,\frac{m_2^{\phi}}{L^{\phi}}
  \Big\}
\right).
\] 
In particular, in this event the estimate holds: 
\[
\st(Z)\ \le\ \frac{1+\varepsilon_1}{(1-\varepsilon_2)^2}\,\mathrm{NSR}(z).
\]
\end{lem}

\begin{proof}
We first control the empirical mean. Let $u_\star:=\mu/\|\mu\|_2$ be the unit vector in the direction of $\mu$, and write
\[
\hat\mu-\mu
=
\frac{1}{n}\sum_{i=1}^n (z_i-\mu).
\]
Then $X_i:=\langle z_i-\mu,u_\star\rangle$ are iid mean-zero random variables with $\|X_i\|_{\psi_\theta}\le K$.
A Bernstein-type inequality \cite[Theorem 3.1]{kuchibhotla2022moving} for $\psi_\theta$ random variables yields for all $t>0$ the estimate
\[
\PP\left( \left|\frac{1}{n}\sum_{i=1}^n X_i\right|\geq t\right)
\leq
2\exp\!\left(
-c\,\min\Big\{n\,\tfrac{t^2}{K^2},\ n^{1\wedge \theta}(\tfrac{t}{K})^{\theta}\Big\},
\right)
\]
for a numerical constant $c>0$.
Setting $t=\varepsilon_2\|\mu\|_2$ and using the estimate
\[
\|\hat\mu\|_2
\ \geq\
\langle \hat\mu,u_\star\rangle
=
\|\mu\|_2 + \left\langle \hat\mu-\mu,u_\star\right\rangle,
\]
we obtain
\begin{equation}\label{eqn:denom_event}
\PP\big(\|\hat\mu\|_2\leq (1-\varepsilon_2)\|\mu\|_2\big)
\leq
2\exp\!\left(
  -c_1\min\Big\{
     n\,\varepsilon_2^2\,\tfrac{\|\mu\|_2^2}{K^2},
     \ n^{1\wedge \theta}\varepsilon_2^{\theta}\,\tfrac{\|\mu\|_2^{\theta}}{K^{\theta}}
  \Big\}
\right),
\end{equation}
for a numerical constant $c_1>0$.

We next control the empirical second moment. Define
$
Y_i:=\|z_i\|_2^2 - m_2.
$
By assumption $\|Y_i\|_{\psi_\phi}\leq L$, so $Y_i$ are iid centered $\psi_\phi$ random variables with parameter $L$.
The same Bernstein-type inequality \cite[Theorem 3.1]{kuchibhotla2022moving} yields that for all $t>0$,
\[
\PP\left(
\left|\frac{1}{n}\sum_{i=1}^n Y_i\right|\geq t
\right)
\leq
2\exp\!\left(
  -c_2\,\min\Big\{
     n\,\tfrac{t^2}{L^2},\ n^{1\wedge \phi}(\tfrac{t}{L})^{\phi}
  \Big\}
\right),
\]
for a numerical constant $c_2>0$.
Substituting $t=\varepsilon_1 m_2$ yields
\begin{equation}\label{eqn:num_event}
\PP\big(\hat m_2>(1+\varepsilon_1)m_2\big)
\leq
2\exp\!\left(
  -c_2\,\min\Big\{
     n\,\varepsilon_1^2\,\tfrac{m_2^2}{L^2},
     \ n^{1\wedge \phi}\varepsilon_1^{\phi}\,\tfrac{m_2^{\phi}}{L^{\phi}}
  \Big\}
\right).
\end{equation}
Combining \eqref{eqn:denom_event} and \eqref{eqn:num_event} and taking a union bound over the two events completes the proof.
\end{proof}

\subsubsection{Bounding the effective rank for random data: polynomial maps}
We are now ready to establish our first main result. Namely, we will bound ${\rm er}(M)$ for {\em squares of polynomial functions applied to Gaussian
inputs}. This covers, for example, a single layer $Z_n=\sigma_{\rm sq}(WX)$ with a quadratic activation $\sigma_{\rm sq}(t):=t^2$ and when
the columns of $X$ are i.i.d.\ Gaussian, and more generally, deeper layers of the form $Z_n=\sigma_{\rm sq}(W[F(x_1),\ldots, F(x_n)])$ where the coordinate functions of $F$ are polynomials. 
Our strategy is to show that the population NSR is small and then pass to the empirical counterpart using Theorem~\ref{thm:small_balls} and the Paley-Zygmund inequality.

\begin{thm}[Stable rank for polynomial-square features under Gaussian data]\label{thm:stab_rankthm3}
Let $x_1,\dots,x_n\in\R^p$ be iid Gaussian, and let $f_1,\dots,f_d:\R^p\to\R$ be polynomials of
degree at most $q\ge 1$. Define the feature map
\[
F(x):=\big(f_1(x)^2,\dots,f_d(x)^2\big)\in\R^d,
\qquad
Z_n:=[F(x_1),\dots,F(x_n)]\in\R^{d\times n}.
\]
Then for any $c\in (0,\tfrac{1}{2})$, as long as $n\ \ge\ 18\cdot 3^{4q}\,\log\!\frac{1}{c}$,
the bound
\[
\st(Z_n)\ \le\ \frac{27}{2c}\cdot 3^{6q}
\]
holds with probability at least $1-2c$.
\end{thm}
\begin{proof}
Set $z:=F(x)\in\R^d$ for a single Gaussian draw $x$, and set $M=\EE[zz^\top]$.
If $\EE\|z\|_2^2=0$ then $Z_n=0$ almost surely and the claim is trivial, so assume
$\EE\|z\|_2^2>0$.
Since at least one $f_j$ is not identically zero, we have $\EE[f_j(x)^2]>0$ for that $j$, hence
$\EE z\neq 0$ and $\mathrm{NSR}(z)$ is well-defined.

We aim to apply Theorem~\ref{thm:small_balls}. To this end, we first verify the small ball condition \eqref{eqn:small_ball}. Let $u_\star$ be a top-eigenvector of $M$ and set $g(x):=\langle F(x),u_\star\rangle$.
Then $\EE[g(x)^2]=\|M\|_{\op}$ and $g$ is a polynomial of degree at most $2q$.
Using Gaussian hypercontractivity \cite[Chapter~3.2]{ledoux2013probability}, we deduce
\[
\EE[g(x)^4]\le 3^{4q}\,(\EE[g(x)^2])^2.
\]
Therefore, the Paley--Zygmund inequality gives
\[
\PP\!\left(g(x)^2\ge \frac13\,\EE[g(x)^2]\right)
\ge \left(1-\frac13\right)^2\frac{(\EE[g(x)^2])^2}{\EE[g(x)^4]}
\ge \frac{4}{9}\cdot 3^{-4q}.
\]
Thus the small ball condition \eqref{eqn:small_ball} holds with $c_0=1/3$ and $p_0=\frac{4}{9}\,3^{-4q}$.
With these constants, Theorem~\ref{thm:small_balls} shows that in the regime
$n\ge \frac{8}{p_0}\log\frac{1}{c}=18\cdot 3^{4q}\log\frac{1}{c}$, with probability at least
$1-2c$, we have
\begin{equation}\label{eq:poly-hw-step}
\st(Z_n)\le \frac{2}{p_0c_0c}\,{\rm er}(M)=\frac{27}{2c}\cdot 3^{4q}\,{\rm er}(M).
\end{equation}

It remains to upper bound $\er(M)$. To this end,
since $\EE z\neq 0$, equation \eqref{eq:er-leq-pop-nsr} implies ${\rm er}(M)\le \mathrm{NSR}(z)$.
Writing $z=(z_1,\dots,z_d)$ with $z_j=f_j(x)^2$, we have
\[
\mathrm{NSR}(z)=\frac{\EE\|z\|_2^2}{\|\EE z\|_2^2}
=\frac{\sum_{j=1}^d \EE[z_j^2]}{\sum_{j=1}^d (\EE z_j)^2}
\le \max_{1\le j\le d}\ \frac{\EE[z_j^2]}{(\EE z_j)^2}
=
\max_{1\le j\le d}\ \frac{\EE[f_j(x)^4]}{\big(\EE[f_j(x)^2]\big)^2}.
\]
Each $f_j$ has degree at most $q$, so hypercontractivity gives
$\EE[f_j(x)^4]\le 3^{2q}\big(\EE[f_j(x)^2]\big)^2$, and hence
${\rm er}(M)\le \mathrm{NSR}(z)\le 3^{2q}$.
Substituting this into \eqref{eq:poly-hw-step} yields
\[
\st(Z_n)\le \frac{27}{2c}\cdot 3^{4q}\cdot 3^{2q}=\frac{27}{2c}\cdot 3^{6q},
\]
on the same event, completing the proof.
\end{proof}

In particular, for a depth $L$ feedforward neural network with quadratic activation functions $\sigma_{\rm sq}$ in each layer, Theorem~\ref{thm:stab_rankthm3} shows that with fixed weight matrices $W_1,\ldots,W_{L}$ and random Gaussian data $X$, the stable rank of the final post-activation matrix $A^{L}$ is bounded by $\st(A^{L})\leq \frac{27}{2c}3^{3\cdot 2^{L}}$ with probability at least $1-2c$. 

Theorem~\ref{thm:stab_rankthm3} gives a dimension-free stable-rank bound in a random-data regime,
but only with constant probability and with a dependence that is exponential in the polynomial
degree (hence doubly exponential in the depth for quadratic compositions). We now turn to the
random-weights regime, where we obtain bounds for general activations.

\subsubsection{Bounding the stable rank with random weights: polynomial-growth activations}\label{sec:bounding_effective_rank}

We now return to the canonical post-activation matrix $A=\sigma(WX)$ with random weights $W$
and fixed input $X$. In this regime the \emph{rows} of $A$ are iid; we therefore apply the
iid-column bounds in Lemma~\ref{lem:snr_stability} to $A^\top$, using the equality $\st(A)=\st(A^\top)$.
The remaining work is to express the population NSR of $\sigma(X^\top w)$ (with Gaussian $w$) in terms
of one-dimensional Gaussian moments of $\sigma$ over the range of nonzero column norms of $X$. To this end, let $\gamma\sim\mathcal N(0,1)$. For $0<a\le b<\infty$ and $q>0$, define the two quantities: 
\[
s_\sigma(a,b)
\;:=\;
\sup_{s\in[a,b]}\ \frac{\EE[\sigma(s\gamma)^2]}{\big(\EE[\sigma(s\gamma)]\big)^2} \qquad \textrm{and}\qquad \kappa_{\sigma,q}(a,b)\;:=\;\inf_{s\in[a,b]}\ \frac{|\EE[\sigma(s\gamma)]|}{s^q}.
\]
Note that these quantities simplify drastically for the important case of $q$-homogeneous activations, meaning those satisfying $\sigma(\alpha t)=\alpha^{q}\sigma(t)$ for all $t\in \R$, $\alpha\geq 0$. For such functions, both $s_\sigma(a,b)$ and $\kappa_{\sigma,q}(a,b)$ are independent of $a$ and $b$ and  reduce to $s_\sigma(a,b)=\frac{\EE\sigma(\gamma)^2}{(\EE\sigma(\gamma))^2}$ and $\kappa_{\sigma,q}(a,b)=|\EE\sigma(\gamma)|$.

The following theorem estimates the stable rank of $A=\sigma(WX)$; without loss of generality, we assume that non of the columns of $X$ are the zero vector since zero columns can simply be deleted.

\begin{thm}[Gaussian weights yield low stable rank]
\label{thm:gaussian_weights_low_sr}
Fix a matrix $X=[x_1,\dots,x_n]\in\R^{d\times n}$ with nonzero columns and let $W\in\R^{m\times d}$ have iid entries
$\mathcal N(0,\tfrac{1}{d})$. Define the post-activation matrix $A=\sigma(WX)\in\R^{m\times n}$, where
$\sigma:\R\to\R$ satisfies $\sigma(0)=0$ and the polynomial growth bound
\[
|\sigma(t)|\le L_\sigma |t|^p
\qquad\text{for all }t\in\R,
\]
for some $p\ge 1$ and $L_\sigma<\infty$. Set
\[
a:=\frac{1}{\sqrt d}\min_{1\leq t\leq n}\|x_t\|_2,
\qquad
b:=\frac{1}{\sqrt d}\max_{1\leq t\leq n}\|x_t\|_2.
\]
Fix a constant $q>0$ satisfying $\kappa_{\sigma,q}(a,b)>0$ and define the scale ratio
\[
r_{p,q}\;:=\;\frac{L_\sigma}{\kappa_{\sigma,q}(a,b)}\cdot \sup_{s\in[a,b]} s^{p-q}.
\]
Then there exists a constant $c=c(p)>0$ such that for any $0<\varepsilon_1\le 1$ and $0<\varepsilon_2<1$, with probability at least
\[
1
-2\exp\!\Big(
  -c\,\min\Big\{m\,\varepsilon_2^2\,r_{p,q}^{-2},\ m^{1\wedge \tfrac{2}{p}}\varepsilon_2^{2/p}\,r_{p,q}^{-2/p}\Big\}
\Big)
-2\exp\!\Big(
  -c\,\min\Big\{m\,\varepsilon_1^2\,r_{p,q}^{-4},\ m^{1\wedge\tfrac{1}{p}}\varepsilon_1^{1/p}\,r_{p,q}^{-2/p}\Big\}
\Big),
\]
we have
\[
\st(A)\ \le\ \frac{1+\varepsilon_1}{(1-\varepsilon_2)^2}\,s_\sigma(a,b).
\]
\end{thm}
\begin{proof}
Write the rows of $W$ as $w_1^\top,\dots,w_m^\top$. Since $\st(A)=\st(A^\top)$ and
\[
A^\top=[\sigma(X^\top w_1),\dots,\sigma(X^\top w_m)]\in\R^{n\times m},
\]
the columns of $A^\top$ are iid copies of $z:=\sigma(X^\top w)\in\R^n$ with
$w\sim\mathcal N(0,\tfrac{1}{d}\,I_d)$. 

Applying Lemma~\ref{lem:snr_stability} with sample size $m$ yields the estimate
\[
\st(A)=\st(A^\top)\le \frac{1+\varepsilon_1}{(1-\varepsilon_2)^2}\,\mathrm{NSR}(z),
\]
with probability
\begin{equation}\label{eqn:prob_success}
    1
-
2\exp\!\left(
  -c_1\min\Big\{
     m\,\varepsilon_2^2\,\frac{\|\mu\|_2^2}{K^2},
     \ m^{1\wedge\theta}\varepsilon_2^{\theta}\,\frac{\|\mu\|_2^{\theta}}{K^{\theta}}
  \Big\}
\right)
-
2\exp\!\left(
  -c_2\min\Big\{
     m\,\varepsilon_1^2\,\frac{m_2^2}{L^2},
     \ m^{1\wedge \phi}\varepsilon_1^{\phi}\,\frac{m_2^{\phi}}{L^{\phi}}
  \Big\}
\right).
\end{equation}
 
It remains now to estimate $\mathrm{NSR}(z)$, $\|\mu\|_2^2$, $K$, $\theta$, $\phi$, which we do in order. Observe $g_t:=\langle w,x_t\rangle\sim \mathcal N(0,s_t^2)$ with
$s_t=\frac{1}{\sqrt d}\|x_t\|_2\in[a,b]$. Therefore, we deduce
\[
\mathrm{NSR}(z)
=
\frac{\sum_{t} \EE[\sigma(g_t)^2]}{\sum_{t} (\EE[\sigma(g_t)])^2}
=
\frac{\sum_{t} \EE[\sigma(s_t\gamma)^2]}{\sum_{t} (\EE[\sigma(s_t\gamma)])^2}
\le \max_{t}\ \frac{\EE[\sigma(s_t\gamma)^2]}{(\EE[\sigma(s_t\gamma)])^2}
\le s_\sigma(a,b),
\]
where $\gamma$ is a standard Gaussian.

Next, noting $\kappa_{\sigma,q}(a,b)>0$ and $s_t\in[a,b]$, we have
$|\EE[\sigma(g_t)]|\ge \kappa_{\sigma,q}(a,b)\,s_t^q$ for all $t$, and therefore
\begin{equation}\label{eq:mu-lb-poly}
\|\mu\|_2^2=\sum_{t}(\EE[\sigma(g_t)])^2 \ \ge\ \kappa_{\sigma,q}(a,b)^2\sum_{t}s_t^{2q}.
\end{equation}
In order to bound the tail growth, we use only the polynomial growth bound
$|\sigma(t)|\le L_\sigma |t|^p$. 
Namely, for any $u\in\R^n$, using $|x-\EE x|\le |x|+\EE|x|$ we have
\[
|\langle z-\mu,u\rangle|
=
\Big|\sum_{t}u_t\big(\sigma(g_t)-\EE\sigma(g_t)\big)\Big|
\le
\sum_{t}|u_t|\,|\sigma(g_t)|
+
\sum_{t}|u_t|\,\EE|\sigma(g_t)|.
\]
Using $|\sigma(t)|\le L_\sigma|t|^p$ and Cauchy--Schwarz, we deduce
\[
\sum_{t}|u_t|\,|\sigma(g_t)|
\le
L_\sigma\sum_{t}|u_t|\,|g_t|^p
\le
L_\sigma\|u\|_2\Big(\sum_{t}|g_t|^{2p}\Big)^{1/2}.
\]
Moreover, since $\EE|\sigma(g_t)|\le L_\sigma\EE|g_t|^p=L_\sigma s_t^p\,\EE|\gamma|^p$, another application of Cauchy--Schwarz gives
\[
\sum_{t}|u_t|\,\EE|\sigma(g_t)|
\le
\|u\|_2\Big(\sum_{t}(\EE|\sigma(g_t)|)^2\Big)^{1/2}
\le
C_p\,L_\sigma\|u\|_2\Big(\sum_{t}s_t^{2p}\Big)^{1/2},
\]
for some constant $C_p$ that depends only on $p$.
To bound the $\psi_{2/p}$ norm of $\big(\sum_{t}|g_t|^{2p}\big)^{1/2}$, set $a_t:=|g_t|^p$ and
$X:=\big(\sum_{t}a_t^2\big)^{1/2}$. Fixing $r\ge 2$, the Minkowski's inequality in $L_{r/2}$ implies
\[
\|X\|_{L_r}^2
=
\Big\|\sum_{t}a_t^2\Big\|_{L_{r/2}}
\le
\sum_{t}\|a_t^2\|_{L_{r/2}}
=
\sum_{t}\|a_t\|_{L_r}^2,
\]
and hence
\[
\|X\|_{L_r}\le \Big(\sum_{t}\|a_t\|_{L_r}^2\Big)^{1/2}.
\]
For $1\le r\le 2$, by monotonicity of $L_r$ norms we have $\|X\|_{L_r}\le \|X\|_{L_2}$, and the same bound follows from the case $r=2$.
Since $g_t\sim\cN(0,s_t^2)$, we have $\|a_t\|_{L_r}=\||g_t|^p\|_{L_r}=\|g_t\|_{L_{pr}}^p\le C_p\,r^{p/2}\,s_t^p$, for some constant $C_p$ depending only on $p$. Therefore, for all $r\ge 1$, we obtain
\[
\|X\|_{L_r}
\le
C_p\,r^{p/2}\Big(\sum_{t}s_t^{2p}\Big)^{1/2}.
\]
By the standard equivalence between moment growth and Orlicz norms (with constants depending only on $p$), we have
\[
\|X\|_{\psi_{2/p}}
\le
C_p\Big(\sum_{t}s_t^{2p}\Big)^{1/2}.
\]
Combining the bounds above yields
\[
\|\langle z-\mu,u\rangle\|_{\psi_{2/p}}
\le
C_p\,L_\sigma\Big(\sum_{t}s_t^{2p}\Big)^{1/2}\|u\|_2.
\]
This verifies the first hypothesis of Lemma~\ref{lem:snr_stability} with $\theta=2/p$ and
$
K\ \le\ C_p\,L_\sigma\Big(\sum_{t}s_t^{2p}\Big)^{1/2}.
$

For the second hypothesis, define $Y:=\|z\|_2^2-\EE\|z\|_2^2=\sum_{t}\big(\sigma(g_t)^2-\EE\sigma(g_t)^2\big)$. For any $r\ge 1$, Minkowski's inequality gives
\[
\|Y\|_{L_r}
\le
\sum_{t}\|\sigma(g_t)^2-\EE\sigma(g_t)^2\|_{L_r}
\le
2\sum_{t}\|\sigma(g_t)^2\|_{L_r}.
\]
Using $|\sigma(t)|\le L_\sigma|t|^p$ and $g_t\sim\cN(0,s_t^2)$, we have
\[
\|\sigma(g_t)^2\|_{L_r}
\le
L_\sigma^2\,\||g_t|^{2p}\|_{L_r}
=
L_\sigma^2\,\|g_t\|_{L_{2pr}}^{2p}
\le
C_p\,L_\sigma^2\,r^{p}\,s_t^{2p}.
\]
Therefore $\|Y\|_{L_r}\le C_p\,L_\sigma^2\,r^{p}\sum_{t}s_t^{2p}$, and by the standard equivalence between moment growth and Orlicz norms (with constants depending only on $p$), we obtain
\[
\big\|\|z\|_2^2-\EE\|z\|_2^2\big\|_{\psi_{1/p}}
=
\|Y\|_{\psi_{1/p}}
\le
C_p\,L_\sigma^2\sum_{t}s_t^{2p}.
\]
This verifies the second hypothesis of Lemma~\ref{lem:snr_stability} with $\phi=1/p$ and
$L\ \le\ C_p\,L_\sigma^2\sum_{t}s_t^{2p}$.

Finally, since $s_t\in[a,b]$ for all $t$, we have $s_t^{2p}\le \big(\sup_{s\in[a,b]}s^{p-q}\big)^2 s_t^{2q}$ and hence
\[
\sum_{t}s_t^{2p}
\le
\Big(\sup_{s\in[a,b]}s^{p-q}\Big)^2\sum_{t}s_t^{2q}.
\]
Combining this with \eqref{eq:mu-lb-poly} and $m_2\ge \|\mu\|_2^2$ shows that
\[
\frac{\|\mu\|_2^2}{K^2}\ \ge\ c\,r_{p,q}^{-2},
\qquad
\frac{m_2^2}{L^2}\ \ge\ c\,r_{p,q}^{-4},
\qquad
\frac{m_2^{1/p}}{L^{1/p}}\ \ge\ c\,r_{p,q}^{-2/p},
\qquad
\frac{\|\mu\|_2^{2/p}}{K^{2/p}}\ \ge\ c\,r_{p,q}^{-2/p},
\]
for a constant $c>0$ depending only on $p$. Substituting these bounds into \eqref{eqn:prob_success} yields the stated probability bound.
\end{proof}

As a concrete illustration, Table~\ref{table:activ} instantiates Theorem~\ref{thm:gaussian_weights_low_sr} for a number of common activation functions.

\begin{table}[htbp]
\centering
\begingroup
\renewcommand{\arraystretch}{1.35}
\begin{tabular}{l c c c}
\hline
Activation $\sigma(t)$ & $(p,q)$ & $s_\sigma(a,b)$ & Upper bound on $r_{p,q}(a,b)$ \\
\hline
Absolute value $|t|$ &
$(1,1)$ &
$\displaystyle \frac{\pi}{2}$ &
$\displaystyle \sqrt{\frac{\pi}{2}}$ \\[1.0ex]
ReLU $t_+$ &
$(1,1)$ &
$\displaystyle \pi$ &
$\displaystyle \sqrt{2\pi}$ \\[1.0ex]
Leaky ReLU $\alpha\,t\,\mathbf 1_{\{t\ge0\}}+\beta\,t\,\mathbf 1_{\{t<0\}}$ ($\alpha>\beta$) &
$(1,1)$ &
$\displaystyle \frac{\pi(\alpha^2+\beta^2)}{(\alpha-\beta)^2}$ &
$\displaystyle \sqrt{2\pi}\,\frac{\max\{|\alpha|,|\beta|\}}{\alpha-\beta}$ \\[1.0ex]
Squared ReLU $(t_+)^2$ &
$(2,2)$ &
$\displaystyle 6$ &
$\displaystyle 2$ \\[1.0ex]
Quadratic $t^2$ &
$(2,2)$ &
$\displaystyle 3$ &
$\displaystyle 1$ \\[1.0ex]
GELU $t\,\Phi(t)$ &
$(1,1)$ &
$\displaystyle s_\sigma(a,b)\ \le\ 2\pi\Big(1+\frac{1}{a^2}\Big)$ &
$\displaystyle \Big(1+\frac{1}{\sqrt{2\pi e}}\Big)\,\frac{\sqrt{2\pi(1+a^2)}}{a}$ \\
\hline
\end{tabular}
\endgroup
\caption{The pairs $(p,q)$ correspond to the growth bound $|\sigma(t)|\le L_\sigma|t|^p$ and the mean lower bound $\kappa_{\sigma,q}(a,b)>0$. For homogeneous activations in the table, $s_\sigma(a,b)$ and $r_{p,q}(a,b)$ are scale-invariant. For GELU, the mean scales like $s^2$ as $s\downarrow 0$, so $s_\sigma(a,b)$ grows as $a^{-2}$ and $r_{1,1}(a,b)$ grows as $a^{-1}$.}
\label{table:activ}
\end{table}

Theorem~\ref{thm:gaussian_weights_low_sr} becomes particularly simple for $q$-homogeneous functions.

\begin{cor}[Gaussian weights yield low stable rank, the case of $q$-homogeneous functions]
\label{cor:q_homogeneous_low_sr}
Fix a matrix $X=[x_1,\dots,x_n]\in\R^{d\times n}$ with nonzero columns and let
$W\in\R^{m\times d}$ have iid entries $\mathcal N(0,\tfrac{1}{d})$.
Define the post-activation matrix $A=\sigma(WX)\in\R^{m\times n}$, where
$\sigma:\R\to\R$ satisfies $\sigma(0)=0$, is $q$-homogeneous for some $q\ge 1$, and has nonzero mean $\mu_\sigma\;:=\;\EE[\sigma(\gamma)]\neq 0$ with respect to the standard gaussian $\gamma$.
Then there exists a constant $c>0$, which depends only on $q$, $\sigma(1)$, $\sigma(-1)$, $\mu_{\sigma}$ such that 
 for any $0<\varepsilon_1\le 1$ and
$0<\varepsilon_2<1$, with probability at least
\[
1
-2\exp\!\Big(
  -c\,\min\Big\{m\,\varepsilon_2^2,m^{1\wedge \tfrac{2}{q}} \varepsilon_2^{2/q}\Big\}
\Big)
-2\exp\!\Big(
  -c\,\min\Big\{m\,\varepsilon_1^2,\ m^{1\wedge \frac{1}{q}}\varepsilon_1^{1/q}\Big\}
\Big),
\]
we have
\[
\st(A)\ \le\ \frac{1+\varepsilon_1}{(1-\varepsilon_2)^2}\,\;\frac{\EE[\sigma(\gamma)^2]}{(\EE[\sigma(\gamma)])^2}.
\]
\end{cor}

\subsection{Token indicator matrices and Gaussian embeddings}\label{sec:embedding}

We next establish a base case for transformers. We prove that the token-indicator matrix of a fixed sequence
has stable rank exactly $1/p_{\max}$, where $p_{\max}$ is the empirical frequency of the most
common token; in typical corpora this is a small number due to heavy-tailed token frequencies.
We then show that a Gaussian embedding layer maps this indicator structure to an embedding
matrix whose stable rank is controlled by the same $p_{\max}$ up to a constant factor.

Let $(i_1,\dots,i_n)$ be a sequence of tokens with $i_t\in\{1,\dots,V\}$, and
let $H\in\R^{V\times n}$ denote the matrix whose $t$-th column $h_t$ is the standard basis
vector with a single $1$ in position $i_t$ and zeros elsewhere. For each token
$j\in\{1,\dots,V\}$, define
\[
c_j := |\{t : i_t = j\}|
\qquad\text{and}\qquad
p_j := \frac{c_j}{n},
\]
and set
\[
p_{\max} := \max_{1\le j\le V} p_j.
\]
Then $H H^\top$ is the diagonal matrix with entries $c_1,\dots,c_V$ and therefore
\[
\|H\|_{\op}^2
= \lambda_{\max}(H H^\top)
= \max_{1\le j\le V} c_j
= n p_{\max}.
\]
On the other hand, $\|H\|_F^2$ simply counts the number of nonzero entries of $H$, which is
$n$, and therefore
\begin{equation}\label{eq:indicator-stable-rank}
\st(H)
= \frac{\|H\|_F^2}{\|H\|_{\op}^2}
= \frac{n}{n p_{\max}}
= \frac{1}{p_{\max}}.
\end{equation}
Thus the stable rank of the raw token-indicator matrix is exactly the inverse of the empirical
frequency of the most common token. In the NanoGPT experiment of
Figure~\ref{fig:embedding}, we saw that $\st(H) \approx 26$ throughout training, and hence $p_{\max} \approx 1/26$. {That this ratio is low is not so surprising. Indeed, this is consistent with a Zipf-type law $p_j \propto j^{-s}$ over a vocabulary of size $V = 50{,}257$, for which the observed $p_{\max}$ corresponds to an exponent $s \approx 0.86$. }

We now turn to the embedding matrix. Let $E\in\R^{d\times V}$ denote the embedding matrix,
whose columns $e_1,\dots,e_V$ are sampled independently from $\mathcal N(0,I_d)$. For the
fixed sequence of $n$ tokens $(i_1,\dots,i_n)$ we form the embedded data matrix
\[
X = [x_1,\dots,x_n]\in\R^{d\times n},
\qquad
x_t := e_{i_t}.
\]
Recall that $c_j$ and $p_j$ denote the empirical counts and frequencies of token $j$, and
$p_{\max} := \max_j p_j$. 

\begin{lem}[Token embeddings have low stable rank]
  \label{lem:embedding_rms_stablerank}
 Then there exists
 a universal constants $C>0$ such that for any $\varepsilon\in (0,1)$ we have
  \[
    \PP\!\left(
      \st(X)\ \le\ \frac{1+\varepsilon}{1-\varepsilon}\,\frac{1}{p_{\max}}
    \right)
    \ \ge\ 1-4\exp(-C\varepsilon^2 d).
  \]
\end{lem}

\begin{proof}
Define the set of distinct tokens and its cardinality, respectively: 
\[
S := \{j : c_j>0\}
\qquad\text{and}\qquad
m := |S|
\]
Note that we may write
\[
XX^\top=\sum_{j\in S} c_j\,e_j e_j^\top,
\qquad
\tr(XX^\top)=\sum_{j\in S} c_j\|e_j\|_2^2.
\]
Let $j_\star\in S$ attain the maximum count $c_{j_\star}=np_{\max}$. Since
$XX^\top\succeq c_{j_\star}\,e_{j_\star}e_{j_\star}^\top$, we have
\[
\|X\|_{\op}^2=\lambda_{\max}(XX^\top)\ge c_{j_\star}\|e_{j_\star}\|_2^2.
\]
Consequently, we have
\begin{equation}\label{eq:sr-reduce-embedding}
\st(X)
=\frac{\tr(XX^\top)}{\lambda_{\max}(XX^\top)}
\le
\frac{\sum_{j\in S} c_j\|e_j\|_2^2}{c_{j_\star}\|e_{j_\star}\|_2^2}.
\end{equation}
We control the numerator and denominator in \eqref{eq:sr-reduce-embedding} on a high-probability
event.

Beginning with the denominator in \eqref{eq:sr-reduce-embedding}, we observe that the standard $\chi^2$ lower tail bound
(e.g.\ \cite[Theorem~3.1.1]{vershynin2018high}) implies
\begin{equation}\label{eq:chi2-lb-one}
\PP\big(\|e_{j_\star}\|_2^2<(1-\varepsilon)d\big)\le 2\exp(-C\varepsilon^2 d),
\end{equation}
for a universal constant $C>0$. Moving on to the numerator in 
\eqref{eq:sr-reduce-embedding}, define 
\[
T:=\tr(XX^\top)=\sum_{j\in S} c_j\|e_j\|_2^2.
\]
Note that the triangle inequality for the sub-exponential norm implies $\big\|\|e_j\|_2^2-d\big\|_{\psi_1}\lesssim d$. Note moreover $\EE[T]=d\sum_{j\in S}c_j=dn$.
Therefore, Bernstein's inequality \cite[Corollary 2.9.2]{vershynin2018high} ensures
\begin{equation}\label{eqn:bernstein_tokens}
\PP\left(T>(1+\varepsilon)dn\right))=\PP\left(\sum_{j\in S}c_j(\|e_j\|_2^2-d)>\varepsilon dn\right)\leq  2\exp\left(-C\left( \frac{\varepsilon^2 n^2}{\sum_{j\in S}c_j^2}\wedge \frac{\varepsilon }{ p_{\rm max}}\right)\right)
\end{equation}
for some numerical constant $C$. Since $\sum_{j\in S} c_j^2\le c_{j^{\star}}\cdot\sum_{j\in S}c_j = (np_{\max})n=n^2p_{\max}$, we
have $n^2/\sum_{j\in S}c_j^2\ge 1/p_{\max}\ge 1$, and therefore
the right side of \eqref{eqn:bernstein_tokens} is bounded by $2\exp(-C\varepsilon^2 d)$. Taking a union bound over the two events and continuing with \eqref{eq:sr-reduce-embedding}, we deduce that with probability at least $1-4\exp(-C \varepsilon^2 d)$, we have $\st(X)\leq \frac{(1+\varepsilon)dn}{(1-\varepsilon)c_{j^{\star}}d}=\frac{1+\varepsilon}{1-\varepsilon}\tfrac{1}{p_{\max}}$ as claimed.
\end{proof}

Thus, under the Gaussian embedding model, the activations entering the \emph{first}
transformer block already have low stable rank whenever the empirical token
distribution is not too uniform (equivalently, whenever $p_{\max}$ is not too
small). In particular, if $p_{\max}$ is bounded below by a numerical constant,
then both $X$ and $A^{\mathrm{rms}}$ have stable rank $O(1)$ at initialization. Interestingly, even though $1/p_{\max}$ is an exact formula for $\st(H)$, we see from Figure~\ref{fig:embedding} that the bound $\st(E) \leq 1/p_{\max}$ is generally loose.

\section{Propagation of the stable rank}\label{sec:propagate}
In the previous section, we showed that a nonzero mean of the activation function (e.g.\ ReLU) typically forces the post-activation matrix to have a low stable rank. In this section, we prove propagation bounds at Gaussian initialization for the standard building blocks used in transformer blocks (pointwise nonlinearities, residual connections, gating, and RMSNorm): the stable rank of the output is controlled by the stable ranks of the matrices entering the operation, up to explicit factors. We first establish these bounds for the individual (atomic) operations, and then package them into layerwise bounds for attention and MLP sublayers.

Since the propagation results for the RMSNorm depend on per-column $\ell_2$ norms, we also record the column-norm regularity needed to apply it under iteration. As shown earlier for transformer input sequences (see Section~\ref{sec:embedding}), the embedding layer already produces a data matrix with a low stable rank. Iterating the rules proved here, we conclude in Section~\ref{sec:transformers} that the propagated data matrices inside the transformer tend to have low stable rank as well.

\subsection{Atomic operations}

We first analyze a series of atomic operations that roughly preserve low-stable rank structure. In order to streamline the reading, we provide in the following table a glossary of the main theorems in the section.

\begin{table}[h]
\centering
\begin{tabular}{|c|c|c|c|c|}
\hline
Linear activation & Nonlinear activation & Residual connection & RMSnorm & Gating \\
\hline
Theorem~\ref{lem:frob_gauss_scaling} & Theorem~\ref{thm:nonline_activation} & Theorem~\ref{lem:stablerank_X_plus_WH} & Theorem~\ref{thm:rms-stablerank} & Theorem~\ref{thm:gating} \\
\hline
\end{tabular}
\end{table}

\subsubsection{Linear layer}
We begin by analyzing the stable rank of the post-activation matrix of a linear layer. This is the content of the following theorem.


\begin{thm}[Linear activation]\label{lem:frob_gauss_scaling}
Let $W\in \R^{k\times d}$ have iid Gaussian entries $\mathcal{N}(0,\tfrac{1}{d})$ and fix a matrix $X\in \R^{d\times n}$. Then there exists a constant $c>0$ such that for any $\varepsilon\in (0,1)$ we have
\begin{align}
(1-\varepsilon)\tfrac{1}{d}\|X\|^2_F&\leq \tfrac{1}{k}\|WX\|_F^2\leq (1+\varepsilon)\tfrac{1}{d}\|X\|^2_F, \label{eqn:gaussi_preserve normf} \\
(1-\varepsilon) \tfrac{1}{d}\|X\|_{\op}^2&\leq \tfrac{1}{k}\|WX\|^2_{\op}
\label{eqn:gaussi_preserve norm}
\end{align}
with probability at least $1-4\exp(-c\varepsilon^2 k)$. Consequently, in this event the estimate holds:
$$\st(WX)\leq \frac{1+\varepsilon}{1-\varepsilon}\cdot\st(X).$$
\end{thm}
\begin{proof}
Forming a singular value decomposition of $X$, and using orthogonal invariance of Gaussian matrices we see that the singular values of $WX$ have the same distribution as those of $W\Sigma$, where $\Sigma$ is the diagonal matrix with the singular values $\{s_i\}_{i=1}^d$ of $X$ on the diagonal. Now we may write 
$$\|W\Sigma\|_F^2-\tfrac{k}{d}\|X\|_F^2=\sum_{i=1}^d s^2_i(\|w_i\|_2^2-\tfrac{k}{d}),$$
where $w_i$ are the columns of $W$. Classically, we have the sub-exponential bound $\|\|w_i\|_2^2-\tfrac{k}{d}\|_{\psi_1}\lesssim \frac{1}{d}$; see e.g. \cite[Theorem~3.1.1]{vershynin2018high}. Therefore, Bernstein's inequality \cite[Theorem 2.8.2]{vershynin2018high} yields for any $t\geq 0$ the estimate:
$$\PP\left(|\|W\Sigma\|_F^2-\tfrac{k}{d}\|X\|_F^2|\geq t\right)\leq 2\exp\left(-c\cdot\min\left(\tfrac{d^2t^2}{\sum_{i=1}^d s_i^4},\tfrac{dt}{\max_{i} s_i^2}\right)\right).$$
Setting $t=\varepsilon\tfrac{k}{d} \|X\|_F^2$ and noting that $\|s\|_2^4/\|s\|_4^4\geq \|s\|_2^2/\|s\|^2_{\infty}$ completes the proof of \eqref{eqn:gaussi_preserve normf}. 

Now clearly, we have $\|W\Sigma\|^2_{\op}\geq s_1^2 \|w_1\|^2_2$, where $s_1$ is the maximal singular value of $X$ and $w_1\sim \mathcal{N}(0,\frac{1}{d}I_{k})$ is the first column of $W$. Taking into account concentration of the norm \cite[Theorem 3.1.1]{vershynin2018high}, we deduce that with probability at least $1-\exp(-ck \varepsilon^2)$, we have $\|w_1\|^2_2\geq (1-\varepsilon){\frac{k}{d}}$, which completes the proof.
\end{proof}

\subsubsection{RMSNorm layer}

Next, we analyze the propagation of the stable rank through an RMSNorm block, which rescales all the columns in the data matrix $X\in\R^{d\times n}$ to have $\ell_2$-norm $\sqrt{d}$.

\begin{thm}[RMSNorm]
  \label{thm:rms-stablerank}
  Let $X=[x_1,\dots,x_n]\in\R^{d\times n}$ be a deterministic matrix whose
  columns satisfy 
  \[
    L \;\le\; \|x_t\|_2^2 \;\le\; U
    \qquad\text{for all }t=1,\dots,n,
  \]
  for some values $0<L\le U<\infty$. Define the matrix
  \[
    A^{\mathrm{rms}} := \mathrm{RMSNorm}(X)
    \qquad \textrm{with columns}\qquad
    a_t := \sqrt{d}\,\frac{x_t}{\|x_t\|_2}.
  \]
  Then the estimate holds:
  \[
    \st\big(A^{\mathrm{rms}}\big)
    \;\le\;
    \frac{U}{L}\,\st(X).
  \]
\end{thm}
\begin{proof}
  Let $D$ be the diagonal matrix with diagonal entries $D_{tt} := \sqrt{d}/\|x_t\|_2$, so
  that we may write $A^{\mathrm{rms}} = X D$. The bounds on column norms imply
  \[
    \sqrt{\frac{d}{U}} \;\le\; D_{tt} \;\le\; \sqrt{\frac{d}{L}}
    \qquad\text{for all }t.
  \]
  Setting $s_{\min} := \sqrt{\frac{d}{U}}$ and $s_{\max}:=\sqrt{\frac{d}{L}}$ we therefore deduce
  \[
    s_{\min}\,\|X\|_{\op}
    \;\le\; \|A^{\mathrm{rms}}\|_{\op}
    \;\le\; s_{\max}\,\|X\|_{\op}.
  \]
  On the other hand, we have
  \[
    \|A^{\mathrm{rms}}\|_F^2
    = \sum_{t=1}^n \|a_t\|_2^2
    = nd.
  \]
  Combining these facts and using the estimate $\|X\|_F^2\ge nL$ yields
  \[
    \st(A^{\mathrm{rms}})
    = \frac{\|A^{\mathrm{rms}}\|_F^2}{\|A^{\mathrm{rms}}\|_{\op}^2}
    \;\le\;
    \frac{nd}{s_{\min}^2\,\|X\|_{\op}^2}
    = \frac{U}{d}\,\frac{nd}{\|X\|_{\op}^2}
    \leq \frac{U}{L}\,\frac{\|X\|_F^2}{\|X\|_{\op}^2}
    = \frac{U}{L}\,\st(X),
  \]
  which completes the proof.
\end{proof}

\subsubsection{Pointwise nonlinear layer}

Next, we consider stable rank propogation through a nonlinear activation function $\sigma$ that is square integrable under Gaussian measure. This assumptions allows us to expand the activation function in the Hermite basis. Namely, letting $h_k\colon\R\to\R$ be the orthonormal Hermite polynomials on $\R$, we define the coefficients:
$$\hat{\sigma}_k(s):=\EE_{\gamma}[\sigma(s\gamma)h_k(\gamma)],$$ where $\gamma\sim \mathcal{N}(0,1)$ is a standard normal. Equivalently, using Stein's lemma we may interpret Hermite coefficients as average derivatives:  
$$\hat{\sigma}_k(s)=\frac{s^k}{\sqrt{k}}\cdot\EE_{\gamma} \sigma^{(k)}(s\gamma),$$
as long as $\sigma$ is $k$-times weakly differentiable and polynomially bounded. We will need the following lemma that expresses 
$k(x,x')=\EE_{g}[\sigma(\langle x,g\rangle)\sigma(\langle x',g\rangle)]$, with $g\sim \mathcal{N}(0,\tfrac{1}{d}I_d)$ in a Hermite series. The proof follows by standard techniques, but we include it for completeness.

\begin{lem}[Kernel Hermite expansion]\label{lem:hermite_expand}
Consider a function $\sigma\colon\R\to\R$ that is square integrable with respect to the Gaussian measure and define the kernel function $k(x,x')=\EE_{g}[\sigma(\langle x,g\rangle)\sigma(\langle x',g\rangle)],$ where the expectation is taken with respect to  $g\sim \mathcal{N}(0,\tfrac{1}{d}I_d)$. Then we may write 
$$k(x,x')=\sum_{k=0}^\infty \hat\sigma_k\left(\tfrac{\|x\|_2}{\sqrt{d}}\right)\hat\sigma_k\left(\tfrac{\|x'\|_2}{\sqrt{d}}\right) \left(\tfrac{\langle x,x'\rangle}{\|x\|_2\cdot \|x'\|_2}\right)^{k}.$$
\end{lem}
\begin{proof}
Let $s:=\|x\|_2/\sqrt d$ and $s':=\|x'\|_2/\sqrt d$. Define
$U:=\langle x,g\rangle/s$ and $V:=\langle x',g\rangle/s'$.
Then $(U,V)$ is jointly Gaussian with $\EE[U^2]=\EE[V^2]=1$ and correlation
$
\rho=\EE[UV]=\frac{\langle x,x'\rangle}{\|x\|_2\|x'\|_2}.
$
Hence
\[
k(x,x')=\EE\big[\sigma(\langle x,g\rangle)\sigma(\langle x',g\rangle)\big]
=\EE\big[\sigma(sU)\sigma(s'V)\big].
\]
We now expand the function $\gamma\mapsto\sigma(s\gamma)$ in a Hermite basis (see e.g.\ \cite[Prop.~11.30]{ODonnellAOBF}):
\[
\sigma(s\gamma)=\sum_{k\ge0}\hat\sigma_k(s)\,h_k(\gamma),\qquad\textrm{where}\qquad
\hat\sigma_k(s):=\EE_\gamma[\sigma(s\gamma)h_k(\gamma)].
\]
Hence, we obtain the expression
\[
k(x,x')=\sum_{k,\ell\ge0}\hat\sigma_k(s)\hat\sigma_\ell(s')\,\EE[h_k(U)h_\ell(V)].
\]
For $\rho$-correlated standard Gaussians $(U,V)$, the Hermite polynomials satisfy
$\EE[h_k(U)h_\ell(V)]=\rho^k\mathbf{1}_{k=\ell}$ (see \cite[Prop.~11.31]{ODonnellAOBF}), which completes the proof.
\end{proof}

We are now ready to establish stable rank propagation through a nonlinear activation. The only assumptions we make on $\sigma$ is that $|\sigma(t)|\leq |t|$ for all $t\in \R$ and that the first Hermite coefficient $\hat{\sigma}_1\left(\tfrac{\|x_i\|}{\sqrt{d}}\right)$ is nonzero, where $x_i$ are the columns of the data matrix. Concretely, for the ReLU activation function a quick computation shows $s^{-1}\hat \sigma_1(s)=\tfrac{1}{2}$ for any $s$ and therefore we may set $p=1/2$ in \eqref{eqn:defnp}. Similar computations for other common activation functions appear in Table~\ref{table_common_act_func}.

\begin{table}[h]
\centering
\begin{tabular}{|c|c|}
\hline
Activation $\sigma(t)$ & $s^{-1}\hat{\sigma}_1(s)$ \\ \hline
Linear $\;t$ & $1$ \\ \hline
ReLU $\;\max(0,t)$ & $\tfrac{1}{2}$ \\ \hline
Leaky ReLU $\;\max(t,\alpha t)$ & $\tfrac{1+\alpha}{2}$ \\ \hline
Exact GELU $\;t\,\Phi(t)$ & $\tfrac{1}{2}$ \\ \hline
HardTanh $\;\mathrm{clip}(t,-1,1)$ & $2\Phi(1/s)-1$ \\ \hline
$\tanh(t)$ & $\mathbb{E}[\mathrm{sech}^2(s\gamma)]$ \\ \hline
Softsign $\;t(1+|t|^{-1})$ & $\mathbb{E}\!\left[(1+|s\gamma|)^{-2}\right]$ \\ \hline
SiLU $\;t(1+e^{-t})^{-1}$ & $\tfrac12$ \\ \hline
\end{tabular}
\caption{Values of $s^{-1}\hat{\sigma}_1(s)=\mathbb{E}[\sigma'(s\gamma)]$ for common activation functions, where $\gamma\sim\mathcal N(0,1)$ and $\Phi$ denotes the standard normal CDF.\label{table_common_act_func}}
\end{table}

\begin{thm}[Nonlinear activation]\label{thm:nonline_activation}
Consider a matrix $W\in\R^{k\times d}$ with iid Gaussian entries $\mathcal{N}(0,\tfrac{1}{d})$ and let $X\in\R^{d\times n}$ be an arbitrary matrix with columns $x_i\in\R^{k}$. Let $\sigma\colon\R\to\R$ be a function satisfying $|\sigma(t)|\leq |t|$ for all $t\in \R$ and whose first Hermite coefficient satisfies
\begin{equation}\label{eqn:defnp}
p:=\min_{i=1\ldots,n}~\tfrac{d}{\|x_i\|_2^2}\cdot\hat{\sigma}^2_1\left(\tfrac{\|x_i\|_{2}}{\sqrt{d}}\right)>0.
\end{equation}
Then there exists a constant $c>0$ such that for any $\varepsilon\in (0,1)$ we have
\begin{align}
(1-\varepsilon)\frac{p}{d}\|X\|^2_F&\leq \frac{1}{k}\|\sigma(WX)\|_F^2\leq (1+\varepsilon)\frac{1}{d}\|X\|^2_F,\label{eqn:gaussi_preserve frob}\\
(1-\varepsilon) \frac{p}{d}\|X\|_{\op}^2&\leq \frac{1}{k}\|\sigma(WX)\|^2_{\op},
\label{eqn:gaussi_preserve norm2}
\end{align}
with probability at least $1-8\exp(-c\varepsilon^2p^2 k)$. Consequently, in this event the estimate holds:
$$\st(\sigma(WX))\leq \frac{1+\varepsilon}{1-\varepsilon}\cdot\frac{1}{p}\cdot\st(X).$$
\end{thm}
\begin{proof}
First, taking into account that $|\sigma(t)|\leq |t|$ for all $t$, we deduce
$$\|\sigma(WX)\|^2_F\leq \|WX\|_F^2.$$
Therefore Lemma~\ref{lem:frob_gauss_scaling} shows that with probability at least $1-2\exp(-c\varepsilon^2 k)$ we have $\frac{1}{k}\|\sigma(WX)\|^2_F\leq (1+\varepsilon)\frac{1}{d}\|X\|_F^2$, thereby establishing the right side of \eqref{eqn:gaussi_preserve frob}. 

Define now the matrix 
$$M_k:=\frac{1}{k}\sigma(WX)^{\top}\sigma(WX)=\frac{1}{k}\sum_{i=1}^k \sigma(X^{\top}w_i)\sigma(X^{\top}w_i)^{\top},$$
where $w_i^\top$ are the rows of $W$. We first analyze the expectation of $M_k$. To this end, define $z:=\sigma(X^{\top}w)$ where $w\sim \mathcal{N}(0,\tfrac{1}{d}I_d)$, and set $M:=\EE zz^{\top}=\EE[M_k]$.  Note that the $ij$ entry of $M$ is simply $\EE_w\sigma(\langle x_i,w\rangle)\sigma(\langle x_j,w\rangle)$ where $x_i$ and $x_j$ are the $i$'th and $j$'th columns of $X$, respectively. Therefore, applying Lemma~\ref{lem:hermite_expand} entrywise, we may write 
$$M=\sum_{q=0}^{\infty} \widehat{\Sigma}_q\left(\frac{1}{d}X^{\top}X\right)^{\odot q}\widehat{\Sigma}_q,$$
where $\odot q$ denotes a $q$-fold Hadamard product and $\widehat{\Sigma}_q$ is a diagonal matrix with diagonal entries $\hat{\sigma}_q(\tfrac{\|x_i\|_{2}}{\sqrt{d}})\cdot\frac{d^{q/2}}{\|x_i\|^q_2}$. Since the Hadamard product of positive semidefinite matrices is positive semidefinite (Schur product theorem \cite[Sec.~1.2]{Bhatia2007PDM}), every summand on the right side is positive. Dropping all but the term $q=1$ we obtain the lower bound in PSD order
\begin{equation}\label{eqn:lower_bound_m}
M\succeq  \frac{p}{d} X^{\top}X.
\end{equation}
Therefore we obtain the two lower bounds, $\displaystyle\|M\|_{\op}\geq \frac{p}{d}\cdot \|X\|^2_{\op}$ and $\displaystyle\tr(M)\geq \frac{p}{d}\cdot \|X\|^2_{F}$. We now argue using concentration that the operator norm of $M_k$ can not be much smaller than the operator norm of $M$. Indeed, since $\sigma$ is $1$-Lipschitz, Gaussian concentration \cite[Theorem 5.2.3]{vershynin2018high} shows that the vector $z-\EE[z]$ is sub-Gaussian with parameter $K=C\|X\|_{\op}/\sqrt{d}$ for some constant $C$.  Therefore, applying Lemma~\ref{lem:lower_op_norm} we deduce $\|M_k\|_{\op}\geq (1-\varepsilon)\|M\|_{\op}$ with probability at least
$$1-2\exp\left(-ck\left(\tfrac{\varepsilon^2d^2\|M\|^2_{\op}}{\|X\|^4_{\op}}\wedge\tfrac{d\varepsilon\|M\|_{\op}}{\|X\|^2_{\op}}\right)\right)-2\exp\left(-\tfrac{ck\varepsilon^2d\|M\|_{\op}}{\|X\|^2_{\op}}\right).$$
Using \eqref{eqn:lower_bound_m}, we deduce that this probability can be lower bounded by
$1-4\exp(-ck\varepsilon^2(p\wedge p^2))$. Thus in this event, the claimed estimate \eqref{eqn:gaussi_preserve norm2} holds.

Finally, we argue that the trace of $M_k$ can not be much smaller than the trace of $M$. To see this, write $z=\sigma(X^{\top}w)$ and $z_i=\sigma(X^{\top}w_i)$ and let $x_j$ denote the $j$'th row to $X$. Then, we may write
$\tr(M_k)-\tr(M)=\frac{1}{k}\sum_{i=1}^k\|z_i\|_2^2-\EE\|z\|_2^2.$
Since $\sigma$ is 1-Lipschitz with $\sigma(0)=0$, we have $(z_i)_j^2=\sigma(x_j^{\top}w_i)^2\leq (x_j^{\top}w_i)^2$. Therefore we deduce $$\|(z_i)_j^2\|_{\psi_1}\leq \|(x_j^{\top}w_i)^2\|_{\psi_1}=\|x_j^{\top}w_i\|^2_{\psi_2}=\frac{\|x_j\|_2^2}{d},$$
where the first equality follows from \cite[Lemma 2.8.5]{vershynin2018high} and the second equality from the fact that $x_j^{\top}w_i$ is a centered Gaussian with variance $\frac{\|x_j\|_2^2}{d}$. Using the triangle inequality we deduce that $\|z_i\|^2_2$ is sub-Gaussian with parameter $\|X\|_F^2/d$. The centering lemma \cite[Exercise 2.44]{vershynin2018high} therefore implies $\|\|z_i\|^2_2-\EE\|z_i\|^2_2\|_{\psi_1}\lesssim \|X\|_F^2/d$. Applying Bernstein's inequality \cite[Theorem 2.9.1]{vershynin2010introduction}, we therefore deduce for any $\varepsilon>0$ the estimate:
$$\PP(|\tr(M_k)-\tr(M)|\geq \varepsilon \tr(M))\leq 2\exp\left(-c\left(\frac{kd^2\varepsilon^2\tr(M)^2}{\|X\|^4_F}\wedge \frac{kd\varepsilon\tr(M)}{\|X\|_F^2}\right)\right)$$
Taking into account as we have seen that $\tr(M)\geq \frac{p}{d}\|X\|_F^2$ and $p<1$ due to the standing growth assumption on $\sigma$, the right side is bounded by  
$2\exp\left(-ck\varepsilon^2 p^2\right)$. Taking a union bound over all the encountered events completes the proof of \eqref{eqn:gaussi_preserve frob} and \eqref{eqn:gaussi_preserve norm2}. 
Within these events, we then compute
$$\st(\sigma(WX))\leq \frac{\|M_k\|_F}{\|M_k\|_{\op}}\leq \frac{(1+\varepsilon)\tfrac{1}{d}\|X\|_F^2}{(1-\varepsilon)\tfrac{p}{d}\|X\|^2_{\op}}= \frac{1+\varepsilon}{1-\varepsilon}\cdot\frac{1}{p}\cdot\st(X),$$
as we had to show.
\end{proof}

\subsubsection{Residual connections}
Next, we establish stable rank propagation through residual skip connections. 

\begin{thm}[Residual skip connection]\label{lem:stablerank_X_plus_WH}
Let $X \in \R^{d \times n}$ and $H \in \R^{k \times n}$ be fixed matrices and let $W \in \R^{d \times k}$ have
iid Gaussian entries $
\mathcal N(0,\frac{1}{k})$.
Then there exists a constant $c>0$ such that for every
$\varepsilon \in (0,1)$, the estimate
\begin{equation}\label{eq:stablerank_X_plus_WH_goal}
  \st(X + WH)
  \;\le\;
  \frac{1+\varepsilon}{1-\varepsilon}
  \cdot
  \frac{\tfrac{1}{d}\|X\|_F^2 + \frac{1}{k}\,\|H\|_F^2}
       {\max\Big\{\tfrac{1}{d}\|X\|_{\op}^2,\; \frac{1}{k}\,\|H\|_{\op}^2\Big\}}.
\end{equation}
holds with probability at least $1 - 10\exp(-c\varepsilon^2 d)$.
\end{thm}
\begin{proof}
Define the matrix
$G := WH$.
We bound the Frobenius and operator norms of $X+G$ separately, beginning with the former.  To this end, we compute
$$\|X+G\|_F^2=\|X\|_F^2+\|G\|_F^2+2\langle X,G\rangle.$$
Using Lemma~\ref{lem:frob_gauss_scaling}, we see that the estimate $\|G\|_F^2\leq (1+\tfrac{\varepsilon}{2})\frac{d}{k}\|H\|_F^2$ holds with probability at least $1-4\exp(-c\varepsilon^2d)$.
Let us bound the last term 
$$\langle X,G\rangle=\tr(WHX^{\top})=\sum_{i,j}W_{ij}(HX^{\top})_{ji}.$$
Note that the square sum of coefficients is simply $\sum_{i,j} (HX^{\top})_{ji}^2=\|HX^{\top}\|_F^2$. Therefore 
Hoeffding's inequality \cite[Theorem 2.7.3]{vershynin2018high} shows that for any $\varepsilon>0$ the estimate $|\langle X,G\rangle|\leq \varepsilon\sqrt{\frac{d}{k}} \|HX^{\top}\|_F$ holds with probability at least $1-2\exp(-cd\varepsilon^2)$. Using submultiplicativity of the Frobenius norm and Young's inequality, we further estimate 
$$\sqrt{\frac{d}{k}} \|HX^{\top}\|_F\leq \sqrt{\frac{d}{k}}  \|H\|_F\cdot\|X\|_F\leq \frac{d}{2k}\|H\|_F^2+\frac{1}{2}\|X\|_F^2.$$
Thus in this event, we have the estimate 
$$\|X+G\|_F^2\leq (1+\varepsilon)(\|X\|_F^2+\frac{d}{k}\|H\|_F^2).$$

Next, we lower bound the operator norm $\|X+G\|_{\op}$ by $\|X\|_{\op}$ and $\|H\|_{\op}$. To this end, fix a unit vector $v$. Then expanding the square we compute
\begin{equation}\label{eqn:blupty}
\|(X+G)v\|^2_2=\|Xv+Gv\|^2_2=\|Xv\|^2_2+\|g\|^2_2+2\langle Xv,g\rangle
\end{equation}
where we set $g:=Gv=WHv\sim \mathcal{N}(0,\tfrac{\|Hv\|^2_2}{k}I_{d})$. Hoeffding's inequality shows that for any $t>0$ we have
$|\langle Xv,g\rangle|\leq t\sqrt{\frac{d}{k}}\|Xv\|_{2}\|Hv\|_2$
with probability at least $1-2\exp(-cdt^2)$. 
Concentration of the norm in turn shows that for any $\varepsilon\in (0,1)$ we have $|\|g\|_2-\sqrt{\frac{d}{k}}\|Hv\|_2|\leq \varepsilon\sqrt{\frac{d}{k}}\|Hv\|_2$ with probability at least $1-2\exp(-cd\varepsilon^2)$.

{\em Case I: Lower-bounding with $\|X\|_{\op}$.} Suppose now that the unit vector $v$ satisfies $\|Xv\|_2=\|X\|_{\op}$.
Then setting $t=\varepsilon(1-\varepsilon)^2$, we deduce
\begin{align*}
\|g\|_2^2+2\langle Xv,g\rangle&\geq (1-\varepsilon)^2\frac{d}{k}\|Hv\|^2_2 - 2t\sqrt{\frac{d}{k}}\|Hv\|_{2}\|X\|_{\op}\\
&\geq (1-\varepsilon)^2\frac{d}{k}\|Hv\|^2_2-t\frac{d}{k}\|Hv\|_2^2-t\|X\|_{\op}^2\\
&\geq -t\|X\|_{\op}^2,
\end{align*}
where the second to last inequality follows from Young's inequality. Thus in this event we have 
$$\|(X+G)v\|^2_2\geq (1-\varepsilon(1-\varepsilon)^2)\|X\|_{\op}^2.$$
Clearly $(1-\varepsilon)^2$ can be treated as a constant and absorbed into $\varepsilon$.

{\em Case II: Lower-bounding with $\|H\|_{\op}$.} Suppose 
now that the unit vector $v$ satisfies $\|Hv\|_2=\|H\|_{\op}$.
Then using Young's inequality, we deduce
\begin{align*}
\|Xv\|_2^2+2\langle Xv,g\rangle&\geq \|Xv\|^2_2 - 2t\sqrt{\frac{d}{k}}\|Hv\|_{2}\|Xv\|_{2}\geq -t^2\frac{d}{k}\|Hv\|_2^2.
\end{align*}
Consequently, we have 
$$\|(X+G)v\|_2^2=\|g\|_2^2+\|Xv\|_2^2+2\langle Xv,g\rangle\geq ((1-\varepsilon)^2-t^2)\frac{d}{k}\|Hv\|^2_2.$$
Using the inequality $(1-\varepsilon)^2-t\geq 1-2\varepsilon-t$ and replacing $t$ and $\varepsilon$ by $\varepsilon/3$ completes the proof. 
\end{proof}

\paragraph{Residual connections preserve column norms.}

In Section~\ref{sec:transformers} we will apply the blockwise calculus rules from this subsection to control the stable rank of activations inside a transformer at Gaussian initialization. Several of those bounds, most notably RMSNorm (Theorem~\ref{thm:rms-stablerank}), require uniform control of the column norms of the matrices entering a block, not just their stable rank. Thus we must understand how column norms change when we pass through a residual connection. At initialization, a typical residual update has the form $x_t \mapsto x_t + W h_t$, where $W$ is a Gaussian matrix. The following lemma records the needed fact: uniformly over columns, $\tfrac{1}{d}\|x_t+Wh_t\|_2^2$ concentrates around $\tfrac{1}{d}\|x_t\|_2^2+\frac{1}{k}\|h_t\|_2^2$, ensuring that the column-norm regularity required by the subsequent stable-rank bounds is preserved under such residual updates.

\begin{lem}[Residual connection and column norms]
\label{lem:gaussian-residual-preserves-column-norms}
Let $X=[x_1,\dots,x_n]\in\R^{d\times n}$ and $H=[h_1,\dots,h_n]\in\R^{k\times n}$ be
deterministic and let $W\in\R^{d\times k}$ have iid entries $\cN(0,\tfrac{1}{k})$.  Then for any $\varepsilon\in(0,1)$, with probability at least
$1-C\,n\exp(-c\,\varepsilon^2 d)$, simultaneously for all $t=1,\ldots,n$ we have
\[
(1-\varepsilon)\Bigl(\tfrac{1}{d}\|x_t\|_2^2+\tfrac{1}{k}\|h_t\|_2^2\Bigr)
\;\le\;
\tfrac{1}{d}\|x_t+Wh_t\|_2^2
\;\le\;
(1+\varepsilon)\Bigl(\tfrac{1}{d}\|x_t\|_2^2+\tfrac{1}{k}\|h_t\|_2^2\Bigr).
\]
\end{lem}

\begin{proof}
Fix $t$ and set $\sigma_t^2:=\|h_t\|_2^2/k$.  Then $g_t:=Wh_t$ follows the law $\cN(0,\sigma_t^2 I_d)$, and therefore we may write
$g_t=\sigma_t z$ with $z\sim\cN(0,I_d)$ and $\|g_t\|_2^2=\sigma_t^2\|z\|_2^2$.
The standard $\chi^2$ tail bound implies
\[
\PP\Big(\big|\|z\|_2^2-d\big|\ge \tfrac{\varepsilon}{4}d\Big)\le 2\exp(-c\varepsilon^2 d),
\]
and hence on this event we have
\begin{equation}\label{eq:gt-conc-new}
(1-\tfrac{\varepsilon}{4})\,\sigma_t^2 d \;\le\; \|g_t\|_2^2 \;\le\; (1+\tfrac{\varepsilon}{4})\,\sigma_t^2 d.
\end{equation}
Now assume that $x_t \neq 0$; otherwise, the result follows from~\eqref{eq:gt-conc-new}.

Note that $\langle x_t,g_t\rangle$ is Gaussian with mean $0$ and variance $\sigma_t^2\|x_t\|_2^2$, and therefore
\[
\PP\Big(|2\langle x_t,g_t\rangle|\ge \tfrac{\varepsilon}{2}\big(\|x_t\|_2^2+\sigma_t^2 d\big)\Big)
\le 2\exp\!\Big(-c\,\varepsilon^2\,\frac{\big(\|x_t\|_2^2+\sigma_t^2 d\big)^2}{\sigma_t^2\|x_t\|_2^2}\Big)
\le 2\exp(-c\varepsilon^2 d),
\]
where the last inequality follows from the algebraic identity $(\|x_t\|_2^2+\sigma_t^2 d)^2/(\sigma_t^2\|x_t\|_2^2)\ge d$.

On the intersection of these two events, we deduce that the quantity
\[
\|x_t+g_t\|_2^2=\|x_t\|_2^2+\|g_t\|_2^2+2\langle x_t,g_t\rangle
\]
lies between
\[
\|x_t\|_2^2+(1-\tfrac{\varepsilon}{4})\sigma_t^2 d-\tfrac{\varepsilon}{2}\big(\|x_t\|_2^2+\sigma_t^2 d\big)
\quad\text{and}\quad
\|x_t\|_2^2+(1+\tfrac{\varepsilon}{4})\sigma_t^2 d+\tfrac{\varepsilon}{2}\big(\|x_t\|_2^2+\sigma_t^2 d\big),
\]
which are respectively bounded below by $(1-\varepsilon)(\|x_t\|_2^2+\sigma_t^2 d)$ and above by
$(1+\varepsilon)(\|x_t\|_2^2+\sigma_t^2 d)$ for $\varepsilon\in(0,1)$.
Since $\sigma_t^2 d=(d/k)\|h_t\|_2^2$, this proves the desired inequality for the fixed $t$.
Finally, the failure probability for a fixed $t$ is at most $C\exp(-c\varepsilon^2 d)$, and a union bound over
$t=1,\dots,n$ yields the stated probability $1-Cn e^{-c\varepsilon^2 d}$ after adjusting constants.
\end{proof}

\subsubsection{Gated nonlinear layer}

The final building block we will consider is the gating transformation, which has the form 
$$X_{\rm gate}:=\sigma(VZ)\odot WX$$ for fixed data matrices $X$ and $Z$ and random weight matrices $V$ and $W$. Here, $\odot$ denotes the Hadamard product. The key to establishing propagation of the stable rank through the gating block is the following theorem, which states that under reasonable conditiones the operator norm and trace of a matrix of the form $\sum_{i=1}^k (a_i a_i^\top)\odot H$ is well distributed among the summands.

\begin{thm}[Dispersion of operator norm and trace]
\label{thm:sum-dominates-max-relu-LU}
Consider a positive semidefinite matrix $H\in\mathbb{R}^{n\times n}$ and a sequence of random vectors $a_1,\ldots, a_k\in \R^n$. Define now the matrices
\[
X_i \;:=\; (a_i a_i^\top)\odot H \;\in\;\mathbb{R}^{n\times n}
\qquad\textrm{and their sum}\qquad
S \;:=\; \sum_{i=1}^k X_i.
\]
We suppose that $a_i$ are iid realizations of the vector $a=\sigma(Z^{\top}v)$ where $v\sim N(0,\frac{1}{d}I_{d})$ and $Z=[z_1,\dots,z_n]\in\mathbb{R}^{d\times n}$ is a fixed matrix. We suppose moreover that 
$\sigma\colon\R\to\R$ is an activation function that is $1$-Lipschitz with $\sigma(0)=0$.
Then the following are true.
\begin{enumerate}
\item {\bf (Trace dispersion)} Suppose that the diagonal entries of $H$ satisfy 
$$
L \;\le\; H_{jj} \;\le\; U
\qquad\text{for all }j=1,\dots,n,
$$
for some values $0<L\le U<\infty$. Suppose moreover that the first Hermite coefficient is nonzero: 
\begin{equation}\label{eqn:defnp2}
p:=\min_{i=1\ldots,n}~\tfrac{d}{\|z_i\|_2^2}\cdot\hat{\sigma}^2_1\left(\tfrac{\|z_i\|_{2}}{\sqrt{d}}\right)>0.
\end{equation}
Then for any constant $q>0$, there exist constants $c,C<\infty$ such that the estimate 
$$\frac{\tr(S)}{\displaystyle\max_{1\leq i\leq k} \tr(X_i)}\geq \frac{cL}{U}\frac{kp}{\log k}$$
holds with probability at least $1-\exp(-C p^2k)-k^{-q}$.
\item {\bf (Operator norm dispersion)} Suppose that the columns $z_j$ of $Z$ satisfy
\[
L \;\le\; \|z_j\|_2^2 \;\le\; U
\qquad\text{for all }j=1,\dots,n,
\]
for some values $0<L\le U<\infty$. Suppose moreover that for some numerical constant $\kappa>0$ we have $\left|\underset{g\sim \mathcal{N}(0,s^2)}{\EE}\sigma(g)\right|\geq \kappa s$
    for all $s\in [\sqrt{L/d},\sqrt{U/d}]$.
Then for any constant $q>0$, there exist constants $c,C<\infty$ such that the estimate
$$\frac{\|S\|_{\op}}{\displaystyle\max_{1\leq i\leq k}\|X_i\|_{\op}}\geq \frac{cd}{\log(kn)}\cdot\frac{k}{d}\cdot\frac{L}{U},$$
holds with probability at least $1-\exp\left(-Ck\frac{L^2}{U^2}\right)-\frac{1}{(nk)^q}$. 
\end{enumerate}
\end{thm}

\begin{proof}
We first dispense with the Trace Dispersion claim, which is straightforward. Define the matrix $A=\sigma(Z^{\top}V)$ where the columns $v_i$ of $V\in\R^{d\times k}$ are Gaussian $N(0,\frac{1}{d}I_{d})$. Thus $a_i$ are the columns of $A$. Summing the diagonal entries of $X_i$ and using the upper and lower bounds on the diagonal entries of $H$ directly yields the estimates: 
\begin{align*}
L\|a_i\|_2^2&\leq \tr(X_i)\leq U\|a_i\|_2^2\qquad \forall i=1,\ldots, n,\\
L\|A\|_F^2&\leq \tr(S)\leq U \|A\|_F^2.
\end{align*}
We first lower bound $\|A\|^2_F$. To this end, Theorem~\ref{thm:nonline_activation} (equation~\eqref{eqn:gaussi_preserve frob})  shows that there exist constant $c_1,c_2$ such that 
$$\PP\Big(
    \|A\|_F^2
    \;\ge\;
    c_1p\frac{k}{d}\,\|Z\|_F^2
  \Big)
  \;\ge\;
  1 - 8\,\exp\big(-c_2p^2k\big).$$
Let us next upper bound the maximal row norms $\max_i \|a_i\|_2$.
Since $\sigma$ is 1-Lipschitz and satisfies $\sigma(0)=0$, we have $\left\|a_i\|^2\leq \|Z^{\top}v_i\right\|_2^2$. The Hanson-Wright inequality \cite[Exercise 6.13]{vershynin2018high} together with monotonicity of the subgaussian norm then implies $$\left\|\|a_i\|_2-\frac{1}{\sqrt{d}}\|Z\|_F\right\|_{\psi_2}\lesssim \frac{1}{\sqrt{d}}\|Z\|_{\op}.$$
Therefore, for any $t\geq 0$ we have 
$$\PP\left(\|a_i\|_2\geq \frac{1+t}{\sqrt{d}}\|Z\|_F\right)\leq \exp(-ct^2).$$
Taking a union bound over $i=1,\ldots, k$, we deduce $$\PP\left(\max_{i=1,\ldots,k}\|a_i\|_2\geq \frac{1+t}{\sqrt{d}}\|Z\|_F\right)\leq \exp(\log(k)-ct^2).$$
With the choice $t^2=\frac{q+1}{c}\log(k)$, we deduce that there is a constant $C>0$ such that
$$\PP\left(\max_{i=1,\ldots,k}\|a_i\|_2\leq C\sqrt{\frac{{\log(k)}}{d}}\cdot\|Z\|_F\right)\geq 1-\frac{1}{k^q}.$$
Summarizing, we have established that there exist constants $c,C>0$ such that the estimate 
$$\frac{\tr(S)}{\displaystyle\max_{1\leq i\leq k} \tr(X_i)}\geq \frac{cL}{U}\cdot \frac{\frac{pk}{d}\|Z\|^2_F}{\frac{\log(k)}{d}\|Z\|^2_F}=\frac{cL}{U}\frac{pk}{\log k}$$
holds with probability at least $1-8\exp(-C p^2 k)-k^{-q}$.

The remainder of the proof focuses on establishing dispersion of the operator norm.
For any vector $a\in\mathbb{R}^n$, define the diagonal matrix $D_a:=\mathrm{Diag}(a)\in\mathbb{R}^{n\times n}$.
Then clearly equality holds:
\begin{equation}
(a a^\top)\odot H \;=\; D_a\,H\,D_a.
\label{eq:hadamard-diag-LU}
\end{equation}

\medskip
\noindent\textbf{Upper bound on $\|X_i\|_{\op}$.}
Using \eqref{eq:hadamard-diag-LU} and submultiplicativity, we deduce
\[
\|X_i\|_{\mathrm{op}}=\|D_{a_i} H D_{a_i}\|_{\mathrm{op}}
\le \|D_{a_i}\|_{\mathrm{op}}^2\,\|H\|_{\mathrm{op}}
= \|a_i\|_\infty^2\,\|H\|_{\mathrm{op}}.
\]
Fix indices $(i,j)$. Since $v_i\sim\mathcal{N}(0,\frac1d I_d)$, we have
$v_i^\top z_j \sim \mathcal{N}(0,\|z_j\|_2^2/d)$. Since $\sigma$ is $1$-Lipschitz with $\sigma(0)=0$, the entry $(a_i)_j=\sigma(v_i^\top z_j)$ satisfies 
\[
\mathbb{P}\Big(|(a_i)_j| \ge t\Big)
\;\le\;
\mathbb{P}\Big(|v_i^\top z_j| \ge t\Big)
\;\le\;
2\exp\!\Big(-\frac{d t^2}{2\|z_j\|_2^2}\Big)
\;\le\;
2\exp\!\Big(-\frac{d t^2}{2U}\Big),
\]
for any $t\ge 0$.
Taking a union bound over all $kn$ pairs $(i,j)$ yields
\[
\mathbb{P}\Big(\max_{1\le i\le k}\|a_i\|_\infty \ge t\Big)
\;\le\;
\exp\!\Big(\log(2kn)-\frac{d t^2}{2U}\Big).
\]
With the choice $t^2=(q+1)\frac{2U}{d}\log\!\big(2kn\big)$, the right-hand side is at most
$\frac{1}{(kn)^{q}}$. Thus with probability at least $1-\frac{1}{(kn)^{q}}$, we have
\begin{equation}\label{eqn:upper_bound_OptX}
\max_{1\le i\le k}\|X_i\|_{\mathrm{op}}
\le
\|H\|_{\mathrm{op}}\max_{1\le i\le k}\|a_i\|_\infty^2
\le
\frac{2(q+1)U}{d}\,\|H\|_{\mathrm{op}}\,
\log (2kn).
\end{equation}

\medskip
\noindent\textbf{Deterministic lower-bound on $\|S\|_{\op}$.}
Define $\lambda:=\|H\|_{\mathrm{op}}$ and let $u\in\mathbb{R}^n$ be a unit top eigenvector of $H$, that is
$Hu=\lambda u$.
In particular, we clearly have $H\succeq \lambda uu^\top$.
We compute
\begin{align}
\|S\|_{\mathrm{op}}
\;\ge\;
u^\top S u
=\sum_{i=1}^k u^\top D_{a_i} H D_{a_i} u
&=\sum_{i=1}^k (D_{a_i}u)^\top H (D_{a_i}u)\notag
\\
&\ge\;
\lambda \sum_{i=1}^k (D_{a_i}u)^\top (u u^\top) (D_{a_i}u)
=
\lambda \sum_{i=1}^k\underbrace{(u^\top D_{a_i}u)^2}_{=:T_i}.\label{eqn:expansition_TSs}
\end{align}

\medskip
\noindent\textbf{A lower bound on $\mathbb{E}T_i$ via the mean outer product.}
Let $v\sim \mathcal{N}(0,\frac1d I_d)$ and set $a:=\sigma(Z^\top v)\in\mathbb{R}^n$. Observe that each coordinate $a_j$ has the form $a_j=\sigma(z_j^{\top}v)\sim\sigma(s\gamma)$ where $s=\sqrt{\|z_j\|^2_2/d}$ and $\gamma\sim \mathcal{N}(0,1)$. The assumption that for some numerical constant $\kappa>0$ we have $\left|\underset{g\sim \mathcal{N}(0,s^2)}{\EE}\sigma(g)\right|\geq \kappa s$
    for all $s\in [\sqrt{L/d},\sqrt{U/d}]$ therefore ensures that the mean $\mu:=\EE[a]$ satisfies $\mu_j\geq \kappa\sqrt{\frac{L}{d}}$ for all $j$.

Define the kernel matrix $K:=\mathbb{E}[a a^\top]\in\mathbb{R}^{n\times n}$ and observe the estimate
\begin{equation}
K=\mathbb{E}[a a^\top]
=\mathrm{Cov}(a)+\mu\mu^\top
\succeq \mu\mu^\top.
\label{eq:K-lower-by-mean-LU}
\end{equation}
Define now the weights $w_j:=u_j^2$, which clearly satisfy $w_j\geq 0$ and $\sum_{j} w_j=1$. Observe that $T_i$ have the same distribution as $T:=(w^{\top}a)^2$. Therefore, we deduce 
\begin{equation}\label{eqn:lower_T_exp}    
\EE[T]=\EE (w^{\top}a)^2= w^{\top}Kw\geq (w^{\top}\mu)^2\geq \min_j \mu^2_j\gtrsim \frac{L}{d},
\end{equation}
where the second inequality follows from the fact that the weight vector $\{w_j\}_j$ lies on the simplex.

\medskip
\noindent\textbf{Concentration of $\sum_{i=1}^k T_i$ below its mean.}
We will now show that $T-\mathbb{E}T$ is subexponential with parameter $\asymp U/d$.
To see this, write $v=g/\sqrt{d}$ with $g\sim\mathcal{N}(0,I_d)$ and define
\[
f(g) := w^{\top}a= \sum_{j=1}^n w_j \sigma\left(\tfrac{1}{\sqrt{d}}g^\top z_j\right).
\]
Since $\sigma$ is $1$-Lipschitz, for any $g,g'\in\mathbb{R}^d$, we have
\[
|f(g)-f(g')|
\le \frac{1}{\sqrt{d}}\sum_{j=1}^n w_j |(g-g')^\top z_j|
\le \frac{\|g-g'\|_2}{\sqrt{d}} \sum_{j=1}^n w_j\|z_j\|_2
\le \sqrt{\frac{U}{d}}\,\|g-g'\|_2,
\]
Thus $f$ is $\sqrt{\tfrac{U}{{d}}}$-Lipschitz. By Gaussian concentration for Lipschitz functions \cite[Theorem 5.2.3]{vershynin2018high}, the random variable $Y:=f(g)=w^{\top}a$ 
satisfies
\begin{equation}
\|Y-\mathbb{E}Y\|_{\psi_2} \lesssim \sqrt{\frac{U}{{d}}}.
\label{eq:Y-subgaussian-LU}
\end{equation}
Therefore we deduce
\[
\|(Y-\mathbb{E}Y)^2-\mathbb{E}(Y-\mathbb{E}Y)^2\|_{\psi_1}\lesssim \|(Y-\mathbb{E}Y)^2\|_{\psi_1}\asymp \|Y-\mathbb{E}Y\|^2_{\psi_2}
\lesssim
\frac{U}{d},
\]
where the first inequality follows from the centering lemma \cite[Exercise 2.44]{vershynin2018high} while the equality follows from \cite[Lemma 2.8.5]{vershynin2018high}. Moreover, expanding the squares on the left and isolating the cross term gives
$$\|(Y-\mathbb{E}Y)^2-\mathbb{E}(Y-\mathbb{E}Y)^2\|_{\psi_1}\geq \|Y^2-\EE Y^2\|_{\psi_1}-2\underbrace{|\EE Y|}_{\lesssim  \|\mu\|_{\infty}}\cdot\underbrace{\|(Y-\EE(Y))\|_{\psi_1}}_{\lesssim \sqrt{\frac{U}{{d}}}}.$$
Thus, we have arrived at 
$$\|T-\EE T\|_{\psi_1}= \|Y^2-\EE Y^2\|_{\psi_1}\lesssim \frac{U}{d}+\sqrt{\frac{U}{{d}}}\cdot \|\mu\|_{\infty}\lesssim \frac{U}{d},$$
where we use the bound that $\|\mu\|_{\infty} \lesssim \sqrt{U/d}$.
Applying the Bernstein inequality \cite[Theorem 2.9.1]{vershynin2018high} we deduce for any $s\geq 0$ the estimate 
$\sum_{i=1}^k T_i\geq k\cdot \EE T-s$
holds with probability at least $$1-\exp\left(-c\left(\frac{d^2s^2}{kU^2}\wedge \frac{ds}{U}\right)\right).$$
Recall that from \eqref{eqn:lower_T_exp} we further have the lower bound $\EE T\geq \frac{CL}{d}$ for some numerical constant $C$. Therefore setting $s=\frac{kCL}{2d}$, the estimate $\sum_{i=1}^k T_i\geq \frac{kCL}{2d}$ holds with probability at least $1-\exp\left(-C'k\frac{L^2}{U^2}\right)$. In this event, returning to \eqref{eqn:expansition_TSs}, we have 
$$\|S\|_{\op}\geq CL\frac{k}{d}\|H\|_{\op}.$$
Therefore, taking a union bound with the event which ensured \eqref{eqn:upper_bound_OptX}, we deduce that the estimate 
$$\frac{\|S\|_{\op}}{\max_{1\leq i\leq k}\|X_i\|_{\op}}\gtrsim \frac{d}{\log(kn)}\cdot\frac{k}{d}\cdot\frac{L}{U},$$
holds with probability at least $1-\exp\left(-C'k\frac{L^2}{U^2}\right)-\frac{1}{(nk)^q}$. This completes the proof.
\end{proof}

Armed with Theorem~\ref{thm:sum-dominates-max-relu-LU} we can now establish propagation of the stable rank for the gating block, induced by a nonlinear activation with a nonzero mean (e.g. ReLU).

\begin{thm}[Gating]\label{thm:gating}
    Consider the matrix 
    $$Q=\sigma(VZ)\odot WX$$
    where $Z\in \R^{d_1\times n}$ and $X\in \R^{d_2\times n}$ are fixed and $V\in \R^{k\times d_1}$ and $W\in \R^{k\times d_2}$ have iid Gaussian entries, $\mathcal{N}(0,\frac{1}{d_1})$ and $\mathcal{N}(0,\frac{1}{d_2})$ respectively. Suppose moreover that we are in the proportionate regime $k\asymp d_1$.
    Suppose that the column norms of $X$ and $Z$ satisfy 
    $$L\leq \|x_i\|^2_2\leq U \quad \textrm{and}\quad L\leq \|z_i\|^2_2\leq U\qquad \forall i\in [n].$$ Suppose moreover that the activation function $\sigma$ is $1$-Lipschitz with $\sigma(0)=0$ and that for some numerical constants $p,\kappa>0$ we have 
    $$\left|\underset{g\sim \mathcal{N}(0,s^2)}{\EE}\sigma(g)\right|\geq \kappa s\qquad \textrm{and}\qquad \hat{\sigma}_1(s)\geq \sqrt{p}\cdot s\qquad \forall s\in [\sqrt{L/d_1},\sqrt{U/d_1}].$$
 Then for any constant $q>0$, there exist constants $c,C<\infty$ such that the estimate 
    $$\st(Q)\leq \frac{(1+\varepsilon)^2}{(1-\varepsilon)^3}\cdot \frac{U^2}{L^2}\cdot \kappa^{-2}\cdot \st(X)$$
    holds with probability at least $1-C\left(\exp\left(-c\varepsilon^2 \frac{k}{\log^2(k)}\right)-\exp\left(-c\varepsilon^2 k\frac{L^2}{U^2}\right)-k^{-q}\right).$
\end{thm}
\begin{proof}
We proceed by analyzing the gram matrix of $Q$ by first conditioning on $V$. 
Setting the notation, define the matrices $A:=\sigma(VZ)$ and $B:=WX$, let $q_j^{\top}$ denote the $j$'th rows of $Q$, and define the matrices $$S:=Q^{\top}Q=\sum_{j=1}^k q_jq_j^{\top},\qquad  M_j:=\EE[q_jq_j^{\top}\mid V], \qquad M:=\EE [S\mid V].$$ 
    Let us first compute $q_jq_j^{\top}$ and $M_j$ in terms of the rows $b_j^{\top}$ and $w_j^{\top}$ of $B$ and $W$, respectively. To this end, observe that equality $q_j^{\top}=a_j^{\top}\odot (w_j^{\top}X)$ holds and therefore we may write the outer product
    $$q_j^{\top}q_j=a_ja_j^{\top}\odot (X^{\top}w_jw_j^{\top}X).$$
    In particular, taking the expectation with respect to $W$ and with $V$ fixed, we deduce the equalities: 
    \begin{align*}
    M_j=\frac{1}{d_2}(a_ja_j^{\top})\odot (X^{\top}X)\qquad \textrm{and}\qquad M=\frac{1}{d_2}(A^{\top}A)\odot (X^{\top}X)
    \end{align*}
    Thus conditioned on $V$, the vectors $q_j$ are independent mean-zero Gaussian with covariance $M_j$.

     We will now show that the effective rank of $S$ is not much larger than the effective rank of $M$ with high probability over the pair $(V,W)$. We do so by estimating the operator norm and trace, in order. Fix a constant $\varepsilon\in (0,1)$ and define the event 
    $$\mathcal{E}_1=\left\{\|S\|_{\mathrm{op}} \ge (1-\varepsilon)\|M\|_{\mathrm{op}}\right\}.$$
    Observe that $q_j$ is a mean-zero Gaussian random vector with covariance $M_j$. Therefore, we now apply Theorem~\ref{thm:lower_tail_top_eig} with $K_j=1$ and with $V$ fixed. We thus deduce that there exists a constant $c>0$ (independent of $\varepsilon$) such that
    \begin{equation}\label{eqn:upper_bound_needed}
    \displaystyle
\mathbb{P}(\mathcal{E}_1^c\mid V
) \leq \
\exp\!\left(
-c\,\min\!\left\{\frac{\varepsilon^2\|M\|^2_{\op}}{\sum_{j=1}^k  \|M_j\|_{\op}^2},\,\frac{\varepsilon\|M\|_{\op}}{\displaystyle\max_{1\leq j\leq k} \|M_j\|_{\op}}\right\}
\right).
\end{equation}
Now applying Theorem~\ref{thm:sum-dominates-max-relu-LU} we see that for any constant $q>0$, there exist constants $c,C<\infty$ such that the estimate $$\frac{\|M\|_{\op}}{\displaystyle\max_{1\leq j\leq k}\|M_j\|_{\op}}\geq \frac{cd_2}{\log(kn)}\cdot\frac{k}{d_2}\cdot\frac{L}{U},$$
holds with probability at least $1-\exp\left(-Ck\frac{L^2}{U^2}\right)-\frac{1}{(nk)^q}$.
Within this event, which we denote by $\mathcal{E}_2$, the right-side of \eqref{eqn:upper_bound_needed} is bounded by $\omega:=\exp(-\frac{c'\varepsilon^2 k}{\log^2(kn)} )$ for some numerical constant $c'<\infty$, which follows by upper bounding the sum by the $k$ times the maximum. Thus the probability of the complementary event $\mathcal{E}_1$ can be lower bounded as 
\begin{equation}\label{eqn:conditioning}
\begin{aligned}
\PP[\mathcal{E}_1^c]=\EE[\PP(\mathcal{E}_1^c\mid V)]&=\EE[\PP(\mathcal{E}_1^c\mid V){\bf 1}_{\mathcal{E}_2}]+\EE[\PP(\mathcal{E}_1^c\mid V){\bf 1}_{\mathcal{E}^c_2}]\\
&\leq \omega+\PP(\mathcal{E}_2^c)\\
&\leq\exp\left(-\tfrac{c'\varepsilon^2k}{\log^2(kn)} \right)+\exp\left(-Ck\tfrac{L^2}{U^2}\right)+\frac{1}{(nk)^q}.
\end{aligned}
\end{equation}

Next, we estimate the trace of $S$. To this end, fix $\varepsilon>0$ and define the event  
$$\mathcal{F}_1=\left\{\tr(S) \le (1+\varepsilon)\tr(M)\right\}.$$
 We now apply Theorem~\ref{thm:trace} with $K_j=1$ and with $V$ fixed. We thus deduce that there exists a constant $c>0$ (independent of $\varepsilon$) such that 
 \begin{equation}\label{eqn:herereagaian}
\displaystyle
\mathbb{P}(\mathcal{F}_1^c\mid V
) \leq 2\exp\left(-c\left(\frac{\varepsilon^2\tr(M)^2}{\sum_{i=1}^n \tr(M_i)^2}\wedge \frac{\varepsilon \tr(M)}{\displaystyle\max_{i=1,\ldots, k} \tr(M_i)}\right)\right)
\end{equation}
Trace dispersion of Theorem~\ref{thm:sum-dominates-max-relu-LU} guarantees 
$$\frac{\tr(M)}{\displaystyle\max_{1\leq i\leq k} \tr(M_i)}\geq \frac{cL}{U}\frac{k}{\log k}$$
holds with probability at least $1-2n\exp(-C k)-k^{-q}$. Within this event, the right side of \eqref{eqn:herereagaian} is bounded by $\exp\left(\frac{-c'\varepsilon^2 k}{\log^2(k)}\right)$.
A completely analogous argument to \eqref{eqn:conditioning} shows that 
$$\PP[\mathcal{F}_1^c]\leq \exp\left(\frac{-c'\varepsilon^2 k}{\log^2(k)}\right)+2n\exp(-C k)+k^{-q}.$$
It remains now to estimate $\tr(M)$ and $\|M\|_{\op}$. Clearly, we have 
$$\tr(M)\leq \frac{U}{d_2}\|A\|_F^2.$$
Since $\sigma$ is $1$-Lipschitz with $\sigma(0)=0$, we have $\|A\|_F^2\leq \|VZ\|_F^2$. Using Theorem~\ref{lem:frob_gauss_scaling} the estimate $\|VZ\|_F^2\leq (1+\varepsilon)\frac{k}{d_1}\|Z\|_F^2$ holds with probability at least $1-2\exp(-C\varepsilon^2 k)$.

Next, we lower-bound the operator norm $\|M\|_{\op}$. To this end, define $\mu=\EE[a]$. 
Fix $j$ and note that $(a_i)_j=\sigma(v_i^\top z_j)$, where conditional on $z_j$ we have
$v_i^\top z_j\sim \mathcal N(0,\|z_j\|_2^2/d)$ and $\|z_j\|_2^2\le U$. Hence
$\|v_i^\top z_j\|_{\psi_2} \leq \,\|z_j\|_2/\sqrt d\le \sqrt{U/(d)}$.
Since $\sigma$ is $1$-Lipschitz with $\sigma(0)=0$, we have $|\sigma(t)|\le |t|$, and therefore
$\|(a_i)_j\|_{\psi_2}\le \sqrt{U/(d)}$, and moreover $\|(a_i)_j-\mu_j\|_{\psi_2}\le 2\sqrt{U/(d)}$.
Thus, the sub-Gaussian Hoeffding bound yields
$$\PP\!\left(|\hat\mu_j-\mu_j|\ge \varepsilon|\mu_j|\right)
\le 2\exp\!\left(-c\,\varepsilon^2 k\,\frac{\mu_j^2}{U/d}\right).$$
Using the assumption $|\mu_j|=|\EE\sigma(v^\top z_j)|\ge \kappa\sqrt{\|z_j\|_2^2/d}\ge \kappa\sqrt{L/d}$, we get
$$|\hat{\mu}_j-\mu_j|\leq \varepsilon|\mu_j|$$
with probability at least $1-2\exp\!\big(-c\,\varepsilon^2 k\,\kappa^2 L/U\big)$.
Taking a union bound over $j=1,\dots,n$ we deduce
$\sup_j|\hat{\mu}_j-\mu_j|\le \varepsilon|\mu_j|$ with probability at least
$1-2n\exp\!\big(-c\,\varepsilon^2 k\,\kappa^2 L/U\big)$.
Using the elementary observation $\frac{1}{k}A^TA\succeq \hat \mu\hat \mu^{\top}$, we may now estimate
\[
  \tfrac{d_2}{k}\cdot M= (\tfrac{1}{k}A^\top A)\odot(X^\top X)
  \succeq (\hat\mu\hat\mu^\top)\odot (X^\top X)=\Diag(\hat\mu)\,X^\top X\,\Diag(\hat\mu)\succeq (1-\varepsilon)^2\min_{j}\mu_j^2\cdot X^\top X.
\]
Therefore we deduce $$\|M\|_{\op}\geq \frac{(1-\varepsilon)^2 k}{d_2}\min_{j}\mu_j^2 \|X\|_{\op}^2\geq \frac{(1-\varepsilon)^2L\kappa^2}{d_1d_2} \|X\|_{\op}^2,$$
where the last inequality follows from the assumption that $\left|\underset{g\sim \mathcal{N}(0,s^2)}{\EE}\sigma(g)\right|\geq \kappa s$
    for all $s\in [\sqrt{L/d_1},\sqrt{U/d_1}]$. 
    
Therefore, in the intersection of the considered events we have 
$$\st(Q)\leq \frac{1+\varepsilon}{1-\varepsilon}\cdot\er(M)\leq \frac{1+\varepsilon}{1-\varepsilon}\cdot\frac{(1+\varepsilon) k U\|Z\|_F^2/(d_1d_2)}{(1-\varepsilon)^2 k L\kappa^2 \|X\|_{\op}^2/(d_1d_2)}\leq \frac{(1+\varepsilon)^2}{(1-\varepsilon)^3}\cdot \frac{U^2}{L^2}\cdot \kappa^{-2}\cdot \st(X),$$
where the last inequality uses the relation $\|Z\|_F^2\leq n U$ and $\|X\|_F^2\geq n L$. The proof is complete.
\end{proof}

\paragraph{Column norm bounds for Gated activations.}

Finally, we bound the column norms of gated activations. This lemma will be used in conjunction with Lemma~\ref{lem:gaussian-residual-preserves-column-norms} to control the evolution of column norms throughout the transformer architecture.

\begin{lem}[Gated column norms]
\label{lem:gated-colnorm-cond}
Fix a vector $a\in\R^{d}$ satisfying $\|a\|_2^2=d$ and a function $\sigma\colon\R\to\R$ that $1$-Lipschitz with $\sigma(0)=0$ and a second moment $m_2:=\EE[\sigma(\gamma)^2]$ for $\gamma\sim\cN(0,1)$ that is bounded from above and below by a positive numerical constant.   Let $w^{\mathrm{in}},w^{\mathrm{gate}}\in\R^{d}$ be independent random vectors with
iid entries $\cN(0,\tfrac{1}{d})$.
Define
\[
z:=\langle w^{\mathrm{in}},a\rangle,\qquad h:=\langle w^{\mathrm{gate}},a\rangle,\qquad
y:=\sigma(z)\,h.
\]
Define now the vector
$b:=(y_1,\dots,y_k)\in\R^{k}$ where $y_i$ are iid copies of $y$.  Then there exists a numerical constant $c$ such that for every $\varepsilon\in(0,1)$, we have
\[
\PP\!\left(\|b\|_2^2 \ \le\ (1+\varepsilon)^2\,m_2\,k\right)
\ \ge\
1-4\exp(-c\varepsilon^2 k)-2k\exp\!\left(-\tfrac12 k^{1/3}\right)
-2\exp\!\left(-c\varepsilon^2(1-\varepsilon)m_2\,k^{2/3}\right).
\]
\end{lem}

\begin{proof}
Let $(z_i,h_i)_{i=1}^k$ be iid with $z_i,h_i\sim\cN(0,1)$ and set
$y_i=\sigma(z_i)h_i$.  Define $s_i:=\sigma(z_i)^2$ and
\[
S:=\sum_{i=1}^k s_i,\qquad T:=\sum_{i=1}^k s_i(h_i^2-1).
\]
Then clearly, we may write $\|b\|_2^2=\sum_{i=1}^k s_i h_i^2=S+T$.  Conditional on $(s_i)_{i=1}^k$, the random variables
$h_i^2-1$ are independent, mean zero, and sub-exponential with a constant parameter
$C>0$.  Bernstein's
inequality therefore yields a numerical $c_0>0$ such that, for every $u>0$, it holds:
\[
\PP\!\left(T\ge u \,\middle|\, (s_i)_{i=1}^k\right)
\le
\exp\!\left(
-c_0\min\left\{\frac{u^2}{\sum_{i=1}^k s_i^2},\ \frac{u}{\max_{i\le k}s_i}\right\}
\right).
\]
Setting $u:=\varepsilon S$ and noting $\sum_{i=1}^k s_i^2\le (\max_{i\le k}s_i)\sum_{i=1}^k s_i=(\max_i s_i)S$,
we deduce
\[
\min\left\{\frac{u^2}{\sum_{i=1}^k s_i^2},\ \frac{u}{\max_{i\le k}s_i}\right\}
\ge
\min\left\{\frac{\varepsilon^2 S^2}{(\max_i s_i)S},\ \frac{\varepsilon S}{\max_i s_i}\right\}
=
\frac{\varepsilon^2 S}{\max_i s_i},
\]
and hence
\begin{equation}\label{eq:cond-bernstein-gated}
\PP\!\left(T\ge \varepsilon S \,\middle|\, (s_i)_{i=1}^k\right)
\le
\exp\!\left(-c_0\varepsilon^2\,\frac{S}{\max_i s_i}\right).
\end{equation}
Define the events
\[
\mathcal{E}:=\{T\geq \varepsilon S\},\qquad
\mathcal{E}_{\mathrm{sum}}:=\Big\{(1-\varepsilon)m_2k\le S\le (1+\varepsilon)m_2k\Big\},
\qquad
\mathcal{E}_{\max}:=\Big\{\max_{i\le k}s_i\le k^{1/3}\Big\}.
\]
Then we compute
\begin{align}
\PP[\mathcal{E}]=\EE[\PP(\mathcal{E}\mid (s_i)_{i=1}^k)]&=\EE[\PP(\mathcal{E}\mid (s_i)_{i=1}^k)){\bf 1}_{\mathcal{E}_{\mathrm{sum}}\cap \mathcal{E}_{\mathrm{max}}}]+\EE[\PP(\mathcal{E}\mid (s_i)_{i=1}^k)){\bf 1}_{(\mathcal{E}_{\mathrm{sum}}\cap \mathcal{E}_{\mathrm{max}})^c}]\notag\\
&\leq \exp\!\left(-c_0\varepsilon^2(1-\varepsilon)m_2\,k^{2/3}\right)+\PP(\mathcal{E}_{\mathrm{sum}}^{c})+\PP(\mathcal{E}_{\max}^{c}),\label{eq:main-uncond-bound}
\end{align}
where the last inequality follows from \eqref{eq:cond-bernstein-gated} and a union bound over the two events $\mathcal{E}_{\mathrm{sum}}^{c}$ and ${\mathcal E}_{\max}^{c}$.

We now bound $\PP(E_{\mathrm{sum}}^{c})$ and $\PP(E_{\max}^{c})$.  Since $\sigma$ is $1$-Lipschitz with
$\sigma(0)=0$, we have $|\sigma(t)|\le |t|$ for all $t$, and hence $0\le s_i\le z_i^2$.  In particular,
$s_i-m_2$ is sub-exponential with $\|s_i-m_2\|_{\psi_1}\le C$ for a numerical constant $C>0$. Therefore, Bernstein's inequality yields $\PP(\mathcal{E}_{\mathrm{sum}}^{c})\le 2\exp(-c\varepsilon^2 k)$ for a numerical
$c>0$.

For $\mathcal{E}_{\max}$, using $s_i\le z_i^2$ and the Gaussian tail bound
$\PP(z_i^2\ge t)\le 2\exp(-t/2)$, we obtain
\[
\PP(\mathcal{E}_{\max}^{c})
\le
\PP\!\left(\max_{i\le k}z_i^2>k^{1/3}\right)
\le
2k\exp\!\left(-\tfrac12 k^{1/3}\right),
\]
where the last inequality follows from a union bound.
Substituting these estimates into \eqref{eq:main-uncond-bound} shows that, with probability at least
\[
1-2\exp(-c\varepsilon^2 k)-2k\exp\!\left(-\tfrac12 k^{1/3}\right)
-\exp\!\left(-c_0\varepsilon^2(1-\varepsilon)m_2\,k^{2/3}\right),
\]
we have $T\le \varepsilon S$. Taking a further union bound with the event $\mathcal{E}_{\rm sum}$ we further deduce,
\[
\|b\|_2^2=S+T\le (1+\varepsilon)S\le (1+\varepsilon)^2 m_2k,
\]
which is the claim after adjusting constants in the last term.
\end{proof}

\subsection{Compound operations: attention and MLP layers}\label{sec:compoundoperations}
In Section~\ref{sec:transformers}, we will analyze transformer architectures at Gaussian initialization.
Those architectures (see Section~\ref{sec:multiple}) are compositions of two compound sublayers
(attention and MLP), each wrapped with RMS normalization and a residual connection.
To firther iterate transfomer blocks in Section~\ref{sec:transformers} using the atomic calculus rules above, we need two
layerwise propagation inputs: (a) a uniform column-norm envelope for the matrices entering RMSNorm (since
Theorem~\ref{thm:rms-stablerank} depends on the ratio of maximal to minimal column norm), and (b) control
of how the stable rank changes when passing through one attention sublayer and one MLP sublayer.

The next subsubsections provide these packaged bounds by composing the atomic lemmas already proved
(RMSNorm, Gaussian linear maps, pointwise nonlinearities, and Gaussian residual connection in
Lemma~\ref{lem:gaussian-residual-preserves-column-norms} and Theorem~\ref{lem:stablerank_X_plus_WH}).
For the attention operation, the proofs use only the fact that the headwise mixing operators
are column-stochastic; the resulting bounds are uniform in the number of heads and do not depend on
whether masks or positional encodings are used.

For MLP sublayers, we distinguish two activation regimes.
In the \emph{mean-spike-inducing} regime (Section~\ref{sec:propagate}, Theorem~\ref{thm:gaussian_weights_low_sr}), a nonzero Gaussian mean introduces a rank-one contribution to the second moment and forces the post-activation matrix to have uniformly low stable rank; this is recorded in Assumption~\ref{ass:stablesigmarank} and used in Lemma~\ref{lem:mlp-standard}.
In the \emph{propagation} regime treated below, the activation does not provide such a uniform reset; instead, the post-activation matrix has stable rank controlled by that of the RMS-normalized input, and we additionally track a uniform upper bound on the column norms needed for the residual connection and subsequent RMSNorm. Two natural examples are gated MLPs and pointwise activations applied to normalized inputs.

Finally, we include one mixture-of-experts calculation as a one-off illustration of the breadth of the
calculus rules: under arbitrary routing, stable rank still propagates assuming only $\hat\sigma_1(1)>0$
from Theorem~\ref{thm:nonline_activation}.  We will not return to MoE layers elsewhere in the paper.

\subsubsection{Attention layer propagation}

We now derive an attention layer estimate by combining the atomic rules already proved:
RMS normalization (Theorem~\ref{thm:rms-stablerank}), Gaussian linear maps (Theorem~\ref{lem:frob_gauss_scaling}),
and Gaussian residual connections for column norms and stable rank
(Lemma~\ref{lem:gaussian-residual-preserves-column-norms} and Theorem~\ref{lem:stablerank_X_plus_WH}).
The only property of the attention kernels used below is that they act as column-stochastic mixing operators. See the companion Figure~\ref{fig:mha-sublayer}.

\newcommand{\mha}{\mathrm{MHA}}
\newcommand{\concat}{\mathrm{Concat}}

\begin{definition}[Multi-headed value aggregation map]
\label{def:mha-value-agg}
{\rm
Fix integers $n_{\mathrm{head}}\ge 1$ and $d\ge 1$ with  $d$ divisible by $n_{\mathrm{head}}$, and set
$d_{\mathrm{head}}:=d/n_{\mathrm{head}}$.
Fix matrices $P^{(h)}\in\R^{n\times n}$ for $h\in[n_{\mathrm{head}}]$, each with columns that are probability vectors.
Fix matrices $W_O\in\R^{d\times d}$ and $W_V^{(h)}\in\R^{d_{\mathrm{head}}\times d}$ for $h\in[n_{\mathrm{head}}]$.
Given any input $X\in\R^{d\times n}$ define
\[
A^{\mathrm{rms}}:=\RMSNorm(X),\qquad
Y^{(h)}:=A^{\mathrm{rms}}P^{(h)},\qquad
H^{(h)}:=W_V^{(h)}Y^{(h)}\in\R^{d_{\mathrm{head}}\times n},
\]
\[
H:=\concat_{h\in[n_{\mathrm{head}}]} H^{(h)}\in\R^{d\times n},\qquad
X^{\mathrm{att}}:=X+W_OH,\qquad
A^{\mathrm{rms}}_{\mathrm{mlp}}:=\RMSNorm(X^{\mathrm{att}}),
\]
and set $\mha(X):=X^{\mathrm{att}}$.
}
\end{definition}

\paragraph{Where the matrices $P^{(h)}$ come from.}
In a standard decoder block, $P^{(h)}$ is obtained from the headwise query/key projections
$Q^{(h)}=W_Q^{(h)}A^{\mathrm{rms}}$ and $K^{(h)}=W_K^{(h)}A^{\mathrm{rms}}$ via a (masked) softmax:
\[
P^{(h)}=\softmax\!\bigl(d_{\mathrm{head}}^{-1/2}(K^{(h)})^\top Q^{(h)}\bigr),
\]
possibly with positional encodings (RoPE) folded into $Q^{(h)},K^{(h)}$.
For the bounds below we do not use this structure; we only require that each $P^{(h)}$ is column-stochastic.
To keep notation light we write $\mha(X)$, although it depends on all parameters
$\{P^{(h)}\}_{h\in[n_{\mathrm{head}}]}$, $\{W_V^{(h)}\}_{h\in[n_{\mathrm{head}}]}$, and $W_O$.

\input{transformer_block_mha_attention}

\begin{lem}[Attention/RMS sublayer bounds]
\label{lem:block-att-common}
Fix constants $c_-, c_+ > 0$. Let $X \in \RR^{d \times n}$ be a matrix. Suppose the squared column norms of $X$ lie between $c_-d$ and $c_+d$.  We work in the setting of Definition~\ref{def:mha-value-agg}. Assume that $W_O\in\R^{d\times d}$ and
$\{W_V^{(h)}\in\R^{d_{\mathrm{head}}\times d}\}_{h\in[n_{\mathrm{head}}]}$ have iid
entries $\cN(0,1/d)$.  Fix $\delta\in(0,\tfrac14)$. Then there exist numerical constants $c,C>0$ such that,
whenever $d_{\mathrm{head}}\ge C\delta^{-2}\log(nn_{\mathrm{head}})$, the following hold with probability at least
$1-Ce^{-c\delta^2 d_{\mathrm{head}}}$:
\begin{enumerate}
\item[(i)] \textbf{Column norms of $X^{\mathrm{att}}$.} For all $t=1,\dots,n$, we have
\[
(1-\delta)c_-\,d \;\le\; \|x_t^{\mathrm{att}}\|_2^2 \;\le\; (1+\delta)(c_+ + 1)\,d.
\]
\item[(ii)] \textbf{Stable ranks.} The matrices $A^{\mathrm{rms}},X^{\mathrm{att}},A^{\mathrm{rms}}_{\mathrm{mlp}}$ satisfy
\begin{align*}
\st(A^{\mathrm{rms}})
&\le \frac{c_+}{c_-}\,\st(X),\\
\st(X^{\mathrm{att}})
&\le \frac{1+\delta}{1-\delta}\Bigl(1+\frac{2}{c_-}\Bigr)\st(X),\\
\st(A^{\mathrm{rms}}_{\mathrm{mlp}})
&\le
\frac{(1+\delta)^2}{(1-\delta)^2}\,
\frac{c_+ + 1}{c_-}\Bigl(1+\frac{2}{c_-}\Bigr)\st(X).
\end{align*}
\end{enumerate}
\end{lem}

\begin{proof}
Throughou the proof, set $\eta:=\delta/4$. We go  in order through the operations making up the multi-head attention layer.

\paragraph{RMS normalization of the input.}
By the column norm assumptions on $X$, we have $c_-d\le \|x_t\|_2^2\le c_+d$ for all $t$.
Applying Theorem~\ref{thm:rms-stablerank} with $L=c_-d$ and $U=c_+d$ yields
\[
\st(A^{\mathrm{rms}})\le \frac{c_+}{c_-}\st(X),
\qquad
\|A^{\mathrm{rms}}\|_F^2=\sum_{t=1}^n \|a_t\|_2^2=nd,
\]
since each RMS-normalized column satisfies $\|a_t\|_2^2=d$.

\paragraph{Bounds for $Y^{(h)}$ and $H^{(h)}$; concatenation.}
Fix a head $h\in[n_{\mathrm{head}}]$ and write $Y^{(h)}=[y^{(h)}_1,\dots,y^{(h)}_n]$.
Since each column of $P^{(h)}$ is a probability vector, Jensen's inequality yields
\[
\|y^{(h)}_t\|_2^2
=\Big\|\sum_{i=1}^n P^{(h)}_{it}a_i\Big\|_2^2
\le \sum_{i=1}^n P^{(h)}_{it}\|a_i\|_2^2
= d,
\qquad \forall t=1,\dots,n,
\]
and hence $\|Y^{(h)}\|_F^2\le nd$.

Condition on $Y^{(h)}$ and let $W_V^{(h)}\in\R^{d_{\mathrm{head}}\times d}$ have iid entries
$\cN(0,1/d)$.  By Theorem~\ref{lem:frob_gauss_scaling} (with $k=d_{\mathrm{head}}$ and $\varepsilon=\eta$),
with probability at least $1-4e^{-c\eta^2 d_{\mathrm{head}}}$, we have
\[
\|H^{(h)}\|_F^2=\|W_V^{(h)}Y^{(h)}\|_F^2 \le (1+\eta)\frac{d_{\mathrm{head}}}{d}\,\|Y^{(h)}\|_F^2
\le (1+\eta)\,n\,d_{\mathrm{head}}.
\]
Applying the same theorem to each column $y^{(h)}_t\in\R^{d\times 1}$ and taking a union bound over
$(h,t)\in[n_{\mathrm{head}}]\times[n]$, whenever $d_{\mathrm{head}}\ge C\eta^{-2}\log(nn_{\mathrm{head}})$ we also have
with probability at least $1-Ce^{-c\eta^2 d_{\mathrm{head}}}$, the estimates
\[
\|h^{(h)}_t\|_2^2=\|W_V^{(h)}y^{(h)}_t\|_2^2
\le (1+\eta)\frac{d_{\mathrm{head}}}{d}\,\|y^{(h)}_t\|_2^2
\le (1+\eta)\,d_{\mathrm{head}}
\qquad \forall (h,t).
\]
On this event, we have
\[
\|H\|_F^2=\sum_{h=1}^{n_{\mathrm{head}}}\|H^{(h)}\|_F^2 \le (1+\eta)\,n\,d,
\qquad
\|h_t\|_2^2=\sum_{h=1}^{n_{\mathrm{head}}}\|h^{(h)}_t\|_2^2 \le (1+\eta)\,d
\quad \forall t.
\]
Denote this event by $E_V$.

\paragraph{Column norms of $X^{\mathrm{att}}=X+W_OH$.}
Condition on $H$ and let $W_O\in\R^{d\times d}$ have i.i.d.\ entries $\cN(0,1/d)$, independent of the
$\{W_V^{(h)}\}$.  Apply Lemma~\ref{lem:gaussian-residual-preserves-column-norms} with parameter $\eta$
(here $k=d$): on an event $E_O$ with probability at least $1-Ce^{-c\eta^2 d}$, simultaneously for all $t$, we have
\[
(1-\eta)\big(\|x_t\|_2^2+\|h_t\|_2^2\big)
\le \|x_t^{\mathrm{att}}\|_2^2
\le (1+\eta)\big(\|x_t\|_2^2+\|h_t\|_2^2\big).
\]
On $E_V\cap E_O$, using $c_-d\le \|x_t\|_2^2$ and $\|h_t\|_2^2\ge 0$ gives
\[
\|x_t^{\mathrm{att}}\|_2^2 \ge (1-\eta)c_-d \ge (1-\delta)c_-d.
\]
Using $\|x_t\|_2^2\le c_+d$ and $\|h_t\|_2^2\le (1+\eta)d$ (from $E_V$) yields
\[
\|x_t^{\mathrm{att}}\|_2^2
\le (1+\eta)\bigl(c_+d+(1+\eta)d\bigr)
\le(1+\eta)^2(c_+ + 1)d
\le (1+\delta)(c_+ + 1)d,
\]
since $(1+\eta)^2\le 1+\delta$ for $\eta=\delta/4$ and $\delta\in(0,1)$.
This proves (i).

\paragraph{Stable rank of $X^{\mathrm{att}}$.}
Apply Theorem~\ref{lem:stablerank_X_plus_WH} (with parameter $\delta$) to $X^{\mathrm{att}}=X+W_OH$:
\[
\st(X^{\mathrm{att}})
\le \frac{1+\delta}{1-\delta}\cdot\frac{\|X\|_F^2+\|H\|_F^2}{\|X\|_{\op}^2}
=\frac{1+\delta}{1-\delta}\Bigl(1+\frac{\|H\|_F^2}{\|X\|_F^2}\Bigr)\st(X).
\]
On $E_V$ we have $\|H\|_F^2\le (1+\eta)nd$, while $\|X\|_F^2=\sum_t\|x_t\|_2^2\ge nc_-d$, hence
$\|H\|_F^2/\|X\|_F^2\le (1+\eta)/c_-\le 2/c_-$. Therefore, we deduce
\[
\st(X^{\mathrm{att}})
\le \frac{1+\delta}{1-\delta}\Bigl(1+\frac{2}{c_-}\Bigr)\st(X).
\]

\paragraph{Stable rank of $A^{\mathrm{rms}}_{\mathrm{mlp}}=\RMSNorm(X^{\mathrm{att}})$.}
Apply Theorem~\ref{thm:rms-stablerank} to $X^{\mathrm{att}}$ using (i), i.e.\ with
$L=(1-\delta)c_-d$ and $U=(1+\delta)(c_+ + 1)d$:
\[
\st(A^{\mathrm{rms}}_{\mathrm{mlp}})
\le \frac{(1+\delta)(c_+ + 1)}{(1-\delta)c_-}\st(X^{\mathrm{att}}),
\]
and substitute the bound on $\st(X^{\mathrm{att}})$.

\paragraph{Probability.}
With the condition $d_{\mathrm{head}}\ge C\eta^{-2}\log(nn_{\mathrm{head}})$, the event $E_V$ fails with probability at most
$Ce^{-c\eta^2 d_{\mathrm{head}}}$ after absorbing the union bound over $(h,t)$.  The event $E_O$ fails with probability at most
$Ce^{-c\eta^2 d}\le Ce^{-c\eta^2 d_{\mathrm{head}}}$, and Theorem~\ref{lem:stablerank_X_plus_WH} fails with probability at most
$Ce^{-c\delta^2 d}\le Ce^{-c\delta^2 d_{\mathrm{head}}}$.  Since $\eta=\delta/4$, a final union bound (conditionally) yields the stated probability.
\end{proof}

We next move on to the MLP layer of the transformer. We isolate two cases depending on the weather the activation function is mean-spike-inducing (Section~\ref{sec:mean_spike_mlp})
or is not (Section~\ref{sec:not_MSI_mlp}).

\subsubsection{Mean-spike-inducing (MSI) MLP layer propagation}\label{sec:mean_spike_mlp}

This subsection packages the mean-spike-inducing (MSI) MLP sublayer as a composition of the atomic operations already analyzed:
RMS normalization (Theorem~\ref{thm:rms-stablerank}), a Gaussian linear map followed by a pointwise activation (captured here by Assumption~\ref{ass:stablesigmarank}), and a Gaussian residual update (Lemma~\ref{lem:gaussian-residual-preserves-column-norms} and Theorem~\ref{lem:stablerank_X_plus_WH}).  We impose the following assumption. For the activations we will use later, Assumption~\ref{ass:stablesigmarank} follows by combining
Theorem~\ref{thm:gaussian_weights_low_sr} and Theorem~\ref{thm:nonline_activation}.
In particular, for the ReLU activation, one may take $\st_\sigma=2\pi$, $p_\sigma=1/2$, and $P(\sigma, d, k,n)=1-c\exp(-Ck)$.

\begin{assumption}\label{ass:stablesigmarank}
{\rm The activation function $\sigma$ is $1$-Lipschitz with $\sigma(0) = 0$. For any matrix $X \in \RR^{d\times n}$ with equal column norms and any matrix $W \in \RR^{k \times d}$ with iid entries $\mathcal{N}(0,\tfrac{1}{d})$, we have 
$$
\st(\sigma(WX)) \leq \st_\sigma \qquad \text{ and } \qquad \|\sigma(WX)\|_F^2 \geq \frac{kp_\sigma}{d}\|X\|_F^2
$$
with probability at least $P(\sigma, d, k,n)$, where $\st_\sigma$ and $p_\sigma$ are some numerical constants.}
\end{assumption}

{\rm
Now, fix dimensions $d,k,n\ge 1$, an activation $\sigma:\R\to\R$, and weight matrices
$W_1\in\R^{k\times d}$ and $W_2\in\R^{d\times k}$.
For any input $X\in\R^{d\times n}$ define
\[
A^{\mathrm{rms}}:=\RMSNorm(X),
\qquad
B:=\sigma\!\bigl(W_1A^{\mathrm{rms}}\bigr),
\qquad
X^+:=X+W_2B.
\]
}
See the companion Figure~\ref{fig:ngmlp-sublayer}.
The following theorem analyzes this setup.

\input{transformer_block_ngmlp}
\begin{lem}[Mean-spike-inducing (MSI) MLP sublayer bounds]
\label{lem:mlp-standard}
Fix constants $c_-,c_+>0$. Let $X\in\R^{d\times n}$ be a matrix whose squared column norms lie
between $c_-d$ and $c_+d$.  Assume $\sigma$ satisfies Assumption~\ref{ass:stablesigmarank}.
Assume $W_1\in\R^{k\times d}$ and $W_2\in\R^{d\times k}$ are independent Gaussian matrices with i.i.d.\ entries of variance
$1/d$ and $1/k$, respectively.  Fix $\delta\in(0,\tfrac14)$.
Then there exist numerical constants $c,C>0$ such that, whenever $\min\{d,k\}\ge C\delta^{-2}\log n$, the following holds
with probability at least
\[
1-Ce^{-c\delta^2 d}-Ce^{-c\delta^2 k}-\bigl(1-P(\sigma,d,k,n)\bigr).
\]
\begin{enumerate}
\item[(i)] \textbf{Column norms of $X^+$.} For all $t=1,\dots,n$,
\[
(1-\delta)c_-\,d \;\le\; \|x_t^+\|_2^2 \;\le\; (1+\delta)^2(c_+ + 1)\,d.
\]
\item[(ii)] \textbf{Stable ranks.} We have
\[
\st(B)\le \st_\sigma,
\qquad
\st(X^+)\le \frac{1+\delta}{1-\delta}\left(1+\frac{c_+}{p_\sigma}\right)\st_\sigma. 
\]
\end{enumerate}
\end{lem}
\begin{proof}
Throughout the proof let $\eta:=\delta/2$. We work through the MLP layer in order.

\paragraph{Column norms of $X^+$.}
Write $A^{\mathrm{rms}}=[a_1,\dots,a_n]$.  By definition of $\RMSNorm$, $\|a_t\|_2^2=d$ for all $t$.
For each $t$, define $u_t:=W_1a_t$.  Conditional on $a_t$, we have $u_t\sim \cN(0,I_k)$, and since
$\sigma$ is $1$-Lipschitz with $\sigma(0)=0$, we have $|\sigma(u)|\le |u|$ entrywise and hence
\[
\|b_t\|_2^2 \le \|u_t\|_2^2.
\]
A $\chi^2$ tail bound and a union bound over $t=1,\dots,n$ imply that, whenever $k\ge C\eta^{-2}\log n$,
with probability at least $1-Ce^{-c\eta^2 k}$, we have
\begin{equation}\label{eq:bt-ub-mlp}
\|b_t\|_2^2 \le (1+\eta)k
\qquad\text{for all }t.
\end{equation}

Condition on $(X,B)$ and apply Lemma~\ref{lem:gaussian-residual-preserves-column-norms} with parameter
$\eta$ to $X^+=X+W_2B$ (here $W_2$ has i.i.d.\ $\cN(0,1/k)$ entries).  On an event of probability at least
$1-Ce^{-c\eta^2 d}$ (absorbing the union bound over $t$ using $d\ge C\eta^{-2}\log n$), simultaneously for
all $t$, we deduce
\[
(1-\eta)\Bigl(\|x_t\|_2^2+\frac{d}{k}\|b_t\|_2^2\Bigr)
\le \|x_t^+\|_2^2
\le (1+\eta)\Bigl(\|x_t\|_2^2+\frac{d}{k}\|b_t\|_2^2\Bigr).
\]
Using $\|b_t\|_2^2\ge 0$ and $\|x_t\|_2^2\ge c_-d$ yields
\[
\|x_t^+\|_2^2 \ge (1-\eta)c_-d \ge (1-\delta)c_-d.
\]
Using $\|x_t\|_2^2\le c_+d$ and \eqref{eq:bt-ub-mlp} yields
\[
\|x_t^+\|_2^2
\le (1+\eta)\Bigl(c_+d+(1+\eta)d\Bigr)
=(1+\eta)^2(c_+ + 1)d
\le (1+\delta)^2(c_+ + 1)d,
\]
which proves (i).

\paragraph{Stable rank of $B$.}
By Assumption~\ref{ass:stablesigmarank} applied to the equal-column-norm matrix $A^{\mathrm{rms}}$ and the
Gaussian matrix $W_1$, on an event $E_\sigma$ with $\PP(E_\sigma)\ge P(\sigma,d,k,n)$ we have
\[
\st(B)\le \st_\sigma,
\qquad
\|B\|_F^2 \ge \frac{k p_\sigma}{d}\,\|A^{\mathrm{rms}}\|_F^2.
\]
Since $\|A^{\mathrm{rms}}\|_F^2=nd$, this gives $\|B\|_F^2\ge p_\sigma k n$ on $E_\sigma$.

\paragraph{Stable rank of $X^+$.}
Condition on $B$ and apply Theorem~\ref{lem:stablerank_X_plus_WH} (with parameter $\delta$) to
$X^+=X+W_2B$ to get
\[
\st(X^+)
\le \frac{1+\delta}{1-\delta}\left(
\frac{k}{d}\frac{\|X\|_F^2}{\|B\|_{\op}^2}+\st(B)\right).
\]
Using $\|X\|_F^2\le nc_+d$ and, on $E_\sigma$, the bounds $\st(B)\le \st_\sigma$ and
$\|B\|_{\op}^2= \|B\|_F^2/\st(B)\ge p_\sigma k n/\st_\sigma$, we obtain
\[
\frac{k}{d}\frac{\|X\|_F^2}{\|B\|_{\op}^2}
\le
\frac{k}{d}\cdot \frac{nc_+d}{p_\sigma k n/\st_\sigma}
=
\frac{\st_\sigma\,c_+}{p_\sigma}.
\]
Therefore, on $E_\sigma$, we have
\[
\st(X^+)
\le
\frac{1+\delta}{1-\delta}\left(\frac{\st_\sigma\,c_+}{p_\sigma}+\st_\sigma\right)
=
\st_\sigma\,\frac{1+\delta}{1-\delta}\left(1+\frac{c_+}{p_\sigma}\right),
\]
which is (ii).

\paragraph{Probability.}
Take a union bound (conditionally) over: the column-norm bound \eqref{eq:bt-ub-mlp} (probability
$\ge 1-Ce^{-c\eta^2 k}$ after $k\ge C\eta^{-2}\log n$), the residual column-norm event for $X^+$
(probability $\ge 1-Ce^{-c\eta^2 d}$ after $d\ge C\eta^{-2}\log n$), the event from
Theorem~\ref{lem:stablerank_X_plus_WH} (probability $\ge 1-Ce^{-c\delta^2 d}$), and $E_\sigma$
(probability $\ge P(\sigma,d,k,n)$).  Since $\eta=\delta/2$, this yields the stated probability.
\end{proof}

\subsubsection{MLP layer propagation without mean induced spikes}\label{sec:not_MSI_mlp}

Unlike the mean-spike-inducing MLP sublayer (Lemma~\ref{lem:mlp-standard}), here we treat MLP sublayers in which the post-activation matrix does \emph{not} reset to uniformly low stable rank. Instead, stable rank propagates from the RMS-normalized input, and we additionally track a uniform upper bound on the column norms of the intermediate matrix in order to control the subsequent residual update and to re-enter RMSNorm in the next block.
We package the two quantitative requirements on the intermediate matrix $B$ into the following assumption.

\begin{assumption}[Propagation interface]
\label{ass:propagation}
{\rm
Fix $\delta\in(0,1)$ and dimensions $d,k,n\ge 1$.
Let $X\in\R^{d\times n}$ be deterministic and set $A^{\mathrm{rms}}:=\RMSNorm(X)$.
Let $B=[b_1,\dots,b_n]\in\R^{k\times n}$ be a (possibly random) matrix.
We assume there exist constants $C_{\mathrm{PROP}}\ge 1$ and $m_\ast\ge 0$ and a success probability
$P_{\mathrm{PROP}}(\delta)\in[0,1]$ such that, with probability at least $P_{\mathrm{PROP}}(\delta)$, it holds:
\[
\st(B)\le C_{\mathrm{PROP}}\;\st\!\bigl(A^{\mathrm{rms}}\bigr)
\qquad\text{and}\qquad
\|b_t\|_2^2\le m_\ast\,k\quad\text{for all }t\in[n].
\]}
\end{assumption}

There are two natural constructions of $B$: a gated activation (see Figure~\ref{fig:gmlp-sublayer})and a pointwise activation with $\hat\sigma_1(1)^2>0$ applied to RMS-normalized inputs. The next lemma verifies that both satisfy Assumption~\ref{ass:propagation} at Gaussian initialization with explicit constants.

\input{transformer_block_gmlp}

\begin{lem}[Gated and Hermite activations satisfy Assumption~\ref{ass:propagation}]
\label{lem:propagation-examples}
Fix $\delta\in(0,1)$ and let $A=[a_1,\dots,a_n]\in\R^{d\times n}$ satisfy $\|a_t\|_2^2=d$ for all $t\in[n]$.
\begin{enumerate}
\item[(a)] \textbf{Gated case.}
Assume the setting of Theorem~\ref{thm:gating}
(in particular, $k\asymp d$) and set $m_1=|\EE\sigma(\gamma)|$ and $m_1=\EE\sigma(\gamma)^2$ where $\gamma\sim \mathcal{N}(0,1)$ is a standard Gaussian.
Let $W_{\mathrm{in}},W_{\mathrm{gate}}\in\R^{k\times d}$ be independent with iid entries $\cN(0,1/d)$ and set
\[
B:=\sigma(W_{\mathrm{in}}A)\had (W_{\mathrm{gate}}A)\in\R^{k\times n}.
\]
Then Assumption~\ref{ass:propagation} holds (with $A^{\mathrm{rms}}$ replaced by $A$) with
\[
C_{\mathrm{PROP}}=\frac{(1+\delta)^2}{(1-\delta)^3}\cdot \frac{1}{m_1^2},
\qquad
m_\ast=(1+\delta) m_2,
\]
and with probability at least
\[
P_{\mathrm{PROP}}(\delta)
\ \ge\
1
-
n\Bigl(
2e^{-c\delta^2 k}
+2k e^{-\tfrac12 k^{1/3}}
+2e^{-c\delta^2(1-\delta)m_2\,k^{2/3}}
\Bigr)
-
C\exp\!\Bigl(-c\delta^2 \tfrac{k}{\log^2 k}\Bigr)
-\frac{1}{k^{q}},
\]
for any fixed $q>0$, with constants $c,C$ depending only on $q$.

\item[(b)] \textbf{Hermite case.}
Suppose $p:=\hat\sigma_1(1)^2 >0$.
Let $W_1\in\R^{k\times d}$ have i.i.d.\ entries $\cN(0,1/d)$ and set
\[
B:=\sigma(W_1A)\in\R^{k\times n}.
\]
Then Assumption~\ref{ass:propagation} holds (with $A^{\mathrm{rms}}$ replaced by $A$) with
\[
C_{\mathrm{PROP}}=\frac{1+\delta}{1-\delta}\cdot \frac{1}{p},
\qquad
m_\ast=1+\delta,
\]
and with probability at least
\[
P_{\mathrm{PROP}}(\delta)
\ \ge\
1
-
8\exp(-c\delta^2p^2k)
-
C n\exp(-c\delta^2k).
\]
\end{enumerate}
\end{lem}

\begin{proof}
\textbf{(a) Gated case.}
For the stable-rank bound, apply Theorem~\ref{thm:gating} with $Z=X=A$, $V=W_{\mathrm{in}}$, $W=W_{\mathrm{gate}}$. 
Since $\|a_t\|_2^2=d$ for all $t$, the parameters in Theorem~\ref{thm:gating} satisfy $L=U=d$ and the mean condition there holds
with $\kappa=m_1$, yielding
\[
\st(B)\le \frac{(1+\delta)^2}{(1-\delta)^3}\cdot \frac{1}{m_1^2}\,\st(A)
\]
with the stated probability contribution.

For the column-norm envelope, apply Lemma~\ref{lem:gated-colnorm-cond} to each column $a_t$ (with parameter $\delta/3$)
and take a union bound over $t\in[n]$. Since $(1+\delta/3)^2\le 1+\delta$ for $\delta\in(0,1)$, this gives
$\|b_t\|_2^2\le (1+\delta)m_2k$ for all $t$ with the stated probability after adjusting constants.

\textbf{(b) simple nonlinear activation case.}
Since $\|a_t\|_2^2=d$, the parameter $p$ in \eqref{eqn:defnp} for Theorem~\ref{thm:nonline_activation} equals $\hat\sigma_1(1)^2$.
Applying Theorem~\ref{thm:nonline_activation} to $W_1$ and $A$ gives
\[
\st(B)\le \frac{1+\delta}{1-\delta}\cdot \frac{1}{p}\,\st(A)
\]
with probability at least $1-8\exp(-c\delta^2p^2k)$.

For the column-norm envelope, fix $t\in[n]$. Since $\sigma$ is $1$-Lipschitz with $\sigma(0)=0$, we have $|\sigma(u)|\le |u|$, hence
$\|b_t\|_2^2\le \|W_1a_t\|_2^2$. But $W_1a_t\sim\cN(0,I_k)$, so a $\chi^2$ tail bound yields
$\PP(\|W_1a_t\|_2^2>(1+\delta)k)\le \exp(-c\delta^2k)$.
A union bound over $t\in[n]$ gives the stated probability.
\end{proof}

We now show how Assumption~\ref{ass:propagation} feeds into the residual update $X^+=X+W_2B$. The only additional hypothesis needed is that $B$ is independent of the Gaussian output matrix $W_2$.

\begin{lem}[Propagation-activation MLP sublayer bounds]
\label{lem:mlp-gated}
Fix $\delta\in(0,\tfrac14)$ and constants $c_-,c_+>0$.
Let $X\in\R^{d\times n}$ be a matrix whose squared column norms lie between $c_-d$ and $c_+d$.
Set $A^{\mathrm{rms}}:=\RMSNorm(X)$.
Let $B=[b_1,\dots,b_n]\in\R^{k\times n}$ be a random matrix satisfying Assumption~\ref{ass:propagation}
(with this $\delta$, constants $C_{\mathrm{PROP}},m_\ast$, and success probability $P_{\mathrm{PROP}}(\delta)$).
Assume $W_2\in\R^{d\times k}$ has i.i.d.\ entries $\cN(0,1/k)$ and is independent of $B$, and define
\[
X^+ := X + W_2 B.
\]
Then there exist numerical constants $c,C>0$ such that, whenever $d\ge C\delta^{-2}\log n$, the following holds
with probability at least
\[
1
-
Ce^{-c\delta^2 d}
-
\bigl(1-P_{\mathrm{PROP}}(\delta)\bigr).
\]
\begin{enumerate}
\item[(i)] \textbf{Column norms of $X^+$.} For all $t\in[n]$, it holds:
\[
(1-\delta)c_-\,d \;\le\; \|x_t^+\|_2^2 \;\le\; (1+\delta)\,(c_+ + m_\ast)\,d.
\]
\item[(ii)] \textbf{Stable ranks.} On the same event, it holds:
\[
\st(B)\le C_{\mathrm{PROP}}\,\st(A^{\mathrm{rms}})
\le \frac{c_+}{c_-}\,C_{\mathrm{PROP}}\,\st(X),
\]
and
\[
\st(X^+)
\le
\frac{1+\delta}{1-\delta}\left(1+\frac{m_\ast}{c_-}\right)\st(X).
\]
\end{enumerate}
\end{lem}

\begin{proof}
Let $E_{\mathrm{prop}}$ denote the event in Assumption~\ref{ass:propagation}, so $\PP(E_{\mathrm{prop}})\ge P_{\mathrm{PROP}}(\delta)$.

\paragraph{Column norms.}
Condition on $(X,B)$ and apply Lemma~\ref{lem:gaussian-residual-preserves-column-norms} (with parameter $\delta$) to
$X^+=X+W_2B$.
On an event $E_{\mathrm{col}}$ with probability at least $1-Cn\exp(-c\delta^2 d)$, simultaneously for all $t\in[n]$,
\[
(1-\delta)\Bigl(\|x_t\|_2^2+\frac{d}{k}\|b_t\|_2^2\Bigr)
\le \|x_t^+\|_2^2
\le (1+\delta)\Bigl(\|x_t\|_2^2+\frac{d}{k}\|b_t\|_2^2\Bigr).
\]
On $E_{\mathrm{prop}}\cap E_{\mathrm{col}}$, using $\|x_t\|_2^2\ge c_-d$ gives the lower bound
$\|x_t^+\|_2^2\ge (1-\delta)c_-d$.
Using $\|x_t\|_2^2\le c_+d$ and $\|b_t\|_2^2\le m_\ast k$ from $E_{\mathrm{prop}}$ yields
\[
\|x_t^+\|_2^2
\le (1+\delta)\Bigl(c_+d+m_\ast d\Bigr)
= (1+\delta)(c_+ + m_\ast)d,
\]
which proves (i).

\paragraph{Stable ranks.}
On $E_{\mathrm{prop}}$ we have $\st(B)\le C_{\mathrm{PROP}}\st(A^{\mathrm{rms}})$.
Since the squared column norms of $X$ lie between $c_-d$ and $c_+d$, Theorem~\ref{thm:rms-stablerank} gives
$\st(A^{\mathrm{rms}})\le \frac{c_+}{c_-}\st(X)$, proving the displayed bound for $\st(B)$.

For $X^+$, condition on $B$ and apply Theorem~\ref{lem:stablerank_X_plus_WH} (with parameter $\delta$) to $X^+=X+W_2B$.
Lower bounding the denominator in \eqref{eq:stablerank_X_plus_WH_goal} by $\tfrac{1}{d}\|X\|_{\op}^2$ yields
\[
\st(X^+)
\le
\frac{1+\delta}{1-\delta}\left(1+\frac{d}{k}\frac{\|B\|_F^2}{\|X\|_F^2}\right)\st(X).
\]
On $E_{\mathrm{prop}}$, we have $\|B\|_F^2=\sum_{t=1}^n\|b_t\|_2^2\le nm_\ast k$, while
$\|X\|_F^2=\sum_{t=1}^n\|x_t\|_2^2\ge nc_-d$, so
\[
\frac{d}{k}\frac{\|B\|_F^2}{\|X\|_F^2}\le \frac{m_\ast}{c_-},
\]
which gives the stated bound on $\st(X^+)$.

\paragraph{Probability.}
By independence of $W_2$ from $B$, the residual column-norm event and Theorem~\ref{lem:stablerank_X_plus_WH} hold with the same probabilities after conditioning on $(X,B)$.
A union bound gives failure probability at most $(1-P_{\mathrm{PROP}}(\delta))+Cn e^{-c\delta^2 d}+10e^{-c\delta^2 d}$.
Under $d\ge C\delta^{-2}\log n$, we have $Cn e^{-c\delta^2 d}\le Ce^{-c'\delta^2 d}$ after adjusting constants, which yields the claim.
\end{proof}

\subsubsection{Mixture of Experts layer propagation}

We include this final MoE calculation as an  illustration of the calculus rules developed above.
Unlike the MLP sublayers studied earlier, we will not return to this layer in a later section. Instead, we present this result simply to illustrate that the atomic lemmas already suffice to control propagation through a mixture-of-experts update with \emph{arbitrary} routing: the routing weights enter only through the simplex bound $\sum_{e} r_{e,t}^2\le 1$, while stable-rank control of each expert activation is supplied directly by Theorem~\ref{thm:nonline_activation}. See the companion Figure~\ref{fig:ngmoe-sublayer}.

\newcommand{\ngmoe}{\mathrm{MoE\mbox{-}MLP}}

\begin{definition}[MoE-MLP sublayer map]
\label{def:ngmoe}
{\rm
Fix dimensions $d,k,n,n_{\mathrm{exp}}\ge 1$, an activation $\sigma:\R\to\R$, expert weight matrices
$\{W_1^{(e)}\}_{e\in[n_{\mathrm{exp}}]}\subset \R^{k\times d}$ and $\{W_2^{(e)}\}_{e\in[n_{\mathrm{exp}}]}\subset \R^{d\times k}$,
and routing weights $R=(r_{e,t})\in[0,1]^{n_{\mathrm{exp}}\times n}$ with $\sum_{e\in[n_{\mathrm{exp}}]} r_{e,t}=1$ for each $t\in[n]$.
Let $D^{(e)}:=\diag(r_{e,1},\dots,r_{e,n})$.
For any input $X\in\R^{d\times n}$ define
\[
\begin{aligned}
A^{\mathrm{rms}}&:=\RMSNorm(X),\\
B^{(e)}&:=\sigma\!\bigl(W_1^{(e)}A^{\mathrm{rms}}\bigr)\qquad (e\in[n_{\mathrm{exp}}]),\\
X^+&:=\ngmoe(X):=X+\sum_{e\in[n_{\mathrm{exp}}]} W_2^{(e)}\,B^{(e)}\,D^{(e)}.
\end{aligned}
\]
We suppress the dependence of $\ngmoe(\cdot)$ on $(\{W_1^{(e)},W_2^{(e)}\}_{e\in[n_{\mathrm{exp}}]},\sigma,R)$ in the notation.
}
\end{definition}

\input{transformer_moe_layer}

\begin{lem}[MoE-MLP sublayer bounds]
\label{lem:mlp-moe}
Fix constants $c_-,c_+>0$ and an integer $n_{\mathrm{exp}}\ge 1$. Let $X\in\R^{d\times n}$ be a matrix whose squared column norms lie
between $c_-d$ and $c_+d$. Work in the setting of Definition~\ref{def:ngmoe}.
Assume $\sigma$ is $1$-Lipschitz with $\sigma(0)=0$, and assume $\hat \sigma(1)^2 > 0$.
Assume $\{W_1^{(e)}\}_{e\in[n_{\mathrm{exp}}]}$ and $\{W_2^{(e)}\}_{e\in[n_{\mathrm{exp}}]}$ are mutually independent Gaussian matrices with i.i.d.\ entries of variance
$1/d$ and $1/k$, respectively.  Fix $\delta\in(0,\tfrac14)$.
Then there exist numerical constants $c,C>0$ such that, whenever $\min\{d,k\}\ge C\delta^{-2}\log(n\,n_{\mathrm{exp}})$, the following holds
with probability at least
\[
1-Ce^{-c\delta^2 d}-Ce^{-c\delta^2 k}-8n_{\mathrm{exp}}e^{-c\delta^2\hat \sigma(1)^4k}.
\]
\begin{enumerate}
\item[(i)] \textbf{Column norms of $X^+$.} For all $t\in[n]$,
\[
(1-\delta)c_-\,d \;\le\; \|x_t^+\|_2^2 \;\le\; (1+\delta)^2(c_+ + 1)\,d.
\]
\item[(ii)] \textbf{Stable ranks.} For all $e\in[n_{\mathrm{exp}}]$, it holds:
\[
\st\!\bigl(B^{(e)}\bigr)
\le
\frac{1+\delta}{1-\delta}\,\frac{1}{\hat\sigma_1^2(1)}\,\st(A^{\mathrm{rms}})
\le
\frac{1+\delta}{1-\delta}\,\frac{c_+}{c_-}\,\frac{1}{\hat\sigma_1^2(1)}\,\st(X),
\]
and
\[
\st(X^+)\le \frac{1+\delta}{1-\delta}\left(1+\frac{1+\delta}{c_-}\right)\st(X).
\]
\end{enumerate}
\end{lem}

\begin{proof}
Let $\eta:=\delta/2$.

\paragraph{Column norms of $X^+$.}
Write $A^{\mathrm{rms}}=[a_1,\dots,a_n]$.  By definition of $\RMSNorm$, $\|a_t\|_2^2=d$ for all $t$.
For each expert $e\in[n_{\mathrm{exp}}]$ and token $t\in[n]$, define $u_t^{(e)}:=W_1^{(e)}a_t$ and $b_t^{(e)}:=\sigma(u_t^{(e)})$.
Conditional on $a_t$, we have $u_t^{(e)}\sim \cN(0,I_k)$, and since $\sigma$ is $1$-Lipschitz with $\sigma(0)=0$, we have
$|\sigma(u)|\le |u|$ entrywise and hence $\|b_t^{(e)}\|_2^2 \le \|u_t^{(e)}\|_2^2$.
A $\chi^2$ tail bound and a union bound over $(e,t)\in[n_{\mathrm{exp}}]\times[n]$ imply that, whenever
$k\ge C\eta^{-2}\log(n\,n_{\mathrm{exp}})$, with probability at least $1-Ce^{-c\eta^2 k}$,
\begin{equation}\label{eq:bt-ub-moe-mlp}
\|b_t^{(e)}\|_2^2 \le (1+\eta)k
\qquad\text{for all }e\in[n_{\mathrm{exp}}],\ t\in[n].
\end{equation}

Define the block matrices
\[
W_2^{\mathrm{moe}}
:=\frac{1}{\sqrt{n_{\mathrm{exp}}}}\bigl[\,W_2^{(1)}\ \cdots\ W_2^{(n_{\mathrm{exp}})}\,\bigr]
\in\R^{d\times (k n_{\mathrm{exp}})},
\qquad
B^{\mathrm{moe}}
:=\begin{bmatrix}
\sqrt{n_{\mathrm{exp}}}\,B^{(1)}D^{(1)}\\
\vdots\\
\sqrt{n_{\mathrm{exp}}}\,B^{(n_{\mathrm{exp}})}D^{(n_{\mathrm{exp}})}
\end{bmatrix}
\in\R^{(k n_{\mathrm{exp}})\times n},
\]
so that $X^+=X+W_2^{\mathrm{moe}}B^{\mathrm{moe}}$.
By construction, $W_2^{\mathrm{moe}}$ has i.i.d.\ $\cN(0,1/(k n_{\mathrm{exp}}))$ entries and is independent of $(X,B^{\mathrm{moe}})$.
Write $B^{\mathrm{moe}}=[b^{\mathrm{moe}}_1,\dots,b^{\mathrm{moe}}_n]$.
On the event \eqref{eq:bt-ub-moe-mlp}, for each $t\in[n]$,
\[
\|b_t^{\mathrm{moe}}\|_2^2
=
n_{\mathrm{exp}}\sum_{e\in[n_{\mathrm{exp}}]} r_{e,t}^2\,\|b_t^{(e)}\|_2^2
\le
n_{\mathrm{exp}}(1+\eta)k\sum_{e\in[n_{\mathrm{exp}}]} r_{e,t}^2
\le
n_{\mathrm{exp}}(1+\eta)k,
\]
since $\sum_{e} r_{e,t}=1$ implies $\sum_{e} r_{e,t}^2\le 1$.

Condition on $(X,B^{\mathrm{moe}})$ and apply Lemma~\ref{lem:gaussian-residual-preserves-column-norms} with parameter $\eta$
to $X^+=X+W_2^{\mathrm{moe}}B^{\mathrm{moe}}$.  On an event of probability at least $1-Ce^{-c\eta^2 d}$ (absorbing the union bound over $t$
using $d\ge C\eta^{-2}\log n$), simultaneously for all $t$,
\[
(1-\eta)\Bigl(\|x_t\|_2^2+\frac{d}{k n_{\mathrm{exp}}}\|b_t^{\mathrm{moe}}\|_2^2\Bigr)
\le \|x_t^+\|_2^2
\le (1+\eta)\Bigl(\|x_t\|_2^2+\frac{d}{k n_{\mathrm{exp}}}\|b_t^{\mathrm{moe}}\|_2^2\Bigr).
\]
Using $\|b_t^{\mathrm{moe}}\|_2^2\ge 0$ and $\|x_t\|_2^2\ge c_-d$ yields
\[
\|x_t^+\|_2^2 \ge (1-\eta)c_-d \ge (1-\delta)c_-d.
\]
On \eqref{eq:bt-ub-moe-mlp}, using $\|x_t\|_2^2\le c_+d$ and $\|b_t^{\mathrm{moe}}\|_2^2\le n_{\mathrm{exp}}(1+\eta)k$ yields
\[
\|x_t^+\|_2^2
\le (1+\eta)\Bigl(c_+d+(1+\eta)d\Bigr)
=(1+\eta)^2(c_+ + 1)d
\le (1+\delta)^2(c_+ + 1)d,
\]
which proves (i).

\paragraph{Stable ranks of $B^{(e)}$.}
Fix $e\in[n_{\mathrm{exp}}]$ and apply Theorem~\ref{thm:nonline_activation} to the Gaussian matrix $W_1^{(e)}$ and the input $A^{\mathrm{rms}}$,
with $\varepsilon=\delta$.  Since $\|a_t\|_2=\sqrt d$ for all $t$, the parameter $p$ in \eqref{eqn:defnp} reduces to $\hat\sigma_1^2(1)$,
and Theorem~\ref{thm:nonline_activation} yields
\[
\st\!\bigl(B^{(e)}\bigr)\le \frac{1+\delta}{1-\delta}\,\frac{1}{\hat\sigma_1^2(1)}\,\st(A^{\mathrm{rms}})
\]
with probability at least $1-8e^{-c\delta^2p^2k}$.
A union bound over $e\in[n_{\mathrm{exp}}]$ gives the first inequality in (ii) simultaneously for all experts.
The second inequality in (ii) follows from Theorem~\ref{thm:rms-stablerank}, which gives $\st(A^{\mathrm{rms}})\le \frac{c_+}{c_-}\st(X)$.

\paragraph{Stable rank of $X^+$.}
Apply Theorem~\ref{lem:stablerank_X_plus_WH} (with parameter $\delta$) to $X^+=X+W_2^{\mathrm{moe}}B^{\mathrm{moe}}$:
\[
\st(X^+)
\le
\frac{1+\delta}{1-\delta}\cdot
\frac{\tfrac{1}{d}\|X\|_F^2+\tfrac{1}{k n_{\mathrm{exp}}}\|B^{\mathrm{moe}}\|_F^2}{\max\{\tfrac{1}{d}\|X\|_{\op}^2,\tfrac{1}{k n_{\mathrm{exp}}}\|B^{\mathrm{moe}}\|_{\op}^2\}}.
\]
Lower bound the denominator by $\tfrac{1}{d}\|X\|_{\op}^2$ to obtain
\[
\st(X^+)
\le
\frac{1+\delta}{1-\delta}\left(1+\frac{d}{k n_{\mathrm{exp}}}\frac{\|B^{\mathrm{moe}}\|_F^2}{\|X\|_F^2}\right)\st(X),
\qquad
\text{since }\st(X)=\frac{\|X\|_F^2}{\|X\|_{\op}^2}.
\]
Using $\|X\|_F^2=\sum_t\|x_t\|_2^2\ge nc_-d$ and, on \eqref{eq:bt-ub-moe-mlp},
\[
\|B^{\mathrm{moe}}\|_F^2=\sum_{t=1}^n \|b_t^{\mathrm{moe}}\|_2^2 \le n\cdot n_{\mathrm{exp}}(1+\eta)k,
\]
we obtain
\[
\frac{d}{k n_{\mathrm{exp}}}\frac{\|B^{\mathrm{moe}}\|_F^2}{\|X\|_F^2}
\le
\frac{d}{k n_{\mathrm{exp}}}\cdot \frac{n\cdot n_{\mathrm{exp}}(1+\eta)k}{nc_-d}
=
\frac{1+\eta}{c_-}
\le
\frac{1+\delta}{c_-},
\]
which yields the stated bound on $\st(X^+)$.

\paragraph{Probability.}
Take a union bound over: the event \eqref{eq:bt-ub-moe-mlp} (probability $\ge 1-Ce^{-c\eta^2 k}$ after $k\ge C\eta^{-2}\log(n\,n_{\mathrm{exp}})$),
the residual column-norm event for $X^+$ (probability $\ge 1-Ce^{-c\eta^2 d}$ after $d\ge C\eta^{-2}\log n$),
the event from Theorem~\ref{lem:stablerank_X_plus_WH} (probability $\ge 1-Ce^{-c\delta^2 d}$),
and the $n_{\mathrm{exp}}$ activation events from Theorem~\ref{thm:nonline_activation}.
Since $\eta=\delta/2$, this yields the stated probability after adjusting constants.
\end{proof}

\section{MLPs with Gaussian initialization}\label{sec:deep-mlp}

We now combine the one-layer results to control the post-activation matrices of a depth-$L$ bias-free MLP
at Gaussian initialization (random weights and fixed data). There are two distinct behaviors depending on the activation.
If $\EE[\sigma(G)]\neq 0$ (mean-spike inducing), then every post-activation matrix has stable rank bounded
by a constant depending only on the one-dimensional Gaussian moments of $\sigma$; in particular, stable
rank is \emph{reset} at every layer, independently of the stable rank of the input data.
If instead the first Hermite coefficient of $\sigma$ is uniformly nondegenerate, then stable rank
propagates through depth and can inflate at most exponentially.

\paragraph{Architecture.}
Consider a feedforward neural network parametrized by weight matrices $W_1,\ldots, W_L$ and a pointwise
activation $\sigma\colon\R\to\R$ applied entrywise.
Let $X\in\R^{d_0\times n}$ be a data matrix, and define the post-activation matrices by
\begin{equation}\label{eqn:post_activ}
A_0=X
\qquad\textrm{and}\qquad
A_\ell=\sigma(W_\ell A_{\ell-1})
\qquad \forall \ell=1,\ldots, L.
\end{equation}
Unless stated otherwise, we assume that $W_\ell\in\R^{d_\ell\times d_{\ell-1}}$ has iid entries $\cN(0,1/d_{\ell-1})$,
and that the matrices $\{W_\ell\}_{\ell=1}^L$ are mutually independent.

\paragraph{Mean-spike inducing activations.}
Theorem~\ref{thm:gaussian_weights_low_sr} bounds $\st(\sigma(WX))$ by a one-dimensional ratio
$s_\sigma(a_X,b_X)$ when $\kappa_{\sigma,q}(a_X,b_X)>0$, i.e., as soon as the Gaussian mean
$\EE[\sigma(sG)]$ does not vanish on the scale interval induced by the column norms of $X$ and obeys a power-$q$ lower bound there.
For homogeneous activations this dependence on $(a_X,b_X)$ disappears, giving  input-independent bound.

\begin{cor}[All post-activations are low-stable rank with mean-induced spikes]\label{cor:mlp-mean-spike}
Assume that $\sigma$ is $p$-homogeneous for some $p\ge 1$, satisfies $\sigma(0)=0$, and the polynomial growth bound
\[
|\sigma(t)|\le L_\sigma |t|^p
\qquad\text{for all }t\in\R,
\]
for some constant $L_\sigma<\infty$.
Let $\gamma\sim\cN(0,1)$ be a standard Gaussian and define
\[
m_1:=\big|\EE[\sigma(\gamma)]\big|,
\qquad
m_2:=\EE[\sigma(\gamma)^2],
\]
and assume $m_1>0$. Set $r:=L_\sigma/m_1$.
Fix $0<\varepsilon_1\le 1$ and $0<\varepsilon_2<1$.
Then there exist a constant $c>0$ that depends only on $p$, $\sigma(1)$, $\sigma(-1)$, $m_1$ such that, with probability at least
\[
1-\sum_{\ell=1}^L\Bigg(
2\exp\!\Big(-c\,\min\{d_\ell\,\varepsilon_2^2,\ d_\ell^{1\wedge\tfrac{2}{p}}\varepsilon_2^{2/p}\}\Big)
+
2\exp\!\Big(-c\,\min\{d_\ell\,\varepsilon_1^2,\ d_\ell^{1\wedge\tfrac{1}{p}}\varepsilon_1^{1/p}\}\Big)
\Bigg),
\]
the following holds simultaneously for every $\ell=1,\dots,L$:
\[
\st(A_\ell)\ \le\ \frac{1+\varepsilon_1}{(1-\varepsilon_2)^2}\cdot \frac{m_2}{m_1^2}.
\]
\end{cor}

\begin{proof}
This follows immediately by applying Corollary~\ref{cor:q_homogeneous_low_sr} conditionally layer by layer.
\end{proof}

\paragraph{Propagation through nonlinearities.}
When $\EE[\sigma(G)]=0$, Corollary~\ref{cor:mlp-mean-spike} does not apply.
In this case we control stable rank using the nondegeneracy of the first Hermite coefficient
(Theorem~\ref{thm:nonline_activation}).
For $s>0$, write $\hat\sigma_1(s)$ for the first Hermite coefficient of $\sigma(s\gamma)$
(with $\gamma\sim\cN(0,1)$), and define the uniform Hermite nondegeneracy parameter
\begin{equation}\label{eq:psigma-global}
p_\sigma\;:=\;\inf_{s>0}\ \frac{\hat\sigma_1(s)^2}{s^2}.
\end{equation}
For activations such as ReLU, leaky ReLU, exact GELU, and SiLU, the ratio $\hat\sigma_1(s)/s$ is constant,
so $p_\sigma$ is a positive absolute constant; see Table~\ref{table_common_act_func}.

\begin{cor}[Propagation of stable rank]\label{cor:mlp-propagation}
Assume that $\sigma$ is $1$-Lipschitz and satisfies $\sigma(0)=0$, and that $p_\sigma>0$ as in
\eqref{eq:psigma-global}.
Fix $\varepsilon\in(0,1)$.
Then there exists a numerical constant $c>0$ such that, with probability at least
\[
1-\sum_{\ell=1}^L 8\exp\!\Big(-c\,\varepsilon^2\,p_\sigma^2\,d_\ell\Big),
\]
the following holds simultaneously for every $\ell=1,\dots,L$:
\[
\st(A_\ell)\ \le\ \left(\frac{1+\varepsilon}{1-\varepsilon}\cdot\frac{1}{p_\sigma}\right)^\ell \st(X).
\]
\end{cor}

\begin{proof}
This follows immediately by applying Theorem~\ref{thm:nonline_activation} conditionally layer by layer.
\end{proof}

\section{Transformers at Gaussian initialization}\label{sec:transformers}

In this section we bound the stable rank of the matrices that appear inside an $L$-layer decoder-only
transformer at Gaussian initialization.  The input is the embedding matrix $X^{(0)}\in\R^{d\times n}$ for a
fixed token sequence.  By Section~\ref{sec:embedding}, the embedding already has low stable rank, governed
by the inverse frequency of the most common token.  Our goal is to show that this low--stable--rank
property persists throughout the network: at every depth we obtain explicit high-probability bounds for the
RMS-normalized activations (and, in particular, for $A^{\mathrm{rms}}_{\mathrm{final}}=\RMSNorm(X^{(L)})$).

The proofs follow by an induction over depth, based on the blockwise bounds proved in earlier sections.  Two quantities are
propagated from layer to layer.  First, we maintain control on the ratio of column norms of $X^{(\ell)}$,
because the RMSNorm bound (Theorem~\ref{thm:rms-stablerank}) multiplies stable rank by the ratio between
maximal and minimal column norms.  Second, we propagate stable rank by composing the attention sublayer
bound (Lemma~\ref{lem:block-att-common}) with an MLP sublayer bound: Lemma~\ref{lem:mlp-standard} in the
mean-spike-inducing (MSI) case and Lemma~\ref{lem:mlp-gated} for activations that propagate stable rank as in Assumption~\ref{ass:propagation}.
The MSI regime yields polynomial (quadratic) growth in depth, while the nonMSI propagation regime yields an
exponentially increasing bound.

We work with the layerwise updates recorded in Assumption~\ref{ass:L-layer-transformer}.  In particular, at each
layer and for each head we represent attention by a mixing matrix $P^{(h)}\in\R^{n\times n}$ whose columns are
probability vectors.  In a standard transformer $P^{(h)}$ is obtained from queries and keys via a softmax
(and may incorporate masking, RoPE, etc.); see the discussion surrounding the attention definition.
However, all bounds in this section use only column-stochasticity of $P^{(h)}$ together with Gaussian
initialization of the \emph{value} and \emph{output} projections.  Consequently, we impose no assumptions on
the initialization of the query/key maps (or on the particular way $P^{(h)}$ is generated), beyond the
measurability/independence conditions recorded in Assumption~\ref{ass:L-layer-transformer}.

\begin{assumption}[$L$-layer transformer at Gaussian initialization]
\label{ass:L-layer-transformer}
{\rm
Fix integers $L\ge 1$, $n_{\mathrm{head}}\ge 1$, and $d,k\ge 1$ where $d$ is divisible by $n_{\mathrm{head}}$, and set
$d_{\mathrm{head}}:=d/n_{\mathrm{head}}$.
Let $X^{(0)}\in\R^{d\times n}$ be the (pre-normalization) embedding matrix for a fixed token sequence,
constructed from a Gaussian embedding as in Lemma~\ref{lem:embedding_rms_stablerank}, with empirical token
frequencies $p_j=c_j/n$ and $p_{\max}:=\max_j p_j$.

For each depth $\ell\in\{0,\dots,L-1\}$ and head $h\in[n_{\mathrm{head}}]$, let
$P^{(h,\ell)}\in\R^{n\times n}$ be a (possibly random) matrix whose columns are probability vectors.
Let $W_O^{(\ell)}\in\R^{d\times d}$ and $W_V^{(h,\ell)}\in\R^{d_{\mathrm{head}}\times d}$ be weight matrices.
Given any block input $X^{(\ell)}\in\R^{d\times n}$ define the attention-sublayer intermediates
\[
\begin{aligned}
A^{\mathrm{rms},(\ell)} &:= \RMSNorm\!\bigl(X^{(\ell)}\bigr),\\
Y^{(h,\ell)} &:= A^{\mathrm{rms},(\ell)}P^{(h,\ell)}, \qquad (h\in[n_{\mathrm{head}}]),\\
H^{(h,\ell)} &:= W_V^{(h,\ell)}Y^{(h,\ell)}, \qquad (h\in[n_{\mathrm{head}}]),\\
H^{(\ell)} &:= \concat_{h\in[n_{\mathrm{head}}]} H^{(h,\ell)}\in\R^{d\times n},\\
X^{\mathrm{att},(\ell)} &:= X^{(\ell)} + W_O^{(\ell)}H^{(\ell)},\\
A^{\mathrm{rms},(\ell)}_{\mathrm{mlp}} &:= \RMSNorm\!\bigl(X^{\mathrm{att},(\ell)}\bigr).
\end{aligned}
\]
Next define the MLP sublayer output
\[
X^{(\ell+1)} := X^{\mathrm{att},(\ell)}+W_2^{(\ell)}B^{(\ell)},
\]
where $B^{(\ell)}\in\R^{k\times n}$ is a (possibly random) matrix that is independent of $W_2^{(\ell)}$.
Finally, we write $x_t^{(\ell)}$ and $x_t^{\mathrm{att},(\ell)}$ for the $t$th columns of $X^{(\ell)}$ and
$X^{\mathrm{att},(\ell)}$, respectively.

All trainable weight matrices are mutually independent across all layers and are initialized with independent
Gaussian entries with variance $1/m$, where $m$ is the input dimension of the map: in particular
$(W_O^{(\ell)})_{ij},(W_V^{(h,\ell)})_{ij}\sim\cN(0,1/d)$
and $(W_2^{(\ell)})_{ij}\sim \cN(0,1/k)$.

The mixing matrices $\{P^{(h,\ell)}\}_{h,\ell}$ may be arbitrary and may depend on all randomness external
to the collection of Gaussian weights listed above (including, e.g., query/key projections, masking, RoPE,
etc.).  The only requirement is that for each fixed $(h,\ell)$ the matrix $P^{(h,\ell)}$ is independent of
$W_O^{(\ell)}$, $\{W_V^{(h',\ell)}\}_{h'\in[n_{\mathrm{head}}]}$, and $W_2^{(\ell)}$.
}
\end{assumption}

We now state the stable-rank bounds for the MSI and nonMSI architectures.
Before the induction, we record one additional fact about the embedding matrix, complementary to
Lemma~\ref{lem:embedding_rms_stablerank}.  It supplies the uniform column-norm control needed to invoke
Theorem~\ref{thm:rms-stablerank} at depth $0$, and to start the column-norm recursion used in both proofs. The proof is standard and we omit it.

\begin{lem}[Gaussian embedding column norms]\label{lem:embedding-colnorms}
Fix $\varepsilon\in(0,1)$.  Let $E\in\R^{d\times V}$ have independent columns
$e_1,\dots,e_V\sim\cN(0,I_d)$.  For a fixed token sequence $(i_1,\dots,i_n)\in\{1,\dots,V\}^n$ define
\[
X^{(0)}=[x_1,\dots,x_n]\in\R^{d\times n},
\qquad
x_t:=e_{i_t}.
\]
Then there exist numerical constants $c,C>0$ such that whenever $d\ge C\varepsilon^{-2}\log n$, the
estimate
\[
(1-\varepsilon)d \;\le\; \|x_t\|_2^2 \;\le\; (1+\varepsilon)d
\qquad\text{for all }t=1,\dots,n
\]
holds with probability at least $1-Ce^{-c\varepsilon^2 d}$.  
\end{lem}

\subsection{MSI transformer}

We first consider the mean-spike-inducing (MSI) regime, in which the MLP post-activations are given by a pointwise
activation applied to a Gaussian projection of $A^{\mathrm{rms},(\ell)}_{\mathrm{mlp}}$.
The theorem below gives stable-rank bounds for the MLP post-activations and for the RMS-normalized activations at all
depths.  In particular, the MLP post-activations satisfy $\st(B^{(\ell)})\le \st_\sigma$, and the stable ranks of the RMS activations grow at most quadratically in $\ell$, with explicit
dependence on $\st_\sigma/p_\sigma$.

\begin{thm}[Low--stable--rank representations in an $L$-layer transformer]
\label{cor:deep-transformer-stablerank-linear}
Suppose that Assumption~\ref{ass:L-layer-transformer} holds.  Assume moreover that for each $\ell=0,\dots,L-1$,
\[
B^{(\ell)}:=\sigma\!\bigl(W_1^{(\ell)}A^{\mathrm{rms},(\ell)}_{\mathrm{mlp}}\bigr),
\]
where $\sigma$ satisfies Assumption~\ref{ass:stablesigmarank} and the matrices
$\{W_1^{(\ell)}\in\R^{k\times d}\}_{\ell=0}^{L-1}$ are mutually independent with i.i.d.\ entries $\cN(0,1/d)$
and are independent of the randomness in Assumption~\ref{ass:L-layer-transformer}.
Then there exist numerical constants
$c,C>0$ such that whenever
\[
d\ge C L^2\log(nL),
\qquad
k\ge C L^2\log(nL),
\]
with probability at least
\[
1- C L\exp\!\Big(-c\,\frac{d}{L^2}\Big)- C L\exp\!\Big(-c\,\frac{k}{L^2}\Big)
- L\bigl(1-P(\sigma,d,k,n)\bigr),
\]
the following holds:
\begin{enumerate}
\item[(i)] \textbf{MLP post-activations.} For every $\ell=0,\dots,L-1$,
\[
\st\!\bigl(B^{(\ell)}\bigr)\ \le\ \st_\sigma.
\]

\item[(ii)] \textbf{Embedding-layer stable ranks.} We have
\[
\st\!\bigl(X^{(0)}\bigr)\ \le\ \frac{1.01}{p_{\max}},
\qquad
\st\!\bigl(A^{\mathrm{rms},(0)}\bigr)\ \le\ \frac{1.03}{p_{\max}},
\qquad
\st\!\bigl(A^{\mathrm{rms},(0)}_{\mathrm{mlp}}\bigr)\ \le\ \frac{6.03}{p_{\max}}.
\]

\item[(iii)] \textbf{RMS activations at depth $\ell\ge 1$.} For every $\ell=1,\dots,L-1$,
\[
\begin{aligned}
\st\!\bigl(A^{\mathrm{rms},(\ell)}\bigr)
&\le 2.1\,\st_\sigma(1+\ell)\;+\;4.3\,\frac{\st_\sigma}{p_\sigma}\,\ell(\ell+1),\\
\st\!\bigl(A^{\mathrm{rms},(\ell)}_{\mathrm{mlp}}\bigr)
&\le 6.4\,\st_\sigma(1+\ell)\;+\;13\,\frac{\st_\sigma}{p_\sigma}\,\ell(\ell+1),\\
\st\!\bigl(A^{\mathrm{rms}}_{\mathrm{final}}\bigr)
&\le 2.1\,\st_\sigma(1+L)\;+\;4.3\,\frac{\st_\sigma}{p_\sigma}\,L(L+1).
\end{aligned}
\]
\end{enumerate}
\end{thm}
\begin{proof}
Set $\delta:=\frac{1}{512L}$ and $\varepsilon:=\delta$, and write $\alpha:=\frac{1+\delta}{1-\delta}$.
Let $\Omega$ be the event on which Lemmas~\ref{lem:embedding_rms_stablerank} and \ref{lem:embedding-colnorms}
hold with parameter $\varepsilon$, and for every $\ell=0,\dots,L-1$, Lemmas~\ref{lem:block-att-common} and
\ref{lem:mlp-standard} hold with parameter $\delta$ at layer $\ell$.
A union bound yields the stated lower bound on $\PP(\Omega)$ after adjusting constants and using
$d,k\ge C\delta^{-2}\log(nL)$.  Work throughout on $\Omega$.

\paragraph{Embedding layer.}
Lemma~\ref{lem:embedding_rms_stablerank} yields
\[
\st\bigl(X^{(0)}\bigr)\le \alpha\,\frac{1}{p_{\max}}.
\]
Lemma~\ref{lem:embedding-colnorms} yields $(1-\varepsilon)d\le \|x_t^{(0)}\|_2^2\le (1+\varepsilon)d$ for all $t$,
hence Theorem~\ref{thm:rms-stablerank} gives
\[
\st\bigl(A^{\mathrm{rms},(0)}\bigr)\le \alpha\,\st\bigl(X^{(0)}\bigr)\le \alpha^2\,\frac{1}{p_{\max}}.
\]
Since $\delta\le \frac{1}{512}$, we have $\alpha\le \frac{513}{511}<1.01$, hence the first two bounds in (ii).

Applying Lemma~\ref{lem:block-att-common}(ii) to $X^{(0)}$ with $c_-=1-\varepsilon$ and $c_+=1+\varepsilon$ gives
\[
\st\bigl(A^{\mathrm{rms},(0)}_{\mathrm{mlp}}\bigr)
\le
\alpha^2\,
\frac{2+\varepsilon}{1-\varepsilon}\Bigl(1+\frac{2}{1-\varepsilon}\Bigr)\st\bigl(X^{(0)}\bigr).
\]
Using $\varepsilon=\delta\le \frac{1}{512}$ and $\st(X^{(0)})\le \alpha/p_{\max}$, a direct bound gives
\[
\frac{(1+\delta)^2}{(1-\delta)^2}\cdot \frac{2+\delta}{1-\delta}\cdot\Bigl(1+\frac{2}{1-\delta}\Bigr)\cdot \alpha
\ \le\ \frac{31}{5},
\]
which proves the last bound in (ii).

\paragraph{Column-norm control.}
For each $\ell\ge 0$, define $c_-^{(\ell)},c_+^{(\ell)}$ by
\[
c_-^{(\ell)}d\le \|x_t^{(\ell)}\|_2^2\le c_+^{(\ell)}d
\qquad \forall t.
\]
On $\Omega$, Lemma~\ref{lem:embedding-colnorms} gives $c_-^{(0)}=1-\delta$ and $c_+^{(0)}=1+\delta$.
Lemma~\ref{lem:block-att-common}(i) gives
\[
(1-\delta)c_-^{(\ell)}d \le \|x_t^{\mathrm{att},(\ell)}\|_2^2 \le (1+\delta)(c_+^{(\ell)}+1)d
\qquad\forall t,
\]
and Lemma~\ref{lem:mlp-standard}(i), applied to $X^{\mathrm{att},(\ell)}$, gives
\begin{equation}\label{eq:col-recur-cor-tight}
c_-^{(\ell+1)}=(1-\delta)^2c_-^{(\ell)},
\qquad
c_+^{(\ell+1)}=(1+\delta)^2\Bigl((1+\delta)(c_+^{(\ell)}+1)+1\Bigr).
\end{equation}
Hence $c_-^{(\ell)}=(1-\delta)^{2\ell+1}\ge 1-(2\ell+1)\delta\ge \frac{509}{512}$ for all $\ell\le L$.
For the upper envelope, write \eqref{eq:col-recur-cor-tight} as $c_+^{(\ell+1)}=a\,c_+^{(\ell)}+b$ with
$a=(1+\delta)^3$ and $b=(1+\delta)^2(2+\delta)$.  Then
\[
c_+^{(\ell)}
=a^\ell c_+^{(0)}+b\sum_{j=0}^{\ell-1}a^j
\le a^\ell\bigl(c_+^{(0)}+b\ell\bigr)
\le a^L\bigl((1+\delta)+b\ell\bigr).
\]
Since $\delta L\le 1/512$, we have $a^L\le e^{3\delta L}\le e^{3/512}<1.01$ and $b\le (1+\delta)^2(2+\delta)\le 2.01$.
Thus, for all $\ell\le L$,
\begin{equation}\label{eq:cplus-linear-tight}
c_+^{(\ell)}\le 1.03+2.03\,\ell,
\qquad\text{and hence}\qquad
c_+^{(\ell)}+1\le 2.03(1+\ell).
\end{equation}

\paragraph{Stable ranks of $B^{(\ell)}$ and $X^{(\ell)}$.}
Item (ii) of Lemma~\ref{lem:mlp-standard}, applied at layer $\ell$, yields $\st(B^{(\ell)})\le \st_\sigma$, proving (i).
Let $s_\ell:=\st(X^{(\ell)})$.  Lemma~\ref{lem:mlp-standard}(ii), applied to $X^{\mathrm{att},(\ell)}$ whose
upper column-norm parameter is $(1+\delta)(c_+^{(\ell)}+1)$, yields
\[
s_{\ell+1}\le \alpha\left(1+\frac{(1+\delta)(c_+^{(\ell)}+1)}{p_\sigma}\right)\st_\sigma.
\]
Using $\alpha<1.01$ and \eqref{eq:cplus-linear-tight} gives, for $\ell\le L$,
\begin{equation}\label{eq:s-linear-tight}
s_\ell\le 1.01\left(1+\frac{2.04\,\ell}{p_\sigma}\right)\st_\sigma.
\end{equation}

\paragraph{RMS activations.}
Fix $\ell\in\{1,\dots,L-1\}$.
Theorem~\ref{thm:rms-stablerank}, \eqref{eq:cplus-linear-tight}, and $c_-^{(\ell)}\ge 509/512$ yield
\[
\st\bigl(A^{\mathrm{rms},(\ell)}\bigr)
\le \frac{c_+^{(\ell)}}{c_-^{(\ell)}}\,s_\ell
\le \frac{1.03+2.03\,\ell}{509/512}\,s_\ell
\le 2.05(1+\ell)\,s_\ell.
\]
Combining with \eqref{eq:s-linear-tight} gives
\[
\st\bigl(A^{\mathrm{rms},(\ell)}\bigr)
\le 2.05\cdot 1.01\,(1+\ell)\st_\sigma
+2.05\cdot 1.01\cdot 2.04\,\frac{\ell(\ell+1)}{p_\sigma}\st_\sigma
\le 2.1\,\st_\sigma(1+\ell)+4.3\,\frac{\st_\sigma}{p_\sigma}\,\ell(\ell+1).
\]

For $A^{\mathrm{rms},(\ell)}_{\mathrm{mlp}}$, Lemma~\ref{lem:block-att-common}(ii), \eqref{eq:cplus-linear-tight}, and $c_-^{(\ell)}\ge 509/512$ yield
\[
\st\bigl(A^{\mathrm{rms},(\ell)}_{\mathrm{mlp}}\bigr)
\le
\frac{(1+\delta)^2}{(1-\delta)^2}\,
\frac{c_+^{(\ell)}+1}{c_-^{(\ell)}}\Bigl(1+\frac{2}{c_-^{(\ell)}}\Bigr)s_\ell
\le 6.3(1+\ell)\,s_\ell,
\]
using $\delta\le \frac{1}{512}$.
Combining with \eqref{eq:s-linear-tight} gives
\[
\st\bigl(A^{\mathrm{rms},(\ell)}_{\mathrm{mlp}}\bigr)
\le 6.3\cdot 1.01\,(1+\ell)\st_\sigma
+6.3\cdot 1.01\cdot 2.04\,\frac{\ell(\ell+1)}{p_\sigma}\st_\sigma
\le 6.4\,\st_\sigma(1+\ell)+13\,\frac{\st_\sigma}{p_\sigma}\,\ell(\ell+1).
\]

Finally, applying the same RMSNorm estimate as for $A^{\mathrm{rms},(\ell)}$ to $X^{(L)}$ and using
\eqref{eq:s-linear-tight} with $\ell=L$ yields
\[
\st\bigl(A^{\mathrm{rms}}_{\mathrm{final}}\bigr)
\le 2.1\,\st_\sigma(1+L)+4.3\,\frac{\st_\sigma}{p_\sigma}\,L(L+1),
\]
thereby completing the proof.
\end{proof}

\subsection{Transformer without mean induced spikes}

We now consider the nonMSI propagation regime.  In contrast
to the MSI case, the MLP sublayer does not reset stable rank, and the hidden-state recursion becomes
multiplicative.  The theorem below records the resulting exponential-in-depth bounds for the stable ranks
of the RMS-normalized activations and the MLP post-activations. Recall that Lemma~\ref{lem:propagation-examples} verified Assumption~\ref{ass:propagation} at Gaussian initialization for two natural
classes of such activations: gated MLPs and pointwise activations with $\hat\sigma_1(1)^2>0$.

\begin{thm}[Low--stable--rank representations in an $L$-layer transformer]
\label{thm:deep-transformer-stablerank-gated}
Suppose that Assumption~\ref{ass:L-layer-transformer} holds.  Assume that for each $\ell=0,\dots,L-1$ the matrix $B^{(\ell)}$
satisfies Assumption~\ref{ass:propagation} with $\delta=\frac{1}{512L}$, constants $(C_{\mathrm{PROP}},m_\ast)$, and success
probability $P_{\mathrm{PROP}}(\tfrac{1}{512L})$.  Set
\[
\rho:=3.1(1+m_\ast).
\]
Then there exist numerical constants $c,C>0$ such that whenever
\[
d\ge C L^2\log(nL),
\qquad
k\ge C L^2\log(nL),
\]
with probability at least
\[
1
- C L\exp\!\Big(-c\,\frac{d}{L^2}\Big)
- L\Bigl(1-P_{\mathrm{PROP}}\!\bigl(\tfrac{1}{512L}\bigr)\Bigr),
\]
the following holds:
\begin{enumerate}
\item[(i)] \textbf{Embedding-layer stable ranks.} We have
\[
\st\!\bigl(X^{(0)}\bigr)\ \le\ \frac{1.01}{p_{\max}},
\qquad
\st\!\bigl(A^{\mathrm{rms},(0)}\bigr)\ \le\ \frac{1.03}{p_{\max}},
\qquad
\st\!\bigl(A^{\mathrm{rms},(0)}_{\mathrm{mlp}}\bigr)\ \le\ \frac{6.03}{p_{\max}}.
\]

\item[(ii)] \textbf{Activation stable ranks.}  For every $\ell=0,\dots,L-1$,
\begin{align*}
\st\!\bigl(A^{\mathrm{rms},(\ell)}\bigr)
&\le \Bigl(1.02+1.02(1+m_\ast)\ell\Bigr)\,\frac{1.01}{p_{\max}}\,\rho^{\ell}, \\
\st\!\bigl(A^{\mathrm{rms},(\ell)}_{\mathrm{mlp}}\bigr)
&\le \Bigl(6.13+3.09(1+m_\ast)\ell\Bigr)\,\frac{1.01}{p_{\max}}\,\rho^{\ell}, \\
\st\!\bigl(B^{(\ell)}\bigr)
&\le C_{\mathrm{PROP}}\Bigl(6.13+3.09(1+m_\ast)\ell\Bigr)\,\frac{1.01}{p_{\max}}\,\rho^{\ell}, \\
\st\!\bigl(A^{\mathrm{rms}}_{\mathrm{final}}\bigr)
&\le \Bigl(1.02+1.02(1+m_\ast)L\Bigr)\,\frac{1.01}{p_{\max}}\,\rho^{L}.
\end{align*}
\end{enumerate}
\end{thm}
\begin{proof}
Set $\delta:=\frac{1}{512L}$ and $\varepsilon:=\delta$, and write $\alpha:=\frac{1+\delta}{1-\delta}$.
Let $\Omega$ be the event on which Lemmas~\ref{lem:embedding_rms_stablerank} and \ref{lem:embedding-colnorms}
hold with parameter $\varepsilon$, and for every $\ell=0,\dots,L-1$, Lemmas~\ref{lem:block-att-common} and
\ref{lem:mlp-gated} hold with parameter $\delta$ at layer $\ell$.
A union bound gives the stated lower bound on $\PP(\Omega)$ after adjusting constants and using
$d,k\ge C\delta^{-2}\log(nL)$.  Work throughout on $\Omega$.

\paragraph{Embedding layer.}
The argument is identical to the MSI case: using Lemmas~\ref{lem:embedding_rms_stablerank},
\ref{lem:embedding-colnorms}, and Lemma~\ref{lem:block-att-common} at $\ell=0$ with $\varepsilon=\delta\le 1/512$
gives the bounds in (i).

\paragraph{Column-norm control.}
For each $\ell\ge 0$, define $c_-^{(\ell)},c_+^{(\ell)}$ by
\[
c_-^{(\ell)}d\le \|x_t^{(\ell)}\|_2^2\le c_+^{(\ell)}d
\qquad \forall t.
\]
On $\Omega$, Lemma~\ref{lem:embedding-colnorms} gives $c_-^{(0)}=1-\delta$ and $c_+^{(0)}=1+\delta$.
Lemma~\ref{lem:block-att-common}(i) gives
\[
(1-\delta)c_-^{(\ell)}d \le \|x_t^{\mathrm{att},(\ell)}\|_2^2 \le (1+\delta)(c_+^{(\ell)}+1)d
\qquad\forall t,
\]
and Lemma~\ref{lem:mlp-gated}(i), applied to $X^{\mathrm{att},(\ell)}$, gives
\begin{equation}\label{eq:col-recur-gated}
c_-^{(\ell+1)}=(1-\delta)^2c_-^{(\ell)},
\qquad
c_+^{(\ell+1)}=(1+\delta)\Bigl((1+\delta)(c_+^{(\ell)}+1)+m_\ast\Bigr).
\end{equation}
Hence $c_-^{(\ell)}=(1-\delta)^{2\ell+1}\ge 1-(2\ell+1)\delta\ge \frac{509}{512}$ for all $\ell\le L$.
Writing \eqref{eq:col-recur-gated} as $c_+^{(\ell+1)}=a\,c_+^{(\ell)}+b$ with
$a=(1+\delta)^2$ and $b=(1+\delta)\bigl((1+\delta)+m_\ast\bigr)$, we obtain
\[
c_+^{(\ell)}
\le a^\ell c_+^{(0)}+b\sum_{j=0}^{\ell-1}a^j
\le a^L\bigl((1+\delta)+b\ell\bigr)
\le 1.01+1.01(1+m_\ast)\ell,
\qquad \ell\le L,
\]
since $\delta L\le 1/512$ implies $a^L(1+\delta)^2 <1.01$

\paragraph{Hidden-state stable rank recursion.}
Set $s_\ell:=\st(X^{(\ell)})$.  Lemma~\ref{lem:block-att-common}(ii) and $c_-^{(\ell)}\ge 509/512$ imply
\begin{equation}\label{eq:att-sr-gated}
\st\!\bigl(X^{\mathrm{att},(\ell)}\bigr)
\le \alpha\Bigl(1+\frac{2}{c_-^{(\ell)}}\Bigr)s_\ell
\le 3.03\,s_\ell.
\end{equation}
Lemma~\ref{lem:mlp-gated}(ii), applied to $X^{\mathrm{att},(\ell)}$ whose lower column-norm parameter is
$(1-\delta)c_-^{(\ell)}\ge 0.992$, yields
\begin{equation}\label{eq:mlp-sr-gated}
s_{\ell+1}
=\st\!\bigl(X^{(\ell+1)}\bigr)
\le \alpha\left(1+\frac{m_\ast}{(1-\delta)c_-^{(\ell)}}\right)\st\!\bigl(X^{\mathrm{att},(\ell)}\bigr)
\le (1.01+1.02\,m_\ast)\,\st\!\bigl(X^{\mathrm{att},(\ell)}\bigr).
\end{equation}
Combining \eqref{eq:att-sr-gated} and \eqref{eq:mlp-sr-gated} gives
\[
s_{\ell+1}\le 3.1(1+m_\ast)\,s_\ell=\rho\,s_\ell,
\]
and iterating yields $s_\ell\le s_0\rho^\ell$ for all $\ell$.
Since $\st(X^{(0)})\le 1.01/p_{\max}$ on $\Omega$, we have
\begin{equation}\label{eq:sl-gated-bound}
s_\ell\le \frac{1.01}{p_{\max}}\,\rho^\ell
\qquad \text{for all }\ell=0,\dots,L.
\end{equation}

\paragraph{Activation stable ranks.}
For $A^{\mathrm{rms},(\ell)}=\RMSNorm(X^{(\ell)})$, Theorem~\ref{thm:rms-stablerank} yields
\[
\st\!\bigl(A^{\mathrm{rms},(\ell)}\bigr)\le \frac{c_+^{(\ell)}}{c_-^{(\ell)}}\,s_\ell
\le \Bigl(1.02+1.02(1+m_\ast)\ell\Bigr)\,\frac{1.01}{p_{\max}}\,\rho^\ell,
\]
using $c_-^{(\ell)}\ge 509/512$, $c_+^{(\ell)}\le 1.01+1.01(1+m_\ast)\ell$, and \eqref{eq:sl-gated-bound}.
For $A^{\mathrm{rms},(\ell)}_{\mathrm{mlp}}=\RMSNorm(X^{\mathrm{att},(\ell)})$, Lemma~\ref{lem:block-att-common}(ii) yields
\[
\st\!\bigl(A^{\mathrm{rms},(\ell)}_{\mathrm{mlp}}\bigr)
\le
\frac{(1+\delta)^2}{(1-\delta)^2}\,
\frac{c_+^{(\ell)}+1}{c_-^{(\ell)}}\Bigl(1+\frac{2}{c_-^{(\ell)}}\Bigr)s_\ell
\le \Bigl(6.13+3.09(1+m_\ast)\ell\Bigr)\,\frac{1.01}{p_{\max}}\,\rho^\ell,
\]
using the same bounds on $c_\pm^{(\ell)}$ and \eqref{eq:sl-gated-bound}.

For $B^{(\ell)}$, Lemma~\ref{lem:mlp-gated}(ii) gives
\[
\st\!\bigl(B^{(\ell)}\bigr)
\le C_{\mathrm{PROP}}\,\st\!\bigl(A^{\mathrm{rms},(\ell)}_{\mathrm{mlp}}\bigr),
\]
and substituting the bound for $\st(A^{\mathrm{rms},(\ell)}_{\mathrm{mlp}})$ yields the claimed inequality for $\st(B^{(\ell)})$.

Finally, $A^{\mathrm{rms}}_{\mathrm{final}}=\RMSNorm(X^{(L)})$ satisfies the same bound as $A^{\mathrm{rms},(L)}$ by
Theorem~\ref{thm:rms-stablerank}, giving the last display in (ii).
\end{proof}

\section{Nuclear Rank in random feature models}\label{sec:atinitialization}

The layerwise criterion \eqref{eq:sd-vs-gd-condition} shows that spectral updates are most
advantageous precisely when the nuclear rank of the gradient is large compared to the stable
rank of the incoming activations. At the same time, there is no a priori reason for the
gradient nuclear rank to be large early on in training; in fact, we will see that it can be
bounded by a numerical constant in simple,
yet representative, random-feature models.
The goal of this section is to
show that a \emph{few} gradient steps are already enough to change this picture entirely leading to a nuclear rank that scales linearly with dimension.

Throughout the section, we consider the quadratic random-feature objective
\[
  \mathcal L(W) := \tfrac12\|WA-Y\|_F^2\qquad \textrm{for}\qquad W\in\R^{m\times k},
\]
and study two choices of labels $Y$. In the first model, the ``realizable" case, the labels are
generated by a ground-truth matrix $W_\sharp$ via $Y = W_\sharp A$, and the Gram matrix
$B := AA^\top$ follows the rank-one spike–plus–bulk structure in
Assumption~\ref{assump:spiked}. The main difficulty here is spectral: at initialization the
gradient is $-W_\sharp B$ and inherits the poor conditioning of $B$, so its nuclear rank is
small, whereas after a few steps the relevant block becomes $W_\sharp (I-\eta B)^tB$, and we must
show that the polynomial $(I-\eta B)^tB$ flattens the spectrum of $B$: it squeezes the top spike
down to constant scale while leaving the bulk eigenvalues at constant scale, so that the nuclear 
rank of the resulting gradient grows with the feature dimension. In the second
model, a teacher–student setup, the labels are generated by a different feature map
$\overline A$ as $Y = \overline W\,\overline A$ with
\[
  A = \sigma(VX),\qquad \overline A = \sigma(\overline V X),
\]
so the gradient involves the cross-Gram matrix $\overline A A^\top$. Here we work under the
spiked assumptions~\ref{assump:spiked} and~\ref{assump:spiked2}, which encompass both the basic
rank-one random-feature setting and its multi-spike extensions, and we show that these models
are realized with high probability for ReLU random features with Gaussian inputs and weights:
the gradient nuclear rank is $O(1)$ at initialization but becomes dimension-dependent (of order
$\Theta(d)$) after a few gradient steps. Moreover, we show that it then remains $\Omega(d)$
over a window of at least $\Theta(d)$ gradient-descent iterates
(Theorems~\ref{thm:multimodel1} and~\ref{thm:main_thrmgrfjishdfklds}); a refinement of the same arguments
yields longer $\Theta(d\log d)$ windows at the cost of a slightly lower nuclear rank, as explained in the subsequent remark.

Before turning to the formal statements, we illustrate these effects numerically in both
random-feature models. Figures~\ref{fig:rf-nr-vs-step-both} and~\ref{fig:rf-first-step-both}
show how the gradient nuclear rank evolves over the first few gradient steps and how its
first-step value scales with the feature dimension~$k$.

\begin{figure}[t]
\centering
\begin{subfigure}[t]{0.48\textwidth}
  \centering
  \includegraphics[width=\linewidth]{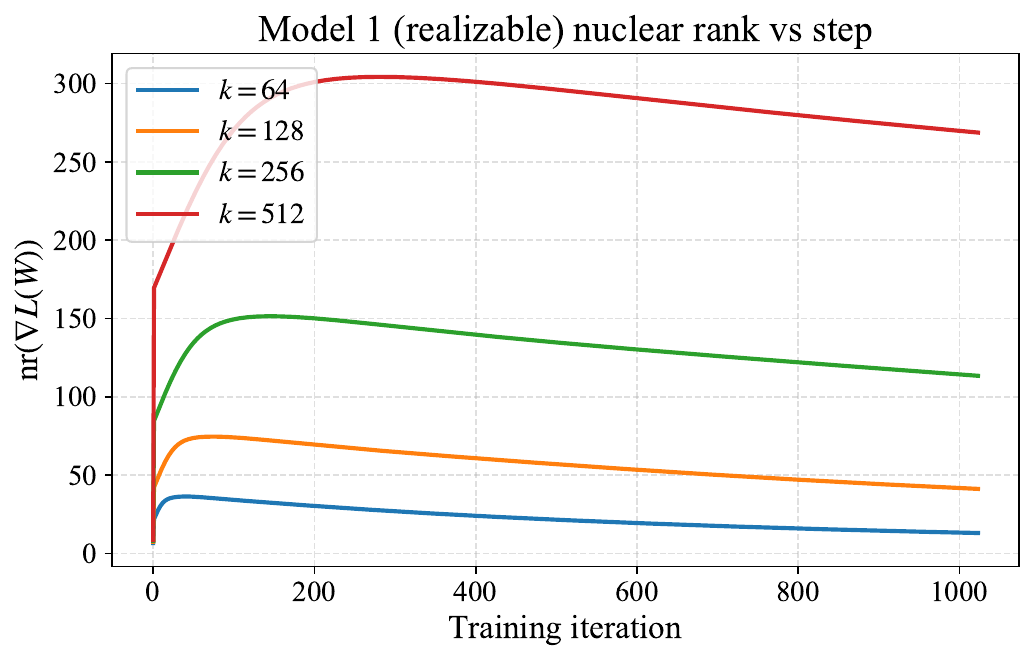}
  \caption{Realizable random-feature model.}
  \label{fig:rf-model1-nr-vs-step}
\end{subfigure}\hfill
\begin{subfigure}[t]{0.48\textwidth}
  \centering
  \includegraphics[width=\linewidth]{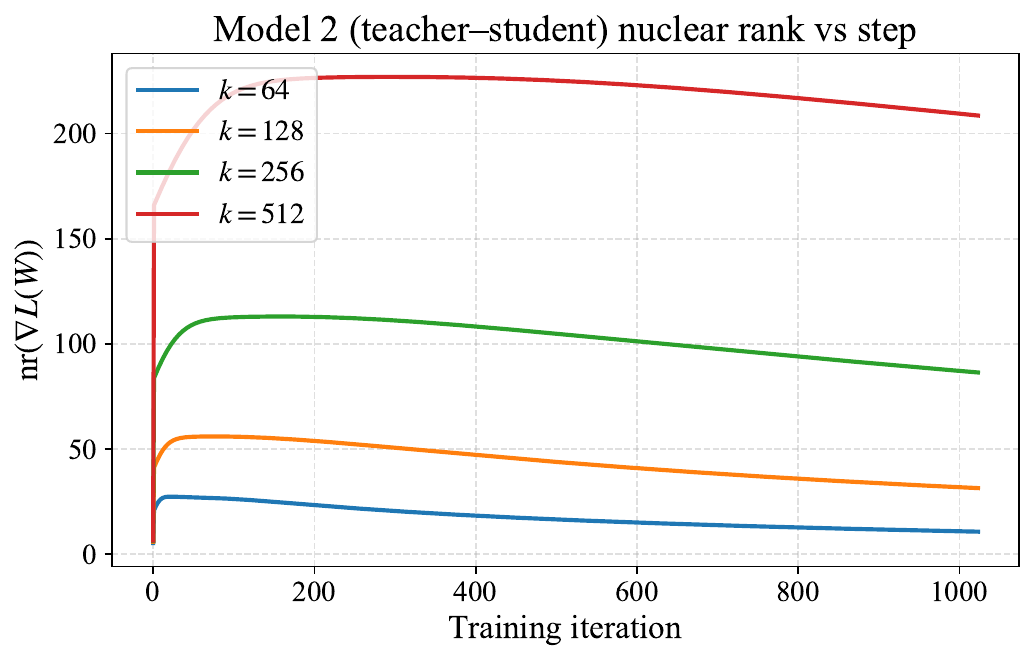}
  \caption{Teacher--student random-feature model.}
  \label{fig:rf-model2-nr-vs-step}
\end{subfigure}
\caption{Gradient nuclear rank as a function of the gradient step. 
We plot $\nr(\nabla_W \cL(W_t))$ for several feature dimensions~$k$. Here $n = 2k$ and $d = k/2$.
In both the realizable (left) and teacher--student (right) models, the nuclear rank grows rapidly in the first few steps and does so more strongly for larger~$k = 2d$, indicating that spectral updates become increasingly advantageous as the feature dimension grows.}
\label{fig:rf-nr-vs-step-both}
\end{figure}

\begin{figure}[t]
\centering
\begin{subfigure}[t]{0.48\textwidth}
  \centering
  \includegraphics[width=\linewidth]{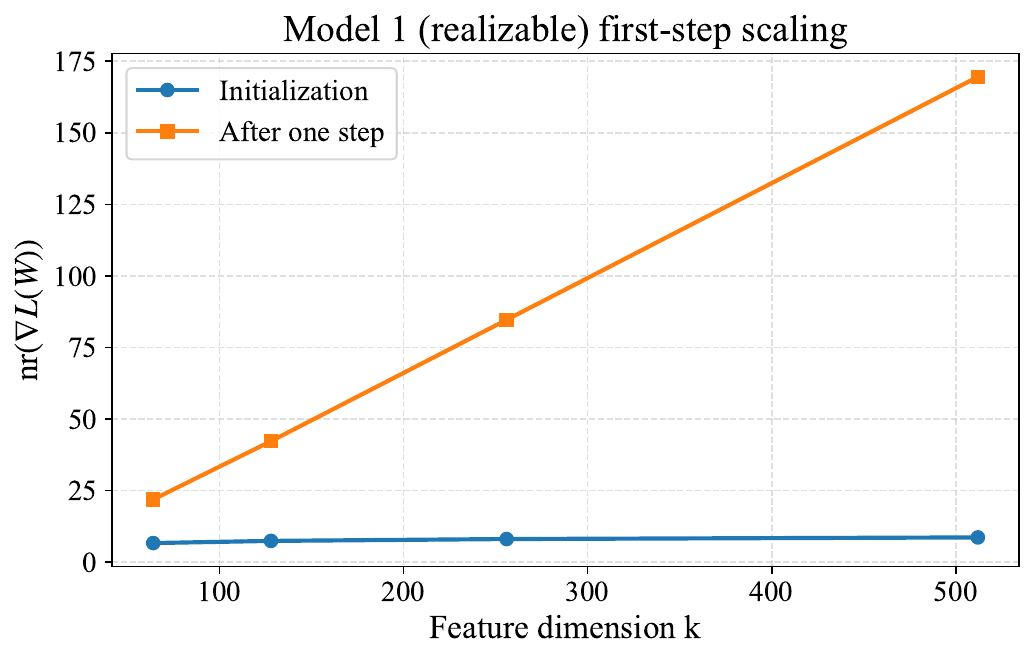}
  \caption{Realizable random-feature model.}
  \label{fig:rf-model1-first-step}
\end{subfigure}\hfill
\begin{subfigure}[t]{0.48\textwidth}
  \centering
  \includegraphics[width=\linewidth]{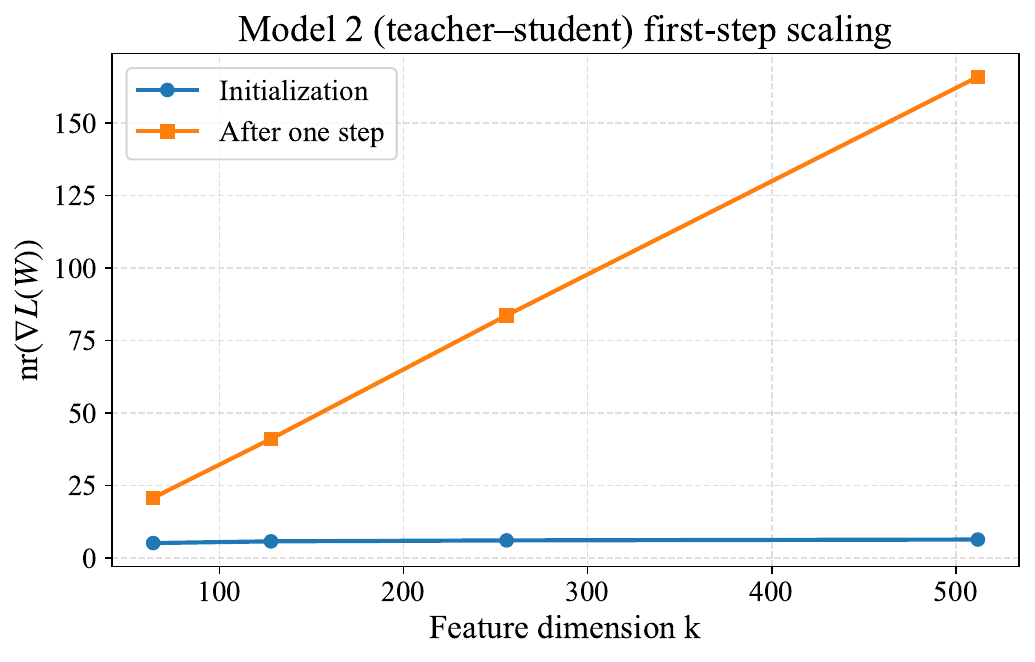}
  \caption{Teacher--student random-feature model.}
  \label{fig:rf-model2-first-step}
\end{subfigure}
\caption{First-step scaling of the gradient nuclear rank with feature dimension. 
For each feature dimension~$k$ we plot the nuclear rank at initialization, $\nr(\nabla_W \cL(W_0))$, and after one gradient step, $\nr(\nabla_W \cL(W_1))$. Here $n = 2k$ and $d = k/2$. 
In both the realizable (left) and teacher--student (right) models, the initialization nuclear rank is essentially independent of~$k$, while the post-first-step nuclear rank grows approximately linearly with~$k$. 
This illustrates the regime captured by our theory: the condition $\nr(\nabla_W \cL(W_1)) \gg \st(A)$ is increasingly activated as the feature dimension increases.}
\label{fig:rf-first-step-both}
\end{figure}

\paragraph{Notation.} For any $V\in \R^{m\times k}$ we  order its singular values as 
$s_1(V)\geq s_1(V) \geq \ldots \geq s_{\min\{k,m\}}(V).$
Similarly, the eigenvalues of any symmetric matrix $B\in \R^{k\times k}$ are 
$\lambda_1(B)\geq \lambda_2(B)\geq \ldots \geq \lambda_k(B).$ 

\subsection{Random feature model I}
In this section, we consider the problem of random feature regression, discussed in Section~\ref{sec:rand_regress} of the introduction. We show two interesting facts that hold under mild assumptions: (1) when initializing at the origin $W_0$, the nuclear rank $\nr(\nabla \mathcal{L}(W_0))$ is upper bounded by a numerical constant and (2) after a single gradient step the nuclear rank $\nr(\nabla \mathcal{L}(W_1))$ grows with the dimension. Together these facts show that one can expect the nuclear rank to be large at least early on in training and therefore for spectral descent to outperform the Euclidean gradient method.

Setting the stage, consider the problem 
\begin{equation}\label{eqn:rfm_app}
    \min_{W\in \mathbb{R}^{m\times k}}~ \mathcal{L}(W):=\tfrac{1}{2n}\|WA-Y\|_F^2\qquad \textrm{where}\qquad Y=W_{\sharp}A.
\end{equation}
The dimensions of the involved matrices are $W,W_{\sharp}\in \R^{m\times k}$ and $A\in \R^{k\times n}$ and we will be interested in the proportionate regime $$d\asymp k\qquad \textrm{as }\quad d\to \infty.$$ 
We will focus on the setting when the Gram matrix $B:=\tfrac{1}{n}AA^{\top}$ has a ``spike+bulk'' structure:

\begin{assumption}[Spiked model]\label{assump:spiked}{\rm
Suppose that we may write
    \begin{equation}\label{eqn:spike_bulk}
B=uu^\top+\Sigma,
\end{equation}
for some vector $u\in\R^{k}$ satisfying $\|u\|_2^2\asymp d$ and a positive definite matrix $\Sigma\in\R^{k\times k}$ satisfying $c_1 I_k\preceq \Sigma \preceq c_2 I_k$ for some constants $0< c_1\leq c_2<\infty$.}
\end{assumption}

In particular, the rank-one interlacing theorem \cite[Chapter 4]{HornJohnson2013} implies separation of eigenvalues:
$$\lambda_1(B)\asymp d\qquad \textrm{and}\qquad \lambda_{i}(B)\in [c_1,c_2]\qquad \forall i\geq 2.$$
Spiked data matrices play an prominent role in random matrix theory \cite{johnstone2001distribution,baik2005phase}. The main example of spiked data matrices for us is the post-activation matrix generated from random data. For the sake of concreteness, we focus on the ReLU activation function, although many other activation functions elicit the same phenomenon.

\subsubsection{ReLU post-activation matrices follow the spiked model}
Consider the matrix $A=\sigma(VX)$ for some fixed matrix $V\in \R^{k\times d}$ and where $X\in \R^{d\times n}$ is a random matrix with iid standard Gaussian entries $\mathcal{N}(0,1)$ and $\sigma$ is the ReLU activation function. We will show that under a mild incoherence condition on the rows of $V$, the post-activation matrix $A$ indeed follows the spiked model (Assumption~\ref{assump:spiked}).
To see this, we first analyze the expectation of $AA^T$. Observe that setting $a:=\sigma(Vx)$ for a standard Gaussian random vector $x\in \R^{d}$, we may write
$$\EE\left[\tfrac{1}{n}AA^T\right]=\EE[aa^\top]=\mu\mu^\top+\Sigma,$$
where we define
$$\mu:=\EE[a]\qquad \textrm{and}\qquad \Sigma:={\rm Cov}(a)=\EE(a-\mu)(a-\mu)^{\top}.$$
A standard computation shows the explicit formulas:
$$\mu=\frac{1}{2\pi}\begin{bmatrix}\|v_1\|_2&\ldots &\|v_k\|_2\end{bmatrix}^{\top}\qquad\textrm{and}\qquad \Sigma_{ij}
= \frac{\|v_i\|_2\|v_j\|_2}{2\pi}
\varphi\left(\frac{\langle v_i,v_j\rangle}{\|v_i\|_2\|v_j\|_2}\right)
,$$
where we set $\varphi(t)=\left(
\sqrt{1-t^2}
+ (\pi - \arccos t)\, t
- 1
\right)$; see Figure~\ref{eqn:kappa_rep} for an illustration of  $\varphi$.

\begin{figure}[h]
\centering
\begin{subfigure}[t]{0.48\textwidth}
  \centering
  \includegraphics[width=\linewidth]{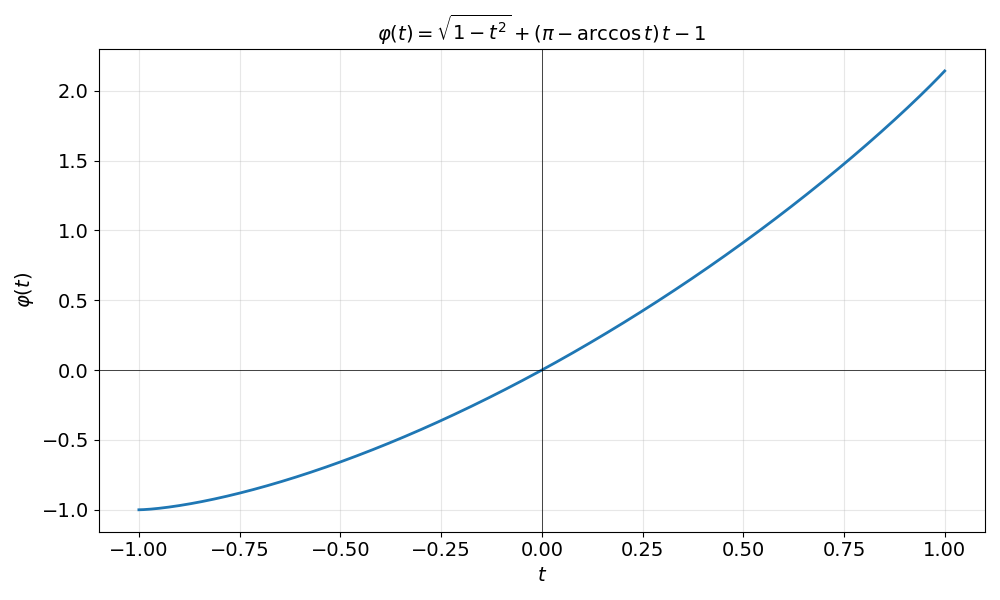}
  \caption{The reparametrizing function $\varphi$\label{eqn:kappa_rep}}
\end{subfigure}
\caption{Illustration of $\varphi$.}
\end{figure}

Thus $\EE[\frac{1}{n}AA^{\top}]$ has the spike+bulk structure \eqref{eqn:spike_bulk} with $\|\mu\|^2_2=\frac{1}{4\pi^2}\|V\|^2_F$. Now in order to estimate the extremal eigenvalues of $\Sigma$, we will have to assume that the rows of $V$ have roughly the same norm and are nearly orthogonal---a form of incoherence. This is the content of Theorem~\ref{thm:incoherence}. The proof of Theorem~\ref{thm:incoherence} relies on the following  lemma describing a few standard properties of the Hadamard product.  Namely, item~\ref{it1} is elementary, item~\ref{it2} appears in \cite[Thm.~1.4.1]{Bhatia2007PDM}, and item~\ref{it3} is the Schur product theorem \cite[Sec.~1.2]{Bhatia2007PDM}. 
\begin{lem}[Properties of Hadamard products]\label{lem:gen_hadamard}
For any two matrices $Q, P\in \mathbb{R}^{k\times k}$, it holds:
\begin{enumerate}
    \item \label{it1}
    $\|Q\odot P\|_{\op}\leq \max_{ij}|Q_{ij}|\cdot\,\|P\|_{F}$
    \item \label{it2} $\|Q\odot P\|_{\op}\leq  \max_{i} Q_{ii}\cdot \|P\|_{\op}$ if $Q\succeq 0$.
    \item \label{it3} $Q\odot P\succeq 0$ if $Q\succeq 0$ and $P\succeq 0$.
\end{enumerate}
\end{lem}

\begin{thm}[Perfect conditioning of the covariance]\label{thm:incoherence}
Suppose that the rows $v_i$ of $V$ are incoherent in the following sense: there exist constants $\varepsilon\in (0,1)$ and $c>0$ satisfying  
\begin{equation}\label{eqn:incoherence}
1-\varepsilon\leq \|v_i\|^2_2\leq 1+\varepsilon\qquad \textrm{and}\qquad \frac{|\langle v_i,v_j\rangle|}{\|v_i\|\|v_j\|}\leq c\sqrt{\frac{\log(k)}{{d}}}\qquad \forall i\neq j. 
\end{equation}
Then $\Sigma$ satisfies the two-sided bound:
\begin{equation}\label{eqn:two_sided}
\left(\tfrac{1+\pi^{-1}(1-\varepsilon)^{-1}}{4}\|V\|_{\op}^2+\tfrac{(\pi-3)(1+\varepsilon)}{4\pi}+O(\tfrac{k\log^2(k)}{d^2})\right)I_{k}\succeq\Sigma\succeq \tfrac{1}{4}VV^{\top} +\left(\tfrac{(\pi-3)(1-\varepsilon)}{4\pi}+O(\tfrac{k\log^2(k)}{d^2})\right)I_k.
\end{equation}
\end{thm}
\begin{proof}
Notice that we may write
$$\Sigma=\tfrac{1}{2\pi}\Lambda K\Lambda$$
where the entries of $K$ are $K_{ij}=\varphi\left(\frac{\langle v_i,v_j\rangle}{\|v_i\|_2\|v_j\|_2}\right)$ and $\Lambda$ is a diagonal matrix with $\Lambda_{ii}=\|v_i\|$. We define $\widehat V\in \R^{k\times d}$ to be a matrix whose $i$'th row is the normalized vector $v_i/\|v_i\|_2$.
We will treat $K$ first and deal with $\Lambda$ at the end of the argument. 

The proof proceeds by carefully Taylor expanding $\varphi$ and treating diagonal and off-diagonal elements separately. First, a quick computation yields the derivatives:
$$\varphi(0)=0,~~\varphi'(0)=\frac{\pi}{2},~~\varphi''(0)=1,~~\varphi^{(3)}(0)=0,~~ \varphi^{(4)}(0)=1, ~~\varphi^{(5)}(0)=0.$$
For any square matrix $Q$, let $\od(Q)$ denote the matrix obtained by zeroing out the diagonal of $Q$, that is $\od(Q)=Q-\Diag(Q)$.
Using Taylor's theorem with remainder, we may write
\begin{equation}\label{eqn:basic_offdiag}
\od(K)=\frac{\pi}{2}\od(\widehat{V}\widehat{V}^{\top})+\frac{1}{2}\od((\widehat{V}\widehat{V}^{\top})^{\odot 2})+\frac{1}{4!} D\odot \od((\widehat{V}\widehat{V}^{\top})^{\odot 4}),
\end{equation}
where the entries of $D$ are fourth-order derivatives of $\varphi$ at some points in the interval $[-\tfrac{c}{\sqrt{d}},\tfrac{c}{\sqrt{d}}]$ and are therefore bounded by some numerical constant. Using Lemma~\ref{lem:gen_hadamard} (Item~\ref{it1}), we deduce 
$$\left\| D\odot \od((\widehat{V}\widehat{V}^{\top})^{\odot 4})\right\|_{\op}\leq \max_{ij} |D_{ij}|\cdot \left\|\od((\widehat{V}\widehat{V}^{\top})^{\odot 4})\right\|_{F}=O(k\log^2(k)d^{-2}),$$
where the last equality follows from the fact that each entry of $\od((\widehat{V}\widehat{V}^{\top})^{\odot 4})$ is bounded by $O(\log^2(k)d^{-2})$. Adding back the diagonal $\Diag(K)=(\pi-1)I_{k}$ to \eqref{eqn:basic_offdiag} we therefore obtain
$$K=\tfrac{\pi}{2}\widehat{V}\widehat{V}^{\top}+\tfrac{1}{2}(\widehat{V}\widehat{V}^{\top})^{\odot 2}+\tfrac{\pi-3}{2}I_k+E$$
where the error term satisfies $\|E\|_{\op}=O(k\log^2(k)d^{-2})$.
Conjugating by $\Lambda$ and multiplying by $\frac{1}{2\pi}$ yields the expression 
\begin{equation}\label{eqn:sigma_eqn}
\Sigma=\frac{1}{4}VV^{\top}+\frac{1}{4\pi}\Lambda^{-1}(VV^T)^{\odot 2}\Lambda^{-1}+\frac{\pi-3}{4\pi}\Lambda^2+\frac{1}{2\pi} E.
\end{equation}
Using Lemma~\ref{lem:gen_hadamard} (Item \ref{it3}) we see that $(VV^{\top})^{\odot 2}$ is positive semidefinite, and therefore we obtain the claimed lower bound: $$\Sigma\succeq \frac{1}{4}VV^{\top} +\tfrac{(\pi-3)(1-\varepsilon)}{4\pi}I_k+\frac{1}{2\pi} E.$$
Conversely, using Lemma~\ref{lem:gen_hadamard} (Item \ref{it2}), we see that $\|({V}{V}^{\top})^{\odot 2}\|_{\op}\leq \|{V}{V}^{\top}\|_{\op}$ and therefore from \eqref{eqn:sigma_eqn} we deduce $$\Sigma\preceq \left(\tfrac{1+\pi^{-1}(1-\varepsilon)^{-1}}{4}\|V\|_{\op}^2+\tfrac{(\pi-3)(1+\varepsilon)}{4\pi}+O(k\log^2(k)d^{-2})\right)I_{k},$$ which completes the proof.
\end{proof}

In words, Theorem~\ref{thm:incoherence} shows that under the incoherence condition~\eqref{eqn:incoherence} and in the regime $k/d^2\to 0$ as $k$ and $d$ tend to infinity, $\EE[nB]$ follows the spiked model \eqref{eqn:spike_bulk} with $\|u\|_2^2\asymp\|V\|_F^2$ and $\Sigma$ satisfying 
$c_1\preceq \Sigma \preceq c_2(1+\|V\|^2_{\rm op})$ where $c_1,c_2>0$ are numerical constants. Reassuringly random matrices $V$ with iid Gaussian entries $\mathcal{N}(0,\frac{1}{d}I)$ satisfy the incoherence condition \eqref{eqn:incoherence} and moreover $\|V\|_{\rm op}$ has order $O(1)$ in the regime $k \asymp d$ while the spike has order $\|\mu\|^2_2\asymp d$. This is the content of the following lemma; we omit the proof since it is standard.

\begin{lem}[Incoherence with random $V$]
Let $v_1,\dots,v_n\in\mathbb{R}^d$ be independent random vector  with
$\,v_i\sim\mathcal N(0,\tfrac1d I_d)$ and suppose  $\frac{\log k}{d}\le \tfrac{1}{4}$. Then we have
\begin{equation}\label{eqn:prob_eqn_poly}
\mathbb{P}\!\left(
\max_{i\neq j}\left|\frac{\langle v_i,v_j\rangle}{\|v_i\|_2\|v_j\|_2}\right|
\le 4\sqrt{\frac{\log k}{d}}
\ \text{ and }\
\max_{i\leq k}\big|\|v_i\|_2^2-1\big|
\le 4\sqrt{\frac{\log k}{d}}
\right)
\ \ge\
1-\frac{4}{k}.
\end{equation}
Moreover, with probability at least $1-4e^{-c\delta^2 d}$ it holds:
\begin{equation}\label{eqn:matrix_norm_bounds}
    \|V\|_{\op}\;\le\;\sqrt{\frac{k}{d}}+(1+\delta)\qquad\textrm{and}\qquad
(1-\delta)\,k \;\le\; \|V\|_{F}^2 \;\le\; (1+\delta)\,k.
\end{equation}
\end{lem}

It remains to take into account concentration of $\frac{1}{n}AA^\top$ around its mean in order to deduce that it follows the spiked model. This is the content of the following theorem, which is immediate from standard results on concentration of covariance matrices. In the theorem, we denote the $i$'th column of A by $a_i$.

\begin{thm}[Spiked model from Gaussian initialization]\label{th:gauss_spike1}
Let $V\in \R^{k\times d}$ have iid Gaussian entries $N(0,\tfrac{1}{d} I_d)$. Then there exist constants $C_0,C_1,C_2>0$ such that in the regime $ k\asymp d$ and $n\geq C_0d$, with probability tending to one as $d\to \infty$, we may write 
$$\tfrac{1}{n}AA^\top=\bar \mu\bar \mu^\top+\overline\Sigma,$$
where the empirical mean $\bar \mu=\frac{1}{n}\sum_{i=1}^n a_i$ and the empirical covariance $\overline{\Sigma}=\frac{1}{n}\sum_{i=1}^n (a_i-\bar \mu)(a_i-\bar \mu)^\top$ satisfy
$ \|\bar \mu\|^2\asymp d $ and $C_1 I\preceq\overline{\Sigma}\preceq C_2 I$.
\end{thm}

\begin{proof}
Suppose that the events \eqref{eqn:prob_eqn_poly} and \eqref{eqn:matrix_norm_bounds} hold, which occurs with probability at least $1-\frac{4}{k}-4\exp(-c\delta^2 d)$. Then Theorem~\ref{thm:incoherence} implies that the bound \eqref{eqn:two_sided} holds say for $\varepsilon=1/2$ as long as $d\geq c\log(k)$ for some constant $c$.

Define the empirical mean $\bar \mu=\frac{1}{n}\sum_{i=1}^n a_i$ and the empirical covariance $\overline{\Sigma}=\frac{1}{n}\sum_{i=1}^n (a_i-\bar \mu)(a_i-\bar \mu)^\top$ and note the equality: 
$$\frac{1}{n}AA^\top=\bar \mu\bar \mu^\top +\overline{\Sigma}.$$
Note that $\langle a,u\rangle=\langle \sigma(Vx),u\rangle$ is a Lipschitz continuous function of the Gaussian vector $x$ with Lipschitz constant $\|V\|_{\op}\cdot \|u\|_2$. 
Therefore, Gaussian concentration \cite[Theorem 5.2.3]{vershynin2018high} shows:
$$\|\langle a-\EE a,u\rangle\|_{\psi_2}\lesssim \underbrace{\|V\|_{\op}}_{\lesssim 1}\|u\|_2\qquad \forall u.$$
Moreover, the minimal eigenvalues of $\Sigma$ is bounded from below by a constant \eqref{eqn:two_sided} and therefore we may write 
$$\|\langle a-\EE a,u\rangle\|_{\psi_2}\lesssim\sqrt{\langle \Sigma u,u\rangle}\qquad \forall u.$$
Appealing to concentration of covariance matrices \cite[Remark 9.2.3]{vershynin2018high}, we deduce
that there exist universal constants $c,C>0$ such that for all $u\ge 1$, with probability at least $1-2e^{-d}$, it holds:
\[
\left\|\overline\Sigma - \Sigma\right\|_{\op}
\ \lesssim \left(\sqrt{\frac{d}{n}}+\frac{d}{n}\right)\,\underbrace{\|\Sigma\|_{\op}}_{\lesssim 1}.
\]
Appealing to concentration of the mean \cite[Theorem 1]{hsu2012tail}, we also obtain:
$$\|\bar \mu-\mu\|_2\lesssim \sqrt{\frac{d}{n}},$$
with probability $1-e^{-d}$. In particular, taking into account \eqref{eqn:matrix_norm_bounds} in this event we may estimate 
$$\sqrt{k}-\sqrt{\frac{d}{n}}\lesssim \|\bar \mu \|_2\lesssim \sqrt{k}+\sqrt{\frac{d}{n}},$$
which completes the proof.
\end{proof}

\subsubsection{The nuclear rank at initialization is constant}
We now look at the consequences of the spiked model \eqref{eqn:spike_bulk} on the nuclear rank of the gradient. Namely, consider initializing gradient descent on the problem~\eqref{eqn:rfm_app} at $W_0$. Then taking into account the expression for the gradient 
$$
\nabla \mathcal{L}(W)=(W-W_{\sharp})B\qquad \textrm{with}\qquad B:=\tfrac{1}{n}AA^\top,$$
we have 
\begin{equation}\label{eq:grad_init}
\nabla \mathcal{L}(W_0)=-W_{\sharp}B.
\end{equation}
Now suppose that the ground truth matrix $W_{\sharp}$ does not eliminate the spike in the following sense:
$$\|W_{\sharp}\|_{\op}\leq C \qquad \textrm{and}\qquad \|W_{\sharp}u\|_2\geq \kappa\cdot \|u\|_2,$$ 
where $C$ and $\kappa$ are numerical constants.
This is indeed the case for example if the singular values $s_1(W_{\sharp})$ and $s_{k}(W_{\sharp})$ are bounded from above and below by a numerical constant or is true with high probability if the entries of $W_{\sharp}\in \R^{m\times k}$ are iid Gaussian $\mathcal{N}(0,\tfrac{1}{m})$ and $m\gtrsim k$.
Then we can estimate the nuclear rank of the gradient as:
$$\sqrt{\nr(W_{\sharp}B)}=\frac{\|W_{\sharp}B\|_{*}}{\|W_{\sharp}B\|_F}\leq \frac{C\|B\|_{*}}{\|W_{\sharp}uu^\top\|_F- \|W_{\sharp}\Sigma\|_{F}}\leq \frac{C\|u\|_2^2+C\|\Sigma\|_*}{\kappa\|u\|^2_2-C\|\Sigma\|_F}\lesssim \frac{d}{d-\sqrt{d}}\asymp 1.$$
Thus indeed the nuclear rank of the gradient $\nabla \mathcal L(W_0)$ is  bounded by a constant.

\subsubsection{The nuclear rank after one gradient step scales with dimension.}
Let us now analyze what happens after a single gradient step $W_{1}=W_0-\eta \nabla \mathcal{L}(W_0)$. Observe the gradient of the objective function at $W_1$ is 
$$\nabla \mathcal{L}(W_1)=(\eta W_{\sharp}B-W_{\sharp})B=-W_{\sharp}\underbrace{(I-\eta B)B}_{=:T_{\eta}}.$$
We will now see that although $B$ is a poorly conditioned matrix, $T_{\eta}$ is perfectly conditioned for any stepsize $\eta=\frac{1}{c+\|B\|_{\op}}$ with constant $c>0$. 
\begin{lem}[Nuclear rank after one iteration grows with dimension]\label{lem:nuc_rank_grows}
Suppose that Assumption~\ref{assump:spiked} holds and fix a step-size $\eta=\frac{1}{c+\|B\|_{\op}}$ for any constant $c>0$. Then the condition number of $T_{\eta}$ is bounded by a numerical constant in the regime $k\asymp d$ as $d\to\infty$, and consequently, we have 
 $$\nr(\nabla \mathcal{L}(W_1))\gtrsim \nr(W_{\sharp}).$$
\end{lem}
\begin{proof}
The eigenvalues of $T_{\eta}$ have the form $\lambda_i\left(1-\frac{\lambda_i}{c+\lambda_1}\right)$ where $\lambda_i$ is the $i$'th largest eigenvalue of $B$. Now clearly, we have 
$\lambda_1(B)\geq \|u\|_2^2\asymp d$. Moreover by the rank one interlacing theorem we have
$$\lambda_i(\Sigma)\geq \lambda_{i+1}(B)\geq \lambda_{i+1}(\Sigma)\qquad\forall i=1,\ldots,d-1.$$
Consequently, the eigenvalues $\lambda_i(B)$ for $i=2,\ldots, d$ all lie in the interval $[c_1,c_2]$. We therefore deduce for $i\geq 2$ the estimates:
$$1\lesssim c_1\left(1-\frac{c_2}{c+\lambda_1}\right)\leq \lambda_i\left(1-\frac{\lambda_i}{c+\lambda_1}\right)\leq c_2.$$
Moreover, for $i=1$ we have 
$\lambda_1\left(1-\frac{\lambda_1}{c+\lambda_1}\right)=\frac{c\lambda_1}{c+\lambda_1}\asymp 1$. Thus $T_{\eta}$ is bounded by a numerical constant in the regime $k\asymp d$ as $d\to\infty$. Using the elementary observation $\nr(W_{\sharp}T_{\eta})\gtrsim\left(\frac{\lambda_k(T_{\eta})}{\lambda_1(T_{\eta})}\right)^2\nr(W_{\sharp})$ completes the proof.
\end{proof}

Since typical ground truth matrices $W_{\sharp}$ (e.g. Gaussian) have nuclear rank that grows with dimension, Lemma~\ref{lem:nuc_rank_grows} shows that the nuclear rank of $\nabla \mathcal{L}(W_1)$ grows with dimension as well.

\subsubsection{Multi-step and Multi-spike}\label{sec:multimodel1}
The results in the previous section established that, under a spiked model and for a specific choice of stepsize, the nuclear rank of the gradient is initially small at $W_0$ and becomes large already at $W_1$. There are, however, two serious simplifications in this setup. First, we assumed a stepsize of the form $\eta = 1/(c+\|B\|_{\op})$, whereas in practice it is more natural to consider the choice $\eta = 1/(c\|B\|_{\op})$ with $c \ge 1$. Second, we worked with a single spike $uu^\top$, while random feature Gram matrices in applications often exhibit multiple prominent directions (e.g, \cite{benigni2022largest,wang2024nonlinear} ). In this subsection, we show that both of these limitations can be removed simultaneously: under a multi-spiked model and with the more realistic stepsize $\eta = 1/(c\|B\|_{\op})$, the same qualitative conclusions on the growth of nuclear rank continue to hold. The main difference is that the phenomenon no longer occurs at the very first update; instead, after a short burn-in period, gradient descent produces iterates $W_t$ for which $\nabla \mathcal{L}(W_t)$ has large nuclear rank over a long time window.

\begin{assumption}[Multi-spiked model]\label{assump:multispiked}{\rm
Suppose that we may write
    \begin{equation}\label{eqn:spike_bulk}
B=S+\Sigma,
\end{equation}
for some symmetric matrices $S,\Sigma\in\R^{k\times k}$ with $S$ having rank $r\in \mathbb{N}$ and satisfying $c_1d^{\ell}I_k \preceq  S \preceq c_2 d^{u} I_k$ and $c_1 I_k\preceq \Sigma \preceq c_2 I_k$ for some constants $0< c_1\leq c_2<\infty$ and $\ell,u>0$.}
\end{assumption}

In words, Assumption~\ref{assump:multispiked} asserts that the matrix $B$ can be decomposed
into a low-rank ``signal'' part $S$ and a full-rank, well-conditioned ``bulk''
part $\Sigma$. The matrix $S$ has rank $r$ and eigenvalues that grow
polynomially with $d$, between $c_1 d^{\ell}$ and $c_2 d^{u}$, so there are only
$r$ directions in feature space along which $B$ is very large. In contrast,
$\Sigma$ has all of its eigenvalues uniformly bounded above and below by
numerical constants, so it behaves like a well-conditioned background covariance
on the remaining $k-r$ directions. Thus, this assumption generalizes the
rank-one spiked model considered earlier: instead of a single prominent
direction $u$, we now allow a low-dimensional subspace of ``spiky'' directions
with substantially larger eigenvalues than the bulk. The parameters $\ell$ and
$u$ control how quickly the spike eigenvalues can grow, while
a small rank $r$ ensures that the spiky subspace remains negligible in
dimension compared to the overall ambient dimension.

We now turn to the dynamics of gradient descent with stepsize $\eta$, given by
\[
W_{t+1} \;=\; W_t - \eta \nabla \mathcal{L}(W_t)
\;=\; W_t - \eta (W_t - W_\sharp)B.
\]
Recall that we initialize the algorithm at the all zero-matrix $W_0$. Introducing the error matrix $E_t := W_t - W_\sharp$, this recursion becomes
\[
E_{t+1} = E_t(I - \eta B)=E_0 (I - \eta B)^t.
\]
Substituting this into the expression for the gradient, we obtain the convenient expression
\[
\nabla \mathcal{L}(W_t)
= (W_t - W_\sharp)B
= -\,W_\sharp \underbrace{(I - \eta B)^t B}_{=:T_{\eta,t}}.
\]

The following is the main result of the section. It shows that under Assumption~\ref{assump:multispiked}, after a short burn-in period of $d^{u-\ell}\log(d)$,
the gradient's nuclear rank is lower-bounded by the dimension $\nr(\nabla \mathcal{L}(W_t))\gtrsim d$ until iteration $t=Cd^{u}$ where $C$ is an arbitrary numerical constant.\footnote{Note in this statement, we choose $C>0$ first and then let $d$ tend to infinity.} In particular, in the setting $\ell=u$ the burn-in period is of logarithmic size $\log(d)$ while the time-window of the large nuclear ranks is linear in $d$. To get a better sense of the scaling, recall that gradient descent will find an $\varepsilon$-optimal solution to the problem after order $d^u\log(\frac{d}{\varepsilon})$ many steps. Thus, if we treat the accuracy $\varepsilon$ as being constant, the time window when the nuclear rank of the gradient is large (order $d$) is shorter than the full run of the algorithm only by a $\log(d)$ factor.

\begin{thm}[Multi-step gradient descent]\label{thm:multimodel1}
Suppose that Assumption~\ref{assump:multispiked} holds. Fix $\eta=\frac{1}{c\|B\|_{\op}}$ for any constant $c\geq 1$ and suppose that $\|W_{\sharp}\|_{op}\lesssim 1$. Let $\overline{M}\in \R^{k\times (k-r)}$ be the matrix consisting of the last $k-r$ columns of the matrix product $W_{\sharp}U$, where $U\in \R^{k\times k}$ is the orthogonal factor in the (decreasing) ordered eigenvalue decomposition of ${B}$.
Then there exists a constant $C_1>0$ such that for any constant $C_2>0$
the two estimates hold:
\begin{equation}\label{eqn:key_est}
\begin{aligned}
\|\nabla \mathcal{L}(W_t)\|_{*}&\gtrsim \|\overline{M}\|_*-r\\
\|\nabla \mathcal{L}(W_t)\|_{F}&\lesssim \|\overline{M}\|_F+\sqrt{r}.
\end{aligned}
\end{equation}
for any $t$ satisfying $C_1d^{u-\ell}\log(d)\leq t\leq C_2 d^u$.
In particular, if $W_{\sharp}$ has iid Gaussian entries $\mathcal{N}(0,\frac{1}{d})$ independent of $A$, then in the regime $m\asymp k\asymp d$ there exists a constant $C_3$ such that conditionally on the realizations of $A$, with probability tending to one as $d$ tends to infinity, the estimate
$$
\nr(\mathcal{L}(W_t))\geq C_3d\qquad \textrm{holds for any}~ t~ \textrm{satisfying}~C_1d^{u-\ell}\log(d)\leq t\leq C_2 d^u.$$
\end{thm}

\begin{proof}
The eigenvalues of $T_{\eta,t}$ have the form $\lambda_i\left(1-\frac{\lambda_i}{c\lambda_1}\right)^t$ where $\lambda_i$ is the $i$'th largest eigenvalue of $B$. Let us next describe the asymptotic growth of the eigenvalues of $B$.
Clearly, we have 
$\lambda_1(B)\lesssim d^{u}$ and the condition $B\succeq S$ ensures $\lambda_{r}(B)\gtrsim \lambda_{r}(S)\gtrsim d^{\ell}$. Now the rank $r$ interlacing theorem show that for every $j$ with $1\leq j\leq k-r$ we have 
$$\lambda_{j+r}(\Sigma)\leq \lambda_j(B)\leq \lambda_{j-r}(\Sigma).$$
In particular, we deduce $\lambda_{r+1}(B)\leq \lambda_1(\Sigma)\lesssim 1.$ Moreover, the condition $B\succeq\Sigma$ ensures $\lambda_{k}(B)\geq \lambda_{k}(\Sigma)\gtrsim 1$. Summarizing, we have 
$$d^{\ell}\lesssim\lambda_i(B)\lesssim d^{u}\quad \forall i\in [1:r] \qquad \textrm{and}\qquad c_1\leq \lambda_i(B)\leq c_2 \quad \forall i\in[r+1:k].$$
Let us now look at how the eigenvalues of $T_{\eta}$ depend on the iteration counter $t$. For any $i\in [1:r]$, using the standard inequality $1-x\leq \exp(-x)$, we have 
\begin{equation}\label{eqn:upper_bd_needed}
\lambda_i\left(1-\frac{\lambda_i}{c\lambda_1}\right)^t\leq e^{u\log(d)-t\Omega(d^{\ell-u})}
\end{equation}
Consequently, as soon as $t\geq cd^{u-\ell}\log(d)$ for a sufficiently large constant $c$, the right side of \eqref{eqn:upper_bd_needed} is bounded by a numerical constant uniformly over all $i\in [i:r]$. Now for all $i\in [r+1:k]$, using the inequality $1-x\geq \exp(-2x)$ for $x\in [0,0.8]$,
we have the lower bound
\begin{equation}\label{eqn:lower_bd_needed}
\lambda_i\left(1-\frac{\lambda_i}{c\lambda_1}\right)^t\geq c_1e^{-O(2td^{-u})}.
\end{equation}
Consequently, as long as $t\leq Cd^u$ for any constant $C$, the right side of \eqref{eqn:lower_bd_needed} is lower bounded by a positive numerical constant uniformly over all $i\in [r+1:k]$. 
Let us now form an eigenvalue decomposition of $T_{\eta}=U\Lambda U^{\top}$ where the diagonal matrix $\Lambda$ contains the eigenvalues of $T_{\eta}$ according to the increasing order of $\lambda_i$. We now decompose $U$ as $U=\begin{bmatrix} U_1&U_2\end{bmatrix}$ where $U_1$ has $r$ columns and $U_2$ has $k-r$. Similarly we decompose $\Lambda$ into two diagonal matrices $\Lambda_1$ and $\Lambda_2$. We then compute
\begin{align*}
\|W_{\sharp}T_{\eta}\|_*&\geq \|W_{\sharp} U_2\Lambda_2U_2^{\top}\|_*-\|W_{\sharp} U_1\Lambda_1U_1^{\top}\|_*\\
&\gtrsim \|W_{\sharp} U_2\Lambda_2\|_*-\|U_1\Lambda_1U_1^{\top}\|_*\\
&\gtrsim \|W_{\sharp} U_2\Lambda_2\|_*-r\\
&\gtrsim \|W_{\sharp} U_2\|_*-r.
\end{align*}
Similarly we compute 
\begin{align*}
\|W_{\sharp}T_{\eta}\|_F&\leq \|W_{\sharp} U_2\Lambda_2U_2^{\top}\|_F+\|W_{\sharp} U_1\Lambda_1U_1^{\top}\|_F\lesssim \|W_{\sharp }U_2\|_F+\sqrt{r}.
\end{align*}
Suppose now that $W_{\sharp}$ has iid Gaussian entries $\mathcal{N}(0,\frac{1}{d})$ independent of $A$. Note that $\overline M\in \R^{k\times r}$ has iid Gaussian entries $\mathcal{N}(0,\frac{1}{d})$ and therefore conditionally on realizations of $A$ we have $\|\overline{M}\|_{\op}= O_{d,\PP}(1),\|\overline{M}\|_F= O_{d,\PP}(\sqrt{d})$, and $\|\overline{M}\|_*= \Omega_{d,\PP}(d)$. The result follows.
\end{proof}

\begin{remark}[Extending the window to $d^u\log d$]{\rm
In Theorem~\ref{thm:multimodel1} we showed that, under
Assumption~\ref{assump:multispiked} and the Gaussian ground-truth model for $W_\sharp$,
there exist constants $C_1,C_2,C_3>0$ such that with probability tending to one,
\[
\nr\bigl(\nabla\mathcal L(W_t)\bigr)\;\ge\;C_3 d
\qquad
\text{for all }t\in[C_1 d^{u-\ell}\log d,\;C_2 d^u].
\]
Thus on this window the gradient nuclear rank is actually linear in $d$.

The same bulk-versus-spike analysis used in the proof of
Theorem~\ref{thm:multimodel1} yields a slightly larger window (by a logarithmic factor) of large nuclear ranks. Namely, for any
$\varepsilon>0$ there exist constants $c_0$ and $c=c(\varepsilon)>0$ such that, with probability
tending to one as $d\to\infty$, we have
\[
\nr\bigl(\nabla\mathcal L(W_t)\bigr)\;\ge\; d^{1-\varepsilon}
\qquad\text{for all }t\in\bigl[c_0d^{u-\ell}\log d,\;c\,d^u\log d\bigr].
\]
The only required modification to the proof is to realize that by setting the end horizon to be $cd^{u}\log(d)$ for a sufficiently small constant $c(\varepsilon)$, we can be sure that the right side of \eqref{eqn:lower_bd_needed} is lower bounded by $d^{-\varepsilon}$. The rest of the proof proceeds unchanged.}
\end{remark}

\subsection{Random feature model II}
Next, we consider a modification of the random feature model \eqref{eqn:rfm_app} corresponding to a ``teacher-student setup'':
\begin{equation}\label{eqn:rfm_app_second}
    \min_{W\in \mathbb{R}^{m\times k}}~ \mathcal{L}(W):=\tfrac{1}{2n}\|WA-Y\|_F^2\qquad \textrm{where}\qquad Y={\overline{W}}\,{\overline{A}},~~\overline{A}=\sigma(\overline{V}X),~~A=\sigma(VX).
\end{equation}
Here, the dimensions of the involved matrices are $W,\overline{W}\in \R^{m\times k}$ and $V,\overline{V}\in \R^{k\times d}$ and $X\in \R^{d\times n}$. Problem~\eqref{eqn:rfm_app} corresponds exactly to the setting $V=\overline{V}$, that is when the first layer weights have been learned perfectly and only the second layer weights need to be trained. 
As in the previous section, we will be interested in the proportional regime $$k\asymp d\qquad \textrm{as }\quad d\to \infty.$$ 
Analogously to Assumption~\ref{assump:spiked}, we will assume that both matrices $\tfrac{1}{n}AA^{\top}$ and $\tfrac{1}{n}\overline{A}A^\top$ follow a spiked model in the sense of Assumption~\ref{assump:spiked2}. Note that we stipulate that the spike for both matrices has the same right singular vector.

\begin{assumption}[Spiked model]\label{assump:spiked2}{\rm
Suppose that we may write
   \begin{equation}\label{eqn:spike_bulk2}
\begin{aligned}    
 \tfrac{1}{n}A\, A^{\top}=uu^\top+\Gamma\qquad \textrm{and}\qquad
 \tfrac{1}{n}\overline A\, A^{\top}=vu^\top+\Sigma,
\end{aligned}
\end{equation}
for some vector $u,v\in\R^{k}$ satisfying $\|v\|_2\asymp \|u\|_2\asymp \sqrt{d}$ and matrices $\Gamma,\Sigma\in\R^{k\times k}$. Suppose moreover: 
\begin{itemize}
    \item $\Gamma$ is symmetric and for some constants $0< c_1\leq c_2<\infty$, we have $c_1 I_k\preceq \Gamma \preceq c_2I_k$,
    \item there exists $r\asymp d$ satisfying  
    $c_2\geq s_1(\Sigma)\geq s_r(\Sigma)\geq c_1$. 
\end{itemize}
}
\end{assumption}

In particular, as we have already seen, the rank one interlacing theorem \cite[Chapter 4]{HornJohnson2013} implies separation of eigenvalues
$$\lambda_1(\tfrac{1}{n}{A}\,A^{\top})\asymp d\qquad \textrm{and}\qquad \lambda_{i}(\tfrac{1}{n}{A}\,A^{\top})\in [c_1,c_2]\qquad \forall i\geq 2.$$
Weyl inequality \cite[7.3.P16]{HornJohnson2013} and the interlacing theorem for singular values \cite[Theorem 1]{thompson1976behavior} imply:
\begin{align*}
&s_1(\tfrac{1}{n}\overline{A}\,A^{\top})\asymp d\\
&s_i(\tfrac{1}{n}\overline{A}\,A^{\top})\in [c_1,c_2]\quad \forall i\in [2:r].
\end{align*}
The main example of spiked data matrices for us is the post-activation matrix generated from random data. This is the content of the following section.

\subsubsection{ReLU post-activation matrices follow the spiked model}
Consider the matrices $A=\sigma(VX)$ and $\overline{A}=\sigma(\overline{V}\, X)$, where the matrices $V,\overline{V}\in \R^{k\times d}$ have iid Gaussian entries $\mathcal{N}(0,\frac{1}{d})$ and $X\in \R^{d\times n}$ has iid standard Gaussian entries $\mathcal{N}(0,1)$. Recall as before $\sigma$ is the ReLU activation function. We will show that Assumption 
\ref{assump:spiked2} indeed holds with probability tending to one in the proportionate regime $d\asymp k$. Since the requisite properties of $AA^T$ have already been established in Theorem~\ref{th:gauss_spike1}, it only remains to analyze $\overline{A}\,A^{\top}$. To this end, we first focus on the expectation of $\overline{A}\,A^T$.
Define $a:=\sigma(Vx)$ and $\overline a:=\sigma(\overline{V}x)$ where $x\sim \mathcal{N}(0,I_d)$ is a standard Gaussian random vector.
Then, we may write
$$\EE_x\left[\frac{1}{n}\overline{A}\,A^{\top}\right]=\EE_x[\overline{a}a^\top]=\mu\nu^\top+\Sigma,$$
where we define
$$\nu:=\EE_x[\overline a],\quad \mu= \EE_x[a]\qquad \textrm{and}\qquad \Sigma:=\EE_x(\overline a-\mu)(a-\nu)^{\top}.$$
Letting $v_i$ and $\overline{v}_i$ denote the rows of V and $\overline{V}$, respectively, a standard computation shows the explicit formulas:
\begin{align*}
\mu&=\frac{1}{2\pi}\begin{bmatrix}\|v_1\|_2&\ldots &\|v_k\|_2\end{bmatrix}^{\top},\\
\nu&=\frac{1}{2\pi}\begin{bmatrix}\|\overline{v}_1\|_2&\ldots &\|\overline{v}_k\|_2\end{bmatrix}^{\top},\\
\Sigma_{ij}
&= \frac{\|\overline{v}_i\|_2\|{v}_j\|_2}{2\pi}
\varphi\left(\frac{\langle \overline{v}_i,{v}_j\rangle}{\|\overline{v}_i\|_2\|{v}_j\|_2}\right),
\end{align*}
where we set $\varphi(t)=\left(
\sqrt{1-t^2}
+ (\pi - \arccos t)\, t
- 1
\right)$; see Figure~\ref{eqn:kappa_rep} for an illustration of  $\varphi$. The next theorem establishes the asymptotics of $\Sigma$ in the proportionate regime $k\asymp d$ as $d\to \infty$.
Although, this result can be directly proved using the seminar work of \cite{karoui2010spectrum,cheng2013spectrum}, we provide a short self-contained argument here based on the hypercontractivity of Gaussian random vectors.

\begin{thm}[Asymptotics]\label{thm:exp}
Consider the regime where $k$ and $d$ tend to infinity with $\sqrt{k}/d\to 0$. Assume that $V$ and $\bar V$ are matrices with iid Gaussian entries with variance $1/d$. Then for any $\epsilon>0$ it holds:
$$\left\|\Sigma-\left(\tfrac{1}{2\pi}\overline{V}V^\top+\tfrac{1}{d}{\bf 1}{\bf 1}^\top\right)\right\|_{\rm op}=O_{d,\PP}\left(\frac{k^{\tfrac{1}{2}+\epsilon}}{d}\right).$$
\end{thm}

\begin{proof}
The proof begins in the same way as the proof of Theorem~\ref{thm:incoherence}. Namely, we write $$\Sigma=\tfrac{1}{2\pi}\overline{\Lambda} K{\Lambda}$$
where the entries of $K$ are $K_{ij}=\varphi\left(\frac{\langle \overline{v}_i,v_j\rangle}{\|\overline{v}_i\|_2\|v_j\|_2}\right)$ and $\overline{\Lambda}$ and $\overline{\Lambda}$ are diagonal matrices with $\overline{\Lambda}_{ii}=\|\overline{v}_i\|$ and ${\Lambda}_{ii}=\|{v}_i\|$, respectively. We define $\overline{U}\in \R^{k\times d}$ to be a matrix whose $i$'th row is the normalized vector $\overline{v}_i/\|\overline{v}_i\|_2$ and similarly we define $U\in \R^{k\times d}$ to be a matrix whose $i$'th row is the normalized vector $v_i/\|v_i\|_2$
We will treat $K$ first and deal with $\Lambda,\overline{\Lambda}$ at the end of the argument. 
Performing the same Taylor expansion as in the proof of Theorem~\ref{thm:incoherence} we arrive at the expression \begin{equation}\label{eqn:basic_offdiag2}
K=\frac{\pi}{2}\overline{U}{U}^{\top}+\frac{1}{2}(\overline{U}{U}^{\top})^{\odot 2}+\frac{1}{4!} D\odot (\overline{U}{U}^{\top})^{\odot 4},
\end{equation}
where the entries of $D$ and of $\overline{U}{U}^{\top}$ are bounded in absolute value 
 by $c_1\sqrt{\frac{\log k}{d}}$ with probability at least $1-\frac{c_2}{k}$.
In this event, using Lemma~\ref{lem:gen_hadamard} (Item~\ref{it1}), we may bound the last term in \eqref{eqn:basic_offdiag2} as 
$$\left\| D\odot (\overline{U}{U}^{\top})^{\odot 4}\right\|_{\op}\leq \max_{ij} |D_{ij}|\cdot \left\|(\overline{U}{U}^{\top})^{\odot 4}\right\|_{F}\lesssim k\log^2(k)d^{-2}.$$
Next, we deal with the term $(\overline{U}{U}^{\top})^{\odot 2}$ in \eqref{eqn:basic_offdiag2}. To this end, let us first estimate $(\overline{V}{V}^{\top})^{\odot 2}$. We compute $\EE(\overline{V}{V}^{\top})^{\odot 2}=\frac{1}{d}{\bf 1}{\bf 1}^{\top}$. Let $z_i$ be the $i$'th row of $(\overline{V}{V}^{\top})^{\odot 2}-\frac{1}{d}{\bf 1}{\bf 1}^{\top}$.
Note that conditionally on $v_1,\ldots, v_k$, the vectors $z_i$ are independent and identically distributed. Let $\EE_v$ be conditional expectation on $v_1,\ldots, v_k$.
Then using \cite[Theorem 5.48]{vershynin2010introduction}, we deduce
$$\sqrt{\EE_v\|(\overline{V}{V}^{\top})^{\odot 2}-\tfrac{1}{d}{\bf 1}{\bf 1}^{\top}\|^2_{\op}}\leq \left\|\Sigma\right\|^{1/2}_{\op}d^{1/2}+Cm^{1/2}\log^{1/2}(\min\{k,d\})$$
where $\Sigma=\EE_v z_iz_i^{\top}$ is the conditional second moment for any index $i$ and $\displaystyle m=\EE_v[\max_{i=1,\ldots,k}\|z_i\|_2^2]$. Now a straightforward computation shows that the $(j,l)$ entry of $\Sigma$ is given by $\frac{\|v_j\|^2_2\|v_l\|^2_2+2\langle v_j,v_l\rangle^2-\|v_j\|_2^2-\|v_l\|^2_2+1}{d^2}$. Thus we may write
$$\Sigma=\frac{1}{d^2}\left(2(VV^\top)^{\odot 2}+(q-{\bf 1})(q-{\bf 1})^\top\right),$$
where we define the vector $q=(\|v_1\|^2_2,\ldots,\|v_k\|^2_2)$. Thus we deduce
$$\|\Sigma\|_{\op}\leq \frac{2\|(VV^\top)^{\odot 2}\|_{\op}+\|q-{\bf 1}\|_2^2}{d^2}.$$
Now elementary algebra shows $\EE\|q-{\bf 1}\|^2=\frac{2k}{d}$.
Now we break up $(VV^\top)^{\odot 2}$ into the diagonal and off-diagonal parts:
\begin{align*}
(VV^\top)^{\odot 2}&=\|\Diag(q^2)+((VV^\top)^{\odot 2}-\Diag(q^2))\|_{\op}\\
&\leq \|\Diag(q^2)\|_{\op}+\|((VV^\top)^{\odot 2}-\Diag(q^2))\|_{F}\\
&\leq \max_{i}\|v_i\|^4_2+\sqrt{\sum_{i\neq j} \langle v_i,v_j\rangle^4}. 
\end{align*}
Now simple algebra shows $\EE \sum_{i\neq j} \langle v_i,v_j\rangle^4\lesssim \frac{k^2}{d^2}$ and a standard concentration argument shows $\EE \max_{i}\|v_i\|^4_2\lesssim \left(1+\frac{\log k}{d}\right)^2.$ Thus we deduce
$\EE\|\Sigma\|_{\op}\lesssim \frac{1+\frac{k}{d}}{d^2}$.
Next, we compute 
\begin{align*}
m&=\EE_v[\max_{i=1,\ldots,k}\|z_i\|_2^2]\leq [\EE_v\max_{i=1,\ldots,k}\|z_i\|_2^{2p}]^{1/p}\leq [\EE_v \sum_{i=1}^k\|z_i\|_2^{2p}]^{1/p}=k^{1/p}(\EE_v \|z_i\|_2^{2p})^{1/p}.
\end{align*}
Observe that $\|z_i\|_2^2$ is a a degree 4 polynomial of a Gaussian and therefore by Gaussian hypercontractivity \cite[Chapter 3.2]{ledoux2013probability} satisfies 
$(\EE_v \|z_i\|_2^{2p})^{1/p}\leq (p-1)^2 (\EE_v \|z_i\|^4_2)^{1/2}$. Taking the expectation now with respect to $v_1,\ldots, v_k$, we deduce 
$\EE m\leq k^{1/p}(p-1)^2 (\EE \|z_i\|^4_2)^{1/2}\lesssim k^{1/p}(p-1)^2 \frac{k}{d^2}$. Thus we deduce
$$\EE\|(\overline{V}{V}^{\top})^{\odot 2}-\tfrac{1}{d}{\bf 1}{\bf 1}^{\top}\|^2_{\op}\lesssim \frac{1+\frac{k}{d}}{d^2}+k^{2/p}(p-1)^4\frac{k\log(\min\{k,d\})}{d^2}.$$
Applying Markov's inequality completes the proof. 
\end{proof}

Finally, we pass from the expectation $\EE[\overline{A}\,A^{\top}]$ to its empirical version
$\frac{1}{n}\overline{A}\,A^{\top}$.

\begin{cor}[Spiked model from Gaussian initialization]\label{th:gauss_spike2}
 Assume that $V,\bar V\in\R^{k\times d}$ are matrices with iid Gaussian entries $\mathcal{N}(0,\frac{1}{d})$ and let $X\in \R^{k\times n}$ be an independent matrix with iid Gaussian entries $\mathcal{N}(0,1)$. Then in the regime $ \tfrac{k}{d}\to\gamma\in (1,\infty)$ and $n\geq C_0d$, with probability tending to one as $d\to \infty$, Assumption~\ref{assump:spiked2} holds.
\end{cor}
\begin{proof}
Let $a_i$ and $\overline{a}_i$ denote the $i$'th rows of $A$ and $\overline{A}$, respectively. We have already seen in Theorem~\ref{th:gauss_spike1} that there exist constants $C_0,C_1,C_2>0$ such that in the regime $ k\asymp d$ and $n\geq C_0d$, with probability tending to one as $d\to \infty$, we may write 
$$\tfrac{1}{n}AA^\top=\bar \mu\bar \mu^\top+\overline\Sigma,$$
where the empirical mean $\bar \mu=\frac{1}{n}\sum_{i=1}^n a_i$ and the empirical covariance $\overline{\Sigma}=\frac{1}{n}\sum_{i=1}^n (a_i-\bar \mu)(a_i-\bar \mu)^\top$ satisfy
$ \|\bar \mu\|^2\asymp d $ and $C_1 I\preceq\overline{\Sigma}\preceq C_2 I$. Therefore we now focus on $\tfrac{1}{n}\overline{A}A^\top$. To this end, define the empirical means
\[
\bar\mu:=\frac1n\sum_{i=1}^n a_i,
\qquad
\bar\nu:=\frac1n\sum_{i=1}^n \overline a_i,
\]
and the empirical cross--covariance
\[
\overline\Sigma
:=\frac1n\sum_{i=1}^n (\overline a_i-\overline\nu)\,(a_i-\bar\mu)^\top.
\]
A simple expansion shows that
\[
\overline B
=\tfrac1n\overline A A^\top
=\overline{\nu}\,\bar\mu^\top+\overline\Sigma,
\]
so $\overline B$ already has the desired rank--one plus bulk structure with
right spike vector $\bar\mu$.

Exactly as in the proof of Theorem~\ref{th:gauss_spike1} (using Gaussian concentration), both $a$ and $\overline a$ are
subgaussian:
\[
\|\langle a-\EE a,u\rangle\|_{\psi_2}\lesssim \|u\|_2,
\qquad
\|\langle \overline a-\EE \overline a,u\rangle\|_{\psi_2}\lesssim \|u\|_2
\qquad\forall u\in\R^k.
\]
Applying the covariance concentration inequality
\cite[Remark~9.2.3]{vershynin2018high} to the joint vector
$(a^\top,\overline a^\top)^\top\in\R^{2k}$ and reading off the off-diagonal blocks yields, for $n\ge C_0 d$, the estimate
\[
\|\overline\Sigma-\Sigma\|_{\op}
\;\lesssim\;
\sqrt{\frac{d}{n}}+\frac{d}{n},
\]
with probability at least $1-2e^{-d}$.
Similarly, by \cite[Theorem~1]{hsu2012tail} we obtain
\[
\|\bar\mu-\mu\|_2\lesssim \sqrt{\tfrac{d}{n}},
\qquad
\|\overline{\nu}-\nu\|_2\lesssim \sqrt{\tfrac{d}{n}},
\]
so that $\|\bar\mu\|_2\asymp\|\bar\nu\|_2\asymp \sqrt d$ whenever $n\ge C_0 d$. Now Theorem~\ref{thm:exp} showed \[
\Bigl\|\Sigma-\Bigl(\tfrac1{2\pi}\,\overline V V^\top
+\tfrac1d{\bf 1}{\bf 1}^\top\Bigr)\Bigr\|_{\op}
=o_d(1).
\]
Now observe that 
$$\left\|\frac{1}{d}{\bf 1}{\bf 1}^{\top}-\frac{1}{d}{\overline{\nu}}\,\overline{\mu}^{\top}\right\|_{\op}\leq \left\|\frac{1}{d}{\bf 1}{\bf 1}^{\top}-\frac{1}{d}{{\nu}}\,{\mu}^{\top}\right\|_{\op}+\left\|\frac{1}{d}{\nu}{\mu}^{\top}-\frac{1}{d}{\overline{\nu}}\,\overline{\mu}^{\top}\right\|_{\op}\leq o_{d}(1)+\sqrt{\frac{d}{n}}.$$
Thus we may write
$$\tfrac1n\overline A A^\top
=(1+\tfrac{1}{d})\overline{\nu}\,\bar\mu^\top+\tfrac1{2\pi}\,\overline V V^\top+E
$$
where $\|E\|_{\op}\lesssim o_{d}(1)+\sqrt{\frac{d}{n}}$.
Now observe $\overline V\, V^{\top}(\overline V\, V^{\top})^{\top}\succeq \lambda_{d}( V^{\top}V)\overline V\,\overline V^\top$.
Thus standard results \cite[Theorem 4.6.1]{vershynin2010introduction} on the extremal singular values of tall Gaussian matrices (recall we are in the regime $k/d\to \gamma\in (1,\infty)$)  imply that
there exist constants $0<c_1\le c_2<\infty$ such that with probability tending to one as $d\to\infty$, we have 
\begin{equation}\label{eq:Sigma-pop-sv-structure}
c_2 \ \ge\ s_1(\overline V V^\top)\ \ge\ s_d(\overline V V^\top)\ \ge\ c_1.
\end{equation}
Therefore after possibly inflating the constants $c_1,c_2>0$ the bulk $Q:=\tfrac1{2\pi}\,\overline V V^\top+E$ satisfies 
$$c_2 \ \ge\ s_1(Q)\ \ge\ s_d(Q)\ \ge\ c_1\qquad \textrm{and}\qquad  s_{d+1}\lesssim o_{d,\PP}(1)+O_{d,\PP}\left(\sqrt{\frac{d}{n}}\right).$$
Thus the proof is complete.
\end{proof}

\subsubsection{The nuclear rank at initialization is constant.}
We now look at the consequences of the spiked model \eqref{eqn:spike_bulk2} on the nuclear rank of the gradient. Namely, consider initializing gradient descent on the problem~\eqref{eqn:rfm_app_second} at the all-zero matrix $W_0=0$. Then taking into account the expression for the gradient 
$$
\nabla \mathcal{L}(W)=\tfrac{1}{n}(WA-\overline W\,\overline A)A^{\top},$$
we have 
\begin{equation}\label{eq:grad_init}
\nabla \mathcal{L}(W_0)=-\overline{W}\,\overline{B}\qquad \textrm{where}\qquad \overline{B}:=\tfrac{1}{n}\overline{A}A^{\top}.
\end{equation}
Now suppose that the ground truth matrix $\overline{W}$ does not eliminate the spike in the following sense:
$$\|\overline{W}\|_{\op}\leq C \qquad \textrm{and}\qquad \|\overline{W}u\|_2\geq \kappa\cdot \|u\|_2.$$ 
This is indeed the case for example if the singular values $s_1(\overline{W})$ and $s_{k}(\overline{W})$ are bounded from above and below by a numerical constant or is true with high probability if the entries of $\overline{W}\in \R^{m\times k}$ are iid Gaussian $\mathcal{N}(0,\tfrac{1}{m})$ and $m\gtrsim k$.
Then we can estimate the nuclear rank of the gradient as:
$$\sqrt{\nr(\overline{W}\,\overline{B})}=\frac{\|\overline{W}\,\overline{B}\|_{*}}{\|\overline{W}\,\overline{B}\|_F}\leq \frac{C\|\overline{B}\|_{*}}{\|\overline{W}uv^\top\|_F- \|\overline{W}\Sigma\|_{F}}\leq \frac{C\|u\|_2\|v\|_2+C\|\Sigma\|_*}{\kappa\|u\|_2\|v\|_2-C\|\Sigma\|_F}\lesssim \frac{d}{d-\sqrt{d}}\asymp 1.$$
Thus indeed the nuclear rank of the gradient $\nabla \mathcal L(W_0)$ is  bounded by a constant.

\subsubsection{The nuclear rank after one gradient step scales with dimension.}
Let us now analyze what happens after a single gradient step $W_{1}=W_0-\eta \nabla \mathcal{L}(W_0)$. Observe the gradient of the objective function at $W_1$ is 
$$\nabla \mathcal{L}(W_1)=-\overline{W}\,\underbrace{\overline{B}(I-\eta  B)}_{=:T_{\eta}},$$
where recall that we set $\overline{B}=\tfrac{1}{n}\overline{A}\,A^{\top}$ and ${B}=\tfrac{1}{n}{A}\,A^{\top}$.
We will now show that for any stepsize $\eta=\frac{1}{c+\|B\|_{\op}}$ under natural randomness assumptions, 
the nuclear rank of $\nabla \mathcal{L}(W_1)$ grows with dimension---the main result of the section.

\begin{thm}[Nuclear rank after one iteration grows with dimension]\label{lem:key_result_buya}
Suppose that Assumption~\ref{assump:spiked2} holds and fix $\eta=\frac{1}{c+\|B\|_{\op}}$ for any constant $c>0$. Let $\overline{M}\in \R^{k\times r}$ be the matrix consisting of the first $r$ columns of $\overline{W}\,\overline{L}$, where $\overline{L}\in \R^{k\times k}$ is the left orthogonal factor in the ordered singular value decomposition of $\overline{B}$.
Then the following estimates hold in the regime $k\asymp d$ for all large $d$:
\begin{equation}\label{eqn:key_est}
\begin{aligned}
\nr(\nabla \mathcal{L}(W_1))\geq \frac{\|\overline M\|_*}{\|\overline W\|_F}.
\end{aligned}
\end{equation}
In particular, if $\overline{W}$, $\overline{V}$, $V$  have iid Gaussian entries $\mathcal{N}(0,\frac{1}{d})$, and we are in the regime $m\asymp d$ and $\tfrac{k}{d}\to\gamma\in (1,\infty)$, then there exists a constant $C$ such that as long as $n\geq C d$ we have 
\begin{equation}\label{eqn:claim_main}
\nr(\mathcal{L}(W_1))\geq \Omega_{d,\PP}(d).
\end{equation}
\end{thm}
\begin{proof}
Let us write the singular value decompositions
$$B=L S L^{\top}\qquad \textrm{and}\qquad \overline{B}=\overline{L}\,\overline{S}\, \overline{R}^{\top},$$
where $S,\overline{S}\in \R^{k\times k}$ are diagonal matrices with non-negative decreasing entries (singular values),  and $L,\overline{L}\in \R^{k\times k}$ and $\overline{R}\in \R^{k\times k}$ are orthogonal matrices. We then compute
$$\overline{L}^{\top}T_{\eta}L=\overline{S}\, \overline{R}^{\top}L(I-\eta S).$$
The goal is now to show that the product $\overline{R}^{\top}L$ nearly has a block form. To this end, let us write 
$$\overline R=\begin{bmatrix}r_0& R_0\end{bmatrix}\qquad \textrm{and}\qquad L=\begin{bmatrix} \ell_0&L_0\end{bmatrix}$$
for some vectors  $r_0,\ell_0\in \R^{k}$. Now the Davis-Kahan Theorem \cite{davis1970rotation} applied to the two matrices $AA^{\top}$ and $uu^{\top}$ shows that the corresponding top eigenvectors are close 
\begin{equation}\label{eqn:dk1}
\argmin_{q\in \{\pm 1\}}\|\hat{u}-q \ell_0\|_2\lesssim \frac{1}{d},
\end{equation}
where we set $\hat u=u/\|u\|_2$. Without loss of generality, we may assume the optimal orientation is $q=1$. Similarly, applying the Davis-Kahan theorem for singular values \cite{wedin1972perturbation} to the right singular vectors of $\overline{A}A^{\top}$ and $vu^{\top}$, we deduce 
\begin{equation}\label{eqn:dk2}
\argmin_{q\in \{\pm 1\}}\|\hat{u}-q r_0\|_2\lesssim \frac{1}{d}.
\end{equation}
Again without loss of generality, we may assume that  $q=1$ is optimal. Let us now write $\overline{R}^{\top}L$ in block form:
$$\overline{R}^{\top}L=\begin{bmatrix} r_0^{\top}\ell_0 & r_0^{\top}L_0\\
R_0^{\top}\ell_0 & R_0^{\top}L_0
\end{bmatrix}$$
Now using equations~\eqref{eqn:dk1} and \eqref{eqn:dk2} we deduce 
$$|r_0^{\top}\ell_0-1|\lesssim \frac{1}{d},\qquad \|r_0^{\top}L_0\|_2\lesssim \frac{1}{d},\qquad \|R_0^{\top}\ell_0\|_2\lesssim \frac{1}{d}.$$
Moreover, by standard principal--angle
arguments, the singular values of \(R_0^\top L_0\) are the cosines of the
principal angles between the subspaces \(r_0^\perp\) and \(\ell_0^\perp\)
(see, e.g.,~\cite{bjorck1973numerical} or~\cite{yu2015useful}).
Thus we deduce
$$s_{k-1}(R_0^\top L_0)=|r_0^\top\ell_0|\gtrsim 1-\frac{1}{d}.$$
Thus writing $S=\Diag(s)$ and $\overline{S}=\Diag(\overline{s})$ for some vectors $s$ and $\overline s$ we deduce 
\begin{align*}
\overline{S}\, \overline{R}^{\top}L(I-\eta S)&=\Diag(1,\overline{s}_{\geq 2})\cdot\underbrace{\begin{bmatrix} \overline{s}_1 (1-\eta s_1) r_0^{\top}\ell_0 & \overline{s}_1 r_0^{\top}L_0(I-\eta\Diag( s_{\geq 2}))\\
(1-\eta s_1)R_0^{\top}\ell_0 &  R_0^{\top}L_0 (I-\eta\Diag( s_{\geq 2}))
\end{bmatrix}}_{=:Q},
\end{align*}
where $s_{\geq 2}$ and $\overline{s}_{\geq 2}$ denote, respectively, the vectors $s$ and $\bar s$ with the first coordinate removed.
Now observe $\overline{s}_1\asymp d$ and $1-\eta s_1\asymp \left(1-\frac{\|A\|^2_{\rm op}}{c+\|A\|^2_{\op}}\right)\asymp \frac{1}{d}$. Therefore we may write  
$$Q=\begin{bmatrix}
a & z^{\top}\\
0 & D
\end{bmatrix}+E,$$
where $a\asymp 1$, $\|z\|_2\lesssim 1$, $\|E\|_{\op}\lesssim \frac{1}{d^2}$, and $1\lesssim s_{k-1}(D)\leq s_1(D)\lesssim 1$. It follows immediately that the operator norm of the blocked matrix is bounded by a constant, while its minimal singular value is bounded away from zero due to the standard linear algebraic Lemma~\ref{lem:block-upper-triangular-sigma-min}. Thus we deduce
$$1\lesssim s_{k}(Q)\lesssim s_{1}(Q)\lesssim 1.$$
We therefore deduce 
$$\nr(\nabla \mathcal{L}(W_1))=\nr(\overline{W}\,\overline{L}\,\overline{L}^{\top}T_{\eta}{L})=\nr(\overline{W}\,\overline{L}\Diag(1,\overline{s}_{\geq 2}) Q)\asymp \nr(\overline{W}\,\overline{L}\Diag(1,\overline{s}_{\geq 2})).$$
Using the linear algebra Lemma~\ref{lem:lower_bd_nuc} with $P=[1:r]$ we immediately deduce 
$$\|\overline{W}\,\overline{L}\Diag(1,\overline{s}_{\geq 2})\|_*\gtrsim \|\overline M\|_*,$$
while the inequality $\|\overline{W}\,\overline{L}\Diag(1,\overline{s}_{\geq 2})\|_F\lesssim \|\overline W\|_F$ holds trivially. The claimed estimate~\eqref{eqn:key_est} follows immediately.

Suppose now that $\overline{W},V,\overline{V}$  have iid Gaussian entries $\mathcal{N}(0,\frac{1}{d})$. Then we have already seen in Corollary~\ref{th:gauss_spike2} that Assumption~\ref{assump:spiked2} is satisfied with $r=d$. Note that $\overline M\in \R^{k\times r}$ has iid Gaussian entries $\mathcal{N}(0,\frac{1}{d})$ and therefore we have $\|\overline W\|_F= O_{d,\PP}(\sqrt{d})$, $\|\overline{M}\|_*= \Omega_{d,\PP}(d)$, and $\|\overline{W}\|_{\rm op}= O_{d,\PP}(1).$ 
The claimed estimate \eqref{eqn:claim_main} follows immediately.
\end{proof}

\subsubsection{Multi-step and multi-spike for the teacher--student model}

Theorem~\ref{lem:key_result_buya} shows that, under the
rank-one spiked model of Assumption~\ref{assump:spiked2}, the nuclear rank of the gradient
in the teacher--student random feature model jumps from $O(1)$ at initialization to order
$d$ after a single gradient step. In the realizable model, we further saw in Section~\ref{sec:multimodel1}
that a similar phenomenon persists for multi-spiked Gram matrices and over many
gradient steps (Theorem~\ref{thm:multimodel1}).
The goal of this section is to develop an analogous multi-step, multi-spike theory in
the teacher--student setting.
The main technical difference compared to the realizable case is that the gradient now
contains the cross-Gram matrix
$
\overline B := \frac{1}{n}\,\overline A A^\top,
$
rather than the symmetric Gram matrix $B = \tfrac{1}{n}AA^\top$ alone. As a result, the
gradient at time~$t$ takes the form $\overline W\,\overline B$ multiplied by a polynomial
in~$B$, and the relevant spectrum is that of the non-symmetric matrix
$\overline B(I-\eta B)^t$. We will see, however, that under a natural multi-spiked
generalization of Assumption~\ref{assump:spiked2}, the same qualitative picture holds:
after a short burn-in period the gradient nuclear rank is of order~$d$ for a long time
window, so that the spectral-versus-Euclidean advantage is again visible.

We continue to work in the proportional regime $k\asymp d$ as $d\to\infty$ and we use the
same notation as in the rest of this section. In particular,
\[
B := \frac1n AA^\top \in\R^{k\times k},
\qquad
\overline B := \frac1n \overline A A^\top \in\R^{k\times k},
\]
and gradient descent on the objective~\eqref{eqn:rfm_app_second} satisfies
\[
W_{t+1} = W_t - \eta \nabla \mathcal L(W_t),
\qquad
\nabla \mathcal L(W_t) = \frac1n (W_tA - \overline W\,\overline A)A^\top
= W_t B - \overline W\,\overline B.
\]

We first formulate a multi-spiked analogue of Assumption~\ref{assump:spiked2}. As in the
realizable case, we assume that $B$ decomposes into a low-rank ``signal'' part and a
well-conditioned bulk, as in Assumption~\ref{assump:multispiked}, and we impose a
compatible structure on the cross-Gram matrix $\overline B$.

\begin{assumption}[Multi-spiked teacher--student model]\label{assump:multispiked2}{\rm
Suppose that we may write
   \begin{equation}\label{eqn:spike_bulk2}
\begin{aligned}    
 \tfrac{1}{n}A\, A^{\top}=H_1+\Gamma\qquad \textrm{and}\qquad
 \tfrac{1}{n}\overline A\, A^{\top}=H_2+\Sigma,
\end{aligned}
\end{equation}
for some matrices $H_1,H_2\in\R^{k\times k}$. Suppose that the following conditions are true.
\begin{itemize}
\item {\bf Spike properties} $H_1$ and $H_2$ have rank $p=o(d)$, the matrix $H_1$ is symmetric positive semidefinite, and the singular values of $H_1$ and $H_2$ lie in the interval $[d^{\ell},d^u]$ for some constants $u,\ell\in(0,1]$. Moreover, $H_1$ and $H_2$ admit compact singular value decompositions
$$H_1=UQU^\top\qquad \textrm{and}\qquad H_2=V\overline{Q}U^\top$$
for orthogonal matrices $U,V\in\R^{k\times p}$ and nonnegative diagonal matrices $Q,\overline{Q}\in \R^{p\times p}$ with non-increasing diagonal entries. Note that the right orthogonal factors of $H_1$ and $H_2$ coincide.
    \item {\bf Bulk properties.} The matrix $\Gamma$ is symmetric and for some constants $0< c_1\leq c_2<\infty$, we have $c_1 I_k\preceq \Gamma \preceq c_2I_k$. Additionally, there exists $r\asymp d$ satisfying  
    $c_2\geq s_1(\Sigma)\geq s_r(\Sigma)\geq c_1$. 
\end{itemize}
}
\end{assumption}

In words, Assumption~\ref{assump:spiked2} ensures that the matrices $\tfrac{1}{n} A A^\top$ and $\tfrac{1}{n} \overline A A^\top$ decompose into a low-rank \emph{signal} (or ``spike'') part $H_1,H_2$ plus a \emph{bulk} part $\Gamma,\Sigma$.  
The spike matrices $H_1,H_2$ have rank $p=o(d)$, so the signal concentrates in a low-dimensional subspace compared to the ambient dimension, and their nonzero singular values grow like a power of $d$; moreover, the right singular vectors of $H_1$ and $H_2$ in an ordered SVD coincide, so both spikes are aligned in the same latent directions.  
The bulk matrix $\Gamma$ is well-conditioned (its eigenvalues are bounded above and below by fixed constants), while $\Sigma$ has $r\asymp d$ singular values of constant order.  
Overall, this is a multi-spike teacher–student model in which a small number of strong signal directions are embedded in an otherwise well-behaved high-dimensional bulk. In particular it is straightforward to check that Assumption~\ref{assump:spiked2} is a strict generalization of the
rank-one teacher--student model in Assumption~\ref{assump:spiked2}.

We begin analyzing gradient descent in the multi-spiked student-teacher setting by deriving a simple expression for the gradient of $\mathcal{L}$ along the iterates. This is the content of the following lemma which parallels the realizable case.

\begin{lem}[Gradient recursion]\label{lem:grad-recursion-ts}
Let $W_{t+1}=W_t-\eta\nabla\mathcal L(W_t)$ be the gradient descent iterates on the
objective~\eqref{eqn:rfm_app_second}, initialized at $W_0=0$. Then the gradients satisfy
\[
\nabla\mathcal L(W_t)
= -\,\overline W\,\overline B\,(I-\eta B)^t
\quad\text{for all }t\ge 0.
\]
\end{lem}
\begin{proof}
Recall that $\nabla\mathcal L(W) = W B - \overline W\,\overline B$.
Define $G_t := \nabla\mathcal L(W_t)$. Then
\[
G_0 = \nabla\mathcal L(W_0)
= W_0B -\overline W\,\overline B
= -\overline W\,\overline B.
\]
Moreover, using the update $W_{t+1}=W_t-\eta G_t$ we obtain
\[
G_{t+1}
= W_{t+1}B - \overline W\,\overline B
= (W_t-\eta G_t)B - \overline W\,\overline B
= G_t - \eta G_t B
= G_t(I-\eta B).
\]
By induction this yields
$G_t = G_0 (I-\eta B)^t = -\overline W\,\overline B\,(I-\eta B)^t$ for all $t\ge 0$.
\end{proof}

Thus, the time-$t$ gradient is obtained by multiplying the initial gradient
$-\overline W\,\overline B$ on the right by the polynomial $(I-\eta B)^t$. In particular,
the spectral properties of the family of matrices
\[
T_{\eta,t} := \overline B(I-\eta B)^t
\]
control the evolution of $\nabla\mathcal L(W_t)$. We are now ready to estimate the nuclear rank of the gradient after a short burn-in period. This is the content of the following theorem, which closely parallels Theorem~\ref{thm:multimodel1} in the realizable case. Henceforth, for any matrix $M$ we let $M_{i:j}$ denote the submatrix formed by columns $i$ through $j$.

\begin{thm}[Nuclear rank for multi-step GD]\label{thm:main_thrmgrfjishdfklds}
Suppose that Assumption~\ref{assump:spiked2} holds and that $\|\overline{W}\|_{\op}\lesssim 1$. Let $\overline{M}$ consist of the columns $p+1$ through $r-p$ of the matrix ${\overline W}\, \overline{L}$, where $\overline{L}$ is the left orthogonal factor in the ordered (decreasing)  singular values decomposition of $\overline{B}$. Then for any constants $q,C_3>0$ there exists a constant constant $C_1>0$ such that the following estimates hold 
\begin{align*}
\|\nabla \mathcal{L}(W_t)\|_{*}&\gtrsim  \|\overline{M}\|_{*}-pd^{u-\ell}-d^{-q}\\
\|\nabla \mathcal{L}(W_t)\|_{F}&\lesssim \|\overline{W}\|_{F}+\sqrt{p}d^{u-\ell}+d^{-q},
\end{align*}
for all $t$ in the interval $[C_1d^{u-\ell}\log(d), C_3 d^{u}]$.
\end{thm}

\begin{proof}
First, recall that eigenvalue interlacing directly implies 
$$d^{\ell}\lesssim\lambda_i(B)\lesssim d^{u}\quad \forall i\in [1:r] \qquad \textrm{and}\qquad c_1\leq \lambda_i(B)\leq c_2 \quad \forall i\in[r+1:k].$$
Moreover, interlacing of singular values \cite{thompson1976behavior} shows:
$$s_{i+p}(\Sigma)\leq s_{i}(\overline B)\leq s_{i-p}(\Sigma).$$
In particular, we deduce
$$s_{i}(\overline{B})\in [c_1,c_2]\qquad \forall i\in [p+1,r-p].$$
Let us now write the singular value decompositions
$$B=L S L^{\top}\qquad \textrm{and}\qquad \overline{B}=\overline{L}\,\overline{S}\, \overline{R}^{\top},$$
where $S,\overline{S}\in \R^{k\times k}$ are diagonal matrices with non-negative decreasing entries (singular values),  and $L,\overline{L}\in \R^{k\times k}$ and $\overline{R}\in \R^{k\times k}$ are orthogonal matrices. We then compute
$$\overline{L}^{\top}T_{\eta,t}L=\overline{S}\, \overline{R}^{\top}L(I-\eta S)^t.$$
The goal is now to show that the product $\overline{R}^{\top}L$ nearly has a block form. To this end, let us write 
$$\overline R=\begin{bmatrix}R_1& R_2\end{bmatrix}\qquad \textrm{and}\qquad L=\begin{bmatrix} L_1&L_2\end{bmatrix}$$
for some matrices  $R_1,L_1\in \R^{k\times p}$ and $R_2,L_2\in \R^{k\times (k-p)}$. Now the Davis-Kahan Theorem \cite{davis1970rotation} applied to the two matrices $AA^{\top}$ and $UQU^{\top}$ shows that  
\begin{equation}\label{eqn:dk1new}
\|UU^\top-L_1L_1^{\top}\|_{\op}\leq \frac{\|\Gamma\|_{\op}}{d^{\ell}}\lesssim \frac{1}{d^{\ell}},
\end{equation}
Applying the Davis-Kahan theorem for singular values \cite{wedin1972perturbation} to the right singular vectors of $\overline{A}A^{\top}$ and $V\overline{Q}U^{\top}$, we deduce 
\begin{equation}\label{eqn:dk2new}
\|UU^{\top}-R_1R_1^\top\|_{\op}\leq \frac{\|\Sigma\|_{\op}}{d^{\ell}}\lesssim \frac{1}{d^{\ell}}.
\end{equation}
Let us now write $\overline{R}^{\top}L$ in block form:
$$\overline{R}^{\top}L=\begin{bmatrix} R_1^{\top}L_1 & R_1^{\top}L_2\\
R_2^{\top}L_1 & R_2^{\top}L_2
\end{bmatrix}.$$
Let us now partition $S=\begin{bmatrix}S_1& 0\\0 & S_2\end{bmatrix}$ and $\overline{S}=\begin{bmatrix}\overline{S}_1& 0\\0 & \overline{S}_2\end{bmatrix}$ in the obvious way. 
Thus we deduce
\begin{equation}\label{eqn:bs}\overline{S}\, \overline{R}^{\top}L(I-\eta S)^t=\begin{bmatrix} \overline{S}_1R_1^{\top}L_1(I-\eta S_1)^t & \overline{S}_1R_1^{\top}L_2(I-\eta S_2)^t\\
\overline{S}_2R_2^{\top}L_1(I-\eta S_1)^t & \overline{S}_2R_2^{\top}L_2(I-\eta S_2)^t
\end{bmatrix}.
\end{equation}
Now for any constant $a>0$, we can choose a constant $C>0$ such that for $t\geq Cd^{u-\ell}\log(d)$ we have $\|(I-\eta S_1)^t\|_{\op}\leq d^{-a}.$ Therefore, the first column on the right side of \eqref{eqn:bs} is small in operator norm and we may estimate
$$\left\|\overline{S}\, \overline{R}^{\top}L(I-\eta S)^t-\begin{bmatrix} 0 & \overline{S}_1R_1^{\top}L_2(I-\eta S_2)^t\\
0 & \overline{S}_2R_2^{\top}L_2(I-\eta S_2)^t
\end{bmatrix}\right\|_{\op}\lesssim d^{-(a-u)}.$$
Recalling the equality $\nabla\mathcal{L}(W_t)=-\overline{W} T_{\eta,t}$, we therefore deduce
$$\left\|\nabla\mathcal{L}(W_t)-\overline{W}\,\overline{L}\begin{bmatrix} 0 & -\overline{S}_1R_1^{\top}L_2(I-\eta S_2)^t\\
0 & -\overline{S}_2R_2^{\top}L_2(I-\eta S_2)^t
\end{bmatrix}L\right\|_{\op}\lesssim \|\overline{W}\|_{\op}\cdot d^{-(a-u)}.$$
It remains therefore to estimate the nuclear and Frobenius norms of 
$$\overline{W}\,\overline{L}\begin{bmatrix} 0 & -\overline{S}_1R_1^{\top}L_2(I-\eta S_2)^t\\
0 & -\overline{S}_2R_2^{\top}L_2(I-\eta S_2)^t
\end{bmatrix}L\qquad \textrm{or equivalently}\qquad \overline{W}\,\overline{L}\begin{bmatrix}  \overline{S}_1R_1^{\top}L_2\\
\overline{S}_2R_2^{\top}L_2
\end{bmatrix}(I-\eta S_2)^t.$$
Now observe that as long as $t\leq Cd^{u}$ for any constant $C$, the diagonal entries of $(I-\eta S_2)^t$ are bounded from above and below by positive numerical constants. Therefore it remains to estimate the nuclear and Frobenius norms of the matrix 
$$M=\overline W\,\overline{L}\,Q \qquad \textrm{where}\qquad Q:=\begin{bmatrix}  \overline{S}_1R_1^{\top}L_2\\
\overline{S}_2R_2^{\top}L_2
\end{bmatrix}.$$
Using \eqref{eqn:dk1new} and \eqref{eqn:dk2new} we have 
$\|L_1L_1^{\top}-R_1R_1^{\top}\|_{\op}\lesssim d^{-\ell}$. Now a standard argument based on principal angles shows the estimates 
$$\|L_1L_1^{\top}-R_1R_1^{\top}\|_{\op}=\|R_1^{\top}L_2\|_{\op}\qquad \textrm{and}\qquad s_{k-p}(R_2^{\top}L_2)= \sqrt{1-\|L_1L_1^{\top}-R_1R_1^{\top}\|_{\op}^2}.$$ In particular, we deduce $\|\overline{S}_1R_1^{\top}L_2\|_{\rm op}\lesssim d^{u-\ell}$. Therefore, we obtain the upper bound on the Frobenius norm
\begin{align*}
\|Q\|_*&\geq \left\|\overline{W}\,\overline{L}\,\begin{bmatrix}  0\\
\overline{S}_2R_2^{\top}L_2\end{bmatrix}\right\|_*-\left\|\overline{W}\,\overline{L}\begin{bmatrix}  \overline{S}_1R_1^{\top}L_2\\
0
\end{bmatrix}\right\|_*\\
&\geq \|[\overline W\, \overline{L}]_{p+1:k}\overline{S}_2R_2^{\top}L_2\|_{*}-pd^{u-\ell}\\
&\gtrsim \|[\overline W\, \overline{L}]_{p+1:k}\overline{S}_2\|_{*}-pd^{u-\ell}\\
&\gtrsim \|[\overline W\, \overline{L}]_{p+1:r-p}\|_{*} -pd^{u-\ell}, 
\end{align*}
where the last inequality follows from applying the linear algebraic Lemma~\ref{lem:lower_bd_nuc} with $P=[p+1:r-p]$.
Similarly, we obtain the bound on the Frobenius norm:
\begin{align*}
\|Q\|_F&\leq \left\|\overline{W}\,\overline{L}\,\begin{bmatrix}  0\\
\overline{S}_2R_2^{\top}L_2\end{bmatrix}\right\|_F-\left\|\overline{W}\,\overline{L}\begin{bmatrix}  \overline{S}_1R_1^{\top}L_2\\
0
\end{bmatrix}\right\|_F\\
&\lesssim \|[\overline W\, \overline{L}]_{p+1:k}\overline{S}_2R_2^{\top}L_2\|_{F}+\sqrt{p}d^{u-\ell}\\
&\lesssim \|[\overline W\, \overline{L}]_{p+1:k}\overline{S}_2\|_{F}+\sqrt{p}d^{u-\ell}\\
&\lesssim \|\overline{W}\|_{F}+\sqrt{p}d^{u-\ell}.
\end{align*}
The two claimed estimates follow directly.
\end{proof}

Theorem~\ref{thm:main_thrmgrfjishdfklds} thus shows that after a burn-in of length
$d^{u-\ell}\log d$ and up to iteration $d^{u}$, the gradient along the iterations satisfies 
$$\nr(W_t)\gtrsim \frac{\|[\overline W\, \overline{L}]_{p+1:r-p}\|^2_{*}}{\|\overline W\|^2_{*}},$$
assuming that the error terms in the numerator and denominator are negligible. Consequently, under a natural incoherance condition 
$$\|[\overline W\, \overline{L}]_{p+1:r-p}\|_{*}\gtrsim \|\overline W\|_*,$$
we get the desired conclusion $\nr(W_t)\gtrsim \nr(\overline W)$. Intuitively, the incoherance condition holds if the spectrum of $W_t$ is well-spread across the columns, and is true for example 
if $\overline W$ has iid Gaussian entries. This is the content of the following Corollary, which is the main result of the section. 

\begin{cor}[Nuclear rank with multiple spikes]\label{cor:main_multi_spike_hard}
Suppose that Assumption~\ref{assump:spiked2} holds. Suppose moreover that the entries of $\overline W$ are iid Gaussian $\mathcal{N}(1/d)$ and are independent of $A$ and $\overline{A}$, and we are in the regime $u-\ell<\frac{1}{2}$, $p=o_d(d^{1-2(u-\ell)})$, and $r\asymp d\asymp m\asymp k$ with $k/d\to \gamma\in (0,\infty)$ as $d\to\infty$. Then there exist constants $C_1,C_2>0$ such that for any constant $C_3>0$ the probability conditioned on $(A,\overline{A})$ of the event  
\begin{equation}\label{eqn:nr_lowernuc}
\nr(\nabla \mathcal{L}(W_t))\gtrsim d\qquad \textrm{holds for all}~ t\in [C_1d^{u-\ell}\log(d), C_3 d^{u}],
\end{equation}
 tends to one as $d$ tends to infinity.
\end{cor}
\begin{proof}
The proof follows immediately from Theorem~\ref{thm:main_thrmgrfjishdfklds} upon noting that the entries of $\overline{M}=[\overline W\, \overline{L}]_{p+1:r-p}$ are iid Gaussian $\mathcal{N}(1/d)$ and therefore standard Gaussian estimates show $\|\overline{W}\|_F =O_{d,\PP}(\sqrt{d})$ and $\|\overline{M}\|_*=\Omega_{d,\PP}(k(r-2p+1))=\Omega_{d,P}(d)$. 
\end{proof}

Thus, under a random teacher setting the
nuclear rank of $\nabla\mathcal L(W_t)$ is linear in $d$ throughout the entire time
window $t\in[C_1 d^{u-\ell}\log d,\,C_3 d^u]$.
This confirms that the spectral-versus-Euclidean separation identified in the realizable
multi-spike setting persists in the teacher--student regime.

\begin{remark}{\rm
Just as in the realizable case, the window of large nuclear ranks can be extended by a logarithmic factor: Namely, for any
$\varepsilon>0$ there exist constants $c_0$ and $c=c(\varepsilon)>0$ such that, with probability
tending to one as $d\to\infty$, we have
\[
\nr\bigl(\nabla\mathcal L(W_t)\bigr)\;\ge\; d^{1-\varepsilon}
\qquad\text{for all }t\in\bigl[c_0d^{u-\ell}\log d,\;c\,d^u\log d\bigr].
\]
The only required modification to the proof is to realize that by setting the horizon to be $cd^{u}\log(d)$ for a sufficiently small constant $c(\varepsilon)$, we can be sure that the diagonal entries of $(I-\eta S_2)^t$ are lower bounded by $d^{-\varepsilon}$. The proof then proceeds similarly, while taking this observation into account.
}
\end{remark}

\section{A General Layered Model and Spectral Descent}
\label{sec:layered-descent}

The theoretical and empirical results in the earlier sections suggest that spectral updates
should be particularly effective when the features seen by each block have low stable rank and
the corresponding gradients have large nuclear/Frobenius ratios. Until now, this discussion has
been framed in terms of specific architectures (e.g., MLPs) and random-feature models. In this
section we formulate a general layered model that captures a broad class of  architectures
and show that, under mild smoothness assumptions, the same spectral-versus-Euclidean descent
tradeoff holds for a \emph{full} gradient step. The key ingredient is a layerwise Hessian bound
involving a dominant feature term $\sum_\ell\|\Delta W_\ell A_{\ell-1}\|_F^2$ and a smaller
parameter-only term $\sum_\ell\|\Delta W_\ell\|_{\op}^2$, and a one-block comparison which pulls
in both the stable rank of the incoming features and the stable rank and nuclear rank of the
gradient.
The model is designed to encompass:

\begin{itemize}
  \item \textbf{Fully-connected networks (MLPs)}: each block is a dense matrix $W_\ell$, and
        the activation $A_\ell$ is obtained by applying a (possibly vector-valued) nonlinearity
        to the current preactivation $X_\ell$ and, optionally, combining it with earlier
        preactivations and activations (residual/skip connections).
  \item \textbf{Convolutional networks}: after flattening, each convolutional or projection
        layer can be viewed as a matrix $W_\ell$ acting on an unfolded feature map
        $A_{\ell-1}$, with $A_\ell$ produced by pointwise nonlinearities, pooling, and
        residuals that are smooth functions of $(X_1,\dots,X_\ell)$.
  \item \textbf{Transformers and attention-based models}: multihead attention, feedforward
        sublayers, and output projections are linear maps $W_\ell$ applied to token/head
        features $A_{\ell-1}$; the resulting activations $A_\ell$ are smooth functions of all
        previous preactivations through attention, normalization, and residual mixing.
\end{itemize}

In all of these cases, each block is (affine) linear in its parameters $W_\ell$, and the
nonlinear structure is captured by smooth activation maps $\Phi_\ell$ that depend on the
preactivations $(X_1,\dots,X_\ell)$. We will show that the curvature of the composite loss can
be bounded in a layer-wise fashion by a feature term plus a smaller parameter-only term, and
that, once this bound is established, the spectral-versus-Euclidean comparison essentially
reduces to the same nuclear-rank versus stable-rank inequality as in the random-feature model,
with an additional dependence on the stable rank of the gradient.

\subsection{Model and smoothness assumptions}

Fix an integer $L\ge 1$. For each layer $\ell=1,\dots,L$ we have dimensions
$d_\ell$ (output) and $d_{\ell-1}$ (input). We also fix a column dimension $n$ and think of
all activations $A_\ell(W)\in\R^{d_\ell\times n}$ as stacking $n$ data points in columns. Let
$W_\ell\in\R^{d_\ell\times d_{\ell-1}}$ be the learnable weight matrix at layer $\ell$. We
denote the collection of all weights by $W=(W_1,\dots,W_L)$.

\paragraph{Preactivations and activations.}
We define preactivations and activations recursively as follows:

\begin{itemize}
  \item The input feature $A_0\in\R^{d_0\times n}$ is fixed (e.g., the current minibatch).
  \item For each $\ell=1,\dots,L$, the preactivation at layer $\ell$ is
        \[
          X_\ell(W) := W_\ell A_{\ell-1}(W)\in\R^{d_\ell\times n}.
        \]
  \item For each $\ell=1,\dots,L$, the activation at layer $\ell$ is given by a $C^2$-smooth map
        \[
          A_\ell(W) := \Phi_\ell\big(X_1(W),\dots,X_\ell(W)\big)\in\R^{d_\ell\times n}.
        \]
\end{itemize}

Thus each $A_\ell(W)$ depends only on the previous preactivations $(X_1(W),\dots,X_\ell(W))$,
and each $X_\ell(W)$ depends linearly on $W_\ell$ and on $A_{\ell-1}(W)$.

\paragraph{Outer loss.}
We assume a $C^2$-smooth outer loss
\[
  f:\ \mathcal{X}_1\times\cdots\times\mathcal{X}_L\to\R,
  \qquad \mathcal{X}_\ell = \R^{d_\ell\times n},
\]
and define the composite loss
\[
  \mathcal L(W) := f\big(X_1(W),\dots,X_L(W)\big).
\]
We denote by $G_\ell(W) := \nabla_{W_\ell}\mathcal L(W)$ the gradient with respect to block
$W_\ell$.

\subsection{Layerwise Hessian bound}
\label{subsec:layered-hessian}

In the previous subsection we introduced the layered model
\[
  X_\ell(W) = W_\ell A_{\ell-1}(W),
  \qquad
  A_\ell(W) = \Phi_\ell\big(X_1(W),\dots,X_\ell(W)\big),
  \qquad \ell=1,\dots,L,
\]
with $A_0$ fixed, and the composite loss
\[
  \cL(W) := f\big(X_1(W),\dots,X_L(W)\big).
\]
We now fix a parameter $W=(W_1,\dots,W_L)$ and study the Hessian quadratic form
\[
  Q_W(\Delta W)
  := \big\langle \Delta W,\nabla^2\cL(W)\,\Delta W\big\rangle
\]
for perturbations $\Delta W=(\Delta W_1,\dots,\Delta W_L)$. The goal is to show that $Q_W$ is
bounded by a feature term $\sum_\ell\|\Delta W_\ell A_{\ell-1}(W)\|_F^2$ and a parameter-only
term $\sum_\ell\|\Delta W_\ell\|_{\op}^2$, with constants depending only on local derivative
bounds and on the depth~$L$, and \emph{not} on the layer widths.

\paragraph{Notation.}
At the fixed parameter $W$, we introduce the following quantities:
\[
\begin{array}{lcl}
J_{\Phi,\ell}
&:=& \displaystyle
\sup_{\|U\|_F=1} \|D\Phi_\ell(X)[U]\|_F,\\[1.0ex]
H_{\Phi,\ell}
&:=& \displaystyle
\sup_{\|U\|_F=1} \|D^2\Phi_\ell(X)[U,U]\|_F,\\[1.0ex]
J_\Phi
&:=& \displaystyle
\max_{1\le\ell\le L} J_{\Phi,\ell},
\qquad
H_\Phi
:= \max_{1\le\ell\le L} H_{\Phi,\ell},\\[1.0ex]
J_f
&:=& \displaystyle
\big\|\nabla f(X_1,\dots,X_L)\big\|_F,\\[1.0ex]
H_f
&:=& \displaystyle
\sup_{\|U\|_F=1} \big|D^2 f(X_1,\dots,X_L)[U,U]\big|,\\[1.0ex]
W_*
&:=& \displaystyle
\max_{1\le\ell\le L} \|W_\ell\|_{\op},\\[1.0ex]
B_\ell
&:=& \displaystyle
\Delta W_\ell A_{\ell-1}(W),
\qquad
\|B^{(\ell)}\|_F^2 := \sum_{i=1}^\ell \|B_i\|_F^2,\\[1.0ex]
\Gamma
&:=& \displaystyle
\sqrt{2}\,\bigl(1+2J_\Phi^2 W_*^2\bigr)^{\frac{L}{2}},\\[1.0ex]
\Theta
&:=& \displaystyle
\bigl(1+J_\Phi W_*\bigr)^{L-1}.
\end{array}
\]
Here $X=(X_1(W),\dots,X_L(W))$ denotes the collection of preactivations at $W$.

We parameterize the perturbation along the path $W(t) = W+t\Delta W$ and define
\[
  X_\ell(t) := X_\ell(W(t)),
  \qquad
  A_\ell(t) := A_\ell(W(t)).
\]
We then write
\[
  U_\ell := \dot X_\ell(0),
  \quad
  \widetilde U_\ell := \ddot X_\ell(0),
  \qquad
  U^{(\ell)} := (U_1,\dots,U_\ell),
  \quad
  \widetilde U^{(\ell)} := (\widetilde U_1,\dots,\widetilde U_\ell),
\]
and similarly $$V_\ell := \dot A_\ell(0),\qquad \widetilde V_\ell := \ddot A_\ell(0).$$

The next two lemmas bound the first and second derivatives of the preactivations in terms of
the feature perturbations $B_\ell$ and the parameter perturbations $\Delta W_\ell$.

\begin{lem}[First-order propagation]\label{lem:first-order-U}
For every $1\le\ell\le L$ we have
\begin{equation}\label{eq:first-order-bound}
  \|U^{(\ell)}\|_F \;\le\; \Gamma\,\|B^{(\ell)}\|_F.
\end{equation}
\end{lem}

\begin{proof}
Differentiating $X_\ell(t)=W_\ell(t)A_{\ell-1}(t)$ at $t=0$ yields
\[
  U_1 = \Delta W_1 A_0 = B_1,
  \qquad
  U_\ell = B_\ell + W_\ell V_{\ell-1}\quad(\ell\ge 2),
\]
where $V_{\ell-1}:=\dot A_{\ell-1}(0)$. Differentiating
$A_\ell(t)=\Phi_\ell(X^{(\ell)}(t))$ at $t=0$ and using the definition of
$J_{\Phi,\ell}$ gives
\[
  V_\ell = D\Phi_\ell(X)[U^{(\ell)}],
  \qquad
  \|V_\ell\|_F \;\le\; J_{\Phi,\ell}\,\|U^{(\ell)}\|_F \;\le\; J_\Phi\|U^{(\ell)}\|_F.
\]
Define the quantities
\[
  S_\ell := \sum_{i=1}^\ell\|U_i\|_F^2,
  \qquad
  b_\ell^2 := \sum_{i=1}^\ell\|B_i\|_F^2.
\]
Then $S_1=\|B_1\|_F^2\le b_1^2$. For $\ell\ge 2$ we have
\[
  \|U_\ell\|_F
  \;\le\; \|B_\ell\|_F + \|W_\ell\|_{\op}\|V_{\ell-1}\|_F
  \;\le\; \|B_\ell\|_F + W_*J_\Phi\|U^{(\ell-1)}\|_F,
\]
and hence
\[
  \|U_\ell\|_F^2
  \;\le\; 2\|B_\ell\|_F^2 + 2W_*^2J_\Phi^2 S_{\ell-1}.
\]
Therefore
\[
  S_\ell
  = S_{\ell-1} + \|U_\ell\|_F^2
  \;\le\; (1+2W_*^2J_\Phi^2)S_{\ell-1} + 2\|B_\ell\|_F^2.
\]
Setting $a:=1+2W_*^2J_\Phi^2$ and iterating from $\ell=1$ gives
$S_\ell \le 2 a^{\ell-1}b_\ell^2 \le 2 a^{L-1}b_\ell^2$. Thus, we deduce
\[
  \|U^{(\ell)}\|_F = \sqrt{S_\ell}
  \;\le\; \sqrt{2}\,a^{L/2}\,\|B^{(\ell)}\|_F
  = \Gamma\,\|B^{(\ell)}\|_F,
\]
as claimed.
\end{proof}

\begin{lem}[Second-order propagation]\label{lem:second-order-U}
The following estimate holds:
\begin{equation}\label{eq:second-order-bound}
  \|\widetilde U^{(L)}\|_F
  \;\le\;
  2\Theta J_\Phi \Gamma \sqrt{L}
  \left(\sum_{\ell=1}^L\|\Delta W_\ell\|_{\op}^2\right)^{\frac12}
  \|B^{(L)}\|_F
  \;+\;
  \Theta W_* H_\Phi \Gamma^2 L\,\|B^{(L)}\|_F^2.
\end{equation}
\end{lem}

\begin{proof}
Differentiating $X_\ell(t)=W_\ell(t)A_{\ell-1}(t)$ twice at $t=0$ gives
\[
  \widetilde U_1 = 0,\qquad
  \widetilde U_\ell = 2\Delta W_\ell V_{\ell-1} + W_\ell \widetilde V_{\ell-1}
  \quad(\ell\ge 2),
\]
where $V_\ell := \dot A_\ell(0)$ and $\widetilde V_\ell := \ddot A_\ell(0)$.
For $A_\ell(t)=\Phi_\ell(X^{(\ell)}(t))$, the second-order chain rule yields
\[
  \widetilde V_\ell
  = D^2\Phi_\ell(X)[U^{(\ell)},U^{(\ell)}]
    + D\Phi_\ell(X)[\widetilde U^{(\ell)}].
\]
By the definitions of $H_{\Phi,\ell}$ and $J_{\Phi,\ell}$, we have
\[
  \big\| D^2\Phi_\ell(X)[U^{(\ell)},U^{(\ell)}] \big\|_F
  \;\le\; H_{\Phi,\ell}\,\|U^{(\ell)}\|_F^2
  \;\le\; H_\Phi \,\|U^{(\ell)}\|_F^2,
\]
and
\[
  \big\| D\Phi_\ell(X)[\widetilde U^{(\ell)}] \big\|_F
  \;\le\; J_{\Phi,\ell}\,\|\widetilde U^{(\ell)}\|_F
  \;\le\; J_\Phi\,\|\widetilde U^{(\ell)}\|_F.
\]
Using the triangle inequality, we obtain
\begin{equation}\label{eq:Vtilde-bound}
  \|\widetilde V_\ell\|_F
  \;\le\; H_\Phi \|U^{(\ell)}\|_F^2 + J_\Phi\|\widetilde U^{(\ell)}\|_F
  \qquad\text{for all } \ell.
\end{equation}
Define now $M_\ell := \|\widetilde U^{(\ell)}\|_F$ for $\ell=1,\dots,L$. Then
$M_1 = \|\widetilde U_1\|_F = 0$. For $\ell\ge 2$ we have
\[
  M_\ell
  = \|\widetilde U^{(\ell)}\|_F
  \;\le\; \|\widetilde U^{(\ell-1)}\|_F + \|\widetilde U_\ell\|_F
  = M_{\ell-1} + \|\widetilde U_\ell\|_F.
\]
Using the expression for $\widetilde U_\ell$ and the bound
\eqref{eq:Vtilde-bound} for $\widetilde V_{\ell-1}$ yields
\begin{align*}
  \|\widetilde U_\ell\|_F
  &\le 2\|\Delta W_\ell\|_{\op}\,\|V_{\ell-1}\|_F
       + \|W_\ell\|_{\op}\,\|\widetilde V_{\ell-1}\|_F \\
  &\le 2\|\Delta W_\ell\|_{\op}\,\|V_{\ell-1}\|_F
       + W_*\big(H_\Phi \|U^{(\ell-1)}\|_F^2 + J_\Phi\|\widetilde U^{(\ell-1)}\|_F\big),
\end{align*}
and hence
\[
  M_\ell
  \;\le\; (1+W_*J_\Phi)M_{\ell-1}
       + 2\|\Delta W_\ell\|_{\op}\,\|V_{\ell-1}\|_F
       + W_*H_\Phi \|U^{(\ell-1)}\|_F^2.
\]
From Lemma~\ref{lem:first-order-U} we have
$\|V_{\ell-1}\|_F\le J_\Phi\Gamma\|B^{(\ell-1)}\|_F$ and
$\|U^{(\ell-1)}\|_F\le \Gamma\|B^{(\ell-1)}\|_F$, so
\[
  M_\ell
  \;\le\; (1+W_*J_\Phi)M_{\ell-1}
       + 2J_\Phi\Gamma\|\Delta W_\ell\|_{\op}\,\|B^{(\ell-1)}\|_F
       + W_*H_\Phi \Gamma^2 \|B^{(\ell-1)}\|_F^2.
\]
Unrolling this recursion from $\ell=2$ to $\ell=L$ and using $M_1=0$ and
$(1+W_*J_\Phi)^{L-k}\le \Theta$ for all $k$ gives
\[
  M_L
  \;\le\;
  \Theta\sum_{k=2}^L
  \Big(
    2J_\Phi\Gamma\|\Delta W_k\|_{\op}\,\|B^{(k-1)}\|_F
    + W_*H_\Phi \Gamma^2 \|B^{(k-1)}\|_F^2
  \Big).
\]
We bound the two sums separately. First,
\[
  \sum_{k=2}^L \|B^{(k-1)}\|_F^2
  = \sum_{k=2}^L \sum_{i\le k-1}\|B_i\|_F^2
  = \sum_{i=1}^{L-1}(L-i)\|B_i\|_F^2
  \;\le\; L\sum_{i=1}^L\|B_i\|_F^2
  = L\,\|B^{(L)}\|_F^2.
\]
Second, by Cauchy–Schwarz,
\[
  \sum_{k=2}^L \|\Delta W_k\|_{\op}\,\|B^{(k-1)}\|_F
  \;\le\;
  \sqrt{L}
  \left(\sum_{\ell=1}^L\|\Delta W_\ell\|_{\op}^2\right)^{\frac12}
  \|B^{(L)}\|_F.
\]
Substituting these bounds into the expression for $M_L$ gives
\[
  \|\widetilde U^{(L)}\|_F = M_L
  \;\le\;
  2\Theta J_\Phi \Gamma \sqrt{L}
  \left(\sum_{\ell=1}^L\|\Delta W_\ell\|_{\op}^2\right)^{\frac12}
  \|B^{(L)}\|_F
  \;+\;
  \Theta W_* H_\Phi \Gamma^2 L\,\|B^{(L)}\|_F^2,
\]
which is exactly \eqref{eq:second-order-bound}.
\end{proof}

We can now state the layerwise Hessian bound, which is the main conclusion of this subsection.

\begin{proposition}[Layerwise Hessian bound]\label{prop:layered-Hessian-mixed}
Let $W$ be fixed and consider the layered model and loss $\cL$ described
above. Then there exist finite constants $C_F(W),C_{\op}(W)>0$, depending only
on $J_\Phi,H_\Phi,J_f,H_f$, on the block operator norms $\|W_\ell\|_{\op}$ (through
$W_*$), and on the depth $L$, such that for every perturbation
$\Delta W=(\Delta W_1,\dots,\Delta W_L)$, we have
\begin{equation}\label{eq:layered-Hessian-mixed}
  \big|\langle \Delta W,\nabla^2\cL(W)\,\Delta W\rangle\big|
  \;\le\;
  C_F(W)\,\sum_{\ell=1}^L \|\Delta W_\ell A_{\ell-1}(W)\|_F^2
  \;+\;
  C_{\op}(W)\,\sum_{\ell=1}^L \|\Delta W_\ell\|_{\op}^2.
\end{equation}
Moreover, one can take for example
\begin{equation}\label{eq:CF-Cop-example}
  C_F(W)
  := \Gamma^2 H_f
     + J_f \Theta W_* H_\Phi \Gamma^2 L
     + J_f^2 \Theta \Gamma^2 L,
  \qquad
  C_{\op}(W)
  := \Theta J_\Phi^2.
\end{equation}
\end{proposition}

\begin{proof}
Along the path $W(t)=W+t\Delta W$, the chain rule for
$\cL(W(t))=f(X_1(t),\dots,X_L(t))$ gives
\[
  \frac{d^2}{dt^2}\Big|_{t=0} \cL(W(t))
  = D^2 f(X)[U^{(L)},U^{(L)}]
    + \big\langle \nabla f(X),\widetilde U^{(L)}\big\rangle.
\]
By the definitions of $H_f$ and $J_f$, we have
\begin{equation}\label{eq:Hessian-chain-rule-final}
  \big|\langle \Delta W,\nabla^2\cL(W)\,\Delta W\rangle\big|
  \;\le\;
  H_f \|U^{(L)}\|_F^2 + J_f\|\widetilde U^{(L)}\|_F.
\end{equation}
Lemma~\ref{lem:first-order-U} yields $\|U^{(L)}\|_F\le \Gamma\|B^{(L)}\|_F$, and
Lemma~\ref{lem:second-order-U} yields
\[
  \|\widetilde U^{(L)}\|_F
  \;\le\;
  2\Theta J_\Phi \Gamma \sqrt{L}
  \left(\sum_{\ell=1}^L\|\Delta W_\ell\|_{\op}^2\right)^{\frac12}
  \|B^{(L)}\|_F
  \;+\;
  \Theta W_* H_\Phi \Gamma^2 L\,\|B^{(L)}\|_F^2.
\]
Substituting these bounds into \eqref{eq:Hessian-chain-rule-final} and
collecting terms gives
\[
  \big|\langle \Delta W,\nabla^2\cL(W)\,\Delta W\rangle\big|
  \;\le\;
  A_B(W)\,\|B^{(L)}\|_F^2
  \;+\;
  2 C_c(W)\,a\,b,
\]
where
\[
  A_B(W)
  := \Gamma^2 H_f + J_f \Theta W_* H_\Phi \Gamma^2 L,
  \qquad
  C_c(W)
  := J_f \Theta J_\Phi \Gamma \sqrt{L},
\]
and
\[
  a := \left(\sum_{\ell=1}^L\|\Delta W_\ell\|_{\op}^2\right)^{\frac12},
  \qquad
  b := \|B^{(L)}\|_F.
\]
Using the inequality
\[
  2 C_c(W)\,a\,b \;\le\; A\,a^2 + B\,b^2
  \quad\text{for any }A,B>0\text{ with }AB\ge C_c(W)^2,
\]
and choosing, for instance,
\[
  A = \Theta J_\Phi^2,
  \qquad
  B = \frac{C_c(W)^2}{A}
    = J_f^2 \Theta \Gamma^2 L,
\]
we obtain
\[
  \big|\langle \Delta W,\nabla^2\cL(W)\,\Delta W\rangle\big|
  \;\le\;
  \big(A_B(W) + B\big)\,\|B^{(L)}\|_F^2
  \;+\;
  A\,\sum_{\ell=1}^L\|\Delta W_\ell\|_{\op}^2.
\]
Since equality $\|B^{(L)}\|_F^2 = \sum_{\ell=1}^L\|\Delta W_\ell A_{\ell-1}(W)\|_F^2$ holds, this is exactly
\eqref{eq:layered-Hessian-mixed} with $C_F(W)$ and $C_{\op}(W)$ as in
\eqref{eq:CF-Cop-example}.
\end{proof}

\paragraph{Scaling with the number of samples.}
For averaged losses such as the squared loss
\(
  f(X_1,\dots,X_L) = \tfrac{1}{2n}\|X_L - Y\|_F^2,
\)
we have
\[
  J_f = \|\nabla f(X)\|_F = \frac{1}{n}\,\|X_L - Y\|_F,
  \qquad
  H_f = \frac{1}{n},
\]
as we will verify below. If, in addition, the average loss $f$ remains $O(1)$
along training so that typically $\|X_L - Y\|_F^2 = O(n)$, then
$J_f = O(1/\sqrt n)$ and $H_f = O(1/n)$. A completely analogous scaling holds
for softmax cross-entropy: the per-example gradients are uniformly
bounded, so for the averaged loss one has $J_f = O(1/\sqrt n)$ and $H_f = O(1/n)$
regardless of the size of the prediction errors. These factors appear only in
$C_F(W)$, which multiplies the feature term
$\sum_\ell \|\Delta W_\ell A_{\ell-1}(W)\|_F^2$. Since $\|A_{\ell-1}(W)\|_F^2$
is of order $n$ when $A_{\ell-1}$ stacks $n$ data points in columns, the
resulting curvature contribution is of order one in $n$. In particular, the
constants $C_F(W)$ and $C_{\op}(W)$ in Proposition~\ref{prop:layered-Hessian-mixed}
do not grow with the number of samples $n$, nor with the layer widths.

\subsubsection{Derivative constants for MLPs and transformer blocks}

To interpret the layerwise Hessian bound in concrete architectures, it is useful to understand
the size of the derivative constants $J_{\Phi,\ell},H_{\Phi,\ell}$ and $J_f,H_f$. In this
subsubsection we sketch how these constants are instantiated in multilayer perceptrons and
transformer blocks. Throughout, all implicit constants depend only on the choice of nonlinearities
and loss, and not on the dimensions $d_\ell$ or $n$.

\paragraph{MLPs.}
Consider a fully-connected network with layers
\[
  X_1 = W_1 A_0,\quad A_1 = \sigma(X_1),\quad
  X_2 = W_2 A_1,\quad A_2 = \sigma(X_2),\quad\ldots,
\]
and a fixed elementwise activation $\sigma:\R\to\R$. In the layered notation
above, the nonlinear map at layer $\ell$ is simply
\[
  \Phi_\ell(X_1,\dots,X_\ell) = \sigma(X_\ell)\quad\text{(applied entrywise)}.
\]
Assume that $\sigma$ is $C^2$-smooth with
\[
  \sup_{t\in\R}|\sigma'(t)| \le L_1,
  \qquad
  \sup_{t\in\R}|\sigma''(t)| \le L_2.
\]
Then for the Jacobian we have
\[
  D\Phi_\ell(X)[U] = \sigma'(X_\ell)\odot U,
  \qquad
  \|D\Phi_\ell(X)[U]\|_F \le L_1\|U\|_F.
\]
Therefore we deduce
\[
  J_{\Phi,\ell} \le L_1
  \quad\text{and hence}\quad
  J_\Phi = \max_\ell J_{\Phi,\ell} \le L_1.
\]
Similarly, the second derivative takes the form
\[
  D^2\Phi_\ell(X)[U,U] = \sigma''(X_\ell)\odot U\odot U,
\]
and
\[
  \|D^2\Phi_\ell(X)[U,U]\|_F
  \le L_2\|U\odot U\|_F
  \le L_2\|U\|_F^2,
\]
Therefore, we deduce
\[
  H_{\Phi,\ell} \le L_2,
  \qquad
  H_\Phi \le L_2.
\]
Thus, for MLPs with a fixed smooth activation, the activation-side constants
$J_{\Phi,\ell},H_{\Phi,\ell}$ are controlled entirely by scalar bounds on
$\sigma',\sigma''$ and do not depend on width.

For the outer loss $f$, the squared loss and softmax cross-entropy examples
in the scaling paragraph above already show that $J_f$ and $H_f$ inherit only
the $1/n$ factors coming from the averaging in the loss, and do not introduce
any additional dependence on the layer widths.

\paragraph{Transformer blocks and the layered model.}
We now make the structure of a single decoder block explicit in the layered
notation. Fix a block input $X^{(0)}\in\R^{d_{\mathrm{model}}\times T}$ (hidden
dimension $d_{\mathrm{model}}$, sequence length $T$). We enumerate the
trainable matrices inside the block as
\[
  W_1 = W_Q,\quad W_2 = W_K,\quad W_3 = W_V,\quad
  W_4 = W_O,\quad W_5 = W_1,\quad W_6 = W_2,
\]
corresponding to the attention projections and MLP weights. Following the
layered model convention $X_\ell(W)=W_\ell A_{\ell-1}(W)$ and
$A_\ell(W)=\Phi_\ell(X_1(W),\dots,X_\ell(W))$, we introduce the following
internal variables and maps.

\smallskip
\emph{Block structure.}
Define first the RMS-normalized input
\[
  A_0(W) := A^{\mathrm{rms}}(W)
  := \mathrm{RMSNorm}\big(X^{(0)}(W)\big),
\]
so that the first three preactivations are
\[
  X_1(W) = W_1 A_0(W) = W_Q A^{\mathrm{rms}},
  \quad
  X_2(W) = W_2 A_0(W) = W_K A^{\mathrm{rms}},
  \quad
  X_3(W) = W_3 A_0(W) = W_V A^{\mathrm{rms}}.
\]
Next, we form the scaled attention scores and probabilities
\[
  S(W) := d_{\mathrm{model}}^{-1/2}\,X_2(W)^\top X_1(W),
  \qquad
  P(W) := \mathrm{softmax}\big(S(W)\big)\quad\text{(columnwise)},
\]
and the value aggregation
\[
  H(W) := X_3(W)P(W).
\]
The output projection preactivation is
\[
  X_4(W) = W_4 H(W) = W_O H(W),
\]
and the attention residual is
\[
  X^{\mathrm{att}}(W) := X^{(0)}(W) + X_4(W).
\]
We then RMS-normalize $X^{\mathrm{att}}$ as
\[
  A_4(W) := A^{\mathrm{rms}}_{\mathrm{mlp}}(W)
  := \mathrm{RMSNorm}\big(X^{\mathrm{att}}(W)\big),
\]
and form the MLP preactivations
\[
  X_5(W) = W_5 A_4(W) = W_1 A^{\mathrm{rms}}_{\mathrm{mlp}},
  \qquad
  X_6(W) = W_6 B(W) = W_2 B(W),
\]
where
\[
  B(W) := \sigma\big(X_5(W)\big)
\]
is the pointwise MLP activation. Finally, the block output is
\[
  X^{(+)}(W) := X^{\mathrm{att}}(W) + X_6(W).
\]

In this notation, the preactivations are $X_\ell(W)=W_\ell A_{\ell-1}(W)$ for
$\ell=1,\dots,6$, with $A_{-1}(W)=X^{(0)}(W)$ and $A_0(W),A_4(W),B(W)$ as
above. The corresponding activation maps
\[
  A_\ell(W) = \Phi_\ell\big(X_1(W),\dots,X_\ell(W)\big),
  \qquad \ell=0,\dots,6,
\]
collect the nonlinear operations that produce the data matrices feeding each
linear map; for example
\[
  \Phi_0(X^{(0)}) = \mathrm{RMSNorm}\big(X^{(0)}\big),
\]
\[
  \Phi_3(X_1,X_2,X_3) = A^{\mathrm{rms}}(W)
  \quad\text{(the attention inputs, reused for $W_Q,W_K,W_V$),}
\]
\[
  \Phi_4(X_1,\dots,X_4)
  = \mathrm{RMSNorm}\big(X^{(0)} + X_4\big)
  = A^{\mathrm{rms}}_{\mathrm{mlp}}(W),
\]
\[
  \Phi_5(X_1,\dots,X_5) = \sigma(X_5),
  \qquad
  \Phi_6(X_1,\dots,X_6) = X^{(+)} = X^{\mathrm{att}} + X_6.
\]

\smallskip
\emph{Derivative bounds.}
With this structure in hand, the derivative constants $J_{\Phi,\ell}$ and
$H_{\Phi,\ell}$ for the transformer block can be bounded by combining the
primitive bounds for RMSNorm, softmax, and the MLP nonlinearity. As before, we
assume that the MLP activation $\sigma$ satisfies
\[
  \sup_{t\in\R}|\sigma'(t)| \le L_1,
  \qquad
  \sup_{t\in\R}|\sigma''(t)| \le L_2.
\]
For the RMS normalization, a direct computation of the Jacobian and Hessian
shows that, as long as the rescaled squared norms
$d_{\mathrm{model}}^{-1}\|X^{(0)}\|_F^2$ and $d_{\mathrm{model}}^{-1}\|X^{\mathrm{att}}\|_F^2$
remain in a fixed compact interval, there exist constants
$C_{\mathrm{rms},1},C_{\mathrm{rms},2}$ (depending only on the RMSNorm
hyperparameters) such that
\[
  \|D(\mathrm{RMSNorm})(\cdot)[U]\|_F \le C_{\mathrm{rms},1}\,\|U\|_F,
  \qquad
  \|D^2(\mathrm{RMSNorm})(\cdot)[U,U]\|_F \le C_{\mathrm{rms},2}\,\|U\|_F^2.
\]
For the softmax, viewed columnwise on $S\in\R^{T\times T}$, there exist
constants $C_{\mathrm{sm},1},C_{\mathrm{sm},2}$ such that
\[
  \|D(\mathrm{softmax})(S)[U]\|_F \le C_{\mathrm{sm},1}\,\|U\|_F,
  \qquad
  \|D^2(\mathrm{softmax})(S)[U,U]\|_F \le C_{\mathrm{sm},2}\,\|U\|_F^2
\]
for all $S,U$. The linear projections $W_\ell A_{\ell-1}$ always contribute
their spectral norms:
\[
  \|D(W_\ell A_{\ell-1})[U]\|_F \le \|W_\ell\|_{\op}\,\|U\|_F,
  \qquad
  D^2(W_\ell A_{\ell-1})[\cdot,\cdot]\equiv 0.
\]

Combining these primitive estimates through the chain rule, one finds that the
constants $J_{\Phi,\ell},H_{\Phi,\ell}$ are bounded above by expressions of the
form
\[
  J_{\Phi,\ell} \le C_1 \prod_{\ell'\in \mathcal{N}(\ell)} \|W_{\ell'}\|_{\op},
  \qquad
  H_{\Phi,\ell} \le C_2 \left( 1 + \sum_{\ell'\in \mathcal{N}(\ell)} \|W_{\ell'}\|_{\op}^2\right),
\]
where $\mathcal{N}(\ell)$ denotes the (finite) collection of projection
matrices whose outputs feed into the nonlinearities affecting $A_{\ell-1}$, and
$C_1,C_2$ depend only on the scalar nonlinearities (RMSNorm, softmax, $\sigma$)
and not on $d_{\mathrm{model}}$ or $T$. In particular, these bounds show that $J_\Phi$ and $H_\Phi$ remain bounded as the width and
sequence length vary.

\subsection{From feature metric to blockwise Lipschitz constants}
\label{subsec:feature-to-lip}

The mixed Hessian bound of Proposition~\ref{prop:layered-Hessian-mixed} controls the curvature
of $\cL$ in terms of the feature perturbations $\Delta W_\ell A_{\ell-1}(W)$ and the operator
norms $\|\Delta W_\ell\|_{\op}$:
\[
  \big\langle \Delta W,\nabla^2\cL(W)\,\Delta W\big\rangle
  \;\le\;
  C_F(W)\sum_{\ell=1}^L \|\Delta W_\ell A_{\ell-1}(W)\|_F^2
  \;+\;
  C_{\op}(W)\sum_{\ell=1}^L \|\Delta W_\ell\|_{\op}^2.
\]
The feature term in this inequality plays exactly the same role as in the random-feature model:
it yields blockwise curvature constants that are determined by the activations $A_{\ell-1}(W)$.
The parameter-only term $C_{\op}(W)\sum_\ell \|\Delta W_\ell\|_{\op}^2$ contributes an extra,
geometry-independent correction which we will account for in the descent comparison, but it does
not affect which activation matrices are relevant for each block or the stable-rank ratio.

In many architectures of interest, and in particular in the multilayer and transformer examples
discussed in Section~\ref{sec:multiple}, each activation $A_{\ell-1}(W)$ admits a natural
factorization
\begin{equation}\label{eq:factorization}
  A_{\ell-1}(W)
  \;=\; M_{\ell,\mathrm{left}}(W)\,
  \widetilde A_{\ell-1}(W)\,M_{\ell,\mathrm{right}}(W),
\end{equation}
where:
\begin{itemize}
  \item $\widetilde A_{\ell-1}(W)$ is the \emph{core activation} feeding block~$\ell$
        (for example, the post-activation $A_{\ell-1}$ in an MLP, or an RMS-normalized
        activation or MLP post-activation in a transformer block as in
        Section~\ref{sec:multiple}); and
  \item $M_{\ell,\mathrm{left}}(W)$ and $M_{\ell,\mathrm{right}}(W)$ are fixed linear
        operators (at the current $W$) capturing the surrounding linear structure:
        projections, attention kernels, residual maps, etc.
\end{itemize}
The dependence on $W$ is smooth but, crucially, these matrices are independent of $\Delta W$.
Using submultiplicativity of the operator norm, we obtain the Frobenius and operator
bounds, for each block~$\ell$, given by
\begin{align*}
  \|\Delta W_\ell A_{\ell-1}(W)\|_F
  &\le\; \|\Delta W_\ell\|_F\,\|A_{\ell-1}(W)\|_{\op},\\[0.5ex]
  \|\Delta W_\ell A_{\ell-1}(W)\|_F
  &\le\; \|\Delta W_\ell\|_{\op}\,\|A_{\ell-1}(W)\|_{F}.
\end{align*}
Using the factorization~\eqref{eq:factorization} once more, we obtain
\begin{align*}
  \|A_{\ell-1}(W)\|_{\op}
  &\le
  \|M_{\ell,\mathrm{left}}(W)\|_{\op}\,
  \|\widetilde A_{\ell-1}(W)\|_{\op}\,
  \|M_{\ell,\mathrm{right}}(W)\|_{\op},\\[0.5ex]
  \|A_{\ell-1}(W)\|_{F}
  &\le
  \|M_{\ell,\mathrm{left}}(W)\|_{\op}\,
  \|\widetilde A_{\ell-1}(W)\|_{F}\,
  \|M_{\ell,\mathrm{right}}(W)\|_{\op}.
\end{align*}
Combining these inequalities with the feature part of the Hessian bound, we obtain feature-based
blockwise curvature constants
\begin{equation}\label{eq:block-lip-simple}
\begin{aligned}
  L_\ell^{\mathrm{F}}(W)
  &:= C_F(W)\,
  \|M_{\ell,\mathrm{left}}(W)\|_{\op}^2\,
  \|\widetilde A_{\ell-1}(W)\|_{\op}^2\,
  \|M_{\ell,\mathrm{right}}(W)\|_{\op}^2,\\[0.5ex]
  L_\ell^{\mathrm{op}}(W)
  &:= C_F(W)\,
  \|M_{\ell,\mathrm{left}}(W)\|_{\op}^2\,
  \|\widetilde A_{\ell-1}(W)\|_{F}^2\,
  \|M_{\ell,\mathrm{right}}(W)\|_{\op}^2,
\end{aligned}
\end{equation}
for which the feature term in the Hessian satisfies
\begin{align}
  \big\langle \Delta W,\nabla^2\cL(W)\,\Delta W\big\rangle_{\mathrm{feature}}
  &\le
  \sum_{\ell=1}^L L_\ell^{\mathrm{F}}(W)\,\|\Delta W_\ell\|_{F}^2,
  \label{eq:hess-F-simple}\\[0.5ex]
  \big\langle \Delta W,\nabla^2\cL(W)\,\Delta W\big\rangle_{\mathrm{feature}}
  &\le
  \sum_{\ell=1}^L L_\ell^{\mathrm{op}}(W)\,\|\Delta W_\ell\|_{\op}^2.
  \label{eq:hess-op-simple}
\end{align}
By construction, the ratio of these two constants depends only on the stable rank of the
\emph{core} activation $\widetilde A_{\ell-1}(W)$, while the surrounding linear maps cancel:
\begin{equation}\label{eq:stable-rank-ratio-simple}
  \frac{L_\ell^{\mathrm{op}}(W)}{L_\ell^{\mathrm{F}}(W)}
  =
  \frac{\|\widetilde A_{\ell-1}(W)\|_{F}^2}{\|\widetilde A_{\ell-1}(W)\|_{\op}^2}
  =
  \mathrm{st}\!\big(\widetilde A_{\ell-1}(W)\big).
\end{equation}
In words, the feature-based curvature in the spectral geometry is larger than that in the
Frobenius geometry by a factor equal to the stable rank of the core activation.

Two special cases are worth highlighting.

\begin{itemize}
  \item \textbf{Fully-connected networks.} For standard MLPs, we may take
        $M_{\ell,\mathrm{left}}(W)=I$ and $M_{\ell,\mathrm{right}}(W)=I$, so that
        $\widetilde A_{\ell-1}(W)=A_{\ell-1}(W)$ is exactly the layerwise post-activation
        introduced in~\eqref{eqn:post_activ}. In this case, $L_\ell^{\mathrm{F}}$ and
        $L_\ell^{\mathrm{op}}$ are proportional to $\|A_{\ell-1}\|_{\op}^2$ and
        $\|A_{\ell-1}\|_{F}^2$, and their ratio is the stable rank
        $\mathrm{st}(A_{\ell-1})$.
  \item \textbf{Transformer blocks.} In transformer architectures, the effective feature
        matrices $A_{\ell-1}(W)$ entering each block factor as in
        Section~\ref{sec:multiple}: the core activation $\widetilde A_{\ell-1}$ is one of
        the RMS-normalized hidden states or MLP post-activations, while
        $M_{\ell,\mathrm{left}}$ and $M_{\ell,\mathrm{right}}$ collect attention kernels,
        projection matrices, and residual maps. The expressions
        \eqref{eq:block-lip-simple}–\eqref{eq:stable-rank-ratio-simple} then show that the
        feature-based curvature is governed, up to fixed operator–norm factors, by the
        stable rank of the corresponding RMS or MLP activations.
\end{itemize}

The full Hessian bound of Proposition~\ref{prop:layered-Hessian-mixed} adds the parameter-only
term $C_{\op}(W)\sum_\ell \|\Delta W_\ell\|_{\op}^2$ to the right-hand side of
\eqref{eq:hess-F-simple}–\eqref{eq:hess-op-simple}. This term does not depend on the activations
and therefore does not change which feature matrices are relevant for a given block or the
stable-rank ratio~\eqref{eq:stable-rank-ratio-simple}; it will appear only as a small, blockwise
correction in the one-step descent analysis below.

\subsection{Descent bounds and spectral-vs.-Euclidean comparison}
\label{subsec:descent-comparison}

We now use the blockwise constants from
\eqref{eq:block-lip-simple}–\eqref{eq:stable-rank-ratio-simple} together with the full Hessian
bound of Proposition~\ref{prop:layered-Hessian-mixed} to compare one-step decreases under
Euclidean (Frobenius) and spectral updates acting on all matrix blocks. We write
$G_\ell(W) := \nabla_{W_\ell}\cL(W)$ for the block gradients, and let
$\mathcal S\subseteq\{1,\dots,L\}$ denote the subset of matrix-valued blocks on which
we intend to apply spectral updates (for example, the internal MLP and attention weights in a
transformer stack, as in Section~\ref{sec:multiple}). For notational brevity, we fix $W$ and
write $G_\ell$, $L_\ell^{\mathrm{F}}$, $L_\ell^{\mathrm{op}}$, $C_F$, and $C_{\op}$.

\paragraph{Taylor model.}
Expanding $\cL$ around $W$ in the direction $\Delta W=(\Delta W_1,\dots,\Delta W_L)$ yields
\begin{equation}\label{eq:taylor-simple}
  \cL(W+\Delta W)
  \;=\; \cL(W)
  + \sum_{\ell=1}^L \langle G_\ell,\Delta W_\ell\rangle
  + \frac12\langle \Delta W,\nabla^2\cL(W)\,\Delta W\rangle
  + R(\Delta W),
\end{equation}
where $R(\Delta W)$ collects third and higher order terms. Combining the feature-based bounds
\eqref{eq:hess-F-simple}–\eqref{eq:hess-op-simple} with the additional parameter term in
Proposition~\ref{prop:layered-Hessian-mixed}, we obtain the Frobenius-geometry bound
\begin{equation}\label{eq:hess-mixed-simple}
  \langle \Delta W,\nabla^2\cL(W)\,\Delta W\rangle
  \;\le\;
  \sum_{\ell} L_\ell^{\mathrm{F}}\|\Delta W_\ell\|_{F}^2
  \;+\;
  C_{\op}(W)\sum_{\ell} \|\Delta W_\ell\|_{\op}^2,
\end{equation}
valid for all $\Delta W$. For mixed Euclidean/spectral updates that use spectral geometry only
on the subset $\mathcal S$, we will also use the corresponding mixed bound
\[
  \langle \Delta W,\nabla^2\cL(W)\,\Delta W\rangle
  \;\le\;
  \sum_{\ell\in\mathcal S} L_\ell^{\mathrm{op}}\|\Delta W_\ell\|_{\op}^2
  \;+\;
  \sum_{\ell\notin\mathcal S} L_\ell^{\mathrm{F}}\|\Delta W_\ell\|_{F}^2
  \;+\;
  C_{\op}(W)\sum_{\ell} \|\Delta W_\ell\|_{\op}^2,
\]
which is obtained by applying the spectral feature bound on blocks in $\mathcal S$, the
Frobenius feature bound on blocks not in $\mathcal S$, and the same parameter-only term on all
blocks.

\paragraph{Definition of the GD and spectral updates.}
We now define a full Euclidean (Frobenius) gradient step and a mixed spectral step by choosing
blockwise step sizes that are natural for the quadratic model \eqref{eq:taylor-simple} under
the mixed Hessian bounds above.

For the Euclidean step we set, for each block $\ell$,
\[
  a_{\mathrm{GD},\ell}
  := L_\ell^{\mathrm{F}} + \frac{C_{\op}(W)}{\st(G_\ell)},
  \qquad
  \st(G_\ell) := \frac{\|G_\ell\|_F^2}{\|G_\ell\|_{\op}^2},
\]
and define
\[
  W_\ell^{\mathrm{GD}}
  := W_\ell - \frac{1}{a_{\mathrm{GD},\ell}}\,G_\ell,
  \qquad
  \Delta W_\ell^{\mathrm{GD}} := W_\ell^{\mathrm{GD}}-W_\ell
  = -\frac{1}{a_{\mathrm{GD},\ell}}\,G_\ell,
\]
so that $W^{\mathrm{GD}} := (W_1^{\mathrm{GD}},\dots,W_L^{\mathrm{GD}})$ is the full GD update.

For the spectral step we set, for each $\ell\in\mathcal S$,
\[
  a_{\mathrm{Spec},\ell}
  := L_\ell^{\mathrm{op}} + C_{\op}(W),
\]
and define
\[
  W_\ell^{\mathrm{Spec}}
  := W_\ell - \frac{\|G_\ell\|_*}{a_{\mathrm{Spec},\ell}}\,\mathrm{polar}(G_\ell),
  \qquad
  \Delta W_\ell^{\mathrm{Spec}} := W_\ell^{\mathrm{Spec}}-W_\ell
  = -\frac{\|G_\ell\|_*}{a_{\mathrm{Spec},\ell}}\,\mathrm{polar}(G_\ell),
\]
where $\mathrm{polar}(G_\ell)$ is the orthogonal factor in the polar decomposition of $G_\ell$.
For blocks $\ell\notin\mathcal S$ we simply set $W_\ell^{\mathrm{Spec}}:=W_\ell^{\mathrm{GD}}$.
In this way $W^{\mathrm{Spec}} := (W_1^{\mathrm{Spec}},\dots,W_L^{\mathrm{Spec}})$ is the full
mixed update which uses spectral geometry on blocks in $\mathcal S$ and Euclidean geometry on
the others.

\paragraph{One-step bound for the Euclidean step.}
For the Euclidean update, the linear term in \eqref{eq:taylor-simple} is
\[
  \langle G_\ell,\Delta W_\ell^{\mathrm{GD}}\rangle
  = -\frac{1}{a_{\mathrm{GD},\ell}}\|G_\ell\|_{F}^2.
\]
For the quadratic term we use \eqref{eq:hess-mixed-simple} together with
\[
  \|\Delta W_\ell^{\mathrm{GD}}\|_F^2
  = \frac{1}{a_{\mathrm{GD},\ell}^2}\|G_\ell\|_F^2,
  \qquad
  \|\Delta W_\ell^{\mathrm{GD}}\|_{\op}^2
  = \frac{1}{a_{\mathrm{GD},\ell}^2}\|G_\ell\|_{\op}^2
  = \frac{1}{a_{\mathrm{GD},\ell}^2}\frac{\|G_\ell\|_F^2}{\st(G_\ell)}.
\]
Thus the contribution of block $\ell$ to the quadratic term is bounded by
\[
  \frac12\left(
    L_\ell^{\mathrm{F}}\frac{\|G_\ell\|_F^2}{a_{\mathrm{GD},\ell}^2}
    + C_{\op}(W)\frac{\|G_\ell\|_F^2}{a_{\mathrm{GD},\ell}^2\,\st(G_\ell)}
  \right)
  =
  \frac{a_{\mathrm{GD},\ell}}{2}\,\frac{\|G_\ell\|_F^2}{a_{\mathrm{GD},\ell}^2}
  = \frac{1}{2a_{\mathrm{GD},\ell}}\|G_\ell\|_F^2.
\]
Plugging $\Delta W^{\mathrm{GD}}$ into \eqref{eq:taylor-simple} and summing over $\ell$ gives
\[
  \cL(W^{\mathrm{GD}})
  \;\le\;
  \cL(W)
  - \sum_{\ell=1}^L \frac{1}{a_{\mathrm{GD},\ell}}\|G_\ell\|_F^2
  + \frac12\sum_{\ell=1}^L \frac{1}{a_{\mathrm{GD},\ell}}\|G_\ell\|_F^2
  + R(\Delta W^{\mathrm{GD}}),
\]
so the quadratic model predicts the total decrease
\begin{equation}\label{eq:GD-descent}
  \Delta_{\mathrm{GD}}
  := \cL(W) - \cL(W^{\mathrm{GD}})
  \;\ge\;
  \frac12\sum_{\ell=1}^L \frac{\|G_\ell\|_F^2}{a_{\mathrm{GD},\ell}}
  =
  \frac12\sum_{\ell=1}^L
  \frac{\|G_\ell\|_F^2}{L_\ell^{\mathrm{F}} + C_{\op}(W)/\st(G_\ell)}.
\end{equation}

\paragraph{One-step bound for the spectral step.}
For the spectral update on $\ell\in\mathcal S$ we have
\[
  \langle G_\ell,\Delta W_\ell^{\mathrm{Spec}}\rangle
  = -\frac{\|G_\ell\|_*^2}{a_{\mathrm{Spec},\ell}},
  \qquad
  \|\Delta W_\ell^{\mathrm{Spec}}\|_{\op}^2
  = \frac{\|G_\ell\|_*^2}{a_{\mathrm{Spec},\ell}^2},
\]
and for $\ell\notin\mathcal S$ the update coincides with the Euclidean one, so the contribution
to the linear and quadratic terms on those blocks is exactly the same as in the GD case. Using
the mixed Hessian bound (spectral geometry on $\mathcal S$, Frobenius geometry on its
complement), the contribution of block $\ell\in\mathcal S$ to the quadratic term is bounded by
\[
  \frac12\left(
    L_\ell^{\mathrm{op}}\frac{\|G_\ell\|_*^2}{a_{\mathrm{Spec},\ell}^2}
    + C_{\op}(W)\frac{\|G_\ell\|_*^2}{a_{\mathrm{Spec},\ell}^2}
  \right)
  =
  \frac{a_{\mathrm{Spec},\ell}}{2}\,\frac{\|G_\ell\|_*^2}{a_{\mathrm{Spec},\ell}^2}
  = \frac{1}{2a_{\mathrm{Spec},\ell}}\|G_\ell\|_*^2.
\]
Proceeding as above and summing over all blocks, we obtain the total decrease guarantee
\begin{align}\label{eq:Spec-descent}
  \Delta_{\mathrm{Spec}}
  := \cL(W) - \cL(W^{\mathrm{Spec}})
  \;&\ge\;
  \frac12\sum_{\ell\in\mathcal S}
  \frac{\|G_\ell\|_*^2}{a_{\mathrm{Spec},\ell}}
  \;+\;
  \frac12\sum_{\ell\notin\mathcal S} \frac{\|G_\ell\|_F^2}{a_{\mathrm{GD},\ell}}\notag\\
  &=
  \frac12\sum_{\ell\in\mathcal S}
  \frac{\|G_\ell\|_*^2}{L_\ell^{\mathrm{op}} + C_{\op}(W)}
  \;+\;
  \frac12\sum_{\ell\notin\mathcal S}
  \frac{\|G_\ell\|_F^2}{L_\ell^{\mathrm{F}} + C_{\op}(W)/\st(G_\ell)}.
\end{align}

\paragraph{Nuclear-rank vs.\ stable-rank condition.}
Comparing the global guarantees \eqref{eq:GD-descent} and \eqref{eq:Spec-descent}, we see that
whenever we keep the GD update on blocks $\ell\notin\mathcal S$, the total spectral decrease
is guaranteed to be at least as large as the Euclidean decrease provided that, for every
$\ell\in\mathcal S$,
\[
  \frac{\|G_\ell\|_*^2}{L_\ell^{\mathrm{op}} + C_{\op}(W)}
  \;\ge\;
  \frac{\|G_\ell\|_F^2}{L_\ell^{\mathrm{F}} + C_{\op}(W)/\st(G_\ell)}
  \quad\Longleftrightarrow\quad
  \frac{\|G_\ell\|_*^2}{\|G_\ell\|_F^2}
  \;\ge\;
  \frac{L_\ell^{\mathrm{op}} + C_{\op}(W)}{L_\ell^{\mathrm{F}} + C_{\op}(W)/\st(G_\ell)}.
\]
Using the feature-based ratio \eqref{eq:stable-rank-ratio-simple}, we may rewrite the right-hand
side in terms of the stable rank of the core activation feeding block~$\ell$. Setting
\[
  s_\ell := \st\!\big(\widetilde A_{\ell-1}(W)\big)
  = \frac{L_\ell^{\mathrm{op}}}{L_\ell^{\mathrm{F}}},
  \qquad
  \alpha_\ell := \frac{C_{\op}(W)}{L_\ell^{\mathrm{F}}},
\]
we have
\[
  \frac{L_\ell^{\mathrm{op}} + C_{\op}(W)}{L_\ell^{\mathrm{F}} + C_{\op}(W)/\st(G_\ell)}
  = \frac{s_\ell + \alpha_\ell}{1 + \alpha_\ell/\st(G_\ell)}.
\]
Thus the per-block condition for the spectral update on block~$\ell\in\mathcal S$ to improve
upon the Euclidean update on that block is
\begin{equation}\label{eq:nr-vs-stable-refined}
  \nr(G_\ell)
  := \frac{\|G_\ell\|_*^2}{\|G_\ell\|_F^2}
  \;\ge\;
  \frac{\st\!\big(\widetilde A_{\ell-1}(W)\big) + \alpha_\ell}{
    1 + \alpha_\ell/\st(G_\ell)}.
\end{equation}

The right-hand side of \eqref{eq:nr-vs-stable-refined} is increasing in the
denominator parameter $\st(G_\ell)$. Since $\st(G_\ell) \le \nr(G_\ell)$ by definition,
we may therefore replace $\st(G_\ell)$ by $\nr(G_\ell)$ in the denominator, which
only \emph{increases} the right-hand side and yields a stronger (more conservative)
but still sufficient condition for the spectral update to dominate GD. With this
replacement, \eqref{eq:nr-vs-stable-refined} is implied by
\begin{equation}\label{eq:nr-vs-stable-strong}
  \nr(G_\ell)
  \;\ge\;
  \frac{\st\!\big(\widetilde A_{\ell-1}(W)\big) + \alpha_\ell}{
    1 + \alpha_\ell/\nr(G_\ell)}.
\end{equation}
A simple algebraic rearrangement shows that the stronger inequality
\eqref{eq:nr-vs-stable-strong} is in fact \emph{equivalent} to the bare inequality
\begin{align}\label{eq:nuc-vs-stable}
  \nr(G_\ell) \;\ge\; \st\!\big(\widetilde A_{\ell-1}(W)\big).
\end{align}
Indeed, writing $r:=\nr(G_\ell)$, $s:=\st(\widetilde A_{\ell-1}(W))$ and
$\alpha:=\alpha_\ell$, the inequality $\,r \ge (s+\alpha)/(1+\alpha/r)$ is equivalent
to $r(1+\alpha/r)\ge s+\alpha$, i.e.\ $r+\alpha\ge s+\alpha$, hence $r\ge s$. 

Summing over all blocks, it follows that whenever the sufficient condition
\eqref{eq:nuc-vs-stable} holds for every $\ell\in\mathcal S$, the quadratic model guarantees
that the full mixed spectral update $W\mapsto W^{\mathrm{Spec}}$ (spectral on $\mathcal S$,
Euclidean elsewhere) achieves at least as much decrease as the full Euclidean update
$W\mapsto W^{\mathrm{GD}}$ (up to the common remainder term $R(\Delta W)$). This is precisely
the same rule that emerged in our random-feature analysis: spectral updates are most
advantageous on those blocks whose incoming core activations $\widetilde A_{\ell-1}(W)$ have
low stable rank while the corresponding gradients $G_\ell$ have large nuclear rank.

\subsection{Extension: diagonal RMS-normalization weights}
\label{sec:rms-diagonal}

In the layered model, the blocks $W_\ell$ were arbitrary matrices, and the spectral step used
the full polar factor of $\nabla_{W_\ell}\cL(W)$. The RMS-normalization layers are more
structured: their weights are diagonal, and in this case the spectral steepest–descent step
reduces to a ``signSGD''-type update on the RMSNorm weights, with a corresponding
$\ell_1/\ell_2$ analogue of the nuclear-rank condition. We record this diagonal case here.

Each RMSNorm layer carries weights $\gamma\in\R^d$ and acts on a hidden state $X\in\R^{d\times n}$ as
\[
  \mathrm{RMSNorm}_\gamma(X)
  \;=\;
  \Diag(\gamma)\,\widetilde A^{\mathrm{rms}}(X),
  \qquad
  A := \widetilde A^{\mathrm{rms}}(X)\in\R^{d\times n},
\]
where $\widetilde A^{\mathrm{rms}}(X)$ is the parameter–free RMS-normalized activation. In our
notation we write $\Gamma := \Diag(\gamma)$ and regard it as a diagonal block with incoming
activations $A$. A perturbation of the RMSNorm weights has the form
$\Delta\Gamma = \Diag(\delta\gamma)$ and appears in the same feature term
$\|\Delta\Gamma A\|_F^2$ as any other block in the Hessian bound.

For diagonal $\Delta\Gamma$ we have
\[
  \|\Delta\Gamma A\|_F
  \;\le\;
  \|\Delta\Gamma\|_F\,\|A\|_{\op}
  =
  \|\delta\gamma\|_2\,\|A\|_{\op},
  \qquad
  \|\Delta\Gamma A\|_F
  \;\le\;
  \|\Delta\Gamma\|_{\op}\,\|A\|_F
  =
  \|\delta\gamma\|_\infty\,\|A\|_F.
\]
Thus, at the level of the feature term, the RMSNorm block has the same two curvature constants
as any matrix block,
\[
  L^{\mathrm F}_{\mathrm{rms}}(W) = C_F(W)\,\|A\|_{\op}^2,
  \qquad
  L^{\mathrm op}_{\mathrm{rms}}(W) = C_F(W)\,\|A\|_F^2,
\]
and their ratio is
\[
  \frac{L^{\mathrm op}_{\mathrm{rms}}(W)}{L^{\mathrm F}_{\mathrm{rms}}(W)}
  =
  \frac{\|A\|_F^2}{\|A\|_{\op}^2}
  =
  \st(A).
\]

Let $g := \nabla_\gamma\cL(W)\in\R^d$ be the gradient with respect to the RMSNorm weights. At
the level of matrix blocks, this corresponds to the diagonal matrix
$G_{\mathrm{diag}} := \Diag(g)$, whose Frobenius and operator norms are
\[
  \|G_{\mathrm{diag}}\|_F^2 = \|g\|_2^2,
  \qquad
  \|G_{\mathrm{diag}}\|_{\op}^2 = \|g\|_\infty^2.
\]
The analogue of the stable rank of the gradient in this diagonal setting is therefore
\[
  \st_{\mathrm{diag}}(g)
  := \frac{\|G_{\mathrm{diag}}\|_F^2}{\|G_{\mathrm{diag}}\|_{\op}^2}
  = \frac{\|g\|_2^2}{\|g\|_\infty^2}.
\]

For this diagonal block we can repeat exactly the one-step analysis from
Section~\ref{subsec:descent-comparison}, now with the gradient represented by
$G_{\mathrm{diag}}$ and curvature constants $L^{\mathrm F}_{\mathrm{rms}}(W)$ and
$L^{\mathrm op}_{\mathrm{rms}}(W)$. In particular, in analogy with
\eqref{eq:GD-descent}–\eqref{eq:Spec-descent}, we define blockwise step sizes
\[
  a_{\mathrm{GD},\mathrm{rms}}
  := L^{\mathrm F}_{\mathrm{rms}}(W)
     + \frac{C_{\op}(W)}{\st_{\mathrm{diag}}(g)},
  \qquad
  a_{\mathrm{Spec},\mathrm{rms}}
  := L^{\mathrm op}_{\mathrm{rms}}(W) + C_{\op}(W),
\]
and consider the two updates
\[
  \gamma_{\mathrm{GD}}^{+}
  := \gamma - \frac{1}{a_{\mathrm{GD},\mathrm{rms}}}\,g,
  \qquad
  \gamma_{\mathrm{Spec}}^{+}
  := \gamma - \frac{\|g\|_1}{a_{\mathrm{Spec},\mathrm{rms}}}\,\mathrm{sign}(g).
\]

For the Euclidean update, the same quadratic-model calculation as in
\eqref{eq:GD-descent} (with $G_\ell$ replaced by $G_{\mathrm{diag}}$ and
$L_\ell^{\mathrm F}$ by $L^{\mathrm F}_{\mathrm{rms}}$) gives the one-step bound
\[
  \Delta_{\mathrm{GD},\mathrm{rms}}
  \;\ge\;
  \frac{\|g\|_2^2}{2\left(
    L^{\mathrm F}_{\mathrm{rms}}(W) + C_{\op}(W)/\st_{\mathrm{diag}}(g)
  \right)}.
\]
For the spectral update, the same argument as in \eqref{eq:Spec-descent}
(using that $\|\Diag(\mathrm{sign}(g))\|_{\op}=1$ and
$\|\Diag(\mathrm{sign}(g))A\|_F\le\|A\|_F$) yields
\[
  \Delta_{\mathrm{Spec},\mathrm{rms}}
  \;\ge\;
  \frac{\|g\|_1^2}{2\left(
    L^{\mathrm op}_{\mathrm{rms}}(W) + C_{\op}(W)
  \right)}.
\]

Thus, at the level of the quadratic model, the spectral update dominates the
Euclidean update on this diagonal block whenever
\[
  \frac{\|g\|_1^2}{L^{\mathrm op}_{\mathrm{rms}}(W) + C_{\op}(W)}
  \;\ge\;
  \frac{\|g\|_2^2}{L^{\mathrm F}_{\mathrm{rms}}(W) + C_{\op}(W)/\st_{\mathrm{diag}}(g)}.
\]
Using $L^{\mathrm op}_{\mathrm{rms}}/L^{\mathrm F}_{\mathrm{rms}} = \st(A)$ and defining
\[
  s_A := \st(A) = \frac{\|A\|_F^2}{\|A\|_{\op}^2},
  \qquad
  \alpha_{\mathrm{rms}} := \frac{C_{\op}(W)}{L^{\mathrm F}_{\mathrm{rms}}(W)},
\]
we can rewrite this condition exactly as in
\eqref{eq:nr-vs-stable-refined}:
\[
  \frac{\|g\|_1^2}{\|g\|_2^2}
  \;\ge\;
  \frac{s_A + \alpha_{\mathrm{rms}}}{1 + \alpha_{\mathrm{rms}}/\st_{\mathrm{diag}}(g)}.
\]

The right-hand side is increasing in $\st_{\mathrm{diag}}(g)$. On the other hand,
for any nonzero $g$ we have the scalar inequality
\[
  \frac{\|g\|_1^2}{\|g\|_2^2}
  \;\ge\;
  \frac{\|g\|_2^2}{\|g\|_\infty^2}
  = \st_{\mathrm{diag}}(g),
\]
so it is sufficient to impose the stronger condition obtained by replacing
$\st_{\mathrm{diag}}(g)$ in the denominator by $\|g\|_1^2/\|g\|_2^2$. This gives
\[
  \frac{\|g\|_1^2}{\|g\|_2^2}
  \;\ge\;
  \frac{s_A + \alpha_{\mathrm{rms}}}{1 + \alpha_{\mathrm{rms}}\|g\|_2^2/\|g\|_1^2}.
\]
A short algebraic manipulation shows that this is equivalent to
\[
  \frac{\|g\|_1^2}{\|g\|_2^2} \;\ge\; s_A.
\]
In other words, the $\ell_1/\ell_2$ ``nuclear rank'' of the RMSNorm gradient must exceed
the (low) stable rank of the incoming RMS-normalized activation matrix $A$. This is the exact
diagonal analogue of the matrix-level condition $\nr(G_\ell)\ge\st(\widetilde A_{\ell-1}(W))$
from \eqref{eq:nuc-vs-stable}, with the role of $\nr(G_\ell)$ played by the scalar
$\|g\|_1^2/\|g\|_2^2$ and the role of $\st(A_{\ell-1})$ played by $\st(A)$.

\bibliographystyle{plain}
\bibliography{bibliography}

\newpage

\appendix

\section{Sparse Regression and SwiGLU activations}\label{sec:swiglu}

We now return to the sparse regression setup of Figure~\ref{fig:synthetic_stable_rank_sparse} and replace the squared-ReLU activation by SwiGLU.  The architecture and target remain unchanged: we train a four-layer feedforward network on the cubic $f(x)=x_1x_2x_3$ from Gaussian inputs using full-batch \texttt{GD} and \texttt{SpecGD}, applying spectral updates only to the hidden layers as in Section~\ref{sec:multiple}.  Our aim is to examine how the change in activation affects the stable rank of the post-activations and, in turn, the relative behavior of Euclidean and spectral updates.

Figure~\ref{fig:swiglu_sparse_sr} reports the stable rank of the post-activation matrices at initialization as a function of batch size.  In contrast with the ReLU and squared-ReLU experiments, the stable ranks now grow noticeably with the batch size, reflecting the fact that SwiGLU is approximately centered and therefore does not induce a large mean spike in the hidden-layer Gram matrices.

To see how this change in stable rank interacts with training dynamics, Figure~\ref{fig:synthetic_run_grad_vs_training_loss_sparse} compares the training loss of \texttt{GD} and \texttt{SpecGD} for two batch sizes.  For the smaller batch ($n=512$, left panel), the network is still in a mildly low–stable–rank regime and \texttt{SpecGD} enjoys a visible early advantage, much as in the ReLU case.  For the larger batch ($n=8196$, right panel), the initialization stable ranks of the hidden layers are substantially higher, and the performance of \texttt{GD} and \texttt{SpecGD} becomes similar.  As the batch size grows, the inequality $\nr(\nabla\cL(W_\ell)) \gtrsim \st(A_{\ell-1})$ is no longer strongly activated, and the structural advantage of spectral updates predicted in Section~\ref{sec:rand_regress} is correspondingly reduced.

\begin{figure}[H]
\centering
\includegraphics[width=.8\textwidth]{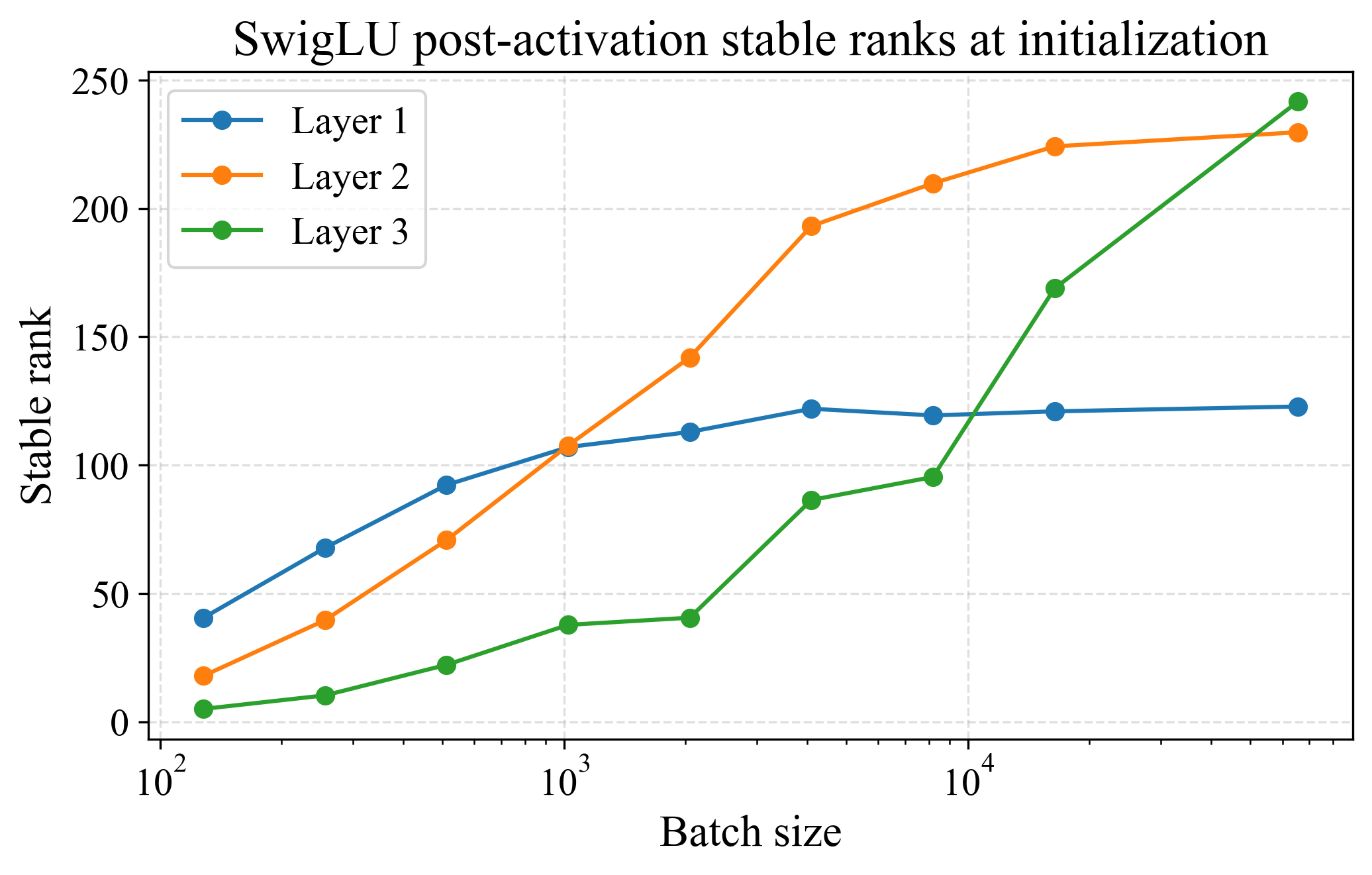}
\caption{Stable rank at initialization of the post-activation matrices in the sparse regression network with SwiGLU activations, as a function of batch size.  For small batches the hidden layers remain only moderately high–stable–rank, but as the batch size increases the stable ranks grow toward their ambient dimension, in contrast with the ReLU and squared-ReLU experiments where the mean-induced spike keeps $\st(A_\ell)$ essentially constant.}
\label{fig:swiglu_sparse_sr}
\end{figure}

\begin{figure}[h]
\centering
\begin{subfigure}[t]{0.48\textwidth}
  \centering
  \includegraphics[width=\linewidth]{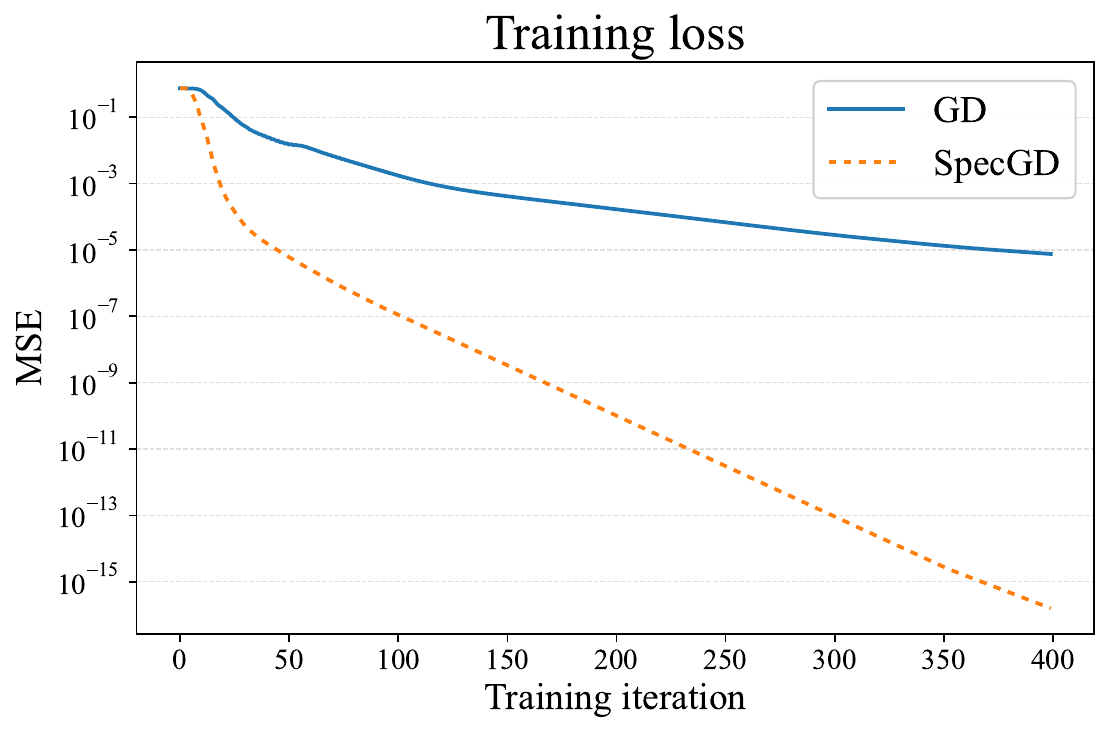}
  \caption{Batch size $n=512$.}
\end{subfigure}\hfill
\begin{subfigure}[t]{0.48\textwidth}
  \centering
  \includegraphics[width=\linewidth]{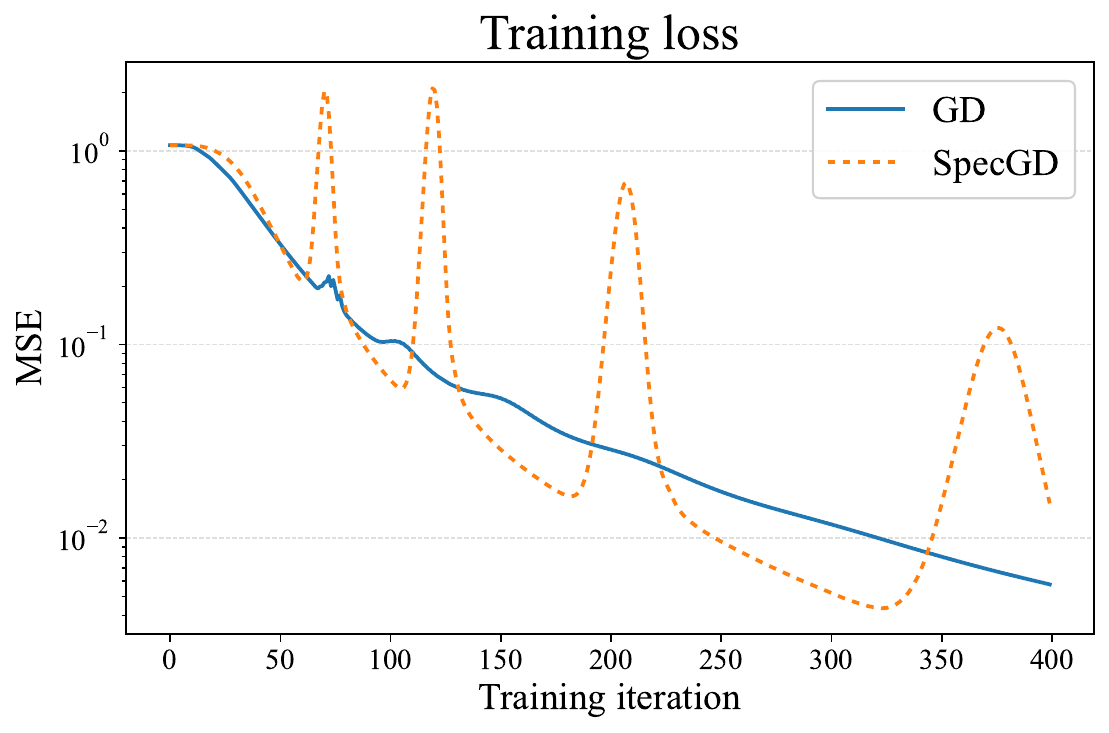}
  \caption{Batch size $n=8196$.}
\end{subfigure}
\caption{Sparse regression with SwiGLU activations: training loss for gradient descent (\texttt{GD}) and spectral gradient descent (\texttt{SpecGD}) at two batch sizes.  For $n=512$ (left) the methods behave similarly to the ReLU experiment, with \texttt{SpecGD} exhibiting a faster initial decrease in loss.  For $n=8196$ (right), where the initialization stable ranks of the hidden layers are much larger (Figure~\ref{fig:swiglu_sparse_sr}), the trajectories of \texttt{GD} and \texttt{SpecGD} essentially coincide.  In both runs, the stable ranks of the first two hidden layers remain close to their initialization values, while the third layer drops rapidly to stable rank $\approx 3$; this drop occurs in the final scalar-output layer, where the weights form a vector and no spectral update is applied in \texttt{SpecGD}.}
\label{fig:synthetic_run_grad_vs_training_loss_sparse}
\end{figure}

\section{NanoGPT with SwiGLU activations}\label{sec:swiglu-nanogpt}

We repeat the NanoGPT-scale run from Figures~\ref{fig:mlpnano}--\ref{fig:gradientattentionvo}, but replace the squared-ReLU MLP nonlinearity by SwiGLU.  The figures below report the same stable-rank and gradient nuclear-rank diagnostics, and are meant as direct pointwise analogues of the corresponding modded-NanoGPT plots in the main text.

\begin{figure}[t]
\centering
\includegraphics[width=\linewidth]{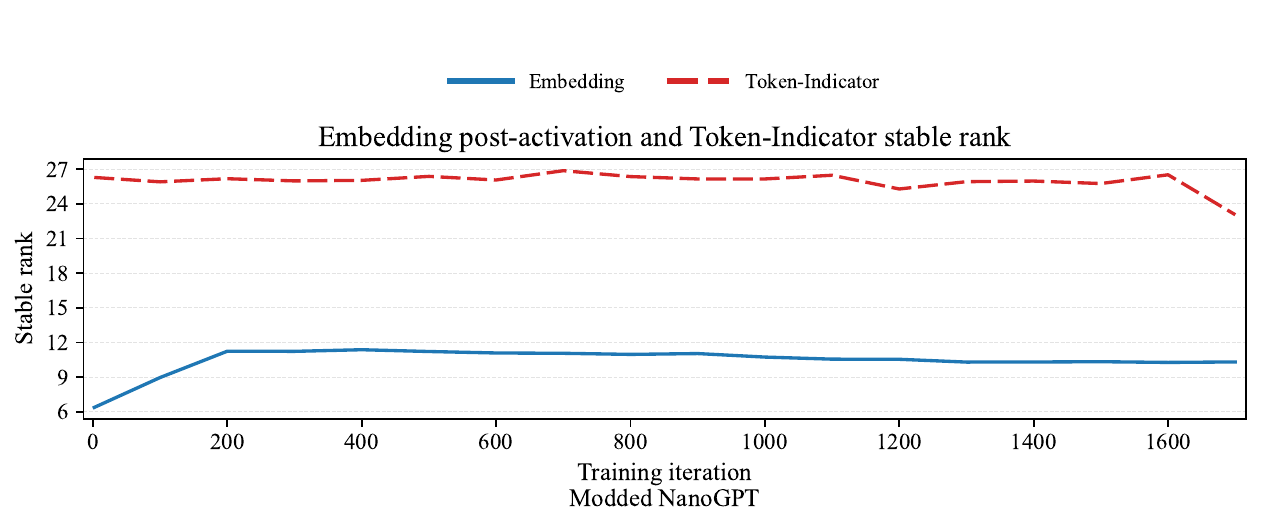}
\caption{Embedding post-activation and token-indicator stable ranks in the SwiGLU run.  This is the SwiGLU analogue of Figure~\ref{fig:embedding}.  As expected, these curves are essentially unchanged by the MLP activation swap: the token-indicator stable rank stays near $1/p_{\max}$ (here $\approx 26$), and the embedding stable rank remains $\approx 10$ throughout training.}
\label{fig:swiglu_nano_embedding}
\end{figure}

\begin{figure}[t]
\centering
\includegraphics[width=\linewidth]{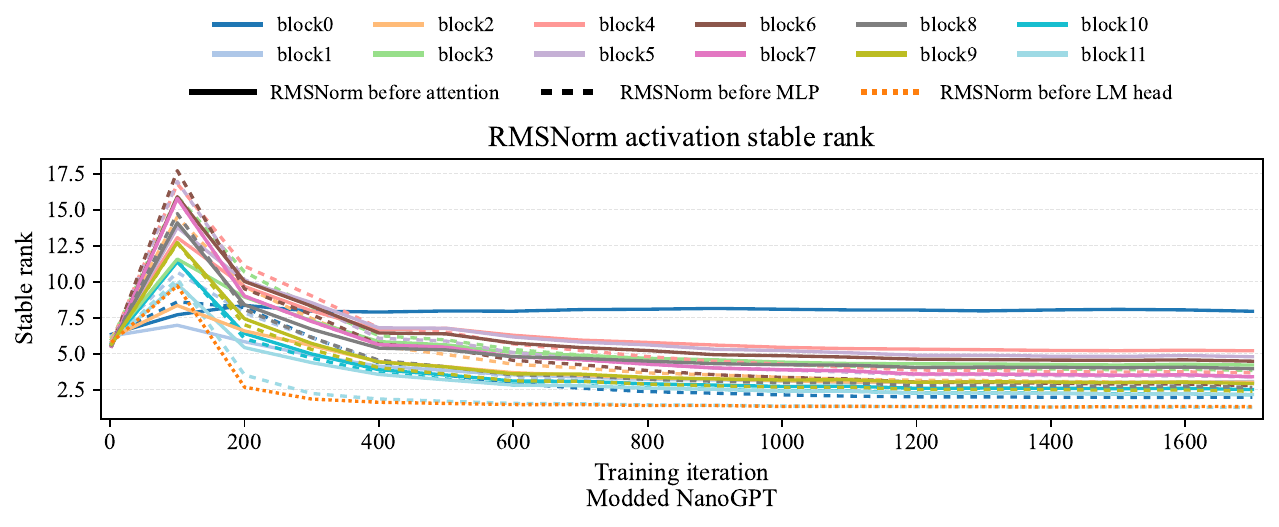}
\caption{Stable rank of RMS-normalized activations $A^{\mathrm{rms}},A^{\mathrm{rms}}_{\mathrm{mlp}},A^{\mathrm{rms}}_{\mathrm{final}}$ in the SwiGLU run.  This is the SwiGLU analogue of Figure~\ref{fig:rmsactivations}.  The RMS features remain low--stable--rank (typically single digits after burn-in), but their stable ranks are systematically larger than in the squared-ReLU run, with a more pronounced early transient before settling.}
\label{fig:swiglu_nano_rms}
\end{figure}

\begin{figure}[t]
\centering
\includegraphics[width=\linewidth]{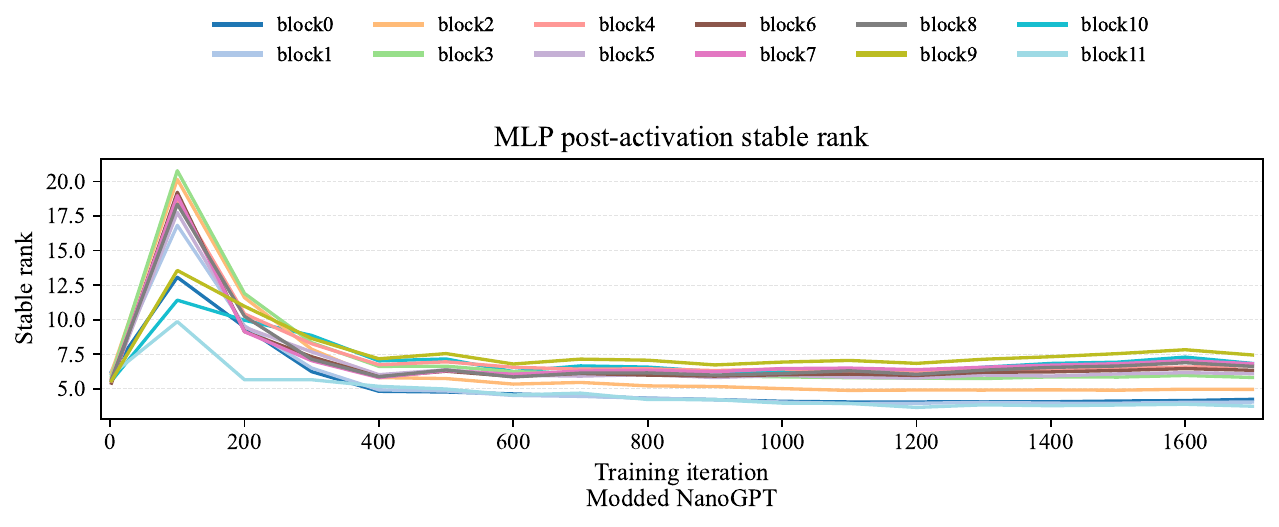}
\caption{Stable rank of MLP post-activations in the SwiGLU run.  This is the SwiGLU analogue of Figure~\ref{fig:mlpnano} (modded-NanoGPT with squared ReLU).  The maximal possible stable rank is $3072$, and the observed values remain far below this ceiling throughout training.  Compared to Figure~\ref{fig:mlpnano}, the post-activation stable ranks are noticeably larger (typically $\approx 5$--$8$ after burn-in, versus $\approx 2$--$3$ in the squared-ReLU run), consistent with the absence of a large mean-induced spike under SwiGLU.}
\label{fig:swiglu_nano_mlp}
\end{figure}

\begin{figure}[t]
\centering
\includegraphics[width=\linewidth]{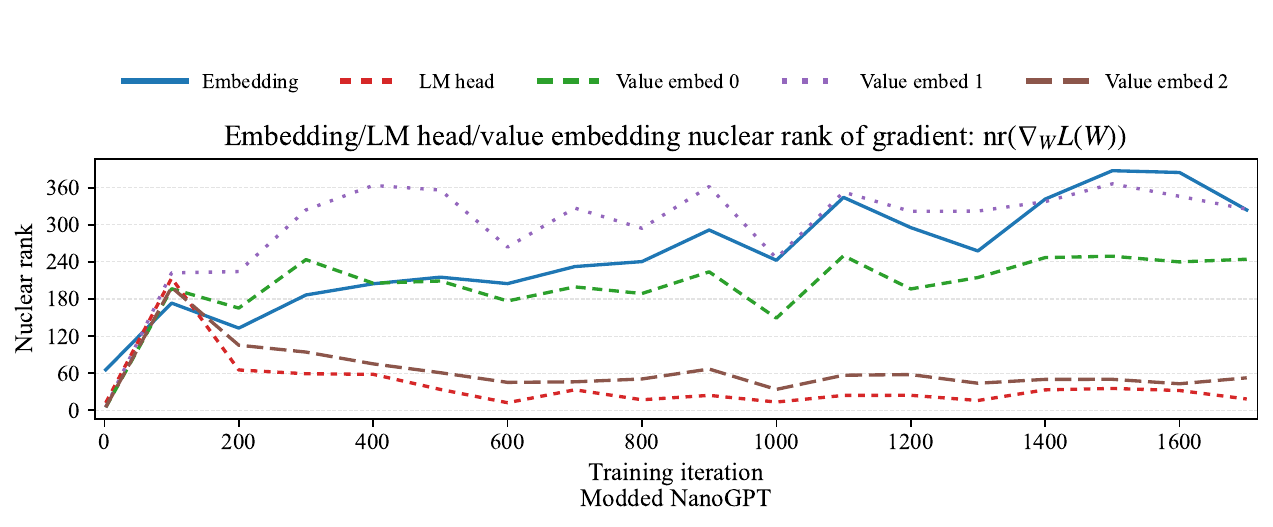}
\caption{Gradient nuclear ranks for the token embedding and language-model head parameters in the SwiGLU run, together with the value-embedding parameters tracked by the same logging code.  This is the SwiGLU analogue of Figure~\ref{fig:gradientembeddingslmhead}.  The embedding gradient nuclear rank grows to the hundreds while the LM-head nuclear rank remains in the tens after burn-in; unlike the squared-ReLU run, the LM-head curve exhibits a larger early transient before relaxing.}
\label{fig:swiglu_nano_grad_embed}
\end{figure}

\begin{figure}[t]
\centering
\includegraphics[width=\linewidth]{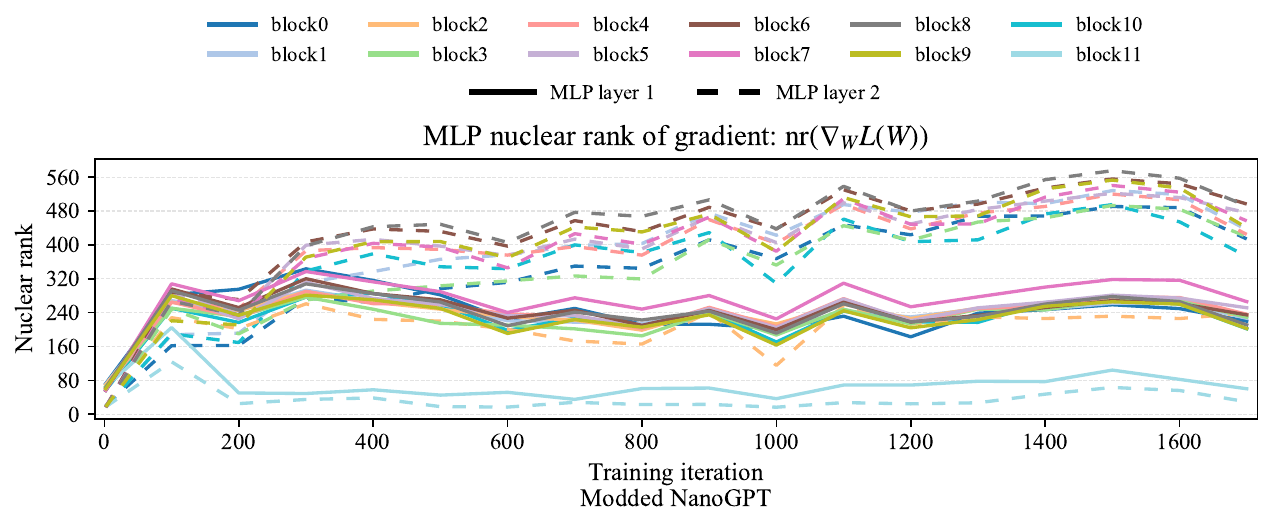}
\caption{MLP gradient nuclear ranks in the SwiGLU run (solid: $W_1$; dashed: $W_2$).  This is the SwiGLU analogue of Figure~\ref{fig:mlpgrad}.  As in the squared-ReLU run, nuclear ranks for the MLP gradients remain large (hundreds) throughout training, with the $W_2$ blocks typically larger than the $W_1$ blocks.}
\label{fig:swiglu_nano_grad_mlp}
\end{figure}

\begin{figure}[t]
\centering
\includegraphics[width=\linewidth]{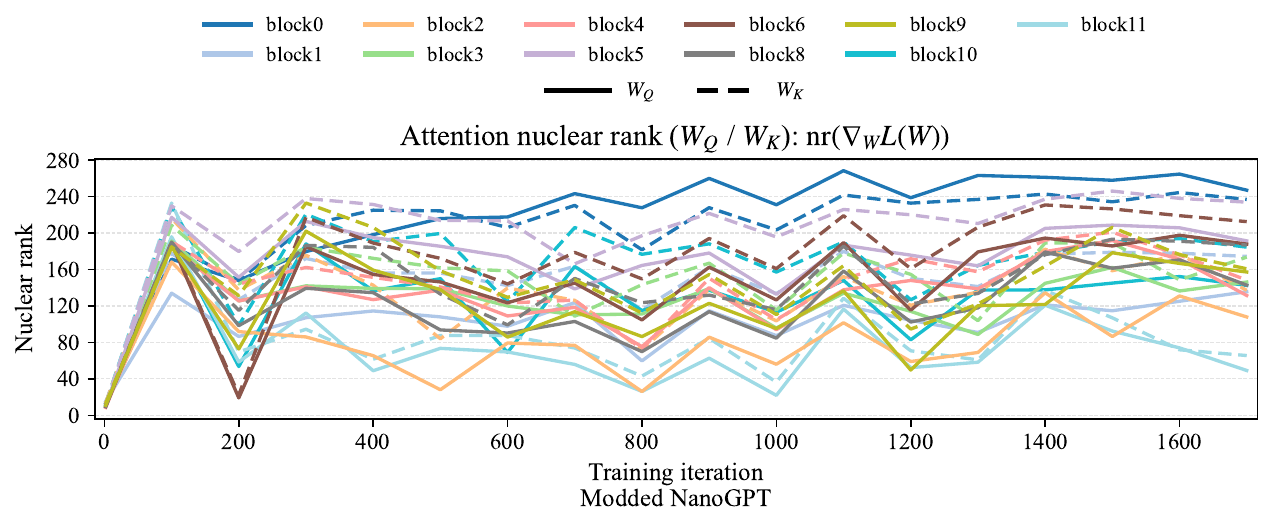}
\caption{Gradient nuclear ranks for the attention parameters $W_Q,W_K$ in the SwiGLU run.  This is the SwiGLU analogue of Figure~\ref{fig:gradientattentionqk}.  Many of the nuclear ranks remain in the hundreds across blocks with modest variation over training.}
\label{fig:swiglu_nano_grad_qk}
\end{figure}

\begin{figure}[t]
\centering
\includegraphics[width=\linewidth]{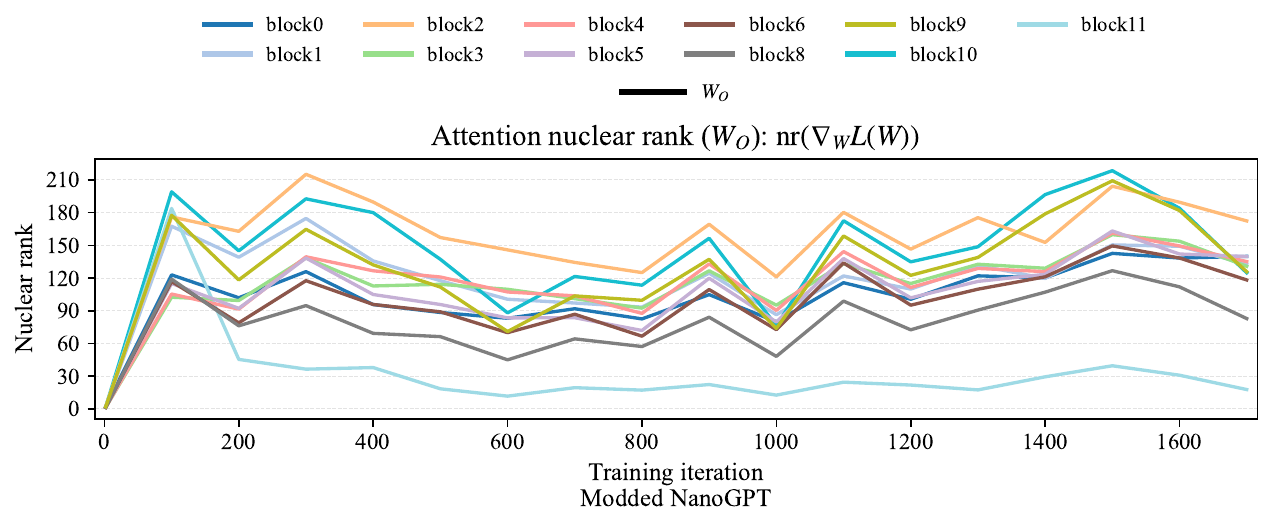}\vspace{1ex}
\includegraphics[width=\linewidth]{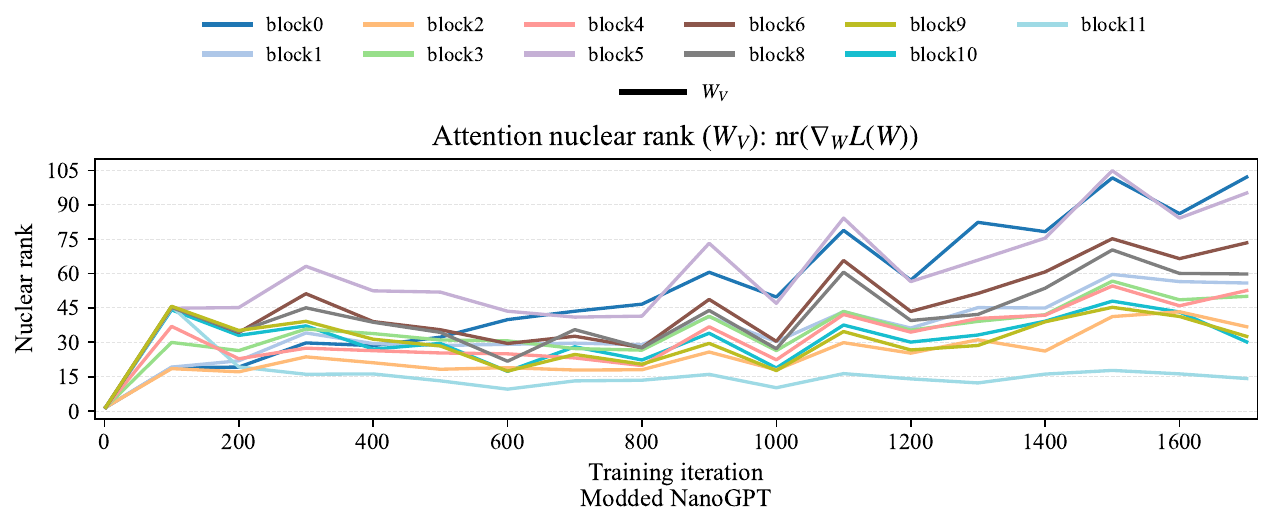}
\caption{Gradient nuclear ranks for the attention parameters $W_O$ (top) and $W_V$ (bottom) in the SwiGLU run.  This is the SwiGLU analogue of Figure~\ref{fig:gradientattentionvo}.  As in the squared-ReLU run, $W_O$ typically has larger nuclear rank than $W_V$, while both remain well above the single-digit stable ranks of the corresponding RMS inputs (Figure~\ref{fig:swiglu_nano_rms}).}
\label{fig:swiglu_nano_grad_ov}
\end{figure}

\section{Trace and operator-norm bounds for the sample second moment under sub-Gaussian tails}
 In this section, we derive a few basic deviation inequalities between the operator/Frobenius norms of the sample and population second-moment matrices. There is a rich literature on bounds of this type, notably \cite{koltchinskii2017covariance,hsu2012quadratic,laurent2000adaptive,tropp2012user,vershynin2012samplecov,srivastava2013covariance}.  Our notation will closely follow the text \cite{vershynin2018high}. 
In particular, $\|\cdot\|_{\psi_2}$ will denote the sub-Gaussian norm, while the sub-exponential norm will be denoted by $\|\cdot\|_{\psi_1}$. More generally, the quasinorm $\|z\|_{\psi_{\theta}}$ for any $\theta>0$ is the infimum over all constants $K>0$ satisfying $\EE\exp(|z/K|^{\theta})\leq 2$.

\begin{thm}[Trace bound]\label{thm:trace}
Consider a sequence of random vectors $z_1,\ldots, z_n\in \R^d$ and define  $S:=\sum_{i=1}^n z_iz_i^{\top}$. Suppose that for each $i\in [n]$, there exists $K_i<\infty$ satisfying
\begin{equation}\label{eq:SG}
    \|\langle z_i,u\rangle\|^2_{\psi_2}\leq K_i^2\EE\langle z_i,u\rangle^2\qquad \forall u\in\R^d.
\end{equation}
Then there exists a constant $c<\infty$ such that for any $\varepsilon\ge 0$, the estimate $\Big|\tr(S)-\tr(\EE S)\Big|
\ \le\ \varepsilon\tr(\EE S)$ holds with probability at least 
$1-2\exp\left(-c\left(\frac{\varepsilon^2\tr(S)^2}{\sum_{i=1}^n (K_i^2\tr(M_i))^2}\wedge \frac{\varepsilon \tr(S)}{\displaystyle\max_{i=1,\ldots, n} K_i^2\tr(M_i)}\right)\right)$.
\end{thm}
\begin{proof}
Set $\mu_i=\EE z_i$ and write $z_i=\mu_i+\tilde z_i$ with $\tilde z_i$ i.i.d., mean zero. Then we can  write:
\[
\tr(S)-\tr(\EE S)
=\sum_{i=1}^n\big(\|z_i\|^2_2-\EE\|z_i\|^2_2\big)
=\underbrace{2\sum_{i=1}^n \langle \tilde z_i,\mu_i\rangle}_{=:A_n}
+\underbrace{\sum_{i=1}^n\big(\|\tilde z_i\|^2-\EE\|\tilde z_i\|_2^2\big)}_{=:B_n}.
\]
We bound $A_n$ and $B_n$ in order. First,  by \eqref{eq:SG}, $\langle \tilde z_i,\mu_i\rangle$ is sub-Gaussian with
$\|\langle \tilde z_i,\mu_i\rangle\|^2_{\psi_2}\le K_i^2{\mu^\top_i M_i\,\mu_i}\leq K_i^2{\|M_i\|_{\op}}\|\mu_i\|^2_{2}\leq K_i^2 \|M_i\|_{\op}^2$, where we set $M_i:=\EE z_iz_i^{\top}$. Therefore, Hoeffding inequality \cite[Theorem 2.2.1]{vershynin2018high}, implies that for all $u> 0$, the estimate 
$
A_n\ \le\ \frac{1}{4}u\tr(S),
$
holds 
with probability at least $1-2\exp\left(-c\left(\frac{t^2\tr(S)^2}{{\sum_{i=1}^n K_i^2\|M_i\|^2_{\op}}}\right)\right)$. 

Next applying $\eqref{eq:SG}$ with $u=e_j$ we see that the $j$'th coordinate of $\tilde z_i$ is $K_i (M_i)_{jj}^{1/2}$ sub-Gaussian, and therefore its square is $K^2_i (M_i)_{jj}$ sub-exponential \cite[Lemma 2.8.5]{vershynin2018high}. Therefore by the triangle inequality, $\|\tilde z\|^2_2$ is sub-exponential with parameter $K^2_i\tr(M_i)$. Applying Bernstein's inequality \cite[Theorem 2.9.1]{vershynin2018high}, we deduce that for any $u\geq 0$ the estimate
$|B_n|\leq u\tr(S)$ holds with probability at least $1-2\exp\left(-c\left(\frac{u^2\tr(S)^2}{\sum_{i=1}^n (K_i^2\tr(M_i))^2}\wedge \frac{u \tr(S)}{\displaystyle\max_{i=1,\ldots, n} K_i^2\tr(M_i)}\right)\right)$, thus completing the proof.
\end{proof}

\begin{thm}[Lower tail for the operator norm for non-iid data]
\label{thm:lower_tail_top_eig}
Let $g_1,\dots,g_n\in\mathbb{R}^d$ be independent random vectors and define
\[
M_i := \mathbb{E}[g_i g_i^\top],\qquad S := \sum_{i=1}^n g_i g_i^\top,
\qquad
M := \mathbb{E}S = \sum_{i=1}^n M_i.
\]
Assume there exist $K_i\ge 1$ such that for every $u\in\mathbb{R}^d$, we have
\begin{equation}
\label{eq:sg_assump_Mt}
\|\langle u,g_i\rangle\|_{\psi_2}\ \le\ K_i\,\sqrt{u^\top M_i u},
\qquad i=1,\dots,n,
\end{equation}
Then there exists an absolute constant $c>0$ such that for any $\varepsilon\in(0,1)$, we have
\begin{equation}\label{eqn:main_lower_oper}
\displaystyle
\mathbb{P}\!\left(\,\|S\|_{\mathrm{op}} \le (1-\varepsilon)\|M\|_{\mathrm{op}}\,\right)
\ \le\
\exp\!\left(
-c\,\min\!\left\{\frac{\varepsilon^2\|M\|^2_{\op}}{\sum_{i=1}^n K_i^4 \|M_i\|_{\op}^2},\,\frac{\varepsilon\|M\|_{\op}}{\displaystyle\max_{1\leq i\leq n} K_i^2\|M_i\|_{\op}}\right\}
\right).
\end{equation}
\end{thm}
\begin{proof}
Let $v\in\mathbb{R}^d$ be a unit eigenvector of $M$ associated with $\|M\|_{\op}$ and set
\[
a_i := v^\top M_i v = \mathbb{E}(v^\top g_i)^2.\] Clearly, we have
$\|S\|_{\mathrm{op}}\ge\ v^\top S v=\sum_{i=1}^n (v^\top g_i)^2.$
The sub-Gaussian assumption \eqref{eq:sg_assump_Mt} shows
$$\|\langle v,g_i\rangle^2-\EE\langle v,g_i\rangle^2\|_{\psi_1}\lesssim\|\langle v,g_i\rangle^2\|_{\psi_1}\asymp\|\langle v,g_i\rangle\|^2_{\psi_2}=K_i^2 a_i,$$
where the first inequality follows from \cite[Exercise 2.44]{vershynin2018high} while the first equality follows from \cite[Lemma 2.8.5]{vershynin2018high}.
Applying the Bernstein inequality \cite[Theorem 2.9.1]{vershynin2018high} gives the estimate
\[
\mathbb{P}\!\left\{\sum_{t=1}^n ((v^\top g_i)^2-\EE  (v^\top g_i)^2)\le -s\right\}
\ \le\
\exp\!\left(
-c\,\min\!\left\{\frac{s^2}{\sum_{t=1}^n (K_i^2 a_i)^2},\,\frac{s}{\max_i (K_i^2 a_i)}\right\}
\right),
\]
for all $s\ge 0$, with an absolute constant $c>0$. Setting $s=\varepsilon\|M\|_{\op}$  completes the proof of \eqref{eqn:main_lower_oper}.
\end{proof}

\begin{lem}[Lower-bounding the operator norm]\label{lem:lower_op_norm}
Let $v_1,\ldots, v_n\in\R^d$ be independent realizations of a random vector $v$ in $\R^d$ and suppose that for some values $K,\theta<\infty$ the estimate holds:
$$\|\langle v-\EE v,u \rangle\|_{\psi_{\theta}}\leq K \|u\|_2\qquad \forall u\in \R^d.$$
Define the sum $S=\sum_{i=1}^n v_iv_i^{\top}$.
Let $V\in \R^{n\times d}$ be a matrix whose rows are $v_1,\ldots,v_n$ and define the second-moment matrix $M:=\EE vv^{\top}$. Then there exists a numerical constant $c$ that depends only on $\theta$ such that for any $\varepsilon\in (0,1)$ the estimate $\|S\|_{\op}\geq (1-\varepsilon)\|\EE S\|_{\op}$ holds with probability at least
$$1-2\exp\left(-c\left(\tfrac{n\varepsilon^2\|M\|^2_{\op}}{K^4}\wedge n^{1\wedge \tfrac{\theta}{2}}\left(\tfrac{\varepsilon\|M\|_{\op}}{K^2}\right)^{\theta/2}\right)\right)-2\exp\left(- c\left(\tfrac{n\varepsilon^2\|M\|_{\op}}{K^2}\wedge n^{1\wedge \theta}\left(\tfrac{\varepsilon \sqrt{\|M\|_{\op}}}{K}\right)^{\theta}\right)\right).$$
\end{lem}
\begin{proof}
Let $V\in \R^{n\times d}$ be a matrix whose rows are $v_1,\ldots,v_n$ and define the second-moment matrix $M:=\EE vv^{\top}$
Define the matrix $M_n=\frac{1}{n}V^{\top}V$ and note $\EE[M_n]=M$.
Fix a unit length vector $u\in\R^d$ satisfying  
$$\|M\|_{\op}=u^{\top}Mu=\EE \langle v,u\rangle^2.$$ Define now $\mu:=\EE v$ and $z_i:=\langle v_i-\mu,u\rangle$. We may then write 
\begin{align}
\frac{1}{n}\|V^{\top}V\|_{\op}\geq \frac{1}{n}u^{\top}V^{\top}Vu=\frac{1}{n}\|Vu\|_2^2&=\frac{1}{n}\sum_{i=1}^{n} \langle v_i,u\rangle^2\notag\\
&=\frac{1}{n}\sum_{i=1}^{n} (z_i^2-\EE z_i^2)+2\langle \mu,u\rangle\cdot \frac{1}{n}\sum_{i=1}^n z_i+\|M\|_{\op},\label{eqn:bounded_we_need}
\end{align}
where the last equality follows from algebraic manipulations.
We now bound the first two terms on the right side of \eqref{eqn:bounded_we_need}. Using basic properties of Orlitz quasi-norms we have
$$\|z_i^2-\EE z_i^2\|_{\psi_{\theta/2}}\lesssim \|z_i^2\|_{\psi_{\theta/2}}\asymp \|z_i\|^2_{\psi_\theta}\leq K^2.$$
Therefore Bernstein's inequality \cite[Theorem 3.1]{kuchibhotla2022moving} yields for any $t\geq 0$ the estimate:
$$\PP\left(\left|\frac{1}{n}\sum_{i=1}^{n} z_i^2-\EE z_i^2\right|\geq t \right)\leq 2\exp\left(-c\left(n\tfrac{t^2}{K^4}\wedge n^{1\wedge \tfrac{\theta}{2}}(\tfrac{t}{K^2})^{\theta/2}\right)\right).$$
Another application of Bernstein's inequality \cite[Theorem 3.1]{kuchibhotla2022moving} implies for any $s\geq 0$ the estimate
$$\PP\left(\left|\frac{1}{n}\sum_{i=1}^{n} z_i\right|\geq s \right)\leq 2\exp\left(-c\left(\tfrac{ns^2}{K^2}\wedge n^{1\wedge \theta} \left(\tfrac{s}{K}\right)^{\theta}\right)\right).$$
Note moreover that we may write the second moment as
$M=\mu\mu^{\top}+{\rm Cov}(v)$
and therefore $|\langle \mu,u\rangle|\leq \|\mu\|_2\leq \sqrt{\|M\|_{\op}}$.
Thus setting $s:=\frac{\varepsilon\sqrt{\|M\|_{\op}}}{4}$ and $t:=\frac{\varepsilon}{2}\|M\|_{\op}$ and taking a union bound over the two events we deduce that the lower bound 
$\frac{1}{n}\|V^{\top}V\|_{\op}\geq (1-\varepsilon)\|M\|_{\op}$ with probability at least
$$1-2\exp\left(-c\left(\tfrac{n\varepsilon^2\|M\|^2_{\op}}{K^4}\wedge n^{1\wedge \tfrac{\theta}{2}}\left(\tfrac{\varepsilon\|M\|_{\op}}{K^2}\right)^{\theta/2}\right)\right)-2\exp\left(- c\left(\tfrac{n\varepsilon^2\|M\|_{\op}}{K^2}\wedge n^{1\wedge \theta}\left(\tfrac{\varepsilon \sqrt{\|M\|_{\op}}}{K}\right)^{\theta}\right)\right),$$
which completes the proof.
\end{proof}

\section{Auxiliary linear algebraic results.}

\begin{lem}[Minimal singular value of blocked matrices]\label{lem:block-upper-triangular-sigma-min}
Let $a\in\R$, let $z\in\R^{\,p}$, and let $D\in\R^{(q-1)\times p}$ with $q-1\le p$.  
Define
\[
M:=\begin{bmatrix} a & z^{\mathsf T} \\ 0 & D \end{bmatrix}
\in\R^{\,q\times (p+1)}.
\]
Then assuming $a\neq 0$, we have
$
\sigma_{q}(M)
\;\ge\;
\frac{\min\{|a|,\sigma_{q-1}(D)\}}{1+\|z\|_2/|a|}.
$
\end{lem}

\begin{proof}
We may write $M=AN$, where
\[
A=\begin{bmatrix} a & 0 \\[1mm] 0 & D \end{bmatrix},
\qquad
N=\begin{bmatrix} 1 & a^{-1}z^{\mathsf T} \\[1mm] 0 & I_{p} \end{bmatrix}.
\]
Clearly we have
\begin{equation}\label{eq:mn-product}
\sigma_{q}(M)\;\ge\;\sigma_{q}(A)\,\sigma_{q}(N).
\end{equation}
Because $A$ is block diagonal,
we have
\begin{equation}\label{eq:Amin}
\sigma_{q}(A)=\min\{|a|,\sigma_{q-1}(D)\}.
\end{equation}
The matrix $N$ is invertible with
\[
N^{-1}
=
\begin{bmatrix}
1 & -a^{-1}z^{\mathsf T} \\[1mm]
0 & I_{p}
\end{bmatrix}
= I_{p+1}+E,
\qquad\textrm{where we set}\qquad
E:=\begin{bmatrix}
0 & -a^{-1}z^{\mathsf T}\\[1mm]
0 & 0
\end{bmatrix}.
\]
Hence
$
\|N^{-1}\|_{\op} \le 1+\frac{\|z\|_2}{|a|}$ and therefore
$\sigma_{q}(N)=\|N^{-1}\|_{\op}^{-1}
\ge \frac{1}{1+\|z\|_2/|a|}.$
Substituting \eqref{eq:Amin} and this bound for $\sigma_{q}(N)$ into \eqref{eq:mn-product} yields the claimed result.
\end{proof}

\begin{lem}[Lower bound on the nuclear rank]\label{lem:lower_bd_nuc}
Consider a matrix $M\in \R^{m\times d}$ and a vector $v\in \R^d$. Then for any subset $P\subset [d]$ the estimate holds:
$$\|M\Diag(v)\|_*\geq \left(\min_{i\in P} |v_i|\right)\cdot\|M_T\|_*,$$
where $M_P\in\R^{m\times |T|}$ is the submatrix of $M$ formed by columns indexed by $P$.
\end{lem}
\begin{proof}
Let $e_P\in\R^{d}$ denote the vector with ones in all coordinates indexed by $P$ and zeros elsewhere and set $q:=\min_{i\in P} |v_i|$. Then observe $$M\Diag(v)\Diag(v)^{\top}M^{\top}\succeq q^2\cdot M e_Pe_P^{\top} M^{\top}=q^2\cdot M_PM_P^{\top}.$$
Taking matrix square roots of both sides and taking the trace completes the proof.
\end{proof}

\section{Shardwise \texttt{SpecGD} and distributed orthogonalization}
\label{app:distributed-specgd}

Large training runs typically store parameters, gradients, and optimizer states \emph{sharded across devices}.
For a single matrix-shaped tensor, this means the gradient (or EMA gradient) matrix
$G\in\R^{P\times Q}$ is available only as $S$ disjoint shards across a device group.
SpecGD-style updates orthogonalize $G$ by applying a polar-factor map; the systems question is how to implement
this orthogonalization when $G$ is sharded.
Any method that targets the \emph{global} polar factor must move information across devices in order to realize
the requisite matrix products (e.g.\ to form $GG^\top$ and apply a polynomial in $GG^\top$ to $G$), even if no device ever materializes the full matrix.

There are several standard ways to organize this global computation (all-gather, distributed matmuls inside the
Newton--Schulz loop, or resharding along a batch-like axis such as layers).
This section focuses on an \emph{alternative strategy}: instead of computing the global polar factor,
one can apply the same polar-factor map independently to each local shard. Variants of this strategy under, e.g.,  tensor parallelism appear in \cite{boreiko2025towards,khaled2025muonbp}.
When the sharding splits the dimension that governs the Newton--Schulz cost,
this shardwise option has two immediate quantitative features:
(i) it introduces no \emph{incremental} optimizer communication, and
(ii) it reduces the \emph{per-device} orthogonalization work by a factor $S$ relative to any method that computes
a global polar factor with ideal $1/S$ parallel scaling.
The quadratic model at the end of the section makes this locality explicit by identifying a blockwise polar step
as the exact minimizer of a block-separable upper model.
Before turning to that calculation, we briefly fix the minimal distributed vocabulary used in the cost comparison
(shards and the two standard reshuffling primitives) and recall why Newton--Schulz governs the arithmetic cost.

\paragraph{Communication primitives and sharded matrices.}
Assume $G\in\R^{P\times Q}$ is sharded across $S$ devices.
For concreteness, consider row-sharding:
each device $s\in[S]$ stores a disjoint block of rows $G^{(s)}\in\R^{(P/S)\times Q}$, and $\{G^{(s)}\}_{s\in[S]}$ partitions $G$.
Two group-wide communication primitives recur below:
\begin{itemize}
\item an \emph{all-gather} broadcasts shards so that every device can access all of the distributed data (often materializing it locally);
\item an \emph{all-to-all} permutes shards so that each device ends with a different slice (a \emph{resharding}).
\end{itemize}
Any group-wide operation of this type is called a \emph{collective}.
When we say \emph{incremental optimizer communication}, we mean collectives that are introduced solely to implement
the orthogonalization step, beyond whatever communication the baseline training step already requires.

\paragraph{Why Newton--Schulz enters and what it costs.}
For a matrix $M$, write $\mathrm{polar}(M)$ for the polar factor (equivalently, $UV^\top$ when $M=U\Sigma V^\top$ is an SVD).
Muon applies (or approximates) $\mathrm{polar}(G)$ to a gradient- or momentum-like matrix $G$ \cite{jordan2024muon}.
A standard practical route computes $\mathrm{polar}(G)$ via an inverse square root of a Gram matrix,
e.g.\ $\mathrm{polar}(G)=(GG^\top)^{-1/2}G$ when $Q\ge P$,
and approximates $(GG^\top)^{-1/2}$ with a fixed small number of Newton--Schulz iterations
using only matrix multiplications and additions \cite{jordan2024muon,essentialai2025layer}.
In the regime $Q\ge P$, each iteration is dominated by multiplying a $P\times P$ matrix by a $P\times Q$ matrix,
so the per-iteration arithmetic cost is on the order of $\Theta(P^2Q)$ flops (up to constant factors and the fixed iteration count).

\paragraph{Three ways to orthogonalize a sharded matrix (global) and their costs.}
When $G$ is row-sharded across $S$ devices, there are three natural ways to obtain the \emph{global} polar factor $\mathrm{polar}(G)$:
\begin{enumerate}
\item All-gather then orthogonalize locally.
  Perform an all-gather so every device can access (and typically materialize) $G$, then compute $\mathrm{polar}(G)$ independently on each device.
  This uses one collective, but it duplicates the full $\Theta(P^2Q)$ orthogonalization work across all $S$ devices.

\item Keep $G$ sharded and distribute the Newton--Schulz algebra.
  Keep $G$ row-sharded and implement the Newton--Schulz multiplications as distributed matmuls.
  This avoids duplicating compute, but the loop requires repeated synchronization of intermediate results across devices.
  Concretely, each Newton--Schulz iteration induces one or more collectives, so the communication cost scales with the iteration count.

\item Reshard to make the orthogonalization local, then reshard back.
  If there is an additional batch-like axis indexing many independent matrices (e.g.\ the layer axis in a deep network),
  one can temporarily reshard so that each device holds full $P\times Q$ matrices for a subset of that axis, orthogonalize locally,
  then all-to-all back \cite{essentialai2025layer}.
  This yields a fixed number of reshuffles (two all-to-alls), rather than one collective per Newton--Schulz iteration.
\end{enumerate}

\paragraph{How sharding arises in large-scale training.}
Sharding of a single matrix-shaped tensor typically reduces to a partition of its rows and/or columns:
\begin{itemize}
\item \emph{Tensor parallelism (TP)} splits matrix multiplications across devices; weights are commonly sharded by rows or columns.
\item \emph{Fully-sharded data parallelism (FSDP/ZeRO)} shards parameters and optimizer states across data-parallel ranks; for matrix-shaped tensors this again induces a row/column partition after flattening/packing.
\item \emph{Pipeline parallelism (PP)} partitions layers across devices; this creates a layer index that can act as the batch-like axis exploited by the resharding strategy above.
\item \emph{Expert parallelism (EP)} in MoE models partitions experts across devices; within any single expert matrix, the effective sharding is still row/column.
\end{itemize}
In what follows we treat the sharding as a partition of matrix coordinates, and we emphasize row/column partitions since they are the
canonical sharding patterns for a single matrix.

\paragraph{Shardwise \texttt{SpecGD}: orthogonalize only the local shard.}
As an alternative to the global strategies above, one can define a \emph{shardwise} direction by applying the same polar-factor map to each local shard:
\[
\widetilde{\mathrm{polar}}(G)\;:=\;
\begin{bmatrix}
\mathrm{polar}(G^{(1)})\\
\vdots\\
\mathrm{polar}(G^{(S)})
\end{bmatrix}.
\]
This introduces \emph{no incremental optimizer communication}: each device computes $\mathrm{polar}(G^{(s)})$ from its locally stored shard.
If $Q\ge P$ and the sharding splits the contracted dimension $P$ (as in row-sharding above), then the Newton--Schulz work is also reduced.
A global orthogonalization has cost $\Theta(P^2Q)$ in total and hence $\Theta(P^2Q/S)$ per device under ideal $1/S$ scaling.
In contrast, shardwise orthogonalization costs
\[
\Theta\!\Bigl(\bigl(P/S\bigr)^2Q\Bigr)\;=\;\Theta(P^2Q/S^2)
\qquad\text{per device},
\]
which is a factor-$S$ reduction in per-device orthogonalization flops relative to any globally orthogonalizing method with ideal $1/S$ scaling.
The tradeoff is structural: $\widetilde{\mathrm{polar}}(G)$ enforces orthogonality only \emph{within} each shard and does not coincide with $\mathrm{polar}(G)$.

\paragraph{Summary of incremental cost (row-sharded $G$).}
Let $|G|$ denote the number of matrix entries (so communication scales in bytes like $|G|$ up to datatype size).
The table records \emph{incremental} optimizer communication beyond the baseline training loop.
\begin{center}
\begin{tabular}{lcc}
\hline
method & extra communication per update & per-device orthog flops \\
\hline
all-gather then local orthog & one all-gather moving $\Theta(|G|)$ entries & $\Theta(P^2Q)$ \\
distributed Newton--Schulz & $\Theta(\text{\#iters})$ collectives & $\Theta(P^2Q/S)$ \\
reshard $\to$ local orthog $\to$ reshard & two all-to-alls & $\Theta(P^2Q/S)$ \\
shardwise \texttt{SpecGD} & none & $\Theta(P^2Q/S^2)$ \\
\hline
\end{tabular}
\end{center}

We now return to the quadratic-model viewpoint from Section~\ref{sec:rand_regress} and record the blockwise majorization
calculation underlying the shardwise polar step.
In this calculation, the partition $\mathcal{P}$ should be read as the coordinate partition induced by device shards (or any other block structure).

\paragraph{Quadratic model and a blockwise operator geometry.}
Fix $A\in\R^{k\times n}$ and $Y\in\R^{m\times n}$ and consider
\begin{equation}
\label{eq:app-dist-LS}
\mathcal{L}(W)\;:=\;\frac{1}{2n}\,\|WA-Y\|_F^2,
\qquad
G:=\nabla \mathcal{L}(W)=\frac{1}{n}(WA-Y)A^\top.
\end{equation}
Let $\mathcal{P}$ be a partition of the coordinate set $[m]\times[k]$.
For each $p\in\mathcal{P}$ let $\Pi_p:\R^{m\times k}\to\R^{m\times k}$ be the associated coordinate projection
(so $\Pi_p\Pi_q=0$ for $p\neq q$ and $\sum_{p\in\mathcal{P}}\Pi_p=\mathrm{Id}$), and set
\[
W_p:=\Pi_p W,\qquad U_p:=\Pi_p U,\qquad G_p:=\Pi_p G.
\]
Write the row/column footprints of block $p$ as
\[
R_p:=\{\,i\in[m]:\exists j\in[k]\text{ with }(i,j)\in p\,\},
\qquad
C_p:=\{\,j\in[k]:\exists i\in[m]\text{ with }(i,j)\in p\,\},
\]
and define the overlap multiplicities
\[
\nu_{\mathcal{P}}:=\max_{i\in[m]}\bigl|\{p\in\mathcal{P}:i\in R_p\}\bigr|,
\qquad
\mu_{\mathcal{P}}:=\max_{j\in[k]}\bigl|\{p\in\mathcal{P}:j\in C_p\}\bigr|,
\qquad
\kappa_{\mathcal{P}}:=\min\{\mu_{\mathcal{P}},\nu_{\mathcal{P}}\}.
\]
Finally, define the blockwise seminorm
\[
\|U\|_{\mathcal{P}}^2:=\sum_{p\in\mathcal{P}}\|U_p\|_{\op}^2.
\]

\begin{lem}[Partitioned spectral majorization]
\label{lem:partitioned-majorization}
For all $U\in\R^{m\times k}$,
\begin{equation}
\label{eq:partitioned-majorization}
\mathcal{L}(W+U)
\;\le\;
\mathcal{L}(W)
+\sum_{p\in\mathcal{P}}\langle G_p,U_p\rangle
+\frac{L_{\mathcal{P}}}{2}\,\|U\|_{\mathcal{P}}^2,
\qquad
L_{\mathcal{P}}:=\frac{\kappa_{\mathcal{P}}\|A\|_F^2}{n}.
\end{equation}
In particular, $\kappa_{\mathcal{P}}=1$ for row partitions ($\nu_{\mathcal{P}}=1$) and for column partitions ($\mu_{\mathcal{P}}=1$).
\end{lem}

\begin{proof}
The identity
\[
\mathcal{L}(W+U)=\mathcal{L}(W)+\langle G,U\rangle+\frac{1}{2n}\|UA\|_F^2
\]
is exact and $\langle G,U\rangle=\sum_{p\in\mathcal{P}}\langle G_p,U_p\rangle$.
It remains to show $\|UA\|_F^2\le \kappa_{\mathcal{P}}\|A\|_F^2\sum_{p\in\mathcal{P}}\|U_p\|_{\op}^2$.

\emph{Row-overlap bound.}
Write $UA=\sum_{p\in\mathcal{P}}U_pA$ and set $X_p:=U_pA$.  Each $X_p$ is supported only on rows in $R_p$, hence
\begin{align*}
\|UA\|_F^2
=\Big\|\sum_{p\in\mathcal{P}}X_p\Big\|_F^2
&=\sum_{i=1}^m\Big\|\sum_{p:\,i\in R_p}(X_p)_{i:}\Big\|_2^2
\le \sum_{i=1}^m \nu_{\mathcal{P}}\sum_{p:\,i\in R_p}\|(X_p)_{i:}\|_2^2 \\
&=\nu_{\mathcal{P}}\sum_{p\in\mathcal{P}}\|X_p\|_F^2
\le \nu_{\mathcal{P}}\sum_{p\in\mathcal{P}}\|U_p\|_{\op}^2\,\|A\|_F^2.
\end{align*}

\emph{Column-overlap bound.}
For each $p$, the block $U_p$ uses only columns in $C_p$, hence $U_pA=U_pA_{C_p:}$. Therefore
\begin{align*}
\|UA\|_F
=\Big\|\sum_{p\in\mathcal{P}}U_pA_{C_p:}\Big\|_F
&\le \sum_{p\in\mathcal{P}}\|U_pA_{C_p:}\|_F
\le \sum_{p\in\mathcal{P}}\|U_p\|_{\op}\,\|A_{C_p:}\|_F \\
&\le \Big(\sum_{p\in\mathcal{P}}\|U_p\|_{\op}^2\Big)^{1/2}
\Big(\sum_{p\in\mathcal{P}}\|A_{C_p:}\|_F^2\Big)^{1/2}.
\end{align*}
Moreover,
\[
\sum_{p\in\mathcal{P}}\|A_{C_p:}\|_F^2
=\sum_{p\in\mathcal{P}}\sum_{j\in C_p}\|A_{j:}\|_2^2
\le \sum_{j=1}^k \mu_{\mathcal{P}}\,\|A_{j:}\|_2^2
=\mu_{\mathcal{P}}\,\|A\|_F^2,
\]
so $\|UA\|_F^2\le \mu_{\mathcal{P}}\|A\|_F^2\sum_{p\in\mathcal{P}}\|U_p\|_{\op}^2$.
Taking the minimum of the row- and column-overlap bounds yields the claim.
\end{proof}

\begin{cor}[Blockwise polar step and one-step decrease]
\label{cor:blockwise-polar-step}
Define $W^+:=W+U^\star$ by
\[
U_p^\star
:=
-\frac{1}{L_{\mathcal{P}}}\,\|G_p\|_*\,\mathrm{polar}(G_p),
\qquad p\in\mathcal{P},
\]
where $\mathrm{polar}(M)$ is the polar factor of $M$.
Then $U^\star$ minimizes the right-hand side of~\eqref{eq:partitioned-majorization} over $U$, and hence
\begin{equation}
\label{eq:partitioned-descent}
\mathcal{L}(W)-\mathcal{L}(W^+)\;\ge\;\frac{1}{2L_{\mathcal{P}}}\sum_{p\in\mathcal{P}}\|G_p\|_*^2
=\frac{n}{2\kappa_{\mathcal{P}}\|A\|_F^2}\sum_{p\in\mathcal{P}}\|G_p\|_*^2.
\end{equation}
\end{cor}

\begin{proof}
The bound~\eqref{eq:partitioned-majorization} is a quadratic upper model in the variables $\{U_p\}_{p\in\mathcal{P}}$,
and it decouples across $p$.  For fixed $p$, minimizing
$\langle G_p,U_p\rangle+\frac{L_{\mathcal{P}}}{2}\|U_p\|_{\op}^2$
over $U_p$ yields the steepest-descent direction $-\mathrm{polar}(G_p)$ with optimal step length $\|G_p\|_*/L_{\mathcal{P}}$.
Summing the optimal values over $p$ gives~\eqref{eq:partitioned-descent}.
\end{proof}

\paragraph{Comparison to Euclidean GD.}
Let $L_F:=\|A\|_{\op}^2/n$ be the Frobenius smoothness constant of $\mathcal{L}(\cdot)$ in~\eqref{eq:app-dist-LS}.
The GD step $W_{\mathrm{GD}}:=W-\frac{1}{L_F}G$ satisfies
\begin{equation}
\label{eq:gd-descent-app}
\mathcal{L}(W)-\mathcal{L}(W_{\mathrm{GD}})
\;\ge\;
\frac{1}{2L_F}\|G\|_F^2
=\frac{n}{2\|A\|_{\op}^2}\|G\|_F^2.
\end{equation}
Combining~\eqref{eq:partitioned-descent} and~\eqref{eq:gd-descent-app}, the blockwise polar step achieves at least as much one-step decrease as GD whenever
\[
\mathrm{nr}_{\mathcal{P}}(G)\;\ge\;\mathrm{st}_{\mathcal{P}}(A),
\qquad
\mathrm{nr}_{\mathcal{P}}(G):=\frac{\sum_{p\in\mathcal{P}}\|G_p\|_*^2}{\|G\|_F^2},
\qquad
\mathrm{st}_{\mathcal{P}}(A):=\frac{L_{\mathcal{P}}}{L_F}
=\kappa_{\mathcal{P}}\frac{\|A\|_F^2}{\|A\|_{\op}^2}
=\kappa_{\mathcal{P}}\st(A).
\]
For row and column partitions, $\kappa_{\mathcal{P}}=1$ and the threshold is $\st(A)$, matching the criterion in the main text.
In the notation of Section~\ref{sec:rand_regress}, the deciding metric is the comparison between a (blockwise) nuclear-rank proxy for the gradient
and a (partition-weighted) stable rank of the incoming features: when $\mathrm{nr}_{\mathcal{P}}(G)$ exceeds $\mathrm{st}_{\mathcal{P}}(A)$, the
blockwise spectral geometry promises at least as much one-step decrease as the Euclidean geometry.

\input{moddednanogpt.tex}

\end{document}

%% file: liweimacros.tex
\numberwithin{equation}{section}

\newcommand{\R}{{\bf R}}

\newcommand{\inclu}[0] {\ar@{^{(}->}}

\newcommand{\diag}{{\rm diag}}

\newcommand{\PP}{\mathbb{P}}

\newcommand{\EE}{\mathbb{E}}

\newcommand{\RR}{\mathbb{R}}

\newcommand{\st}{\mathrm{st}}

\newcommand{\cL}{\mathcal{L}}

\newcommand{\cN}{\mathcal{N}}

\newcommand{\argmin}{\operatornamewithlimits{argmin}}

\newcommand{\Diag}{{\rm Diag}}

\newcommand{\dotp}[1]{\left\langle #1\right\rangle}

\newtheorem{thm}{Theorem}[section]

\newtheorem{definition}[thm]{Definition}
\newtheorem{proposition}[thm]{Proposition}
\newtheorem{lem}[thm]{Lemma}
\newtheorem{cor}[thm]{Corollary}

\newtheorem{assumption}{Assumption}

\newtheorem{remark}{Remark}

\theoremstyle{remark}

\usepackage{mathtools}

%% file: transformer_block_mha_attention.tex
\begin{figure}[h]
\centering
\begin{tikzpicture}[
  >=LaTeX,
  node distance = 1.0cm and 1.6cm,
  every node/.style = {font=\small},
  every path/.style = {->, thick, draw=gray!70},
  state/.style = {
    draw=blue!60!black,
    rounded corners,
    thick,
    inner sep=1.8pt,
    fill=blue!6
  },
  att/.style = {
    draw=teal!60!black,
    rounded corners,
    thick,
    inner sep=1.8pt,
    fill=teal!6
  },
  sum/.style = {
    circle,
    draw=gray!60,
    thick,
    inner sep=1pt,
    fill=gray!10
  }
]

\node[state] (X) {$X$};

\node[att, below=of X] (rms1)
  {$A^{\mathrm{rms}}=\RMSNorm(X)$};

\node[att, below=0.8cm of rms1, align=left] (heads)
  {$\begin{aligned}
    \forall h\in[n_{\mathrm{head}}]:\quad
    &Y^{(h)}=A^{\mathrm{rms}}P^{(h)}\\
    &H^{(h)}=W_V^{(h)}Y^{(h)}
  \end{aligned}$};

\node[att, below=0.8cm of heads] (concat)
  {$H=\concat_{h\in[n_{\mathrm{head}}]} H^{(h)}$};

\node[att, below=0.8cm of concat] (WOH)
  {$W_OH$};

\node[sum, right=2.2cm of WOH] (plus) {$+$};

\node[state, below=of plus] (Xatt)
  {$X^{\mathrm{att}}=X+W_OH$};

\node[att, below=0.8cm of Xatt] (rms2)
  {$A^{\mathrm{rms}}_{\mathrm{mlp}}=\RMSNorm(X^{\mathrm{att}})$};

\draw (X) -- (rms1);
\draw (rms1) -- (heads);
\draw (heads) -- (concat);
\draw (concat) -- (WOH);
\draw (WOH) -- (plus);
\draw (plus) -- (Xatt);
\draw (Xatt) -- (rms2);

\draw (X.east) -| (plus);

\end{tikzpicture}
\caption{Multi-headed value aggregation sublayer (the attention kernels $P^{(h)}$ are treated as column-stochastic matrices).}
\label{fig:mha-sublayer}
\end{figure}

%% file: transformer_block_ngmlp.tex
\begin{figure}[h]
\centering
\begin{tikzpicture}[
  >=LaTeX,
  node distance = 1.0cm and 1.6cm,
  every node/.style = {font=\small},
  every path/.style = {->, thick, draw=gray!70},
  state/.style = {
    draw=blue!60!black,
    rounded corners,
    thick,
    inner sep=1.8pt,
    fill=blue!6
  },
  att/.style = {
    draw=teal!60!black,
    rounded corners,
    thick,
    inner sep=1.8pt,
    fill=teal!6
  },
  sum/.style = {
    circle,
    draw=gray!60,
    thick,
    inner sep=1pt,
    fill=gray!10
  }
]

\node[state] (X) {$X$};

\node[att, below=of X] (rms)
  {$A^{\mathrm{rms}}=\RMSNorm(X)$};

\node[att, below=0.8cm of rms] (W1)
  {$U=W_1A^{\mathrm{rms}}$};

\node[att, below=0.8cm of W1] (B)
  {$B=\sigma(U)$};

\node[att, below=0.8cm of B] (W2B)
  {$W_2B$};

\node[sum, right=2.2cm of W2B] (plus) {$+$};

\node[state, below=of plus] (Xplus)
  {$X^+=X+W_2B$};

\draw (X) -- (rms);
\draw (rms) -- (W1);
\draw (W1) -- (B);
\draw (B) -- (W2B);
\draw (W2B) -- (plus);
\draw (plus) -- (Xplus);

\draw (X.east) -| (plus);

\end{tikzpicture}
\caption{MLP sublayer (RMSNorm, pointwise activation, and residual update).}
\label{fig:ngmlp-sublayer}
\end{figure}

%% file: transformer_block_gmlp.tex
\begin{figure}[h]
\centering
\begin{tikzpicture}[
  >=LaTeX,
  node distance = 1.0cm and 1.6cm,
  every node/.style = {font=\small},
  every path/.style = {->, thick, draw=gray!70},
  state/.style = {
    draw=blue!60!black,
    rounded corners,
    thick,
    inner sep=1.8pt,
    fill=blue!6
  },
  att/.style = {
    draw=teal!60!black,
    rounded corners,
    thick,
    inner sep=1.8pt,
    fill=teal!6
  },
  sum/.style = {
    circle,
    draw=gray!60,
    thick,
    inner sep=1pt,
    fill=gray!10
  }
]

\node[state] (X) {$X$};

\node[att, below=of X] (rms)
  {$A^{\mathrm{rms}}=\RMSNorm(X)$};

\node[att, below left=0.9cm and 1.8cm of rms] (U)
  {$U=W_{\mathrm{in}}A^{\mathrm{rms}}$};

\node[att, below right=0.9cm and 1.8cm of rms] (G)
  {$G=W_{\mathrm{gate}}A^{\mathrm{rms}}$};

\node[att, below=0.9cm of U] (phiU)
  {$\sigma(U)$};

\node[sum, right=1.8cm of phiU] (had) {$\had$};

\node[att, below=0.9cm of had] (B)
  {$B=\sigma(U)\had G$};

\node[att, below=0.9cm of B] (W2B)
  {$W_2B$};

\node[sum, right=3.0cm of W2B] (plus) {$+$};

\node[state, below=of plus] (Xplus)
  {$X^+=X+W_2B$};

\draw (X) -- (rms);
\draw (rms) -- (U);
\draw (rms) -- (G);

\draw (U) -- (phiU);
\draw (phiU.east) -- (had.west);
\draw (G.south) |- (had.east);

\draw (had) -- (B);
\draw (B) -- (W2B);
\draw (W2B) -- (plus);
\draw (plus) -- (Xplus);

\draw (X.east) -- ++(6.0cm,0) |- (plus);

\end{tikzpicture}
\caption{Gated MLP sublayer (two linear maps, pointwise gate, and residual update).}
\label{fig:gmlp-sublayer}
\end{figure}

%% file: transformer_moe_layer.tex
\begin{figure}[h]
\centering
\begin{tikzpicture}[
  >=LaTeX,
  node distance = 1.0cm and 1.6cm,
  every node/.style = {font=\small},
  every path/.style = {->, thick, draw=gray!70},
  state/.style = {
    draw=blue!60!black,
    rounded corners,
    thick,
    inner sep=1.8pt,
    fill=blue!6
  },
  att/.style = {
    draw=teal!60!black,
    rounded corners,
    thick,
    inner sep=1.8pt,
    fill=teal!6
  },
  sum/.style = {
    circle,
    draw=gray!60,
    thick,
    inner sep=1pt,
    fill=gray!10
  }
]

\node[state] (X) {$X$};

\node[att, below=of X] (rms)
  {$A^{\mathrm{rms}}=\RMSNorm(X)$};

\node[att, below=0.8cm of rms, align=left] (experts)
  {$\forall e\in[n_{\mathrm{exp}}]:\quad
    B^{(e)}=\sigma\!\bigl(W_1^{(e)}A^{\mathrm{rms}}\bigr)$};

\node[att, below=0.8cm of experts, align=left] (mix)
  {$H=\displaystyle\sum_{e\in[n_{\mathrm{exp}}]} W_2^{(e)}\,B^{(e)}\,D^{(e)}$};

\node[sum, right=2.6cm of mix] (plus) {$+$};

\node[state, below=of plus] (Xplus)
  {$X^+=X+H$};

\draw (X) -- (rms);
\draw (rms) -- (experts);
\draw (experts) -- (mix);
\draw (mix) -- (plus);
\draw (plus) -- (Xplus);

\draw (X.east) -| (plus);

\end{tikzpicture}
\caption{MoE-MLP sublayer (RMSNorm, expert MLPs with routing weights, aggregation, and residual update).}
\label{fig:ngmoe-sublayer}
\end{figure}

%% file: moddednanogpt.tex
\providecommand{\ensurearchitecturepackages}{%
  \usetikzlibrary{arrows.meta,calc,positioning}%
}
\ensurearchitecturepackages

\begin{figure}[t]
  \centering
  \resizebox{!}{0.9\textheight}{%
  \begin{tikzpicture}[
  x=1cm,
  y=1.2cm,
  >=Latex,
  every node/.style={font=\normalsize, transform shape},
  stage/.style={draw, rounded corners=4pt, align=center, minimum width=3.0cm, minimum height=1.35cm, fill=gray!8},
  embedstage/.style={stage, minimum width=3.2cm},
  dataset/.style={draw, rounded corners=4pt, align=center, fill=white, minimum width=3.0cm, minimum height=1.25cm},
  table/.style={draw, rounded corners=4pt, align=center, fill=orange!20, minimum width=2.8cm, minimum height=1.25cm},
  legend/.style={draw, rounded corners=4pt, align=center, fill=gray!10, minimum width=3.4cm, minimum height=1.25cm},
  flow/.style={->, ultra thick},
  skip/.style={->, dashed, line width=1pt, color=teal!80!black},
  valflow/.style={->, ultra thick, color=orange!80!black},
  mixnode/.style={draw, circle, fill=white, inner sep=1.2pt, line width=0.7pt}
]

    \node[dataset] (tokens) at (0,0) {Tokens\\$\{s_t\}$};
    \node[embedstage] (embed) at (0,2.0) {Token embedding $E$\\RMSNorm};

    \foreach \label/\y in {0/3.8,1/5.4,2/7.0,3/8.6,4/10.2,5/11.8,6/13.4,7/15.0,8/16.6,9/18.2,10/19.8,11/21.4}{
      \node[stage] (b\label) at (0,\y) {Block $\label$\\FlexAttn + ReLU$^2$ MLP};
    }

    \node[stage] (norm) at (0,23.0) {RMSNorm};
    \node[stage] (lmhead) at (0,24.6) {FP8 LM head\\$W^{(\mathrm{LM})}$};
    \node[stage] (softcap) at (0,26.2) {Softcap\\$30\,\sigma(\cdot)$};
    \node[dataset] (loss) at (0,27.8) {Cross-entropy\\loss};

    \draw[flow] (tokens.north) -- node[right]{Lookup} (embed.south);
    \draw[flow] (embed.north) -- (b0.south);
    \foreach \i/\j in {0/1,1/2,2/3,3/4,4/5,5/6,6/7,7/8,8/9,9/10,10/11}{
      \draw[flow] (b\i.north) -- (b\j.south);
    }
    \draw[flow] (b11.north) -- (norm.south);
    \draw[flow] (norm.north) -- (lmhead.south);
    \draw[flow] (lmhead.north) -- (softcap.south);
    \draw[flow] (softcap.north) -- node[right]{Logits} (loss.south);

    \node[table, left=3.1cm of b0] (ve0a) {$V^{(0)}$};
    \node[table, left=3.1cm of b1] (ve1a) {$V^{(1)}$};
    \node[table, left=3.1cm of b2] (ve2a) {$V^{(2)}$};

    \node[mixnode, left=1.4cm of b0] (add0) {$+$};
    \node[mixnode, left=1.4cm of b1] (add1) {$+$};
    \node[mixnode, left=1.4cm of b2] (add2) {$+$};

    \draw[valflow] (ve0a.east) -- (add0);
    \draw[valflow] (ve1a.east) -- (add1);
    \draw[valflow] (ve2a.east) -- (add2);

    \node[font=\scriptsize, anchor=west] at ($(ve0a.east)!0.5!(add0)+(0.0,0.55)$) {inject $V^{(0)}$};
    \node[font=\scriptsize, anchor=west] at ($(ve1a.east)!0.5!(add1)+(0.0,0.55)$) {inject $V^{(1)}$};
    \node[font=\scriptsize, anchor=west] at ($(ve2a.east)!0.5!(add2)+(0.0,0.55)$) {inject $V^{(2)}$};

    \draw[flow] (add0) -- (b0.west);
    \draw[flow] (add1) -- (b1.west);
    \draw[flow] (add2) -- (b2.west);

    \node[table, left=3.1cm of b9]  (ve0b) {$V^{(0)}$};
    \node[table, left=3.1cm of b10] (ve1b) {$V^{(1)}$};
    \node[table, left=3.1cm of b11] (ve2b) {$V^{(2)}$};

    \node[mixnode, left=1.4cm of b9]  (add9) {$+$};
    \node[mixnode, left=1.4cm of b10] (add10) {$+$};
    \node[mixnode, left=1.4cm of b11] (add11) {$+$};

    \draw[valflow] (ve0b.east) -- (add9);
    \draw[valflow] (ve1b.east) -- (add10);
    \draw[valflow] (ve2b.east) -- (add11);

    \node[font=\scriptsize, anchor=west] at ($(ve0b.east)!0.5!(add9)+(0.0,0.55)$) {inject $V^{(0)}$};
    \node[font=\scriptsize, anchor=west] at ($(ve1b.east)!0.5!(add10)+(0.0,0.55)$) {inject $V^{(1)}$};
    \node[font=\scriptsize, anchor=west] at ($(ve2b.east)!0.5!(add11)+(0.0,0.55)$) {inject $V^{(2)}$};

    \draw[flow] (add9) -- (b9.west);
    \draw[flow] (add10) -- (b10.west);
    \draw[flow] (add11) -- (b11.west);

    \draw[skip] (b0.north east) .. controls ($(b0.north east)+(1.3,0.8)$) and ($(b11.south east)+(1.3,-0.8)$) .. (b11.south east);
    \draw[skip] (b1.north east) .. controls ($(b1.north east)+(1.3,0.7)$) and ($(b10.south east)+(1.3,-0.7)$) .. (b10.south east);
    \draw[skip] (b2.north east) .. controls ($(b2.north east)+(1.2,0.7)$) and ($(b9.south east)+(1.2,-0.7)$) .. (b9.south east);
    \draw[skip] (b3.north east) .. controls ($(b3.north east)+(1.1,0.6)$) and ($(b8.south east)+(1.1,-0.6)$) .. (b8.south east);
    \draw[skip] (b4.north east) .. controls ($(b4.north east)+(1.0,0.5)$) and ($(b7.south east)+(1.0,-0.5)$) .. (b7.south east);
    \draw[skip] (b5.north east) .. controls ($(b5.north east)+(0.9,0.4)$) and ($(b6.south east)+(0.9,-0.4)$) .. (b6.south east);


    \begin{scope}[shift={(8.4,10.4)}]
      \node[stage, align=center, minimum width=4.0cm] (template) {Block schematic};
      \node[legend, above=0.9cm of template] (norm1) {RMSNorm};
      \node[legend, above=0.8cm of norm1] (qkv) {$6$ heads ($d_h=128$)\\$Q,K,V = \mathrm{rms}(x) W^{(Q,K,V)}$};
      \node[legend, above=0.8cm of qkv] (flex) {FlexAttention $\mathcal{A}(Q,K,V)$\\(doc-window mask)};
      \node[legend, above=0.8cm of flex] (valmixlegend) {Value mix $V \leftarrow V + \lambda^{(v)} V^{(j)}$\\(blocks $0,1,2,9,10,11$)};
      \node[legend, above=0.8cm of valmixlegend] (addreslegend) {Residual add};
      \node[legend, above=0.8cm of addreslegend] (norm2) {RMSNorm};
      \node[legend, above=0.8cm of norm2] (mlp) {$d \rightarrow 4d \rightarrow d$ MLP\\ReLU$^2$ activation};
      \node[legend, above=0.8cm of mlp] (mlpres) {Residual add};

      \draw[flow] (template.north) -- (norm1.south);
      \draw[flow] (norm1.north) -- (qkv.south);
      \draw[flow] (qkv.north) -- (flex.south);
      \draw[flow] (flex.north) -- (valmixlegend.south);
      \draw[flow] (valmixlegend.north) -- (addreslegend.south);
      \draw[flow] (addreslegend.north) -- (norm2.south);
      \draw[flow] (norm2.north) -- (mlp.south);
      \draw[flow] (mlp.north) -- (mlpres.south);
    \end{scope}
  \end{tikzpicture}}
  \caption{The \texttt{train\_gpt.py} architecture of snapshot~\cite{moddednanogpt2025}. Value embeddings $V^{(0)},V^{(1)},V^{(2)}$ enter through explicit addition nodes in blocks $0,1,2$ and $9,10,11$, while the block schematic (right) shows the RMS normalisations, six-head FlexAttention, and two-layer ReLU-squared MLP shared across the stack. Teal skip arcs connect block $k$ to block $11-k$, indicating the pop step $x^{(i+1)} = \tilde{x}^{(i+1)} + \kappa_i s^{(i)}$.}\label{figure:nanogpt}
\end{figure}
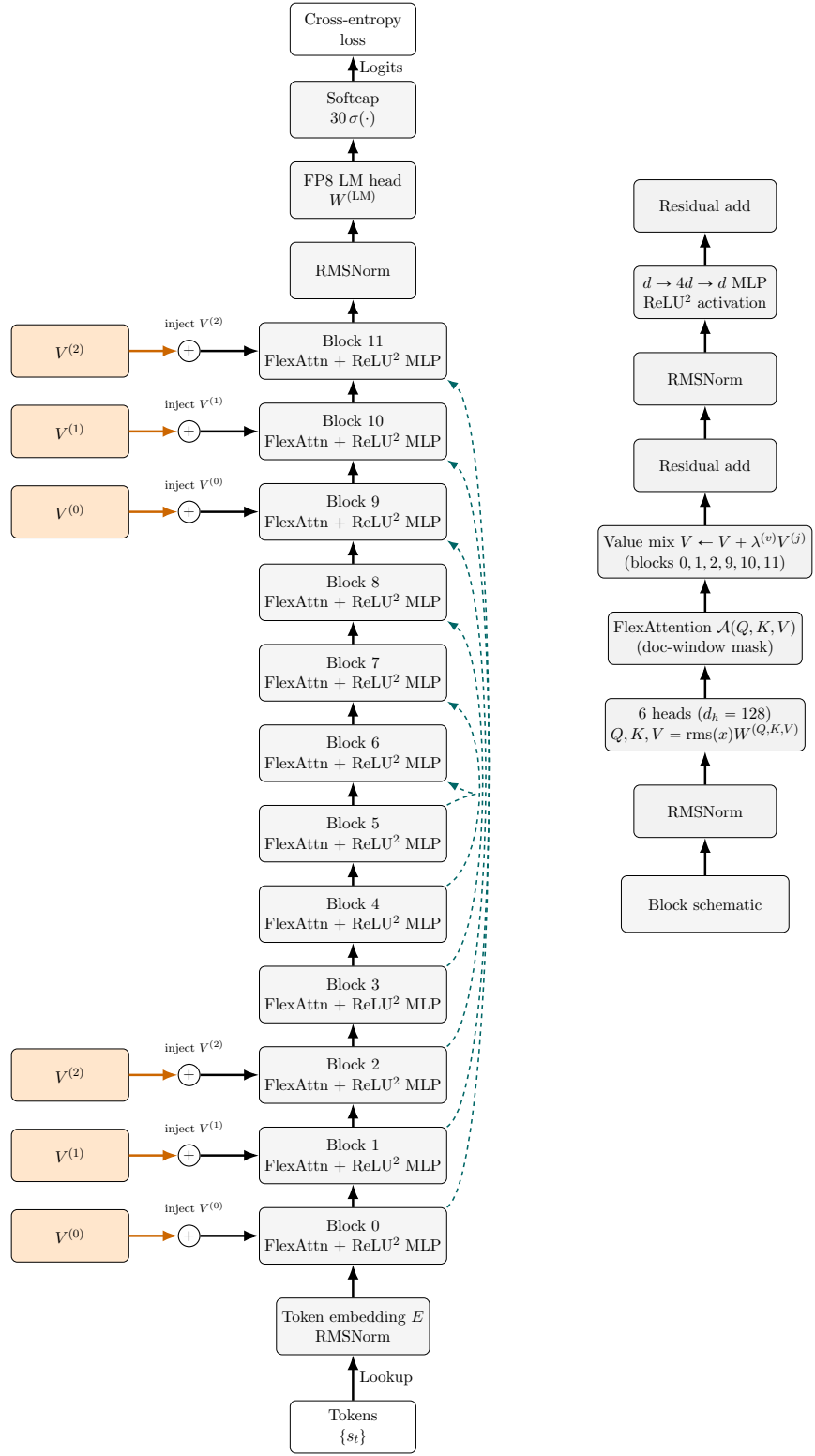